%% file: main.tex
\documentclass{article}
\usepackage{amsmath}
\usepackage{amsfonts}
\usepackage{amssymb}
\usepackage{amsthm}
\usepackage{algorithm} 
\usepackage{algorithmic}
\usepackage{color}
\usepackage[parfill]{parskip}
\usepackage[margin=0.8in]{geometry}
\usepackage{graphicx} 
\usepackage{caption,subcaption}
\usepackage{bbm}
\usepackage{mdframed}
\usepackage{natbib}
\usepackage[hidelinks]{hyperref}
\usepackage{titlesec}
\usepackage{authblk}
\usepackage{placeins}

\DeclareMathOperator*{\argmin}{arg\,min}
\DeclareMathOperator*{\argmax}{arg\,max}

\usepackage{thmtools, thm-restate}
 
\newtheorem{theorem}{Theorem}
\newtheorem{lemma}[theorem]{Lemma}

\newtheorem{definition}[theorem]{Definition}

\newcommand{\E}{\mathbb{E}}
 \newcommand{\KL}{{\mathbf d}_{\rm KL}}
\newcommand{\I}{\mathbb{I}}
\newcommand{\Z}{\mathbb{Z}}
\newcommand{\Lc}{\mathcal{L}}
\newcommand{\diffentropy}{{\bf h}}

\newcommand{\normal}{\mathcal{N}}
\newcommand{\relu}{\mathrm{ReLU}}
\def\E{\mathbb{E}}

\def\H{\mathbb{H}}
\def\diffentropy{\mathbf{h}}
\def\I{\mathbb{I}}
\def\Pr{\mathbb{P}}
\def\R{\mathbb{R}}

\def\1{\mathbf{1}}
\def\subgauss{\nu^2}
\def\proxy{\tilde{\theta}}
\def\proxyset{\tilde{\Theta}}
\def\proxytheta{\tilde{\theta}}
\newcommand{\sphere}{\mathbb{S}^{d-1}}

\definecolor{red}{RGB}{255,0,0}
\definecolor{blue}{RGB}{0,0,255}
\definecolor{green}{RGB}{0,255,0}
\definecolor{orange}{RGB}{255,165,0}
\definecolor{purple}{RGB}{128,0,128}
\definecolor{teal}{RGB}{0,128,128}
\definecolor{x}{RGB}{255,102,102}
\definecolor{darkgreen}{rgb}{0.0,0.5,0.0}

\definecolor{lightgray}{RGB}{230, 230, 230}
\newmdenv[linecolor=lightgray, backgroundcolor=lightgray,innerbottommargin=15pt,innertopmargin=10pt,skipabove=20pt,splittopskip=20pt,splitbottomskip=15pt]{graybox}

\newenvironment{summary}
    {\begin{graybox}
    \addcontentsline{toc}{subsection}{Summary}
    \subsection*{Summary}
    }
    {
    \end{graybox}
    }

\title{Information-Theoretic Foundations for Machine Learning}
\author[1\footnote{Correspondence to \texttt{hjjeon@stanford.edu}.}]{Hong Jun Jeon}
\author[2,3]{Benjamin Van Roy}

\affil[1]{Department of Computer Science, Stanford University}
\affil[2]{Department of Electrical Engineering, Stanford University}
\affil[3]{Department of Management Science and Engineering, Stanford University}

\date{}

\begin{document}

\maketitle

\begin{abstract}
    The progress of machine learning over the past decade is undeniable.  In retrospect, it is both remarkable and unsettling that this progress was achievable with little to no rigorous theory to guide experimentation.  Despite this fact, practitioners have been able to guide their future experimentation via observations from previous large-scale empirical investigations.  However, alluding to Plato's Allegory of the Cave, it is likely that the observations which form the field's notion of reality are but shadows representing fragments of that reality.  In this work, we propose a theoretical framework which attempts to answer what exists outside of the cave.  To the theorist, we provide a framework which is mathematically rigorous and leaves open many interesting ideas for future exploration.  To the practitioner, we provide a framework with simple results that offer intuition to guide future investigations across a wide range of learning paradigms.  Concretely, we provide a theoretical framework rooted in Bayesian statistics and Shannon's information theory which is general enough to unify the analysis of many phenomena in machine learning.  Our framework characterizes the expected log-loss of an optimal Bayesian learner as it learns from a stream of experience.  Unlike existing analyses that weaken with increasing data complexity, our theoretical tools provide accurate insights across diverse machine learning settings.  Throughout this work, we derive theoretical results and demonstrate their generality by applying them to derive insights specific to multiple settings.  These settings range from learning from data which is independently and identically distributed under an unknown distribution, to data which is sequential, to data which exhibits hierarchical structure amenable to meta-learning, and finally to data which is not fully explainable under the learner's beliefs (misspecification).  These results are particularly relevant as we strive to understand and overcome increasingly difficult machine learning challenges in this endlessly complex world.
\end{abstract}

\clearpage

\tableofcontents

\clearpage

\input{sections/introduction}

\input{sections/related_works}

\input{sections/objective}

\input{sections/information_theory}

\input{sections/information_learning}

\input{sections/supervised_learning}

\input{sections/sequence_learning}

\input{sections/meta_learning}

\input{sections/misspecification}

\input{sections/conclusion}

\section*{Acknowledgements}

Financial support from the NSF GRFP fellowship and the Army Research Office (ARO) Grant W911NF2010055 is gratefully acknowledged.

\bibliography{references}
\clearpage

\appendix

\section*{Appendix}
\addcontentsline{toc}{section}{Appendix}

\section{Lower Bounds for Linear Regression}\label{apdx:lin_reg_lb}
We now introduce a lemma relating expected KL divergence to mean-squared error, a distortion measure that is prevalent in the literature. This relation will allow us to derive a lower bound for the rate-distortion function in the Gaussian linear regression setting.

\begin{lemma}
    \label{le:subgaussian1}
    For all $d \in \mathbb{Z}_{++}$ and $\sigma^2 \geq 0$, if $\theta:\Omega\mapsto\Re^d$ has iid components that are each $\subgauss$-subgaussian and symmetric, $X\sim\normal(0, I_d)$, and $Y \sim \normal(\theta^\top X, \sigma^2)$, then for all proxies $\proxytheta\in\proxyset$, $Y-\E[Y|\proxytheta, X]$ is $4\subgauss\|X\|^2_2 + \sigma^2$-subgaussian conditioned on $X$. 
\end{lemma}
\begin{proof}
    \begin{align*}
        \E\left[e^{\lambda(Y-\E[Y|\proxytheta, X])}\big|X\right]
        & \overset{(a)}{=} \E\left[e^{\lambda\left(Y-\E[Y|\theta,X]\right)}\big|X\right]\cdot\E\left[e^{\lambda\left(\E[Y|\theta,X] - \E[Y|\proxytheta, X]\right)}\big|X\right]\\
        & = e^{\frac{\lambda^2\sigma^2}{2}}\cdot\E\left[e^{\lambda\left((\theta-\E[\theta|\proxytheta])^\top X\right)}\big| X\right]\\
        & \overset{(b)}{\leq} e^{\frac{\lambda^2\sigma^2}{2}}\cdot\E\left[e^{-\lambda(\theta^\top X)}\cdot\E\left[e^{-\lambda(\theta^\top X)}|\proxytheta,X\right]\big|X\right]\\
        & \overset{(c)}{\leq} e^{\frac{\lambda^2\sigma^2}{2}}\E\left[\E[e^{-\lambda(\theta^\top X)}|\proxytheta, X]^2\big|X\right]\\
        &\leq e^{\frac{\lambda^2\sigma^2}{2}}\E\left[\E[e^{-2\lambda(\theta^\top X)}|\proxytheta, X]\big|X\right]\\
        & =e^{\frac{\lambda^2\sigma^2}{2}}\E\left[e^{-2\lambda(\theta^\top X)}\big|X\right]\\
        & \overset{(d)}{=}e^{\frac{\lambda^2\sigma^2}{2}}e^{2\lambda^2\subgauss\|X\|^2_2}\\
        & = e^{\frac{\lambda^2\left(\sigma^2+4\subgauss\|X\|^2_2\right)}{2}},
    \end{align*}
    where $(a)$ follows from $Y-\E[Y|\theta,X] = W$ which is independent from $\E[Y|\theta,X]-\E[Y|\proxytheta,X]$, $(b)$ follows from the fact that $\theta^\top X$ is symmetric conditioned on $X$ and Jensen's inequality, $(c)$ follows from the fact that $e^{\theta^\top X} = \E[e^{\theta^\top X}|\theta,\proxytheta, X]$, and $(d)$ follows from the fact that the components of $\theta$ are $\subgauss$-subgaussian.
\end{proof}

\begin{lemma}
    \label{le:subgaussian2}
    If $Y-\E[Y|\proxytheta, X]$ is $\subgauss$-subgaussian conditional on $X$ (uniformly), then for all $\alpha>1$, $Y-\E[Y|\proxytheta, X]$ is $\alpha\subgauss$-subgaussian conditional on $(\proxytheta, X)$.
\end{lemma}
\begin{proof}
    Assume that for some $\alpha > 1$, there exists an event $\mathcal{S}$ s.t. $\Pr(\proxytheta \in \mathcal{S}) > 0$ and $\proxytheta \in \mathcal{S}$ implies that $Y-\E[Y|\proxytheta, X]$ is not $\alpha\subgauss$-subgaussian conditioned on $(\proxytheta, X)$. We have that
    \begin{align*}
        \E\left[e^{\lambda(Y-\E[Y|\proxytheta, X])}\big|X\right]
        & \geq \Pr(\proxytheta\in\mathcal{S})\cdot \E\left[e^{\lambda(Y-\E[Y|\proxytheta, X])}|\proxytheta\in\mathcal{S}, X\right]\\
        & \overset{(a)}{>} e^{\ln\Pr(\proxytheta \in \mathcal{S})}\cdot e^{\frac{\alpha\lambda^2\subgauss}{2}}\\
        & = e^{\frac{\lambda^2\left(\alpha\subgauss + \frac{2}{\lambda^2}\ln\Pr(\proxytheta\in\mathcal{S})\right)}{2}}\\
        & = e^{\frac{\lambda^2\left(\subgauss + (\alpha-1)\subgauss + \frac{2}{\lambda^2}\ln\Pr(\proxytheta\in\mathcal{S})\right)}{2}},
    \end{align*}
    where $(a)$ holds for all $\lambda$ s.t. $|\lambda| \geq |\lambda^*|$ for some $\lambda^*$. Such $\lambda^*$ exists because of the fact that $\proxytheta\in\mathcal{S}$ implies that $Y-\E[Y|\proxytheta, X]$ is not $\alpha\subgauss$-subgaussian conditioned on $(\proxytheta, X)$. As a result, for $\lambda$ such that $\lambda^2 > \max\left\{\frac{2\ln\Pr(\proxytheta\in\mathcal{S})}{(1-\alpha)\subgauss}, \lambda_*^2\right\}$, we have that
    $$\E\left[e^{\lambda(Y-\E[Y|\proxytheta, X])}\big|X\right] > e^{\frac{\lambda^2\subgauss}{2}},$$
    which is a contradiction since $Y-\E[Y|\proxytheta, X]$ is $\subgauss$-subgaussian conditional on $X$. Therefore the assumption that there exists $\alpha > 1$ and $\proxytheta$ s.t. $Y-\E[Y|\proxytheta,X]$ is not $\alpha\nu^2$-subgaussian conditional on $(X, \proxytheta)$ cannot be true. The result follows.
\end{proof}
\begin{lemma}
    \label{le:mse_kl_inequality}
    If $Y-\E[Y|\proxytheta, X]$ is $\subgauss$-subgaussian conditioned on $(\proxytheta, X)$, then
    $$\E\left[\frac{\left(\E[Y|\theta,X] - \E[Y|\proxytheta, X]\right)^2}{2\subgauss}\right] \leq \E\left[\KL(\Pr(Y\in\cdot|\theta, X)\|\Pr(Y\in\cdot|\proxytheta, X))\right].$$
\end{lemma}
\begin{proof}
We begin by stating the Donsker-Varadhan variational form of the KL-divergence. For all probability distributions $P(\cdot)$ and $Q(\cdot)$ over $\Re$ such that $P$ is absolutely continuous with respect to $Q$, for densities $dP, dQ$ w.r.t the Lebesgue measure, 
$$\KL(P(\cdot)\|Q(\cdot)) = \sup_{g:\Re\rightarrow\Re} \left(\int_{y \in \Re} g(y) dP(y) - \ln \int_{y \in \Re} e^{g(y)} dQ(y)\right),$$
where the supremum is taken over measurable functions for which $\int_{y\in\Re} g(y) dP(y)$ is well-defined and the expression on the right $\int_{y\in\Re} e^{g(y)} dQ(y)$ is finite.

Let $P(\cdot) = \Pr(Y\in\cdot|\theta, X_t), Q(\cdot) = \Pr(Y\in\cdot|\proxytheta, X),$ and $Z = Y - \E[Y|\proxytheta, X]$. Then, for arbitrary $\lambda \in \R$, applying the variational form of KL-divergence with $g(Y) = \lambda Z$ gives us
    \begin{align*}
        \KL(\Pr(Y\in\cdot|\theta, X)\|\Pr(Y\in\cdot|\proxytheta, X))
        & \overset{(a)}{=} \KL(\Pr(Y\in\cdot|\theta,\proxytheta,  X)\|\Pr(Y\in\cdot|\proxytheta, X))\\
        & \geq \lambda \E\left[Z|\theta, \proxytheta, X\right]-\ln\E\left[e^{\lambda Z}|\proxytheta, X \right]\\
        &\overset{(b)}{\geq} \lambda\left(\E[Y|\theta, X] -\E[Y|\proxytheta, X]\right) -\frac{\lambda^2\subgauss}{2},
    \end{align*}
    where $(a)$ follows from $Y\perp \proxytheta |(\theta, X)$ and $(b)$ follows from $Z$ being $\subgauss$-subgaussian conditioned on $(\proxytheta, X)$. Since the above holds for arbitrary $\lambda$, maximizing the RHS w.r.t $\lambda$ gives us:
    $$\KL(\Pr(Y\in\cdot|\theta, X)\|\Pr(Y\in\cdot|\proxytheta, X)) \geq \frac{\left(\E[Y|\theta, X] - \E[Y|\proxytheta, X]\right)^2}{2\subgauss}.$$
    The result follows from taking an expectation on both sides.
\end{proof}

We establish a lower bound by first finding a suitable lower bound for the distortion function. For subgaussian random vectors, the following lemma allows us to lower bound the expected KL-divergence distortion by a multiple of the mean squared error.
\begin{restatable}{lemma}{mseInfoInequality}
    \label{le:mse-mutual-info-inequality}
    For all $\proxytheta\in\proxyset$, $d \in \mathbb{Z}_{++}$ and $\sigma^2 \in \Re_{++}$, if $\theta:\Omega\mapsto\Re^d$ consists of iid components each of which are $1/d$-subgaussian and symmetric, $X  \sim \normal(0, I_d)$, and if $Y \sim \normal(\theta^\top X, \sigma^2)$, then
    $$\E\left[\frac{\|X\|^2_2}{2(4\|X\|_2^2 + d\sigma^2)}\right]\E\left[\|\theta - \E[\theta|\proxytheta]\|^2_2\right] \leq \I(Y;\theta|\proxytheta, X).$$
\end{restatable}
\begin{proof}
    $\theta^\top X$ is $\|X\|_2^2/d$-subgaussian conditioned on $X$ and by Lemma \ref{le:subgaussian1}, $Y - \E[Y|\proxytheta, X]$ is $4\|X\|_2^2/d+\sigma^2$-subgaussian conditioned on $\|X\|_2^2$. Lemma \ref{le:subgaussian2} then states that $Y - \E[Y|\proxy, X]$ is $\alpha(4\|X\|_2^2/d+\sigma^2)$-subgaussian conditioned on $(\proxy, X)$ for all $\alpha > 1$. Therefore,
    \begin{align*}
        \E\left[\KL(\Pr(Y\in\cdot|\theta,X)\|\Pr(Y\in\cdot|\proxy, X))\right]
        & \overset{(a)}{\geq} \lim_{\alpha\downarrow 1} \E\left[\frac{\left((\theta - \E[\theta|\proxy])^\top X\right)^2}{2\alpha\left(\frac{4\|X\|_2^2}{d} + \sigma^2\right)}\right]\\
        & = \E\left[\frac{\left((\theta - \E[\theta|\proxy])^\top X\right)^2}{2\left(\frac{4\|X\|_2^2}{d} + \sigma^2\right)}\right]\\
        & = \E\left[\E\left[\frac{\left((\theta - \E[\theta|\proxy])^\top X\right)^2}{2\left(\frac{4\|X\|_2^2}{d} + \sigma^2\right)}\Bigg| \|X\|^2_2\right]\right]\\
        & = \E\left[\frac{\frac{\|X\|^2_2}{d}\E\left[\|(\theta - \E[\theta|\proxy]\|^2_2\right]}{2\left(\frac{4\|X\|_2^2}{d} + \sigma^2\right)}\right]\\
        & \overset{(b)}{=} \E\left[\frac{\|X\|^2_2}{2(4\|X\|_2^2 + d\sigma^2)}\right]\E\left[\|\theta - \E[\theta|\proxy]\|^2_2\right]\\
    \end{align*}
    where $(a)$ follows from Lemma \ref{le:mse_kl_inequality} and $(b)$ follows from the fact that $X\sim \normal(0, I_d)$.
\end{proof}

\section{Theoretical Results for Dirichlet-Multinomial}\label{apdx:dir_mult}

\subsection{General Combinatorics Results}
We begin with the following general combinatorics results that we will eventually apply to bound the estimation error of the finite $N$ Dirichlet-Multinomial process.
\begin{lemma}\label{le:multinomial_terms}
    For all $m, j, n, K, N\in\mathbb{Z}_{++}$ s.t. $j < n$, if 
    $$C_m(j) = \underbrace{\sum_{i=j}^{n-1}\frac{\frac{K}{N}}{K+i}\cdot\sum_{i=j+1}^{n-1}\frac{\frac{K}{N}}{K+i}\cdot \ldots\cdot\sum_{i > j+m-1}^{n-1}\frac{\frac{K}{N}}{K+i}}_{m},$$
    then
    $$\frac{1}{m!}\frac{K^m}{N^m}\ln^m\left(\frac{K+n}{K+m-1+j}\right) \leq C_m(j) \leq \frac{1}{m!}\frac{K^m}{N^m}\ln^m\left(\frac{K+n}{K-1+j}\right).$$
\end{lemma}
\begin{proof}
    We first prove the upper bound via induction. Base Case: $n=1$
    \begin{align*}
        C_1(j) 
        & = \sum_{i=j}^{n-1}\frac{\frac{K}{N}}{K+i}\\
        & \leq \frac{K}{N}\int_{j-1}^{n}\frac{1}{K+x}dx\\
        & = \frac{K}{N}\ln\left(\frac{K+n}{K-1+j}\right)
    \end{align*}
    Assume inductive hypothesis is true for $m=k$.
    \begin{align*}
        C_{k+1}(j) 
        & = \sum_{i=j}^{n-1}\frac{\frac{K}{N}}{K+i} \cdot C_k(i+1)\\
        & \leq \frac{K}{N}\int_{j-1}^{n}\frac{1}{K+x}\cdot C_k(x+1)\\
        & \leq \frac{K^{k+1}}{N^{k+1}}\frac{1}{k!}\int_{j-1}^{n}\frac{1}{K+x}\cdot \ln^k\left(\frac{K+n}{K+x}\right)dx\\
        & \leq -\frac{K^{k+1}}{N^{k+1}}\frac{1}{(k+1)!}\cdot \ln^{k+1}\left(\frac{K+n}{K+x}\right)\bigg|^{n}_{j-1}\\
        & = \frac{K^{k+1}}{N^{k+1}}\frac{1}{(k+1)!}\ln^{k+1}\left(\frac{K+n}{K-1+j}\right).
    \end{align*}
    We now prove the lower bound also via induction. Base Case: $m=1$
    \begin{align*}
        C_1(j) 
        & = \sum_{i=j}^{n-1}\frac{\frac{K}{N}}{K+i}\\
        & \geq \frac{K}{N}\int_{j}^{n}\frac{1}{K+x}dx\\
        & = \frac{K}{N}\ln\left(\frac{K+n}{K+j}\right)
    \end{align*}
     Assume inductive hypothesis is true for $m=k$.
    \begin{align*}
        C_{k+1}(j) 
        & = \sum_{i=j}^{n-1}\frac{\frac{K}{N}}{K+i} \cdot C_k(i+1)\\
        & \geq \frac{K}{N}\int_{j}^{n}\frac{1}{K+x}\cdot C_k(x+1)\\
        & \geq \frac{K^{k+1}}{N^{k+1}}\frac{1}{k!}\int_{j}^{n}\frac{1}{K+x}\cdot \ln^k\left(\frac{K+n}{K+k+x}\right)dx\\
        & \geq -\frac{K^{k+1}}{N^{k+1}}\frac{1}{(k+1)!}\cdot \ln^{k+1}\left(\frac{K+n}{K+k+x}\right)\bigg|^{n}_{j}\\
        & =\frac{K^{k+1}}{N^{k+1}}\frac{1}{(k+1)!}\ln^{k+1}\left(\frac{K+k+n}{K+k+j}\right)\\
        & \geq \frac{K^{k+1}}{N^{k+1}}\frac{1}{(k+1)!}\ln^{k+1}\left(\frac{K+n}{K+k+j}\right).
    \end{align*}
    The result follows.
\end{proof}

\begin{lemma}\label{le:alt_terms}
    For all $i, n, K, N\in\mathbb{Z}_{++}$, if $2 \leq K \leq \sqrt{N}$ and $n \leq N$,  then
    $$\frac{1}{(2i)!}\frac{K^{2i}}{N^{2i}}\ln^{2i}\left(\frac{K+n}{K-1}\right)-\frac{1}{(2i+1)!}\frac{K^{2i+1}}{N^{2i+1}}\ln^{2i+1}\left(\frac{K+n}{K+2i}\right) \geq 0.$$
\end{lemma}
\begin{proof}
    \begin{align*}
        0
        & \overset{(a)}{\leq} 1 - \frac{1}{2i+1}\frac{K}{N}\ln\left(\frac{K+n}{K-1}\right)\\
        & = \ln^{2i}\left(\frac{K+n}{K-1}\right)-\frac{1}{2i+1}\frac{K}{N}\ln^{2i+1}\left(\frac{K+n}{K-1}\right)\\
        & \leq \ln^{2i}\left(\frac{K+n}{K-1}\right)-\frac{1}{2i+1}\frac{K}{N}\ln^{2i+1}\left(\frac{K+n}{K+2i}\right)\\
        & = \frac{1}{(2i)!}\frac{K^{2i}}{N^{2i}}\ln^{2i}\left(\frac{K+n}{K-1}\right)-\frac{1}{(2i+1)!}\frac{K^{2i+1}}{N^{2i+1}}\ln^{2i+1}\left(\frac{K+n}{K+2i}\right)\\
    \end{align*}
    where $(a)$ follows for all $n \leq \left(K-1\right)e^{\frac{3N}{K}}-K$ which is implied by $n \leq N$ for $K \geq 2$.
\end{proof}

\begin{lemma}\label{le:multi_lb}
    For all $n, K, N\in\mathbb{Z}_{++}$, if $2 \leq K \leq \sqrt{N}$ and $n \leq N$, then 
    $$\prod_{i=0}^{n-1}\left(1 - \frac{\frac{K}{N}}{K+i}\right) \geq 1 - \frac{K}{N}\ln\left(1 + \frac{n}{K}\right).$$
\end{lemma}
\begin{proof}
    \begin{align*}
        \prod_{i=0}^{n-1}\left(1 - \frac{\frac{K}{N}}{K+i}\right)
        & \overset{(a)}{\geq} 1 - \sum_{i=0}^{\lfloor\frac{n-1}{2}\rfloor}\frac{1}{(2i+1)!}\frac{K^{2i+1}}{N^{2i+1}}\ln^{2i+1}\left(\frac{K+r}{K+2i}\right)\\
        &\quad + \sum_{i=1}^{\lceil\frac{n-1}{2}\rceil}\frac{1}{(2i)!}\frac{K^{2i}}{N^{2i}}\ln^{2i}\left(\frac{K+r}{K-1}\right)\\
        & \overset{(b)}{\geq} 1 - \frac{K}{N}\ln\left(1 + \frac{n}{K}\right)\\
        &\quad + \sum_{i=1}^{\lfloor\frac{n-1}{2}\rfloor}\frac{1}{(2i)!}\frac{K^{2i}}{N^{2i}}\ln^{2i}\left(\frac{K+n}{K-1}\right)-\frac{1}{(2i+1)!}\frac{K^{2i+1}}{N^{2i+1}}\ln^{2i+1}\left(\frac{K+n}{K+2i}\right)\\
        & \overset{(c)}{\geq} 1 - \frac{K}{N}\ln\left(1 + \frac{n}{K}\right),
    \end{align*}
    where $(a)$ follows from Lemma \ref{le:multinomial_terms}, and $(b)$ follows from Lemma \ref{le:alt_terms}
\end{proof}

\subsection{Lemmas pertaining to Dirichlet Multinomial}
The following result upper bounds the expected number of unique classes drawn from a Dirichlet-multinomial distribution with $n$ draws and $\alpha = [K/N, \ldots, K/N]\in\Re^N$.

\begin{lemma}
    \label{le:num_unique}
    For all $n, K, N \in \mathbb{Z}_{++}$ s.t. $K \leq \sqrt{N}$ and $n \leq N$, if $\proxytheta' \sim {\rm DirMult}(n, \alpha)$ for $\alpha = \left[K/N, \ldots, K/N\right] \in \Re^{N}$, then
    $$\E\left[\sum_{i=1}^{N} \mathbbm{1}_{[\proxytheta'_i > 0]}\right] \leq K\ln\left(1 + \frac{n}{K}\right).$$
\end{lemma}
\begin{proof}
    \begin{align*}
        \E\left[\sum_{i=1}^{N} \mathbbm{1}_{[\proxytheta'_i > 0]}\right]
        & = N\cdot\Pr(\proxytheta'_i > 0)\\
        & \overset{(a)}{=} N\cdot\left(1 - \Pr(\proxytheta'_1=0, \proxytheta'_2+\cdots+\proxytheta'_N = n)\right)\\
        & = N\cdot\left(1 - \frac{\Gamma(K)\Gamma(n+1)}{\Gamma(n+K)}\cdot\frac{\Gamma\left(\frac{K}{N}\right)}{\Gamma\left(\frac{K}{N}\right)\Gamma\left(1\right)}\cdot\frac{\Gamma\left(n+\frac{K}{N}(N-1)\right)}{\Gamma\left(\frac{K}{N}(N-1)\right)\Gamma\left(n+1\right)}\right)\\
        & = N\cdot\left(1 -\frac{\Gamma(K)}{\Gamma(n+K)}\cdot\frac{\Gamma\left(n+K-\frac{K}{N}\right)}{\Gamma\left(K-\frac{K}{N}\right)}\right)\\
        & = N\cdot\left(1 -\frac{\prod_{i=0}^{n-1}\left(K-\frac{K}{N}+i\right)}{\prod_{i=0}^{n-1}K+i}\right)\\
        & = N\cdot\left(1 - \prod_{i=0}^{n-1}\left(1 - \frac{\frac{K}{N}}{K+i}\right)\right)\\
        & \overset{(b)}{\leq} N\cdot\left(1 - \left(1 - \frac{K}{N}\ln\left(1 + \frac{n}{K}\right)\right)\right)\\
        & = K\ln\left(1 + \frac{n}{K}\right).
    \end{align*}
    where $(a)$ follows from the aggregation property of the Dirichlet-multinomial distribution and $(b)$ follows from Lemma \ref{le:multi_lb}.
\end{proof}
Note that this upper bound is \emph{independent} of $N$. In the next section, we will apply the dominated convergence theorem to bound the number of unique basis functions drawn by our learning model.

The extension to a Dirichlet Process $(N \to\infty)$ follows trivially.
\begin{lemma}\label{le:num_unique_inf}
    For all $n, K \in \mathbb{Z}_{++}$, if $\theta$ is distributed according to a Dirichlet process with base distribution ${\rm Uniform}(\sphere)$ with scale parameter $K$ and $\tilde{\theta} \sim {\rm Multinomial}(\theta)$, then
    $$\E\left[\sum_{w\in\mathcal{W}}\mathbbm{1}_{[\proxytheta_w > 0]}\right] \leq K\ln\left(1+\frac{n}{K}\right).$$
\end{lemma}
\begin{proof}
    \begin{align*}
        \E\left[\sum_{w\in\mathcal{W}}\mathbbm{1}_{[\proxytheta_w > 0]}\right]
        & = \E\left[\E\left[\sum_{w\in\mathcal{W}}\mathbbm{1}_{[\proxytheta_w>0 ]}\bigg|\mathcal{W}\right]\right]\\
        & \overset{(a)}{=} \E\left[\E\left[\lim_{N\rightarrow\infty}\sum_{w\in\tilde{\mathcal{W}}}\mathbbm{1}_{[\proxytheta_w' > 0]}\bigg|\mathcal{W}\right]\right]\\
        & \overset{(b)}{=}
        \lim_{N\rightarrow\infty}\E\left[\E\left[\sum_{w\in\tilde{\mathcal{W}}}\mathbbm{1}_{[\proxytheta_{w}' > 0]}\bigg|\mathcal{W}\right]\right]\\
        & \overset{(c)}{\leq} \lim_{N\rightarrow\infty}K\ln\left(1+\frac{n}{K}\right)\\
        & = K\ln\left(1+\frac{n}{K}\right),
    \end{align*}
    where in $(a)$, $\mathcal{W}_N$ is a subset of the first $N$ elements of $\mathcal{W}$ and $\tilde{X}_N$ is ${\rm DirMult}(n, \alpha_N)$ where $\alpha_N = [K/N, \ldots, K/N] \in\Re^N$ and $\mathcal{W}_N$ is the set of classes, $(b)$ follows from the dominated convergence theorem since $|\sum_{w\in\tilde{\mathcal{W}}}\mathbbm{1}_{[\proxytheta_w' > 0]}|\leq n$, and $(c)$ follows from Lemma \ref{le:num_unique}.
\end{proof}

\dirMulEnt*
\begin{proof}
    Let $N_m$ denote the number of \emph{unique} classes from the $m$ independent categorical samples $\tilde{A}_1, \ldots, \tilde{A}_m$.  Then,
    \begin{align*}
        \H(\tilde{A})
        & \overset{(a)}{\leq} \E\left[N_m \cdot \left(\ln(m) + \ln(|\sphere_\epsilon|)\right) \right]\\
        & \overset{(b)}{\leq} K\ln\left(1 + \frac{m}{K}\right)\left(\ln m + d\ln\frac{3}{{\epsilon}^2}\right)\\
    \end{align*}
    where $(a)$ follows from Theorem \ref{th:entropy_code} and the fact that a realization $\tilde{A}$ can be mapped to a codeword of length $N_m (\ln(2m) + \ln(|\sphere_\epsilon|))$ by using $\ln(m)$ nats to encode the number of times a class was drawn and $\ln(|\sphere_\epsilon|)$ nats to encode the class and $(b)$ follows from Lemma \ref{le:num_unique_inf} and the fact that $|\sphere_{\epsilon}| \leq (3/{\epsilon}^2)^d$.
\end{proof}

\end{document}

%% file: sections/introduction.tex
\section{Introduction}\label{sec:introduction}


From conquering games such as go, which were thought to require human-level learning and abstraction capabilities \citep{silver2016mastering}, to  producing systems which are capable of displaying common sense and holding coherent dialogues with humans around the globe \citep{achiam2023gpt}, the progress of machine learning in the past decade has far exceeded expectations.  It is undeniable that these artifacts will be remembered throughout the future of humanity's pursuit of understanding intelligence.

In retrospect, it is both remarkable and unsettling that these milestones were achievable with little to no rigorous theory to guide experimentation.  While theorists have attempted to repurpose existing statistical tools to analyze modern machine learning, the insights have largely been insufficient to explain empirical observations.  \cite{zhang2021understanding} aptly demonstrated this point via a series of simple experiments which elucidated the fundamental incompatibility of empirically observed phenomena with existing notions of generalization.  Despite this shortage in theoretical understanding, practitioners have been able to guide their future experimentation based on prior large-scale empirical investigations.  However, without a clear idea of how these phenomena slot into a larger picture, many research efforts will continue to be led astray.  Alluding to Plato's Allegory of the Cave, it is likely that the \emph{observations} which form the field's notion of reality are but shadows representing \emph{fragments} of that reality.  As scientists, we strive to reach understanding, and since existing theoretical frameworks have failed to provide this understanding, pursuit of a framework which does is an enticing and worthwhile endeavor.

In this work, we propose a theoretical framework which attempts to answer what exists outside of the cave.  To the theorist, we provide a framework which is mathematically rigorous and leaves open many interesting ideas for future exploration.  To the practitioner, we provide a framework whose results are very intuitive, general, and which will help form principles to guide future investigations.  Concretely, we provide a theoretical framework rooted in Bayesian statistics which is general enough to unify the analysis of many machine learning paradigms.  These settings range from classical learning from exchangeable data, to data which exhibits strong sequential or hierarchical structure.  We also consider learning under misspecification, when the learning model is fundamentally incongruent with reality, a learning setting which becomes all the more pertinent as we tackle ever more complex machine learning problems.

Our theoretical framework draws inspiration from Shannon's theory of information \emph{and} his maxim of ``information first, then computation''.  
The turn of the twentieth century brought a wave of interest in communications research; work that would enable the transmission of signals across long distances.  Much of the work in encoding/decoding was approached heuristically, similarly to how deep learning is today.  Shannon's theory and maxim redirected attention to characterizing what was fundamentally \emph{possible} or \emph{impossible}, in the absence of computational constraints.  His theory guided the discovery of algorithms that achieved these fundamental limits and eventually practical implementations as well.

The aforementioned feats of machine learning and artificial intelligence have fueled optimism that anything is learnable with sufficient data and computation.  However, research directions have largely been informed by informal reasoning supported by a plethora of empirical studies.  While work in statistics provides some guidance, the results for the most part lack the generality required to explain the continuing onslaught of novel empirical findings.  This monograph aims to provide a general framework to elucidate what is possible by studying how the limits of performance in machine learning depend on the informational content of the data.  Our framework is based on Shannon's information theory and characterizes the dependence of performance on information in the absence of computational constraints.  We begin by characterizing the performance of an optimal Bayesian learner that observes data generated by a suite of data generating processes of increasing complexity.  Performance is measured by the expected log-loss, of which the reducible component is equal to the expected KL divergence between the clairvoyant predictive distribution will full knowledge of the data generating measure and the predictive distribution produced by our learner.  By characterizing the performance of the optimal Bayesian learner, we express what is fundamentally possible in machine learning and develop intuition that can \emph{guide} fruitful investigation. 

Unlike existing analyses which weaken with increasing data complexity, our theoretical tools provide accurate insights across diverse machine learning settings.  For example, previous theories about learning from sequential data often rely on specific and rigid mixing time assumptions.  In this monograph, we outline the work of \citet{jeon2024informationtheoretic} which leverages information-theoretic tools to characterize the sample complexity of learning from sequences which are autoregressively generated by transformers \citep{vaswani2017attention}.  We also present the extensions which enable the analysis of hierarchical data generating processes which resemble meta-learning and in-context learning in large language models (LLMs).  The fact that these analytic techniques remarkably apply whether data is exchangeable or exhibits more complex structure provides evidence that our findings are fundamental.

In recent years, we have observed that training larger machine learning models on more data continues to produce \emph{significantly} better performance.  
This continual improvement indicates that the data generating processes that we study are \emph{more complex} than the machine learning models which we fit.
We refer to this phenomenon as ``misspecification'' and it is prominently observed in natural language processing (NLP), where researchers have tried to mathematically characterize this improvement in performance \citep{kaplan2020scaling, hoffmann2022training}.  The ``neural scaling laws'' of \cite{kaplan2020scaling} and \cite{hoffmann2022training} characterize the rate at which out-of-sample log-loss decreases with increases in available compute and data.  While these works provide extensive empirical experimentation, their cursory mathematical analyses leave open many questions regarding how scaling laws change depending on the complexity of the data-generating process.  \cite{jeon2024informationtheoreticfoundationsneuralscaling} use the theoretical tools from this monograph to rigorously characterize the error incurred by a misspecified algorithm.  This monograph expands upon these results and delivers an improved theory which characterizes the error of a much more natural misspecified learner: one which performs Bayesian inference with respect to a misspecified prior distribution.  With these results, we study a data-generating process that is identified by a single hidden-layer network of \emph{infinite} width and characterize how an algorithm with finite compute budget should optimally allocate between parameter count and dataset size.  These results notably are consistent (up to logarithmic factors) with the findings of \cite{hoffmann2022training}, where the optimal dataset size and parameter count exhibit a linear relationship.

Despite the fact that our theory does not address computational constraints, empirical studies with neural networks suggest that practical stochastic gradient algorithms suffice to attain the tradeoffs that our theory establishes between information and performance \citep{zhu2022stochastic}. 
 Throughout this work, we derive theoretical results and demonstrate their generality by applying them to derive insights specific to settings ranging from data which is iid under an unknown distribution to data which is sequential to data which exhibits hierarchical structure amenable to meta-learning.  We conclude with a section dedicated to characterizing the performance of \emph{suboptimal} algorithms that are based on the aforementioned \emph{misspecified} models, an exciting and relevant direction for future work.

%% file: sections/related_works.tex
\section{Related Works}\label{sec:related_works}

\subsection{Frequentist and Bayesian Statistics}
We begin with a discussion of frequentist statistics, the predominant framework which encompasses existing theoretical results.  As its name would suggest, in frequentist statistics, probability describes how often an event occurs if a procedure is repeated many times.  For instance, suppose there exists an \emph{unknown} parameter $\theta \in \Re$ and a sample of size $T$: $(X_1, X_2, \ldots, X_T)$ which is drawn iid $\normal(\theta, 1)$.
After observing the sample, the frequentist statistician may construct a $95\%$ confidence interval for the unknown $\theta$.  However, recall that in frequentism, probability is assigned to how often an event occurs if a procedure is repeated many times.  For our example, this entails that if random samples of size $T$ were drawn repeatedly and their corresponding confidence intervals constructed, then $95\%$ of those confidence intervals would contain $\theta$.  Note that the unknown parameter $\theta$ is \emph{fixed} and hence not a \emph{random variable} in the frequentist framework.  As a result, the frequentist framework does not use the tools of probability to model uncertainty pertaining to $\theta$.

In contrast, Bayesian statistics treats \emph{all} unknown quantities as random variables.  A consequence is that these quantities must be assigned subjective probabilities which reflect one's prior beliefs pertaining to their values.  Returning to our example, the Bayesian may assign a prior distribution $\Pr(\theta\in\cdot) = \normal(0, 1)$ which reflects their beliefs prior to observing the sample.  After observing the sample $(x_1, x_2, \ldots, x_T)$, they may construct a 95\% \emph{credible} interval for $\theta$, an interval $(a,b)$ for which $\Pr(\theta\in(a,b)|X_1=x_1, X_2=x_2, \ldots, X_T=x_T) \geq 0.95$.  The posterior distribution $\Pr(\theta\in\cdot|X_1=x_1, X_2=x_2, \ldots, X_T=x_T)$ is computed via Bayes rule.  Note that unlike the frequentist confidence interval which pertains to repeated experimentation, the Bayesian credible interval states that for this particular sample, with 95\% probability, $\theta \in (a,b)$.  However, we note that this probability is \emph{subjective} as it relies upon the prior subjective probability $\Pr(\theta\in\cdot)$.  While this \emph{subjectivism} has been a topic of constant philosophical debate, we note that in the realm of decision theory, it is well known that the choices of a decision maker which abides by axioms of rationality can be explained as the result of a utility function and subjective probabilities assigned to events \citep{savage1972foundations}.

While the debate surrounding these two schools of thought has continued for almost a century, it is prudent to consider the purpose for such theory.  We are reminded of Laplace's prudent remark that ``Probability theory is nothing but common sense reduced to calculation''.  Theory's merit ought to stem from the results it can provide for specific problems \citep{jaynes1976confidence}.  Jaynes and Kempthorne espoused this viewpoint as their background in physics led to their interest in the use of probability to describe and predict phenomena of the physical world.  Machine learning too is rooted in predictions based on data produced by the physical world.  Therefore, we argue that the merits of machine learning theory also ought to stem from its ability to describe and predict phenomena of data generated by the physical world.  To this end, we believe that the results which we derive via our framework both better reflect what is observed empirically and are also general enough to unify many disparate areas of the field.

\subsection{PAC Learning}

The majority of existing theoretical results for the analysis of modern machine learning are set in the probably approximately correct (PAC) learning framework \citep{valiant1984theory}.  In PAC learning, an algorithm is presented with a sample of data and is tasked with returning a hypothesis from a hypothesis set which can accurately perform predictions out-of-sample.  The phrase \emph{probably approximately correct} comes from the detail that these results are phrased as follows: ``For any data distribution, with probability at least $1-\delta$ over the randomness of an iid sample, the excess out-of-sample error is $\leq \epsilon$''.  ``Probably'' refers to the $1-\delta$ probability and ``approximately correct'' the $\epsilon$ tolerance of out-of-sample error.  While this framework has facilitated the development of influential theoretical concepts such as VC dimension \citep{vapnik2013nature} and Rademacher complexity \citep{bartlett2002rademacher}, \cite{zhang2021understanding} have demonstrated that these tools are inherently insufficient to explain modern empirical phenomena.  Namely, they demonstrate empirically that while the Rademacher complexity of a deep neural network leads to vacuous theoretical results, the observed out-of-sample error of these deep neural networks is actually small.

We posit that the looseness of these theoretical results stems from the fact that they hold for \emph{any} data distribution and \emph{uniformly} over the hypothesis set.  While it is true that Rademacher complexity depends on the distribution of the inputs, it does \emph{not} depend on the joint distribution of the input \emph{and} outputs.  Meanwhile, data which we observe from the real world clearly contains inherent structure between input and output which facilitates sample-efficient learning.  Suppose we perform binary classification with input $X\sim \normal(0, I_d)$.  In case $1$, consider a data generating process in which the corresponding class $Y$ only depends on the first element of $X$.  In case $2$, consider a data generating process in which $Y$ depends on all $d$ elements of $X$.  Common sense would dictate that if we observed an equal number of samples from each data generating process and considered the same hypothesis spaces, the \emph{generalization error} of case $1$ ought to be lower than that of case $2$.  However, an analysis via VC dimension or Rademacher complexity would result in the \emph{same} generalization bound for case $1$ and $2$.  This problem is exacerbated by high dimensional input distributions and overparameterized hypothesis classes, both of which are prevalent qualities of deep learning.  Both empirical \citep{jiang2019fantastic} and theoretical \citep{gastpar2023fantastic} analyses of this subject have demonstrated that there are no generalization bounds which are \emph{uniformly} tight across all data distributions for overparameterized neural network classes.


Several lines of analysis have attempted to ameliorate this via \emph{data dependent} bounds.  While these results can be derived for any data distribution, the \emph{choice} of data distribution will impact the resulting error bound.  Therefore, such a result will be vacuous (as expected) for a problem instance with unstructured data, but potentially much tighter for one which exhibits structure.  The main frameworks for data dependent PAC results involve PAC Bayes \citep{mcallester1998some} and the information-theoretic framework of \cite{xu2017information, hellström2024generalizationboundsperspectivesinformation}.  Both frameworks involve an algorithm which produces a predictive distribution of the hypothesis conditioned on the observed data (hence they analyze a Bayesian algorithm under the PAC framework).  While PAC Bayes results typically hold with high probability over random draws of the data while the information-theoretic results typically hold in expectation over random draws of the data.  Therefore, these results upper bound generalization error via the KL divergence or the mutual information between the observed data and the hypothesis respectively.

These data dependent results mark a significant step in understanding the puzzling empirical success of deep learning.  Notably, \cite{dziugaite2017computing} establish PAC-Bayes results for deep neural networks which result in bounds that dramatically improve upon those based on parameter count or VC dimension.  These results reflect the importance that the \emph{data generating process} has on the generalization error.  However, a limitation is the lack of theoretical tools which facilitate analytic derivations.  Namely, the aforementioned KL divergence/mutual information which bound generalization cannot be computed analytically outside of simple problem instances.  This is because these quantities involve the posterior distribution of the hypothesis conditioned on the data, which cannot be expressed analytically outside of simple instances which exhibit a conjugate distribution.  In contrast, our results analyze these quantities in a Bayesian setting, allowing us to develop general tools to bound mutual information analytically without needing to write down these complicated posterior distributions.  The conciseness and generality of these results lead us to believe that they are fundamental.

\subsection{Information Theory}

The results of this work elucidate the tight relation between error in learning and information measures of Shannon's theory \citep{shannon1948mathematical}.  Our framework characterizes the expected KL divergence of an optimal Bayesian learner.  For the case of learning from data $(X_1,\ldots,X_T)$ drawn iid $\mathcal{N}(\theta,1)$, the error is \emph{exactly} $\I(\theta;X_{0:T-1})/T$.  This result is well known and a special case of Theorem \ref{th:error_info}, which we will establish in Chapter \ref{sec:info_learn}.  Analogues of this result arise in the minimax setting via the well-known redundancy capacity theorem.  However, the mutual information between observed data and model parameters can be difficult to characterize for complex problem settings.  In this work, we derive upper \emph{and} lower bounds on in terms of the \emph{rate-distortion function} of $\theta$, a generalization of metric-entropy.  This rate-distortion function facilitates analysis as it is often simple to bound even for complex problem settings.

The learning setting which we explore in this work follows an existing line of work in universal prediction \citep{universalprediction}. Results of similar flavor to ours have been established in the minimax setting in \citep{yangbarron1999}.  This work characterizes the minimax error in the same problem setting in terms of the metric entropy and packing entropy.  We provide a detailed overview of these results in section \ref{subsec:minimax}.  While there is substantial similarity in the results, the practical significance of our methods comes from the relative ease in which priors can be specified in the Bayesian setting in comparison to the difficulty of specifying a suitable hypothesis set in the frequentist (minimax) setting.  While in theory, one may be able to devise a suitable hypothesis set which provides comparable sample complexity results to those of the Bayesian setting, in practice, this pursuit may be impossible for complex problems (such as those pertaining to deep neural networks).

A widely known framework involving information theory and machine learning is the information bottleneck method \citep{tishby2000information}.  On the surface, this work exhibits many similarities to ours as it draws a connection between information theory, rate-distortion theory and machine learning.  However, the two works diverge in their purpose.  The information bottleneck framework describes a learning objective rooted in information theory and prescribes a learning algorithm to solve this optimization problem.  While they leveraged their framework to produce early work on generalization in deep neural networks \citep{tishby2015deep}, the results remained very abstract.  While they devised metrics which they approximate empirically, they do not provide theoretical tools to analyze these metrics analytically.  In contrast, we present our framework with a collection of theoretical tools which facilitate analytic solutions.  This is an important property of a theoretical framework as it allows the researcher to \emph{forecast} what ought to be possible in practice.

As alluded to above, there exist notable information-theoretic generalization results in the PAC learning framework introduced by \cite{russo2019much} and expanded upon by \cite{xu2017information}.  While our work shares similar analytic techniques, we are able to strengthen the results and provide more streamlined theoretical tools by framing results in a Bayesian setting.  In particular, this framing allows us to upper \emph{and} lower bound the mutual information between the data and the data-generating function via the rate-distortion function, which we can characterize analytically for even complex data generating processes.  We find the results of our framing to be fundamental and we hope that the readers share this sentiment.  We also hope that they provide the reader with a new perspective on machine learning.  We note that the work of \citet{pmlr-v178-sefidgaran22a} developed independently but concurrently with \citep{NEURIPS2022_15cc8e4a} provides this rate-distortion characterization to upper bound the KL divergence/mutual information alluded to above in the frameworks of \citep{russo2019much} and \citep{xu2017information}.  Notably missing from this work is the matching lower bound in terms of the rate-distortion function and the application of these results to study complex problem instances involving neural networks and sequential and/or hierarchical data.

The recent advances in LLMs have incited an interest in the connection between learning and \emph{compression}.  In particular, researchers have posited that models which are better able to compress the observed data will achieve lower out-of-sample error.  This point is conveyed in \citep{deletang2024language}.  However, they do not provide a mathematically rigorous connection between learning and compression.  In this work we establish a rigorous connection between learning and optimal \emph{lossy} compression.  The loss incurred by an optimal Bayesian learner is upper and lower bounded by appropriate expressions containing the rate-distortion function (characterization of optimal lossy compression).  For this community, we hope that our work provides clarity to this matter and informs future experimentation and algorithm design. 

%% file: sections/objective.tex
\section{A Framework for Learning} \label{sec:objective}

\subsection{Probabilistic Framework and Notation}\label{subsec:probability}
In our work, we define all random variables with respect to a common probability space $(\Omega, \mathbb{F}, \Pr)$.  Recall that a random variable $\theta$ is a measurable function $\Omega\mapsto\Theta$ from the sample space $\Omega$ to a set $\Theta$.

The probability measure $\Pr:\mathbb{F} \mapsto [0,1]$ assigns probabilities to events in the $\sigma-{\rm algebra}$ $\mathbb{F}$.  In particular, for any event $F \in \mathbb{F}$, $\Pr(F)$ denotes the probability of the event.  For events $F,G\in \mathbb{F}$ for which $\Pr(G) > 0$, $\Pr(F|G)$ denotes the probability of event $F$ conditioned on event $G$.

For each realization $z$ of a random variable $Z$, $\Pr(Z=z)$ is hence a function of $z$.  We denote the value of this function evaluated at $Z$ by $\Pr(Z)$.  Therefore, $\Pr(Z)$ is a random variable (since it takes realizations in $[0,1]$ depending on the value of $Z$).  Likewise for realizations $(y,z)$ of random variables $Y,Z$, $\Pr(Z=z|Y=y)$ is a function of $(y,z)$ and $\Pr(Z|Y)$ is a random variable which denotes the value of this function evaluated at $(Y,Z)$.

If random variable $Z:\Omega\mapsto\Re^K$ has density $p_Z$ w.r.t the Lebesgue measure, the conditional probability $\Pr(F|Z=z)$ is well-defined despite the fact that for all $z$, $\Pr(Z=z) = 0$.  If function $f(z) = \Pr(F|Z=z)$ and $Y:\Omega\mapsto\Re^K$ is a random variable whose range is a subset of $Z$'s, then we use the $\leftarrow$ symbol with $\Pr(F|Z\leftarrow Y)$ to denote $f(Y)$.  Note that this is different from $\Pr(F|Z=Y)$ since this conditions on the event $Z=Y$ while $\Pr(F|Z\leftarrow Y)$ indicates a change of measure.

For any random variable $\theta$, we use the notation $\Pr(\theta\in\cdot)$ to denote the \emph{distribution} of that random variable i.e. $\Pr(\theta\in\cdot) = \Pr(\theta^{-1}(\cdot))$, where $\theta^{-1}(\cdot)$ denotes the pre-image of $\cdot$ (the pre-image must be $\in \mathbb{F}$ due to measurability).  We make a clear distinction between a random variable and its distribution in this way to provide accurate definitions of information-theoretic quantities later in this work.  As mentioned in the introduction, our framework is Bayesian, and hence uses the tools of probability to model uncertainty about the unknown value of a variable of interest (for instance $\theta$).  This involves treating $\theta$ as a random variable with a prior distribution $\Pr(\theta\in\cdot)$ which encodes the designer's prior information about the value of this variable.  The designer will often never directly observe $\theta$, but rather a stream of \emph{data} which will contain information about $\theta$.  Machine Learning is therefore the process of reducing uncertainty about $\theta$ in ways that are relevant for making better predictions about the future of this data stream.

\subsection{Data Generating Process}\label{sec:continual_interaction}

In machine learning, we are interested in discerning the relationship between input and output pairs $(X, Y)$.  Most frameworks of machine learning focus on a static dataset of fixed size and hope to characterize the performance of a predictive algorithm which leverages the information from the dataset for future predictions.  However, any practical system will continually have access to additional observations as it interacts with the environment.  As a result, it is prudent to consider a framework in which the data arrives in an online fashion and the objective is to perform well across all time.

We consider a stochastic process which generates a sequence $((X_t, Y_{t+1}): t \in \Z_{+})$ of data pairs.  For all $t$, we let $H_t$ denote the history $(X_0, Y_1, \ldots, X_{t-1}, Y_{t}, X_t)$ of experience.  We assume that there exists an underlying latent variable $\theta$ which induces a conditional probability measure $\theta(\cdot|H_t) = \Pr(Y_{t+1}\in\cdot|\theta, H_t)$ to the next label $Y_{t+1}$, given $H_t$.  We further assume that for all $t$, $X_{t}\perp \theta|(H_{t-1}, Y_t)$ and $(X_0, X_1, \ldots) \perp \theta$.  In the case of an \emph{iid} data generating process, this conditional probability measure would only depend on $H_t$ via $X_t$.  Furthermore, such a latent variable $\theta$ must exist under an infinite exchangeability assumption on the sequence $((X_t, Y_{t+1}):t \in \Z_{+})$ by de Finetti's Theorem.  While we will first focus on this iid setting, we will also study learning settings in which the future data may be arbitrarily dependent on $H_t$ even when conditioned on $\theta$.  As our framework is Bayesian, we represent our uncertainty about $\theta$ by modeling it as a random variable with prior distribution $\Pr(\theta\in\cdot)$.

\subsection{Log-Loss}

Log-loss, also commonly referred to as {\it cross-entropy loss}, is the predominant loss function used to train deep neural networks.  The log-loss of a predictive distribution $P$ for a random variable $Y$ is $-\ln P(Y)$.  Our theoretical framework focuses on the analysis of expected log-loss.

A natural question is why we restrict our attention to the log-loss. After all, many theoretical frameworks are compatible with a wide range of loss functions.  The reason is that the log-loss uniquely sits at the intersection of several important considerations: empirical relevance, proper learning, and analytic elegance.

Further, the log-loss is a (strictly) \emph{proper} loss function.  This means that the data-generating distribution uniquely minimizes the expected log-loss:
$$\inf_{P}\ \E\left[-\ln P(Y_{t+1})|\theta, H_t\right] \ = \ \E\left[-\ln \Pr(Y_{t+1}|\theta, H_t)|\theta, H_t\right].$$
The fact that this loss function is proper ensures convergence of predictions to the data generating distribution.

Finally, as we will see in Section \ref{sec:avg_reward}, minimizing the expected log-loss is equivalent to minimizing the expected KL divergence between the data generating and predictive distributions.  The KL divergence, by virtue of being an f-divergence, exhibits many convenient analytic properties such as chain rules and data processing inequalities.  And among f-divergences, except for in degenerate cases, the KL divergence is the only f-divergence that offers a proper loss function \citep{binaryAlphabet,duchi2016lecture}.

\subsection{Error}\label{sec:avg_reward}

Minimizing the expected log-loss is equivalent to minimizing the expected KL divergence because, as we will see, the two differ by a constant.  But the expected KL divergence is better suited as a measure of error because its infimum is zero.  As such, we will express results in terms of the KL divergence instead of the log-loss.  For a countable range $\mathcal{Y}$, the KL divergence is defined as follows.

\begin{definition}{\bf (KL divergence)}\label{def:kl}
    Let $P$ and $Q$ be probability mass functions on $\mathcal{Y}$.  The KL divergence between $P$ and $Q$ is:
    $$\KL(P\| Q) = \sum_{y\in\mathcal{Y}} P(y) \ln \frac{P(y)}{Q(y)}.$$
\end{definition}

Note that for uncountable $\mathcal{Y}$, the KL divergence is defined by replacing $P$ and $Q$ with their appropriate Radon-Nikodym derivatives and replacing the sum with an integral.  We now establish the equivalence of minimizing expected log-loss and KL divergence:
\begin{lemma}{\bf (minimizing expected log-loss equals minimizing KL divergence)}
    For any random variable $Y:\Omega\mapsto\mathcal{Y}$,
    $$\argmin_{P}\ \E\left[- \ln P(Y) \right]  = \argmin_{P}\ \KL(\Pr(Y\in\cdot) \| P$$
\end{lemma}
\begin{proof}
    \begin{align*}
        \argmin_{P}\ \E\left[- \ln P(Y) \right]
        & = \argmin_{P}\ -\sum_{y\in \mathcal{Y}} \Pr(Y=y) \ln P(y)\\
        & \overset{(a)}{=}  \argmin_{P}\ \sum_{y\in \mathcal{Y}} \Pr(Y = y) \ln \Pr(Y=y)  - \sum_{y\in \mathcal{Y}} \Pr(Y=y) \ln P(y) \\
        & = \argmin_{P} \KL(\Pr(Y = \cdot) \| P),
    \end{align*}
    where $(a)$ follows from the fact that the first term does not depend on $P$.
\end{proof}
Going forward, we will focus on KL divergence as our primary notion of error.

\subsection{Achievable Error}
For all $t\in \Z_{+}$, our algorithm is tasked with providing a predictive distribution $P_t$ of $Y_{t+1}$ which may be derived from the history of data $H_t$ which has already been observed.  We denote such an algorithm as $\pi$.  For each $t$, this algorithm produces a predictive distribution $P_t = \pi(H_t)$.  To denote the clairvoyant predictive distribution, we use $P^*_t = \Pr(Y_{t+1}\in\cdot|\theta, H_t)$.  We will study the average error relative to this clairvoyant prediction over each horizon $T$:
$$\Lc_{T, \pi} = \frac{1}{T}\sum_{t=0}^{T-1} \E_{\pi} \left[ \KL\left(P^*_t \| P_t \right)\right].$$

As discussed in Section \ref{subsec:probability}, we define all random variables with respect to a common probability space.  As a result, the expectation operator $\E$ integrates over all random variables which we do not condition on.  We use the subscript $\pi$ in $\E_{\pi}$ to indicate the dependence of $P_t$, and thus the expectation, on $\pi$.  As $Y_{t+1}$ is the random variable which represents the next label that is generated by the underlying stochastic process, $P_t(Y_{t+1})$ denotes the probability that our prediction $P_t$ assigned to label $Y_{t+1}$.

Since we are interested in characterizing the limits of what is possible via machine learning, a natural question which arises is: For all $T$, which $\pi$ minimizes $\Lc_{T,\pi}$?  Since log-loss is a proper scoring rule, the optimal algorithm $\pi$ is one such that for all $t$, $P_t = \Pr(Y_{t+1}\in\cdot|H_t)$.  This distribution is commonly referred to as the posterior predictive distribution and going forward we denote it by $\hat{P}_t$.  The following result establishes optimality of $\hat{P}_t$.

\begin{lemma}{\bf (posterior predictive distribution is optimal)}\label{le:bayes_opt}
    For all $t \in \Z_{+}$, 
    $$\hat{P}_t \overset{a.s}{=}\ \argmin_{\pi}\ \E_{\pi}\left[ \KL\left(P^*_t\| P_t\right) \mid H_t\right].$$
\end{lemma}
\begin{proof}
    \begin{align*}
        \argmin_{\pi}\ \E_{\pi}\left[\KL\left(P^*_t\| P_t\right) \mid H_t\right]
        & \overset{a.s.}{=} \argmin_{\pi}\ \E_{\pi}\left[\KL(P^*_t\|\hat{P}_t) + \ln\frac{\hat{P}_t(Y_{t+1})}{P_t(Y_{t+1})}\bigg| H_t\right]\\
        & \overset{a.s.}{=} \argmin_{\pi}\ \E\left[\KL(P^*_t\|\hat{P}_t) | H_t\right] + \KL(\hat{P}_t\|P_t).
    \end{align*}
    The result follows from the fact that KL divergence is non-negative and the tower property.
\end{proof}

This result is rather convenient since it prescribes that across all problem instances, the optimal prediction is $\hat{P}_t$.  This is widely considered an advantage of Bayesianism over frequentism; the Bayesian need not specify an ad hoc algorithm to analyze/solve a problem.  While in practice it may be intractable to compute $\hat{P}_t$ exactly, for the purposes of characterizing \emph{achievable} error, it provides immense utility.  Going forward, we will restrict our attention to the \emph{optimal} achievable error which we denote by:
$$\Lc_T = \frac{1}{T}\sum_{t=0}^{T-1} \E\left[\KL(P^*_t \|\hat{P}_t )\right].$$

The process of learning should result in $\Lc_T$ vanishing as $T$ increases.  While we have established that Bayesian inference provides optimal predictions at each timestep, in its current form, it is unclear how to characterize $\Lc_T$ for problems in which the posterior distribution does not exhibit an analytic form.  In section \ref{sec:info_learn}, we will establish the connection between our learning framework and information theory.  This connection will facilitate the analysis of \emph{arbitrary} learning problems, even those for which $\hat{P}_t$ cannot be expressed analytically.

{\bf Remarks on Optimal Learner:}\\
To the reader, it may be unclear why we are concerned with characterizing the performance of an optimal learner.  After all, the predominant methods which are used to train deep neural networks are stochastic gradient methods.  The first reason is that studying an optimal learner allows us to understand the fundamental limits of \emph{information}.  This follows the maxim of Claude Shannon: ``information first, then computation'', alluding to the idea that performing mathematical analysis via a pristine framework which obviates away the inefficiencies and complexities of computation can provide clear intuition and targets that practical algorithms can aim for.

This point of view is not new in the realm of statistics, as many analyses in the frequentist framework study idealized learners such as the empirical risk minimizer or maximum likelihood estimator.  As for the specific posterior predictive distribution which we study, theoretical analyses exist in the frequentist setting through the work of \citet{clarkebarron1990} and the PAC-Bayesian framework \citep{NIPS2016_84d2004b}.

Despite the aforementioned ignorance with regards to computation, there is evidence that under certain problem settings, the theoretically optimal sample complexity rates are achieved via stochastic gradient methods.  For uniformly strongly convex problems, stochastic gradient methods with Polyak averaging achieve the optimal sample complexity rates suggested by the Cramer-Rao lower bound \citep{gadat2017optimalnonasymptoticboundruppertpolyak}.  Beyond this setting, the empirical analysis of \citet{zhu2022stochastic} has explored the effectiveness of Adam optimizer for learning from data which is generated by a random single-hidden layer teacher network.  When the width of the student network is optimized, Adam optimizer achieves the sample complexity rates prescribed by prior information-theoretic results \citep{NEURIPS2022_15cc8e4a} (also presented in Section \ref{sec:dnn}).  Further exploring where the performance of stochastic gradient methods and theoretical limits may agree or diverge is an interesting direction for future research.

In the following section, we overview requisite definitions and tools from information theory to establish the connection between learning and information theory.  For readers who are new to information theory, we provide the following section for completeness.  Even for readers who are familiar with information theory, there may be details or results in the following section that are worth revisiting.
\newpage
\begin{summary}
\begin{itemize}
\item The algorithm's observations through time $t$ form a {\bf history} $H_t = (X_0,Y_1,\ldots X_{t-1},Y_{t},X_t)$.
\item The algorithm $\pi$ produces for all $t$ a predictive distribution $P_t$ of $Y_{t+1}$ given the history $H_t$.
\item We denote the clairvoyant prediction via $P^*_t = \Pr(Y_{t+1}\in\cdot|\theta, H_t)$.
\item For any horizon $T \in \Z_{+}$, we define the error of an algorithm $\pi$ as
$$\Lc_{T, \pi} = \frac{1}{T}\sum_{t=0}^{T-1} \E_{\pi} \left[ \KL\left(P^*_t \| P_t \right)\right].$$
\item We denote the posterior predictive distribution by $\hat{P}_t = \Pr(Y_{t+1}\in\cdot|H_t)$.
\item $\pi$ which generates the prediction $\hat{P}_t$ for each $t\in \mathbb{Z}_{+}$ minimizes $\Lc_{T,\pi}$.
\item We denote the estimation error incurred by the optimal algorithm by
$$\Lc_T = \frac{1}{T}\sum_{t=0}^{T-1} \E\left[\KL\left(P^*_t \|\hat{P}_t \right)\right].$$
\end{itemize}
\end{summary}
\clearpage

%% file: sections/information_theory.tex
\section{Requisite Information Theory}\label{sec:info_theory}

In this section we outline definitions and results from information theory which we will refer to in later sections of this monograph.  For a comprehensive overview of the topic, we point the reader to \citep{cover2012elements}.

\subsection{Entropy}
In this text, $\H(X)$ denotes the \emph{entropy} of a discrete random variable $X: \Omega\mapsto\mathcal{X}$.  $\H$ is defined as follows:
$$\H(X) = \sum_{x \in \mathcal{X}} \Pr(X=x) \ln \frac{1}{\Pr(X=x)}.$$
Throughout this monograph, we use the convention that $0\ln 0 = 0$.  While there are many colloquial interpretations of entropy which describe it as the expected ``surprise'' associated with outcomes of a random variable, we provide a concrete motivation for entropy based on coding theory.

We begin our exposition of entropy with an introduction to coding theory.  We first define a code for a random variable.

\begin{definition}
    A code $C$ for random variable $X:\Omega\mapsto\mathcal{X}$ is a function that maps $\mathcal{X}\mapsto \{0,1\}^*$, where $\{0,1\}^*$ denotes the set of all binary strings. 
\end{definition}

When we send a text message to our friend, the characters that comprise our message can be thought of as the realizations of random variables.  In many applications involving digital data transfer, a message (outcome of a random variable) is encoded into a binary string which is passed through a communication channel, and decoded at an endpoint.  Since the binary string arrives in a stream, it would be convenient if the message could be uniquely decoded as the data is arriving.  A necessary and sufficient condition for online decodability is that the code $C$ is \emph{prefix-free}; i.e., no element in the image of $C$ is a \emph{prefix} of another element in the image of $C$.  We use $\mathcal{C}_X$ to denote the set of prefix-free codes for a random variable $X$.

Since these codes are stored, transmitted, and decoded, the memory footprint becomes a significant design consideration.  A natural question which arises is: ``how do we devise \emph{optimal} prefix-free codes?''  The notion of optimality which gives rise to Shannon entropy is the following:
$$\argmin_{C\in\mathcal{C}_X}\ \sum_{x\in\mathcal{X}} \Pr(X=x) \cdot \frac{\text{len}\left(C(x)\right)}{\log_2(e)},$$
where $\text{len}(c)$ denotes the length of binary string $c$.  A prefix-free code which minimizes this objective will on average require the fewest number of bits to transmit information.  A naive prefix-free code is one which maps each of the $|\mathcal{X}|$ outcomes of $X$ to a unique binary string of length $\lceil \log_2|\mathcal{X}| \rceil$.  While such a code would be reasonable if all outcomes of $X$ were equally likely, such a code would be highly suboptimal if some outcomes are much more or less likely than others.  A competent coding scheme ought to map more likely outcomes to \emph{shorter} strings and less likely outcomes to \emph{longer} strings.

\begin{figure}[H]
    \centering
    \includegraphics[width=0.4\textwidth]{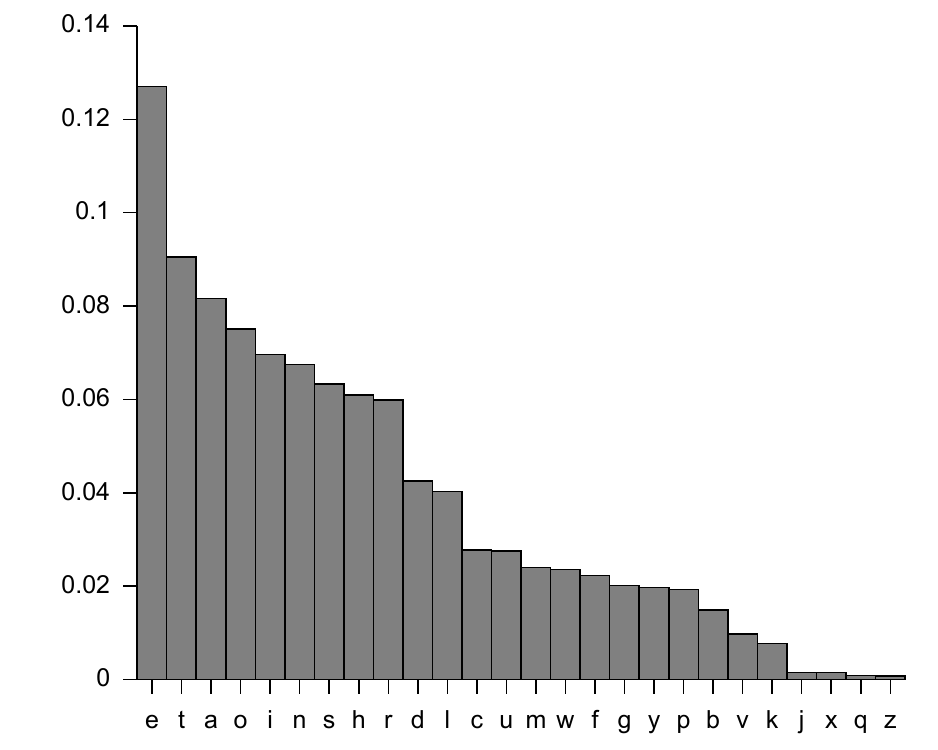}
    \caption{The English alphabet consists of characters of varying frequency across common text corpuses.  Notably, the vowels appear with greater frequency.  A coding scheme ought to map these more frequently appearing characters to shorter strings.}
    \label{fig:frequency}
\end{figure}

The following result establishes the tight connection between the \emph{entropy} of $X$ and its optimal prefix-free code.
\begin{theorem}{\bf (entropy characterizes optimal prefix-free code length)}\label{th:entropy_code}
    For all discrete random variables $X:\Omega\mapsto\mathcal{X}$,
    $$\H(X) \ \leq\ \min_{C\in\mathcal{C}_X}\ \sum_{x\in\mathcal{X}} \Pr(X=x)\cdot \frac{{\rm len}(C(x))}{\log_2(e)}\ \leq\ \H(X)+\frac{1}{\log_2(e)}.$$
\end{theorem}

This result demonstrates that the entropy of a random variable tightly characterizes the fundamental limit to which it can be losslessly compressed.  As a result, the entropy reflects the inherent complexity of a random variable.  This connection is useful to keep in mind as there exist the following analogies between our Bayesian and the frequentist frameworks:
\begin{align*}
    \text{frequentist}
    & \mapsto \text{Bayesian}\\
    \Theta
    & \mapsto \Pr(\theta\in\cdot)\\
    |\Theta|
    & \mapsto \H(\theta),
\end{align*}
where $\theta:\Omega\mapsto \Theta$ and $|\Theta|$ denotes the cardinality of $\Theta$.

\subsection{Conditional Entropy}
$\H(X|Y)$ denotes the \emph{conditional entropy} of a discrete random variable $X:\Omega\mapsto\mathcal{X}$ conditioned on another discrete random variable $Y:\Omega\mapsto\mathcal{Y}$.  The conditional entropy is defined as follows:
$$\H(X|Y) = \sum_{x\in \mathcal{X}, y\in\mathcal{Y}} \Pr(X=x, Y=y) \ln \frac{1}{\Pr(X=x|Y=y)}.$$
Note that \emph{unlike} conditional expectation, conditional entropy is a \emph{number} (and not a random variable).  Upon closer inspection it is clear that conditional entropy is also always non-negative and is $0$ only when $Y$ fully determines $X$.  On the other hand, when $X \perp Y$, we have that $\H(X|Y) = \H(X)$.  We provide the following result which facilitates mathematical manipulations involving the information-theoretic quantities outlined thus far.

\begin{lemma}{\bf (chain rule of conditional entropy)}\label{le:chain_rule_ent}
    For all discrete random variables $X: \Omega\mapsto \mathcal{X}, Y: \Omega\mapsto\mathcal{Y}$,
    $$\H(X, Y)\ =\ \H(X) + \H(Y|X) \ =\ \H(Y) + \H(X|Y).$$
\end{lemma}
\begin{proof}
    \begin{align*}
        \H(X) + \H(Y|X)
        & = \sum_{x\in\mathcal{X}} \Pr(X=x)\ln\frac{1}{\Pr(X=x)} + \sum_{x\in\mathcal{X}, y\in\mathcal{Y}} \Pr(X=x, Y=y)\ln\frac{1}{\Pr(Y=y|X=x)}\\
        & = \sum_{x\in\mathcal{X}, y\in\mathcal{Y}} \Pr(X=x, Y=y)\ln\frac{1}{\Pr(X=x)} + \sum_{x\in\mathcal{X}, y\in\mathcal{Y}} \Pr(X=x, Y=y)\ln\frac{\Pr(X=x)}{\Pr(X=x, Y=y)}\\
        & = \sum_{x\in\mathcal{X}, y\in\mathcal{Y}} \Pr(X=x, Y=y)\ln\frac{1}{\Pr(X=x, Y=y)}\\
        & = \H(X, Y).
    \end{align*}
    The second equality in the lemma statement follows from the same technique shown above.
\end{proof}

If $\H(X,Y)$ denotes the average length of a prefix-free code for $(X,Y)$ jointly, Lemma \ref{le:chain_rule_ent} establishes that $\H(X|Y)$ reflects the average length of a prefix-free code for $X$ after $Y$ is \emph{already} observed.  Evidently if $X\perp Y$, then observing $Y$ does not provide any information which enables a shorter code for $X$ (hence, $\H(X|Y) = \H(X)$).  However, in the other extreme in which $X \overset{a.s.}{=} Y$,  observing $Y$ means that $X$ is also known.  As a result, a trivial code which maps every outcome of $X$ to the null string $\emptyset$ will suffice (hence, $\H(X|Y) = 0$).  

\subsection{Mutual Information}
$\I(X;Y)$ denotes the \emph{mutual information} between two random variables $X:\Omega\mapsto \mathcal{X}$ and $Y:\Omega\mapsto\mathcal{Y}$.  Concretely,
$$\I(X;Y)\ =\ \KL\left(\Pr((X,Y)\in\cdot)\| \Pr(X\in\cdot ) \otimes \Pr(Y\in\cdot)\right),$$
where $\Pr(X\in\cdot) \otimes \Pr(Y\in\cdot)$ denotes the outer product distribution.  Note that KL divergence is a non-symmetric function which maps two \emph{probability distributions} to $\Re_+ \cup \{\infty\}$.  For discrete random variables, we have the following equivalence between mutual information and differences of (conditional) entropies:
$$\I(X;Y)\ =\ \H(X) - \H(X|Y)\ =\ \H(Y) - \H(Y|X).$$
Note that mutual information is symmetric i.e. $\I(X;Y) = \I(Y;X)$ and it is also always non-negative (follows directly as a consequence of Lemma \ref{le:gibbs}).  Intuitively, the mutual information $\I(X;Y)$ represents the amount of information that $X$ conveys about $Y$ and vice versa.  As such, if $X$ fully determines $Y$, then $\I(X;Y) = \H(X) = \H(Y)$.  Meanwhile if $X \perp Y$, then $\I(X;Y) = 0$ as the two random variables convey no information about each other.  As with conditional entropy, we provide the following result which facilitates mathematical analyses involving mutual information:

\begin{lemma}{\bf(chain rule of mutual information)}\label{le:chain_rule}
    For all random variables $X:\Omega\mapsto\mathcal{X}, Y:\Omega\mapsto\mathcal{Y}, Z:\Omega\mapsto\mathcal{Z}$,
    $$\I(Y;Z|X) + \I(X;Z) = \I(X,Y;Z).$$
\end{lemma}
\begin{proof}
    We provide a proof for discrete KL divergence.  The same holds for continuous case with the sums exchanged for integrals and the probability measures exchanged with the Radon-Nikodym derivative.
    \begin{align*}
        \I(X, Y;Z)
        & = \KL\left(\Pr((X, Y, Z)\in \cdot)\ \|\ \Pr((X, Y)\in\cdot) \otimes \Pr(Z\in \cdot)\right)\\
        & = \sum_{(x,y,z) \in \mathcal{X}\times\mathcal{Y}\times\mathcal{Z}}\Pr((X,Y,Z)=(x,y,z))\ln \frac{\Pr((X,Y,Z)=(x,y,z))}{\Pr((X,Y)=(x,y)) \cdot \Pr(Z=z)} \\
        & = \sum_{(x,y,z) \in \mathcal{X}\times\mathcal{Y}\times\mathcal{Z}}\Pr((X,Y,Z)=(x,y,z))\ln \frac{\Pr((Y,Z)=(y,z)|X=x)\cdot\Pr(X=x)}{\Pr((X,Y)=(x,y)) \cdot \Pr(Z=z)} \\
        & = \sum_{(x,y,z) \in \mathcal{X}\times\mathcal{Y}\times\mathcal{Z}}\Pr((X,Y,Z)=(x,y,z))\ln \frac{\Pr((Y,Z)=(y,z)|X=x)}{\Pr(Y=y|X=x) \cdot \Pr(Z=z)} \\
        & = \sum_{(x,y,z) \in \mathcal{X}\times\mathcal{Y}\times\mathcal{Z}}\Pr((X,Y,Z)=(x,y,z))\left(\ln \frac{\Pr((Y,Z)=(y,z)|X=x)}{\Pr(Y=y|X=x) \cdot \Pr(Z=z|X=x)} + \ln\frac{\Pr(Z=z|X=x)}{\Pr(Z=z)}\right) \\
        & = \I(Y;Z|X) + \sum_{(x,z)\in\mathcal{X}\times\mathcal{Z}} \Pr((X,Z)=(x,z))\ln\frac{\Pr((X,Z)=(x,z))}{\Pr(X=x)\cdot\Pr(Z=z)}\\
        & = \I(Y;Z|X) + \I(X;Z).
    \end{align*}
\end{proof}

\begin{figure}[ht]
    \centering
    \includegraphics[width=0.5\textwidth]{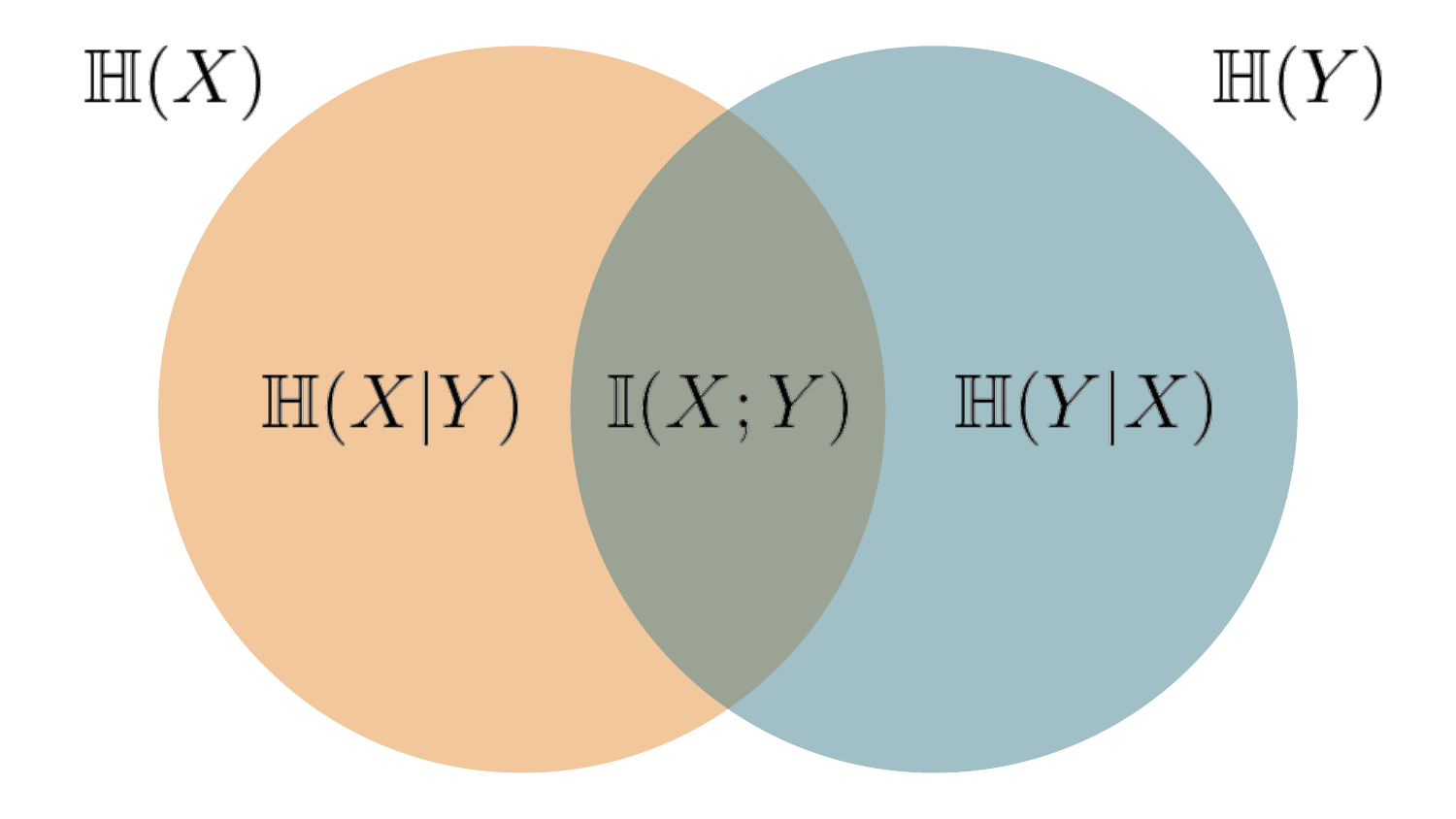}
    \caption{This venn diagram illustrates the relationships between the introduced information-theoretic quantities.}
    \label{fig:venn}
\end{figure}

\subsection{Differential Entropy}\label{subsec:diff_entropy}

While we have defined information-theoretic quantities for discrete random variables, outside of mutual information we have not broached a notion of information regarding \emph{continuous} random variables.  For a continuous random variable $X:\Omega\mapsto\mathcal{X}$ with density $p_X$ (w.r.t the Lebesgue measure), we denote the differential entropy of $X$ by
$$\diffentropy(X) \ =\ \int_{x\in\mathcal{X}} p_X(x)\ln\frac{1}{p_X(x)} d\mu(x),$$
where $\mu(\cdot)$ denotes the Lebesgue measure.  While differential entropy ostensibly resembles discrete entropy, the two are different in almost all regards.  For instance, while the discrete entropy of a random variable is always non-negative, the differential entropy can often be negative.  Look no further than $X = \text{Uniform}([0,1/2])$.  In this case, $\diffentropy(X) = -\ln(2)$.  Furthermore, while discrete entropy is invariant under one-to-one transformations, the differential entropy is not.  For instance, $\diffentropy(2X) = \diffentropy(X) + \ln 2$.  A measure of information should \emph{not} be negative nor should it be dependent on \emph{units} used.  For these reasons, unlike discrete entropy, differential entropy is \emph{not} a meaningful informational quantity by itself.  The correct extension of discrete entropy to continuous random variables exists via rate-distortion theory, which we will present in the following section.

While differential entropy itself is not a meaningful measure of information, \emph{differences} in (conditional) differential entropies are still equal to mutual information.  Concretely, for continuous random variables $X, Y$ with finite (conditional) differential entropies,
$$\I(X;Y)\ =\ \diffentropy(X) - \diffentropy(X|Y)\ =\ \diffentropy(Y) - \diffentropy(Y|X).$$

\subsection{Requisite Results from Information Theory}

We now present an amalgamation of widely known and requisite results from information theory.  Various proofs throughout this monograph will refer to the results of this section.

\begin{lemma}{\bf(log-sum inequality)}
    For all $n \in \mathbb{Z}_{++}$, if $a_1, \ldots, a_n \geq 0$, $b_1, \ldots, b_n \geq 0$, and $a = \sum_{i=1}^{n} a_i, b = \sum_{i=1}^{n} b_i$, then
    $$\sum_{i=1}^{n} a_i \ln \frac{a_i}{b_i} \geq a \ln \frac{a}{b}.$$
\end{lemma}
\begin{proof}
    \begin{align*}
        \sum_{i=1}^{n} a_i \ln \frac{a_i}{b_i}
        & = \sum_{i=1}^{n} b_i \cdot \frac{a_i}{b_i} \ln \frac{a_i}{b_i}\\
        & = b \sum_{i=1}^{n}\frac{b_i}{b}\cdot \frac{a_i}{b_i}\ln\frac{a_i}{b_i}\\
        & \overset{(a)}{\geq} b \left(\sum_{i=1}^{n}\frac{b_i}{b} \cdot \frac{a_i}{b_i}\right) \ln \left(\sum_{i=1}^{n}\frac{b_i}{b}\cdot\frac{a_i}{b_i}\right)\\
        & = b \cdot \frac{a}{b} \ln \frac{a}{b}\\
        & = a \ln \frac{a}{b},
    \end{align*}
    where (a) follows from Jensen's inequality applied to the function $x \ln x$.
\end{proof}

\begin{lemma}{\bf(conditioning reduces entropy)}\label{le:cond_ent}
    For all discrete random variables $X:\Omega\mapsto\mathcal{X}, Y:\Omega\mapsto\mathcal{Y}$,
    $$\H(X) \geq \H(X|Y).$$
\end{lemma}
\begin{proof}
    \begin{align*}
        \H(X)
        & = \sum_{x\in\mathcal{X}}\Pr(X=x)\ln\frac{1}{\Pr(X=x)}\\
        & = \sum_{x \in \mathcal{X}}\sum_{y\in\mathcal{Y}}\Pr(X=x, Y=y)\ln\frac{1}{\Pr(X=x)}\\
        & \overset{(a)}{\geq} \sum_{x\in\mathcal{X}}\sum_{y\in\mathcal{Y}}\Pr(X=x, Y=y)\ln\frac{\Pr(Y=y)}{\Pr(X=x,Y=y)}\\
        & = \sum_{x\in\mathcal{X}}\sum_{y\in\mathcal{Y}}\Pr(X=x, Y=y)\ln\frac{1}{\Pr(X=x|Y=y)}\\
        & = \H(X|Y),
    \end{align*}
    where $(a)$ follows from negating the log-sum inequality and setting $a_i = \Pr(X=x, Y=y), b_i = \Pr(Y=y)$ and $a = \Pr(X=x), b = 1$.
\end{proof}

\begin{lemma}{\bf(conditioning reduces differential entropy)}
For all continuous random variables $X:\Omega\mapsto\mathcal{X}, Y:\Omega\mapsto\mathcal{Y}$, if $\diffentropy(X), \diffentropy(X|Y)$ exist and are both finite, then
$$\diffentropy(X) \geq \diffentropy(X|Y).$$
\end{lemma}
\begin{proof}
    The proof follows from the same reasoning as in Lemma \ref{le:cond_ent} by constructing a sequence of partitions of $\mathcal{X}, \mathcal{Y}$ and taking the associated limits.
\end{proof}

\begin{lemma}{\bf (mutual information as expected KL divergence)}\label{le:kl_mi}
    For all random variables $X:\Omega\mapsto\mathcal{X}, Y:\Omega\mapsto\mathcal{Y}$,
    $$\I(X;Y) = \E\left[\KL\left(\Pr(Y\in\cdot|X)\ \|\ \Pr(Y\in\cdot)\right)\right].$$
\end{lemma}
\begin{proof}
    We prove the result for discrete random variables $X, Y$.  With appropriate technical assumptions, the result can also be extended to continuous random variables which exhibit density functions.
    \begin{align*}
        \I(X;Y)
        & = \sum_{x\in\mathcal{X}}\sum_{y\in\mathcal{Y}}\Pr(X=x,Y=y)\ln\frac{\Pr(X=x, Y=y)}{\Pr(X=x)\cdot\Pr(Y=y)}\\
        & = \sum_{x\in\mathcal{X}}\Pr(X=x)\sum_{y\in\mathcal{Y}}\Pr(Y=y|X=x)\ln\frac{\Pr(Y=y|X=x)}{\Pr(Y=y)}\\
        & = \E\left[\KL\left(\Pr(Y\in\cdot|X)\ \|\ \Pr(Y\in\cdot)\right)\right].
    \end{align*}
\end{proof}

\begin{lemma}{\bf (non-negativity of KL divergence)}\label{le:gibbs}
    For all probability distributions $P:\mathbb{F}\mapsto [0,1], Q:\mathbb{F}\mapsto [0,1]$,
    $$\KL\left(P \ \|\ Q\right) \geq 0.$$
    Furthermore,
    $$\KL\left(P\ \|\ Q\right) = 0 \iff \text{ for all } \nu\in\mathbb{F}, \left(P(\nu) > 0 \implies P(\nu) = Q(\nu)\right).$$
\end{lemma}
\begin{proof}
    \begin{align*}
        \KL\left(P \ \|\ Q\right)
        & = \sum_{\nu\in\Omega} P(\nu) \ln \frac{P(\nu)}{Q(\nu)}\\
        & = \sum_{\nu\in\Omega} - P(\nu) \ln \frac{Q(\nu)}{P(\nu)}\\
        & \overset{(a)}{\geq} - \ln \sum_{\nu\in\Omega} P(\nu) \cdot \frac{Q(\nu)}{P(\nu)}\\
        & = - \ln \sum_{\nu\in\Omega} Q(\nu)\\
        & = 0.
    \end{align*}
    where $(a)$ follows from Jensen's inequality.  To prove the second result, consider the case in which Jensen's inequality holds with \emph{equality}.  This occurs iff $Q(\nu)/P(\nu) = \sum_{\nu\in\Omega}P(\nu)\cdot Q(\nu)/P(\nu) = 1$. This occurs only when $Q(\nu) = P(\nu)$ for all $\nu \in \mathbb{F}$ for which $P(\nu) > 0$.
\end{proof}

\begin{lemma}{\bf (maximum differential entropy)}\label{le:max_entropy}
    For all density functions $q:\Re^{d}\mapsto\Re_{+}$, for all $i,j\in[d]$,
    let $\Sigma_{i,j} =\int_{x\in\Re^{d}} x_ix_j q(x) dx$.  If $X \sim \normal(0, \Sigma)$, then
    $$\diffentropy(X) \geq \int_{x\in\Re^d} q(x)\ln\frac{1}{q(x)}dx.$$
\end{lemma}
\begin{proof}
    Let $p$ denote the probability density function of $X$.
    \begin{align*}
        \diffentropy(X)
        & = \int_{x\in\Re^d} - p(x) \left(\frac{1}{2}\ln\left((2\pi e)^d|\Sigma|\right)- \frac{x^\top\Sigma^{-1}x}{2}\right)dx\\
        & = -\frac{1}{2}\ln\left((2\pi e)^d|\Sigma|\right) + \int_{x\in\Re^d} p(x) \cdot \frac{x^\top \Sigma^{-1} x}{2}dx\\
        & \overset{(a)}{=} -\frac{1}{2}\ln\left((2\pi e)^d|\Sigma|\right) + \int_{x\in\Re^d} q(x) \cdot \frac{x^\top \Sigma^{-1} x}{2}dx\\
        & = \int_{x\in\Re^d} q(x) \ln \frac{1}{p(x)}dx\\
        & = \int_{x\in\Re^d} q(x) \ln \frac{1}{q(x)}dx +  \int_{x\in\Re^d} q(x) \ln \frac{q(x)}{p(x)}dx\\
        & \overset{(b)}{\geq} \int_{x\in\Re^d} q(x) \ln \frac{1}{q(x)}dx
    \end{align*}
    where $(a)$ follows from the equivalence of covariances assumption and $(b)$ follows from Lemma \ref{le:gibbs}.
\end{proof}

\begin{lemma}{\bf (data processing inequality)}
    Let $X, Y, Z$ be random variables for which $X\perp Z|Y$, then
    $$\I(X;Z) \leq \I(Y;Z) \qquad \mathrm{and} \qquad \I(X;Z) \leq \I(X;Y).$$
\end{lemma}
\begin{proof}
\begin{align*}
    \I(X;Z)
    & \leq \I(X,Y;Z)\\
    & = \I(Y;Z) + \I(X;Z|Y)\\
    & \overset{(a)}{=} \I(Y;Z),
\end{align*}
where $(a)$ follows from the independence assumption.  Similarly,
\begin{align*}
    \I(X;Z)
    & \leq \I(X;Y,Z)\\
    & = \I(X;Y) + \I(X;Z|Y)\\
    & = \I(X;Y).
\end{align*}
\end{proof}

\newpage
\begin{summary}
    \begin{itemize}
        \item The \emph{entropy} of a discrete random variable $X: \Omega\mapsto\mathcal{X}$ is
        $$\H(X) = \sum_{x \in \mathcal{X}} \Pr(X=x) \ln \frac{1}{\Pr(X=x)}.$$
        \item The \emph{conditional entropy} of a discrete random variable $X:\Omega\mapsto\mathcal{X}$ conditioned on another discrete random variable $Y:\Omega\mapsto\mathcal{Y}$ is
        $$\H(X|Y) = \sum_{x\in \mathcal{X}, y\in\mathcal{Y}} \Pr(X=x, Y=y) \ln \frac{1}{\Pr(X=x|Y=y)}.$$
        \item The \emph{mutual information} between two random variables $X:\Omega\mapsto \mathcal{X}$ and $Y:\Omega\mapsto\mathcal{Y}$ is
        $$\I(X;Y)\ =\ \KL\left(\Pr((X,Y)\in\cdot)\| \Pr(X\in\cdot ) \otimes \Pr(Y\in\cdot)\right),$$
        where $\Pr(X\in\cdot) \otimes \Pr(Y\in\cdot)$ denotes the outer product distribution.  For discrete random variables,
        $$\I(X;Y)\ =\ \H(X) - \H(X|Y)\ =\ \H(Y) - \H(Y|X).$$
        \item {\bf(chain rule of mutual information)}
        For all random variables $X:\Omega\mapsto\mathcal{X}, Y:\Omega\mapsto\mathcal{Y}, Z:\Omega\mapsto\mathcal{Z}$,
        $$\I(Y;Z|X) + \I(X;Z) = \I(X,Y;Z).$$
        \item The \emph{differential entropy} of a continuous random variable $X:\Omega\mapsto\mathcal{X}$ with density $p_X$ (w.r.t the Lebesgue measure $\mu(\cdot)$) is
        $$\diffentropy(X) \ =\ \int_{x\in\mathcal{X}} p_X(x)\ln\frac{1}{p_X(x)} d\mu(x).$$
        \item Differential entropy can be negative and is unit-dependent.
        \item For continuous random variables $X:\Omega\mapsto \mathcal{X}, Y:\Omega\mapsto\mathcal{Y}$,
        $$\I(X;Y)\ =\ \diffentropy(X) - \diffentropy(X|Y)\ =\ \diffentropy(Y) - \diffentropy(Y|X).$$
        \item {\bf (mutual information as expected KL divergence)}
        For all random variables $X:\Omega\mapsto\mathcal{X}, Y:\Omega\mapsto\mathcal{Y}$,
        $$\I(X;Y) = \E\left[\KL\left(\Pr(Y\in\cdot|X)\ \|\ \Pr(Y\in\cdot)\right)\right].$$
    \end{itemize}
\end{summary}

\clearpage

%% file: sections/information_learning.tex
\section{Connecting Learning and Information Theory}\label{sec:info_learn}

In this section, we will leverage the requisite information-theoretic results of section \ref{sec:info_theory} to derive general upper and lower bounds for the error 
$$\Lc_T = \frac{1}{T}\sum_{t=0}^{T-1} \E\left[\KL(P^*_t \|\hat{P}_t )\right]$$
of the Bayesian posterior predictive $\hat{P}_t$.  The results of this section will facilitate the analysis of concrete problem instances in the following sections.

\subsection{Error is Information}

We now establish the elegant connection between $\Lc_T$ and mutual information.

\begin{theorem}{\bf (optimal error equals total information)}\label{th:error_info}
    For all $T \in \Z_{++}$,
    $$\Lc_T = \frac{\I(H_T;\theta)}{T}.$$
\end{theorem}
\begin{proof}
    \begin{align*}
        \Lc_T
        & = \frac{1}{T}\sum_{t=0}^{T-1} \E\left[-\ln \hat{P}_t(Y_{t+1})\right] - \frac{\H(Y_{1:T}|\theta, X_{0:T-1})}{T}\\
        & = \frac{1}{T}\sum_{t=0}^{T-1} \E\left[\ln \frac{\Pr(Y_{t+1}|\theta, H_t)}{\hat{P}_t(Y_{t+1})}\right]\\
        & = \frac{1}{T}\sum_{t=0}^{T-1} \E\left[\KL\left(\Pr(Y_{t+1}\in\cdot|\theta, H_t) \ \|\ \Pr(Y_{t+1}\in\cdot|H_t)\right)\right]\\
        & = \frac{1}{T}\sum_{t=0}^{T-1}\I(Y_{t+1};\theta|H_t)\\
        & \overset{(a)}{=} \frac{1}{T}\sum_{t=0}^{T-1}\I(Y_{t+1};\theta|H_t) + \frac{1}{T} \I(X_0;\theta) + \frac{1}{T}\sum_{t=1}^{T-1} \I(X_t;\theta|H_{t-1}, Y_t) \\
        & = \frac{\I(H_T;\theta)}{T}.
    \end{align*}
    where $(a)$ follows from the assumption that for all $t$, $X_{t}\perp\theta|(H_{t-1}, Y_t)$.
\end{proof}

The  error incurred by an optimal algorithm over horizon $T$ is \emph{exactly} equal to the total information acquired about $\theta$ from observing the data $H_T$.  Every nat of \emph{information} about $\theta$ can only be acquired via incurring \emph{error} on a prediction which relied on that information.  Evidently, an optimal algorithm never makes the same mistake twice, hence resulting in the equality between total loss incurred and total information acquired.  The duality of error and information has long been known, dating back to results such as the redundancy-capacity theorem which equates the minimax KL divergence error to the channel capacity.  As minimax error characterizes performance under the worst possible scenario, it is natural that it would be equivalent to the channel capacity, which maximizes the mutual information across all possible prior distributions.  We provide results of this flavor in Section \ref{subsec:minimax}.  In this section, we restrict our attention to the Bayesian error and establish novel connections to rate-distortion theory.

A natural question which may arise when inspecting Theorem \ref{th:error_info} is: ``Does $\Lc_T$ decay to $0$ in $T$ and if so, at what rate?''.  Ostensibly, the numerator $\I(\theta;H_T)$ appears to be \emph{growing} in $T$, so it is not immediately obvious that even an optimal learner will experience vanishing error.  A simple instance to initially consider is $\theta$ which is a discrete random variable.  In such an instance, we can always provide the upper bound:
$$\Lc_T \leq \frac{\H(\theta)}{T}.$$
Therefore, for any $\theta$ for which $\H(\theta) < \infty$, we have that $\Lc_T = O(1/T)$.

However, what about a more realistic scenario in which $\theta$ is a \emph{continuous} random variable?  A naive idea would be to simply supplant the discrete entropy $\H(\theta)$ with the \emph{differential entropy} $\diffentropy(\theta)$.  However, while differences in conditional differential entropy equal mutual information (just as with discrete entropy), differential entropy does \emph{not} upper bound mutual information.  This is because differential entropy can be \emph{negative} (as discussed in section \ref{subsec:diff_entropy}).  The appropriate extension of discrete entropy to continuous random variables can be established via \emph{rate-distortion theory}.

\subsection{Characterizing Error via Rate-Distortion Theory}

A rate-distortion function characterizes the optimal compression rate for a given notion of \emph{distortion}.  Equivalently, for a random variable $\theta$, it quantifies the amount of information about $\theta$ necessary to achieve distortion less than or equal to $\epsilon$.  To interpret this statement, consider a random variable $\tilde{\theta}$ that approximates $\theta$.  The mutual information $\I(\theta;\tilde{\theta})$ quantifies the amount of information that $\tilde{\theta}$ conveys about $\theta$.  The rate-distortion function minimizes this rate $\I(\theta;\tilde{\theta})$ across approximations $\tilde{\theta}$ that satisfy the aforementioned distortion constraint.



We now introduce notation that we will use to define a rate-distortion function.  First, let $\tilde{\Theta}$ be the set of random variables $\tilde{\theta}$ such that, for all $t,\ \tilde{\theta}\perp Y_{t+1}|(\theta, H_t)$.  In other words, $\tilde{\theta}$ does not offer any predictive information beyond what $\theta$ does.

Recall that as a gold standard, we defined $P^*_t(\cdot) = \Pr(Y_{t+1} \in \cdot | \theta, H_t)$, the prediction a clairvoyant with full knowledge of the data generating process, identified by $\theta$, would make.  Consider another prediction $\tilde{P}_t(\cdot) = \Pr(Y_{t+1} \in \cdot | \tilde{\theta}, H_t)$, which represents an approximation of the gold standard based on $\tilde{\theta}$ instead of $\theta$.  While rate-distortion theory can be developed for any distortion measure, in this monograph, we will focus on the following distortion measure:
$$\frac{1}{T} \sum_{t=0}^{T-1} \E\left[\KL\left(P^*_t \ \|\ \tilde{P}_t\right)\right].$$
For each $T$, this is the average error incurred by $\tilde{P}_t$ over $T$ timesteps.

Since our notion of distortion depends on $T$, so does our definition of rate-distortion:
\begin{definition}{\bf ($T$-horizon rate-distortion function)}
    For horizon $T$ and threshold $\epsilon \geq 0$, the $T$-horizon rate-distortion function is:
    $$\H_{\epsilon,T}(\theta) = \inf_{\tilde{\theta}\in\tilde{\Theta}_{\epsilon,T}}\ \I(\theta; \tilde{\theta}),$$
    where
    $$\tilde{\Theta}_{\epsilon,T} = \left\{\tilde{\theta} \in \tilde{\Theta}:\ \underbrace{\frac{1}{T}\sum_{t=0}^{T-1} \E\left[\KL\left(P^*_t \ \|\ \tilde{P}_t\right)\right] \leq \epsilon}_{\rm distortion\ threshold} \right\}.$$
\end{definition}

The following result upper and lower bounds the optimal error in terms of the $T$-horizon rate-distortion function:

\begin{theorem}{\bf ($T$-horizon rate-distortion  error bounds)}\label{th:rd_bounds}
    For all $T \in \Z_{++}$,
    $$\sup_{\epsilon \geq 0}\ \min\left\{\frac{\H_{\epsilon, T}(\theta)}{T}, \epsilon\right\}\ \leq\ \Lc_T\ \leq\ \inf_{\epsilon \geq 0}\ \frac{\H_{\epsilon,T}(\theta)}{T} + \epsilon.$$
\end{theorem}
\begin{proof}
    We begin with a proof of the upper bound.
    \begin{align*}
        \Lc_{T}
        & \overset{(a)}{=} \frac{\I(H_T;\theta)}{T}\\
        & = \frac{1}{T}\sum_{t=0}^{T-1} \I(Y_{t+1};\theta| H_t)\\
        & = \frac{1}{T}\sum_{t=0}^{T-1} \I(Y_{t+1};\theta,\tilde{\theta}|H_t)\\
        & = \frac{1}{T}\sum_{t=0}^{T-1} \I(Y_{t+1};\tilde{\theta}|H_t) + \I(Y_{t+1};\theta|\tilde{\theta}, H_t)\\
        & = \frac{\I(H_T;\tilde{\theta})}{T} + \frac{1}{T}\sum_{t=0}^{T-1}\I(Y_{t+1};\theta|\tilde{\theta}, H_t)\\
        & \overset{(b)}{\leq} \inf_{\epsilon \geq 0}\inf_{\tilde{\theta}\in\tilde{\Theta}_{\epsilon,T}}\ \frac{\I(\theta;\tilde{\theta})}{T} + \frac{1}{T}\sum_{t=0}^{T-1}\I(Y_{t+1};\theta|\tilde{\theta}, H_t)\\
        & \overset{(c)}{=} \inf_{\epsilon \geq 0}\inf_{\tilde{\theta}\in\tilde{\Theta}_{\epsilon,T}}\ \frac{\I(\theta;\tilde{\theta})}{T} + \frac{1}{T} \sum_{t=0}^{T-1} \E\left[ \KL(P^*_t\|\tilde{P}_t ) \right]\\
        & \overset{(d)}{\leq} \inf_{\epsilon \geq 0}\ \frac{\H_{\epsilon, T}(\theta)}{T} + \epsilon\\
    \end{align*}
    where $(a)$ follows from Theorem \ref{th:error_info}, $(b)$ follows from the data processing inequality applied to the Markov chain $\tilde{\theta} \rightarrow \theta \rightarrow H_t$, $(c)$ follows from Lemma \ref{le:kl_mi}, and $(d)$ follows from the fact that $\tilde{\theta}\in \tilde{\Theta}_{\epsilon, T}$.

    We now proceed with the lower bound.  Suppose that $\I(H_T;\theta) < \H_{\epsilon, T}(\theta)$.  Let $\tilde{\theta} = \tilde{H}_T \notin \tilde{\Theta}_{\epsilon, T}$ where $\tilde{H}_T$ is another history sampled in the same manner as $H_T$.
    \begin{align*}
        \I(H_T;\theta)
        & = \sum_{t=0}^{T-1}\I(Y_{t+1};\theta|H_t)\\
        & \overset{(a)}{\geq} \sum_{t=0}^{T-1}\I(Y_{t+1};\theta|\tilde{H}_{t}, H_t)\\
        & = \sum_{t=0}^{T-1}\I(Y_{t+1};\theta|\tilde{\theta}, H_t)\\
        & = \sum_{t=0}^{T-1} \E\left[\KL(P^*_t\| \tilde{P}_t)\right]\\
        & \overset{(b)}{\geq} \epsilon T,
    \end{align*}
    where $(a)$ follows from the fact that conditioning reduces entropy and that $Y_{t+1}\perp \tilde{H}_t|(\theta, H_t)$ and $(b)$ follows from the fact that $\tilde{\theta} \notin \tilde{\Theta}_{\epsilon, T}$.  Therefore, for all $\epsilon \geq 0$, $\I(H_{T};\theta) \geq \min\{ H_{\epsilon, T}, \epsilon T\}$.  The result follows.
\end{proof}

Theorem \ref{th:rd_bounds} establishes the tight relation between  error and the rate-distortion function.  The result is very general and facilitates the analysis of concrete problem instances in supervised learning.  In the following section, we will study 3 concrete instances of increasing complexity to demonstrate how Theorem \ref{th:rd_bounds} facilitates analysis.

We also note the qualitative similarity between Theorem \ref{th:rd_bounds} and classical results from PAC learning which all characterize  error as a fraction involving a complexity function of the confidence set and the dataset size (VC-dimension, Rademacher complexity, log-covering number).  In our framework, the rate-distortion function serves as the ``complexity'' function which characterizes the difficulty of learning $\theta$ for the purposes of prediction.  In the following section, we will use this general result to derive concrete error bounds for a suite of problems involving data pairs which are iid when conditioned on $\theta$.

\subsection{Comparisons to Minimax Error}\label{subsec:minimax}

The monograph focuses on the analysis of the expected error under a prior distribution over models.  The preceding sections introduced definitions and concepts that form a basis for such analysis, which some would refer to as a {\it Bayesian approach}.

It is worth noting that for each result we will establish via the Bayesian approach, it is likely that one could derive an analogous result via worst case error over a confidence set, which some refer to as a {\it frequentist approach}.  A motivation for the Bayesian approach, though, is its elegance.  In particular, to translate any given result to its frequentist counterpart, one typically tries to identify an analytically convenient confidence set and, if the arduous task is worth the investment, produces a messier analysis and loss bound.

In this section, we discuss the frequentist approach and its connection to the Bayesian framework of this monograph.  The purpose is to help a reader more familiar with worst case analysis to draw connections as they read this monograph.  However, the content of this section is otherwise unnecessary for reading the rest of this monograph.

Frequentist results study the \emph{minimax error} with respect to a confidence set $\Theta$:
$$\underline{\Lc}_{T} \ =\ 
    \inf_{P\in\mathcal{P}}\sup_{\nu\in\Theta} \frac{1}{T}\sum_{t=0}^{T-1} \E\left[
    \KL(P^*_t\| P)|\theta=\nu\right],$$
where $\mathcal{P}$ denotes the set of all functions mapping histories in $\mathcal{H}$ to probability distributions defined on $\mathcal{Y}$.  Notably, for any horizon $T$ and random variable $\theta: \Omega\mapsto \Theta$, the minimax error $\underline{\Lc}_{T}$ is greater than or equal to the Bayesian error $\Lc_T$.  Therefore, the rate-distortion lower bound from Theorem \ref{th:rd_bounds} also provides a lower bound for $\underline{\Lc}_{T}$.

There exists a large body of prior work which characterizes the minimax error using techniques that mirror ours.  Just as we are able to exactly equate Bayesian error to mutual information, prior work equates minimax error to channel capacity, a worst-case analogue to mutual information.  Furthermore, just as we upper and lower bound Bayesian error via the rate-distortion function, existing work upper and lower bounds frequentist error via metric entropy \citep{yangbarron1999}.  In the following section we present these results along with their requisite definitions and notation.

\subsubsection{Minimax Error Equals Channel Capacity}

While Theorem \ref{th:error_info} established the equivalence between the Bayesian error and mutual information, there exists an analogous result which relates minimax error to the \emph{channel capacity} of the conditional distribution $\Pr(H_T\in\cdot|\theta)$.

\begin{definition} {\bf (channel capacity)}  Let $X:\Omega\mapsto\mathcal{X}, Y:\Omega\mapsto\mathcal{Y}$ be arbitrary random variables.  The \emph{channel capacity} of the conditional probability measure $\Pr(Y\in\cdot|X)$ is defined as:
$$\mathfrak{C}(Y;X) := \sup_{\tilde{X}:\Omega\mapsto\mathcal{X}} \E\left[\KL\left(\Pr(Y\in\cdot|X\leftarrow\tilde{X}) \ \| \ \sum_{x\in\mathcal{X}}\Pr(\tilde{X}=x)\cdot\Pr(Y\in\cdot|X=x) \right)\right].$$
\end{definition}

If $X$ is a continuous random variable, then the sum above should be replaced with an integral and the probability measure replaced with the Radon-Nikodym derivative.  The channel capacity $\mathfrak{C}(Y;X)$ finds the maximum mutual information $\I(Y;X)$ under a thought experiment in which we can \emph{change} the probability measure $\Pr(X\in\cdot)$.  The supremum $\tilde{X}:\Omega\mapsto\mathcal{X}$ and the change of measure $\leftarrow$ in the above definition conveys this point.  With channel capacity defined, we present a classical result which equates minimax error to the channel capacity of $\Pr(H_T\in\cdot|\theta)$.

\begin{theorem}{\bf (redundancy-capacity)}
    For all $T \in \mathbb{Z}_{++}$, if $\mathcal{P}$ is compact and convex, then
    $$\underline{\Lc}_T = \frac{\mathfrak{C}(H_T;\theta)}{T}.$$
\end{theorem}
\begin{proof}
    \begin{align*}
        \underline{\Lc}_T
        & = \inf_{P\in\mathcal{P}}\sup_{\nu\in\Theta} \frac{1}{T}\sum_{t=0}^{T-1} \E\left[
    \KL(\Pr(Y_{t+1}\in\cdot|\theta, X_t)\| P(Y_{t+1}\in\cdot|H_t))|\theta=\nu\right]\\
        & \overset{(a)}{=} \inf_{P\in\mathcal{P}}\sup_{\rho\in\Delta_\Theta} \frac{1}{T}\sum_{t=0}^{T-1}\sum_{\nu\in\Theta} \rho(\nu) \cdot \E\left[
    \KL(\Pr(Y_{t+1}\in\cdot|\theta, X_t)\| P(Y_{t+1}\in\cdot|H_t))|\theta=\nu\right]\\
        & \overset{(b)}{=} \sup_{\rho\in\Delta_{\Theta}} \inf_{P\in\mathcal{P}} \frac{1}{T}\sum_{t=0}^{T-1} \sum_{\nu\in\Theta}\rho(\nu) \cdot \E\left[
    \KL(\Pr(Y_{t+1}\in\cdot|\theta, X_t)\| P(Y_{t+1}\in\cdot|H_t))|\theta=\nu\right]\\
        & \overset{(c)}{=} \sup_{\rho\in\Delta_{\Theta}} \frac{1}{T}\sum_{t=0}^{T-1}\sum_{\nu\in\Theta} \rho(\nu) \cdot \E\left[
    \KL(\Pr(Y_{t+1}\in\cdot|\theta, X_t)\| P_{\rho}(Y_{t+1}\in\cdot|H_t))|\theta=\nu\right]\\
        & \overset{(d)}{=} \sup_{\rho\in\Delta_{\Theta}}\sum_{\nu\in\Theta}\rho(\nu)\cdot \frac{\E\left[\KL\left(\Pr(H_T\in\cdot|\theta)\ \|\ P_{\rho}(H_T\in\cdot)\right)|\theta=\nu\right]}{T}\\
        & \overset{(e)}{=} \sup_{\tilde{\theta}: \Omega\mapsto\Theta} \sum_{\nu\in\Theta}\Pr(\tilde{\theta}=\nu)\cdot \frac{\E\left[\KL\left(\Pr(H_T\in\cdot|\theta)\ \|\ P_{\tilde{\theta}}(H_T\in\cdot)\right)|\theta=\nu\right]}{T}\\
        & = \sup_{\tilde{\theta}: \Omega\mapsto\Theta} \frac{\E\left[\KL\left(\Pr(H_T\in\cdot|\theta\leftarrow \tilde{\theta})\ \| \ P_{\tilde{\theta}}(H_T\in\cdot)\right)\right]}{T}\\
    \end{align*}
    where in $(a)$, $\Delta_{\Theta}$ denotes the probability simplex on $\Theta$, $(b)$ follows from the minimax theorem, $(c)$ follows from Lemma \ref{le:bayes_opt} where 
    $$P_{\rho}(Y_{t+1}\in\cdot|H_t) = \frac{\sum_{\nu\in\Theta} \rho(\nu)\cdot \Pr(H_{t+1}\in\cdot|\theta=\nu)}{\sum_{\nu\in\Theta} \rho(\nu)\cdot\Pr(H_{t}\in\cdot|\theta=\nu)},$$
     $(d)$ follows from the chain rule of KL divergence, where in $(e)$, 
     $$P_{\tilde{\theta}}(H_T\in\cdot) = \sum_{\nu\in\Theta} \Pr(\tilde{\theta}=\nu) \cdot \Pr(H_T\in\cdot|\theta = \nu).$$
     The result follows.
\end{proof}

The result bears many similarities to Theorem \ref{th:error_info}.  Since minimax regret considers worst-case performance across all realizations of $\theta$, it is intuitive that the regret would be the mutual information between $H_T$ and $\theta$ under a suitable worst-case change of measure.  This value turns out to exactly equal the channel capacity.  This result along with Theorem \ref{th:error_info} establish an intimate connection between the optimal error and the information shared between $\theta$ and the observed data $H_T$.  However, just as with mutual information, bounding the channel capacity for complex problems may not be analytically feasible.  In these situations, we must again resort to upper and lower bounds which we explore in the following sections.

\subsubsection{Minimax Upper Bounds}

We begin with the definition of \emph{metric entropy}, the frequentist analogue to the rate-distortion function in the Bayesian framework. 

\begin{definition}{\bf (metric entropy)}
    Let $\epsilon \geq 0$, $\Theta$ be a set, and $\rho:\Theta\times\Theta\to\Re_{+}$ a distortion function.  The metric entropy of set $\Theta$ at tolerance $\epsilon$ is:
    $$\inf_{\tilde{\Theta}\in \Theta_\epsilon}\ \ln|\tilde{\Theta}|$$
    where
    $$\Theta_\epsilon = \left\{ \tilde{\Theta}\subseteq \Theta:\ \sup_{\nu\in\Theta}\inf_{\tilde{\nu} \in \tilde{\Theta}} \rho(\nu, \tilde{\nu}) \leq \epsilon \right\}.$$
\end{definition}

We note that the rate-distortion function generalizes metric entropy: metric entropy relies on discrete quantizations of $\Theta$ whereas rate-distortion theory allows continuous random variables with finite mutual information with $\theta$.  Furthermore, the metric entropy assigns a cost of $\ln|\tilde{\Theta}|$ to a quantization while the rate-distortion function assigns a cost of $\I(\theta;\tilde{\theta})$.  These subtle differences come from the fact that the Bayes error is an average with respect to a prior distribution while the minimax error is a maximum across a confidence set.  In practice, the way that these priors (for Bayesian) and sets (for frequentist) are constructed will impact the eventual bound more than the differences between metric entropy and rate-distortion theory.  However, in this work, we find that analyzing complex problems in machine learning with standard priors (such as Gaussian) produces elegant sample complexity results.  While the corresponding frequentist results are bound to exist, construction of confidence sets that lead to such elegant results is technically challenging.

In order to relate metric entropy to $\underline{\Lc}_T$, we use the following distortion function:

$$\rho(\nu, \tilde{\nu})= \frac{1}{T}\sum_{t=0}^{T-1}\E\left[\KL\left(P^*_t \ \|\ \Pr(Y_{t+1}\in\cdot|\theta=\tilde{\nu}, H_t)\right)\right].$$

We use the notation $\mathfrak{H}_{\epsilon, T}(\Theta)$ to denote the metric entropy:

$$\mathfrak{H}_{\epsilon, T}(\Theta) = \inf_{\tilde{\Theta}\in\Theta_\epsilon}\ \ln |\tilde{\Theta}|,$$
where
$$\Theta_{\epsilon} = \left\{\tilde{\Theta} \subset \Theta;\ \sup_{\nu\in\Theta}\inf_{\tilde{\nu}\in\tilde{\Theta}} \frac{1}{T}\sum_{t=0}^{T-1} \E\left[\KL\left(P^*_t\|\Pr(Y_{t+1}\in\cdot|\theta=\tilde{\nu}, H_t)\right)|\theta=\nu\right] \leq \epsilon \right\}.$$
    

With this requisite notation in place, we present the minimax upper bound (Theorem 2 of \citep{yangbarron1999}).

\begin{theorem}{\bf (minimax upper bound)}\label{th:metric_entropy_ub}
    For all $T \in \mathbb{Z}_{++}$,
    $$\underline{\Lc}_{T} \leq \inf_{\epsilon \geq 0}\ \frac{\mathfrak{H}_{\epsilon, T}(\Theta)}{T} + \epsilon.$$
\end{theorem}
\begin{proof}
    To upper bound $\underline{\Lc}_T$, we characterize the loss of a particular predictive distribution $\hat{P}(\cdot)$.  This predictive distribution is constructed via Bayes rule on a uniform prior distribution on $\Theta_\epsilon$, where $\Theta_\epsilon$ is a minimal $\epsilon$-cover.  We outline some key pieces of notation below:
    \begin{align*}
        \hat{P}(\theta|H_t)
        & = \frac{\frac{1}{|\Theta_\epsilon|}\cdot \Pr(H_t|\theta)}{\sum_{\tilde{\nu}\in\Theta_\epsilon}\ \frac{1}{|\Theta_\epsilon|}\cdot \Pr(H_t|\theta=\tilde{\nu})}\\
        \hat{P}(H_t)
        & = \sum_{\tilde{\nu} \in\Theta_\epsilon}\ \frac{1}{|\Theta_\epsilon|}\cdot\Pr(H_t|\theta=\tilde{\nu})\\
        \hat{P}(Y_{t+1}\in\cdot|H_t)
        & = \sum_{\tilde{\nu}\in\Theta_\epsilon}  \Pr(Y_{t+1}\in\cdot|\theta=\tilde{\nu}, H_t)\cdot \hat{P}(\theta=\tilde{\nu}|H_t).
    \end{align*}
    Finally, let $\tilde{\nu}^* = \argmin_{\nu\in\Theta_\epsilon} $
    \begin{align*}
        \underline{\Lc}_T
        & = \inf_{P}\sup_{\nu\in\Theta} \frac{1}{T}\sum_{t=0}^{T-1}\ \E\left[\KL(\Pr(Y_{t+1}\in\cdot|\theta, X_t)\| P(Y_{t+1}\in\cdot|H_t))|\theta=\nu\right]\\
        & \leq \inf_{\epsilon\geq 0}\sup_{\nu\in\Theta} \frac{1}{T}\sum_{t=0}^{T-1}\ \E\left[\KL(\Pr(Y_{t+1}\in\cdot|\theta, X_t)\| \hat{P}(Y_{t+1}\in\cdot|H_t))|\theta=\nu\right]\\
        & \overset{(a)}{=} \inf_{\epsilon\geq 0}\sup_{\nu\in\Theta} \frac{1}{T} \E\left[ \KL\left( \Pr(H_T\in\cdot|\theta) \| \hat{P}(H_T\in\cdot) \right)\big| \theta=\nu \right]\\
        & = \inf_{\epsilon\geq 0}\sup_{\nu\in\Theta} \frac{1}{T}\cdot \E\left[
        \ln\frac{\Pr(H_T|\theta)}{\hat{P}(H_T)} \bigg|\theta=\nu\right]\\
        & \leq \inf_{\epsilon\geq 0}\sup_{\nu\in\Theta} \frac{1}{T}\cdot \E\left[
        \ln\frac{\Pr(H_T|\theta)}{\frac{1}{|\Theta_\epsilon|} \cdot \Pr(H_t|\theta=\tilde{\nu}^*)} \bigg|\theta=\nu\right]\\
        & \overset{(b)}{=} \inf_{\epsilon\geq 0}\frac{\ln |\Theta_\epsilon|}{T} + \frac{1}{T}\sum_{t=0}^{T-1}\sup_{\nu\in\Theta}\E\left[ \KL\left(\Pr(Y_{t+1}\in\cdot|\theta, H_t) \| \Pr(Y_{t+1}\in\cdot|\theta=\tilde{\nu}^*, H_t) \right)|\theta=\nu \right]\\
        & \overset{(c)}{\leq} \inf_{\epsilon\geq 0}\ \frac{\mathfrak{H}_{\epsilon, T}(\Theta)}{T} + \epsilon,
    \end{align*}
    where $(a)$ and $(b)$ follow from the chain rule of KL divergence and the definition of $\hat{P}$ and $(c)$ follows from the definition of metric entropy.
\end{proof}

We note that Theorem \ref{th:metric_entropy_ub} and Theorem \ref{th:rd_bounds} very closely resemble each other.  For any upper bound on the Bayes regret which we derive via Theorem \ref{th:rd_bounds}, a similar technique can be applied to bound the minimax regret via Theorem \ref{th:metric_entropy_ub}.  However, as mentioned above, the choice of confidence set will dramatically impact the resulting bounds in the frequentist framework.  Our observation is that for a problem such as learning with deep neural networks, naive specification of the confidence set will result in vacuous worst-case sample complexity bounds.  Constructing confidence sets which better reflect the problems we may encounter in the real world and are amenable to analysis is a significant outstanding challenge.  However, Theorem \ref{th:metric_entropy_ub} serves as a bridge which relates the results of this monograph to those of the frequentist framework.

\subsubsection{Minimax Lower Bounds}

The lower bound for Bayesian regret via the rate-distortion function in Theorem \ref{th:rd_bounds}, in tandem with the corresponding rate-distortion upper bounds, provides a tight characterization of the Bayesian error.  Tight characterizations of this sort are not unforeseen as \citet{yangbarron1999} demonstrate similar tight characterizations of the minimax error in terms of metric and packing entropy.  Theorem $1$ of \citep{yangbarron1999} establishes a lower bound on the minimax regret which we detail here for completeness.  The result holds for loss functions which obey a triangle inequality.  It is well known that the square-root KL divergence does not generally exhibit a triangle inequality, however, if the set of densities corresponding to the probability measures $\{ \Pr(Y_{t+1}\in\cdot|\theta=\nu, X_t) : \nu \in \Theta \}$ have uniformly bounded logarithms, then the square root KL divergence satisfies the following \emph{local} triangle inequality:

\begin{lemma}{\bf (local triangle inequality)}\label{le:local_triangle}
    If there exists $M < \infty$ such that
    $$\sup_{P' \in \mathcal{P}} \|\ln dP'\|_\infty < M,$$
    where $dP'$ denotes the Radon-Nikodym derivative of $P'$ w.r.t the Lebesgue measure, then there exists $A \in (0, 1], \epsilon_0 > 0$ such that for any probability measures $P(\cdot), P'(\cdot) \in \mathcal{P}$ and any probability measure $\tilde{P}(\cdot)$ defined on the same $\sigma$ algebra, if
    $$\max\left\{ \KL\left(P(\cdot)\|\tilde{P}(\cdot)\right), \KL\left(P'(\cdot)\|\tilde{P}(\cdot)\right) \right\} \leq \epsilon_0^2,$$
    then
    $$\sqrt{\KL(P(\cdot)\|\tilde{P}(\cdot))} + \sqrt{\KL(P'(\cdot)\| \tilde{P}(\cdot))} \geq A \cdot \sqrt{\KL(P(\cdot)\| P'(\cdot))}.$$
\end{lemma}

Before we provide the lower bound, we first introduce another important quantity for characterizing the complexity of a set $\Theta$.

\begin{definition}{\bf (packing entropy)}
    Let $\epsilon \geq 0$, $\Theta$ be a set, and $\rho:\Theta\times\Theta\to\Re_{+}$ a distortion function.  The packing entropy of set $\Theta$ at tolerance $\epsilon$ is:
    $$\sup_{\tilde{\Theta}\in \Theta_\epsilon}\ \ln|\tilde{\Theta}|$$
    where
    $$\Theta_\epsilon = \left\{ \tilde{\Theta}\subseteq \Theta:\ \inf_{\tilde{\nu} \in \tilde{\Theta}}\inf_{\nu\neq\tilde{\nu}\in\tilde{\Theta}} \rho(\nu, \tilde{\nu}) \geq \epsilon \right\}.$$
\end{definition}

While the metric entropy is the logarithm of the minimum size quantization necessary to \emph{cover} $\Theta$ up to tolerance $\epsilon$, the packing entropy is the logarithm of the maximum size quantization which \emph{packs} $\Theta$ with tolerance $\epsilon$.  This packing ensures that no two elements of the quantization are within $\epsilon$ distortion of one another.  In order to relate packing entropy to $\underline{\Lc}_T$, we use the following distortion:

$$\rho(\nu, \tilde{\nu})= \E\left[\sqrt{\KL\left(\Pr(Y_{1}\in\cdot|\theta=\nu, X_0) \ \|\ \Pr(Y_{1}\in\cdot|\theta=\tilde{\nu}, X_0)\right)}\right].$$

We use the notation $\mathfrak{M}_{\epsilon}(\Theta)$ to denote the packing entropy:

$$\mathfrak{M}_{\epsilon}(\Theta) = \sup_{\tilde{\Theta}\in\Theta_\epsilon}\ \ln |\tilde{\Theta}|,$$
where
$$\Theta_\epsilon = \left\{\tilde{\Theta} \subset \Theta : \inf_{\tilde{\nu}\in\tilde{\Theta}}\inf_{\nu\neq\tilde{\nu}\in\tilde{\Theta}} \E\left[\sqrt{\KL\left(P^*_0\|\Pr(Y_{1}\in\cdot|\theta=\tilde{\nu}, X_0)\right)}|\theta=\nu\right] \geq \epsilon\right\}.$$

With this requisite notation in place, we present the minimax lower bound (Theorem 1 of \citep{yangbarron1999}).

\begin{theorem}{\bf (minimax lower bound)}\label{th:minimax_lower}
    Let $\Theta$ be a set.  For all $t\in \mathbb{Z}_{++}$, let $\epsilon_t > 0$ satisfy $\epsilon_t = \mathfrak{H}_{\epsilon, t}(\Theta)/t$ and $\underline{\epsilon}_{t}$ satisfy $\mathfrak{M}_{\underline{\epsilon}_T}(\Theta) = 4t\epsilon_t + 2\ln 2$.  If $\underline{\epsilon}_{T} < 2\epsilon_0$ then
    $$\underline{\Lc}_T \geq \frac{A^2}{8}\underline{\epsilon}_T^2,$$
    where $A, \epsilon_0$ are the constants from Lemma \ref{le:local_triangle}.
\end{theorem}
\begin{proof}
    Let $\Theta_{\underline{\epsilon}_T}$ be the $\underline{\epsilon}_T$-packing which attains $\mathfrak{M}_{\underline{\epsilon}_T}(\Theta)$.  Let 
    $$\tilde{\theta} = \argmin_{\theta'\in\Theta_{\underline{\epsilon}_T}}\ \E\left[\sqrt{\KL\left(\Pr(Y_{T}\in\cdot|\theta, X_{T-1}) \| \Pr(Y_{T}\in\cdot|H_{T-1}\right)}|\theta=\theta'\right].$$
    Therefore, $\tilde{\theta}$ is the element in the packing $\Theta_{\underline{\epsilon}_T}$ with minimum expected square root KL divergence from the posterior distribution $\Pr(Y_{T}\in\cdot|H_{T-1})$.  Going forward, we assume that $\Pr(\theta\in\cdot) = {\rm Uniform}(\Theta_{\underline{\epsilon}_T})$.

    For brevity, we let 
    \begin{align*}
        d(\theta, P_T)
        & = \E[\sqrt{\KL(\Pr(Y_{T}\in\cdot|\theta, X_{T-1})\| \Pr(Y_{T}\in\cdot|H_{T-1}))}|\theta]\\
        d(\tilde{\theta}, P_T)
        & = \E[\sqrt{\KL(\Pr(Y_{T}\in\cdot|\theta, X_{T-1})\| \Pr(Y_{T}\in\cdot|H_{T-1}))}|\theta=\tilde{\theta}].
    \end{align*}

    Then,

    \begin{align*}
        \theta\neq \tilde{\theta}
        & \overset{(a)}{\implies} d(\theta, P_T) + d(\tilde{\theta}, P_T) \geq A\underline{\epsilon}_T {\rm\ if\ } \max \left\{d(\theta, P_T), d(\tilde{\theta}, P_T)\right\} < \frac{A\underline{\epsilon}_T}{2}\\
        & \implies d(\theta, P_T) + d(\tilde{\theta}, P_T) \geq A\underline{\epsilon}_T {\rm\ if\ } d(\theta, P_T) < \frac{A\underline{\epsilon}_T}{2}\\
        & \implies 2d(\theta, P_T) \geq A\underline{\epsilon}_T {\rm\ if\ } d(\theta, P_T) < \frac{A\underline{\epsilon}_T}{2}\\
        & \implies d(\theta, P_T) \geq \frac{A\underline{\epsilon}_T}{2} {\rm\ if\ } d(\theta, P_T) < \frac{A\underline{\epsilon}_T}{2}\\
        & \implies d(\theta, P_T) \geq \frac{A\epsilon_T}{2},
    \end{align*}
    where $(a)$ follows from the local triangle inequality, the fact that $\Theta_{\underline{\epsilon}_T}$ is an $\underline{\epsilon}_T$ packing, and $\theta \neq \tilde{\theta}$.

    Then,
    \begin{align*}
        \underline{\Lc}_{T}
        & = \inf_{P}\sup_{\nu\in\Theta} \frac{1}{T}\sum_{t=0}^{T-1}\ \E\left[\KL(\Pr(Y_{t+1}\in\cdot|\theta, X_t)\| P(Y_{t+1}\in\cdot|H_t))|\theta=\nu\right]\\
        & \geq \inf_{P}\sup_{\nu\in\Theta}\  \E\left[\KL(\Pr(Y_{T}\in\cdot|\theta, X_{T-1})\| P(Y_{T}\in\cdot|H_{T-1}))|\theta=\nu\right]\\
        & \overset{(a)}{\geq} \inf_{P}\sup_{\nu\in\Theta_{\underline{\epsilon}_T}}\  \E\left[\KL(\Pr(Y_{T}\in\cdot|\theta, X_{T-1})\| P(Y_{T}\in\cdot|H_{T-1}))|\theta=\nu\right]\\
        & \geq \inf_{P}\  \E\left[\KL(\Pr(Y_{T}\in\cdot|\theta, X_{T-1})\| P(Y_{T}\in\cdot|H_{T-1}))\right]\\
        & = \E\left[\KL(\Pr(Y_{T}\in\cdot|\theta, X_{T-1})\| \Pr(Y_{T}\in\cdot|H_{T-1}))\right]\\
        & \geq \left(\frac{A}{2} \underline{\epsilon}_T\right)^2 \cdot \Pr\left(\KL\left(\Pr(Y_{T}\in\cdot|\theta, X_{T-1})\|\Pr(Y_{T}\in\cdot|H_{T-1})\right) \geq \left(\frac{A}{2}\underline{\epsilon}_T\right)^2\right)\\
        & \overset{(b)}{\geq} \left(\frac{A}{2} \underline{\epsilon}_T\right)^2 \cdot \Pr(\tilde{\theta} \neq \theta)\\
        & \overset{(c)}{\geq} \left(\frac{A}{2} \underline{\epsilon}_T\right)^2 \cdot \left(1 - \frac{\I(\theta;H_{T}) + \ln 2}{\H(\theta)}\right)\\
        & \overset{(d)}{\geq} \left(\frac{A}{2} \underline{\epsilon}_T\right)^2 \cdot \left(1 - \frac{\mathfrak{H}_{\epsilon_T, T}(\Theta) + T\epsilon_T +  \ln 2}{\H(\theta)}\right)\\
        & \geq \frac{A^2}{8}\underline{\epsilon}_T^2,
    \end{align*}
    where $(a)$ follows from the fact that $\Theta_{\underline{\epsilon}_T} \subseteq \Theta$, $(b)$ follows from the above implication, $(c)$ follows from Fano's inequality, $(d)$ follows from the fact that expected error lower bounds minimax error and Theorem \ref{th:metric_entropy_ub}.
\end{proof}

We note that \citet{yangbarron1999} provide techniques which relax some of the conditions in Theorem \ref{th:minimax_lower} pertaining to the local triangle inequality by instead constructing covers/packings according to the Hellinger distance (Corollary 3).

As we alluded to above, we observe that for problem settings which involve complex confidence sets $\Theta$, the result of bounding minimax error via Theorem \ref{th:metric_entropy_ub} may be substantially worse than the result of bounding Bayesian error via Theorem \ref{th:rd_bounds}.  Much as the entropy of a discrete random variable may be much smaller than the logarithm of the number of possible realizations, the Bayesian error can be much smaller than the minimax error when the confidence set $\Theta$ is not chosen judiciously.  In this work, we elect to study the Bayesian error as it appears that good sample complexity results are feasible even under \emph{simple} choices of the prior distribution.  Analyses of similar flavor may extend to the frequentist setting if suitable effort is taken in designing the confidence set.

\newpage
\begin{summary}
\begin{itemize}
\item {\bf (optimal  error equals total information)} For horizon $T\in\Z_{++}$, the  error of the optimal algorithm is denoted as $\Lc_{T}$ and is
$$\Lc_{T} = \frac{\I(\theta;H_T)}{T}.$$
\item For all $T\in \Z_{++}$,
$$\Lc_{T} \leq \frac{\H(\theta)}{T}.$$
Therefore, for a discrete random variable $\theta$ with finite entropy, $\Lc_T = O(1/T)$.
\item The extension of this result to continuous random variables can be made via rate-distortion theory.
\item  Let $\tilde{\Theta}$ be the set of random variables $\tilde{\theta}$ such that $\tilde{\theta}\perp Y_{t+1}|(\theta, H_t)$.
\item For tolerance $\epsilon$ and horizon $T$, we define the $T$-horizon rate-distortion function as
$$\H_{\epsilon, T}(\theta) = \inf_{\tilde{\theta}\in\tilde{\Theta}_{\epsilon, T}}\ \I(\theta; \tilde{\theta}),$$
where
$$\tilde{\Theta}_{\epsilon,T} = \left\{\tilde{\theta} \in \tilde{\Theta}: \ \underbrace{\frac{1}{T}\sum_{t=0}^{T-1} \E\left[\KL\left(P^*_t \ \|\ \tilde{P}_t\right)\right] \leq \epsilon}_{\rm distortion\ threshold}\ \right\}.$$
\item {\bf ($T$-horizon rate-distortion error bounds)}
    For all $T \in \Z_{++}$,
    $$\sup_{\epsilon \geq 0}\ \min\left\{\frac{\H_{\epsilon, T}(\theta)}{T}, \epsilon\right\}\ \leq\ \Lc_T\ \leq\ \inf_{\epsilon \geq 0}\ \frac{\H_{\epsilon,T}(\theta)}{T} + \epsilon.$$
\end{itemize}
\end{summary}
\clearpage

%% file: sections/supervised_learning.tex
\section{Learning from iid Data}

In this section, we restrict our attention to the analysis of learning from data that is independently and identically distributed (iid) conditioned on $\theta$.  Formally, we assume that $\{(X_t, Y_{t+1})\}_{t\in \mathbb{Z}_+}$ is iid conditioned on $\theta$.  This specializes the framework introduced in previous sections, and therefore, our results from those sections hold under this iid assumption.  In this section, we will derive stronger theoretical results under the iid assumption.  In particular, we establish results for four specific data generating processes arising from $1)$ linear regression, $2)$ logistic regression, $3)$ finite-width multi-layer perceptrons, $4)$ infinite-width multi-layer perceptrons.  We hope that this suite of examples provides the reader with enough intuition and tools to analyze their own problems of interest.

\subsection{Theoretical Results Tailored for iid Data}

Theorem \ref{th:rd_bounds} provides upper and lower bounds for the optimal error at horizon $T$ via the $T$-horizon rate-distortion function $\H_{\epsilon,T}$.  Under the iid assumption, we can simplify analysis by deriving bounds based on the $0$-horizon rate-distortion function $\H_{\epsilon,0}$.  Going forward, we will write $\H_{\epsilon}$ as shorthand for $\H_{\epsilon,0}$.  Hence,
$$\H_\epsilon(\theta) = \inf_{\tilde{\theta} \in \tilde{\Theta}_\epsilon} \I(\theta; \tilde{\theta}),$$
where 
$$\tilde{\Theta}_\epsilon = \tilde{\Theta}_{\epsilon,0} = \{\tilde{\theta}\in\tilde{\Theta}: \E[\KL(P^*_0\|\tilde{P}_0)] \leq \epsilon\}.$$

The aforementioned bounds on $\Lc_T$ in terms of $\H_{\epsilon}(\theta)$ derive from two lemmas, which we now establish.  The first states that, under the iid assumption, error is monotonically non-increasing with time.  Recall that $\hat{P}_t = \Pr(Y_{t+1} \in \cdot|H_t)$ is our prediction, while $P^*_t = \Pr(Y_{t+1} \in \cdot | \theta, H_t)$ is the gold standard that could be produced by a clairvoyant.
\begin{lemma}{\bf(monotonicity of error)}\label{le:monotonicity}
    If $((X_t, Y_{t+1}) : t\in \Z_{+})$ is an iid stochastic process when conditioned on $\theta$, then for all $t\in\Z_{+}$,
    $$\E\left[\KL(P^*_{t+1}\|\hat{P}_{t+1})\right] \leq \E\left[\KL(P^*_{t}\|\hat{P}_t)\right].$$
\end{lemma}
\begin{proof}
    We have
    \begin{align*}
        \E\left[\KL(P^*_{t+1}\|\hat{P}_{t+1})\right]
        & = \I(\theta; Y_{t+2} | H_{t+1})\\
        & = \diffentropy(Y_{t+2}|H_{t+1}) - \diffentropy(Y_{t+2}|\theta, H_{t+1})\\
        &\overset{(a)}{=} \diffentropy(Y_{t+2}|H_{t+1}) - \diffentropy(Y_{t+2}|\theta, H_{t-1}, Y_t, X_{t+1})\\
        &\overset{(b)}{\leq} \diffentropy(Y_{t+2}|H_{t-1}, Y_t, X_{t+1}) - \diffentropy(Y_{t+2}|\theta, H_{t-1}, Y_t, X_{t+1})\\
        &= \I(Y_{t+2};\theta | H_{t-1}, Y_{t}, X_{t+1}) \\
        &\overset{(c)}{=} \I(Y_{t+1};\theta | H_{t-1}, Y_t, X_t) \\
        &= \I(Y_{t+1};\theta | H_t) \\
        & = \E\left[\KL(P^*_{t}\|\hat{P}_t)\right],
    \end{align*}
    where $(a)$ follows since $Y_{t+2} \perp (X_{t}, Y_{t+1})|(\theta, X_{t+1})$, $(b)$ follows from the fact that conditioning reduces differential entropy, and $(c)$ follows from the fact that $(X_t, Y_{t+1})$ and $(X_{t+1}, Y_{t+2})$ are identically distributed conditioned on $(H_{t-1}, Y_t)$.
\end{proof}

This result is intuitive.  Since $H_{t+1}$ provides more information than $H_t$, the prediction made at time $t+1$ ought to be better.

A second lemma bounds the $T$-horizon distortion via the $0$-horizon distortion.  Recall that $\tilde{P}_t = \Pr(Y_{t+1} \in \cdot | \tilde{\theta}, H_t)$ is the prediction based on an approximation $\tilde{\theta}$ to $\theta$.
\begin{lemma}{\bf ($0$-horizon distortion upper bounds $T$-horizon distortion)}\label{cor:distortion_ub}
    If $((X_t, Y_{t+1}) : t\in \Z_{+})$ is iid conditioned on $\theta$ and, for all $t \in\Z_{+}$, $Y_{t+1}\perp\tilde{\theta}|(\theta, X_t)$ then, for all $T\in \Z_{++}$,
    $$\underbrace{\frac{1}{T}\sum_{t=0}^{T-1} \E\left[\KL(P^*_t \| \tilde{P}_t )\right]}_{T-{\rm horizon\ distortion}} \leq \underbrace{\E\left[ \KL (P^*_0 \| \tilde{P}_0)\right]}_{0-{\rm horizon\ distortion}}.$$
\end{lemma}
\begin{proof}
    \begin{align*}
        \frac{1}{T}\sum_{t=0}^{T-1} \E\left[\KL(P^*_t \| \tilde{P}_t )\right]
        & = \frac{\I(H_T;\theta|\tilde{\theta})}{T}\\
        & = \frac{1}{T}\sum_{t=0}^{T-1} \I(X_t, Y_{t+1};\theta|\tilde{\theta}, H_{t-1}, Y_t)\\
        & \overset{(a)}{=} \frac{1}{T}\sum_{t=0}^{T-1} \I(Y_{t+1};\theta|\tilde{\theta}, H_{t})\\
        & \overset{(b)}{\leq} \I(Y_1;\theta|\tilde{\theta}, X_0)\\
        & =  \E\left[ \KL (P^*_0 \| \tilde{P}_0)\right],
    \end{align*}
    where $(a)$ follows from the assumption that $X_t \perp \theta$ and $(b)$ follows from the same proof technique as in Lemma \ref{le:monotonicity}.
\end{proof}

These two results lead to simpler upper and lower error bounds than those presented in Theorem \ref{th:rd_bounds}.  In particular, under the iid assumption, for any $T$, we can upper and lower bound $\Lc_{T}$ via the $0$-horizon instead of $T$-horizon rate-distortion function.  The $0$-horizon rate-distortion function is simpler to analyze, and this will allow us to establish tighter bounds later for specific data generating processes.

\begin{theorem}{\bf (0-horizon rate-distortion error bound for iid data)}\label{th:horizon-indep-rd}
    If $((X_t, Y_{t+1}) : t\in \Z_{+})$ is iid conditioned on $\theta$, then, for all $T\in \Z_{++},$
    $$\sup_{\epsilon \geq 0}\ \min\left\{ \frac{\H_\epsilon(\theta)}{T},
    \epsilon \right\}\ \leq\ \Lc_{T} \ \leq\ \inf_{\epsilon \geq 0}\ \frac{\H_{\epsilon}(\theta)}{T} + \epsilon.$$
\end{theorem}
\begin{proof}
    We begin with the lower bound.  Fix $T \in \mathbb{Z}_{++}$.  Let $\tilde{\theta} = (\tilde{H}_{T-2}, \tilde{Y}_{T-1})$ be independent from but distributed identically with $(H_{T-2}, Y_{T-1})$ when conditioned on $\theta$.

    Fix $\epsilon \geq 0$.  If $\Lc_{T} < \H_\epsilon(\theta)/T$ then 
    $$\Lc_{T} < \H_\epsilon(\theta)/T \ \implies\ \I(H_T;\theta)
        < \H_\epsilon(\theta)\ \implies\ \I(\tilde{\theta};\theta) < \H_\epsilon(\theta).$$
    Since the rate of $\tilde{\theta}$ is \emph{lower} than the rate-distortion function $\H_\epsilon(\theta)$, $\tilde{\theta} \notin \tilde{\Theta}_\epsilon = \{\theta\in\tilde{\Theta}: \E[\KL(P^*_0\|\tilde{P}_0)] \leq \epsilon\}$.  As a result,
    \begin{align*}
        \Lc_{T}
        &\overset{(a)}{=} \frac{\I(H_T;\theta)}{T} \\
        &\overset{(b)}{=} \frac{1}{T}\sum_{t=0}^{T-1} \I(Y_{t+1};\theta | H_t) \\
        &\overset{(c)}{\geq} \I(Y_T;\theta | H_{T-1}) \\
        &= \I(Y_T;\theta | \tilde{\theta}, X_{T-1}) \\
        & = \E\left[\KL(P^*_{T-1} \| \tilde{P}_{T-1})\right] \\
        & \overset{(d)}{=} \E\left[\KL(P^*_{0} \| \tilde{P}_{0})\right] \\
        &\overset{(e)}{>} \epsilon,
    \end{align*}
    where $(a)$ follows from Lemma \ref{th:error_info}, $(b)$ follows from the chain rule of mutual information, $(c)$ follows from Lemma \ref{le:monotonicity}, $(d)$ follows from the iid assumption, and $(e)$ follows from the fact that $\tilde{\theta} \notin \tilde{\Theta}_\epsilon$.  Therefore, 
    $$\Lc_T \geq \min\{\H_\epsilon(\theta)/T,\ \epsilon\}.$$
    Since this holds for any $\epsilon \geq 0$, the lower bound result follows.

    We next establish the upper bound.  For all $\tilde{\theta} \in \tilde{\Theta}_{\epsilon}$,
    \begin{align*}
        \Lc_T
        & = \frac{\I(H_T;\theta)}{T}\\
        & = \frac{1}{T} \sum_{t=0}^{T-1} \I(Y_{t+1};\theta|H_t)\\
        & \overset{(a)}{=} \inf_{\tilde{\theta}}\ \frac{1}{T} \sum_{t=0}^{T-1} \I(Y_{t+1};\theta,\tilde{\theta}|H_t)\\
        & = \inf_{\tilde{\theta}}\ \frac{1}{T} \sum_{t=0}^{T-1} \I(Y_{t+1};\tilde{\theta}|H_t) + \I(Y_{t+1};\theta|\tilde{\theta}, H_t)\\
        & \leq \inf_{\tilde{\theta}}\ \frac{\I(H_T;\tilde{\theta})}{T} + \frac{1}{T} \sum_{t=0}^{T-1}\E\left[\KL(P^*_t\|\tilde{P}_t)\right]\\
        & \overset{(b)}{\leq} \inf_{\tilde{\theta}}\ \frac{\I(H_T;\tilde{\theta})}{T} + \E\left[\KL(P^*_0\|\tilde{P}_0)\right]\\
        & = \inf_{\epsilon \geq 0} \frac{\H_\epsilon(\theta)}{T} + \epsilon,
    \end{align*}
    where $(a)$ follows from our regularity assumption and $(b)$ follows from Lemma \ref{cor:distortion_ub}.
\end{proof}

In the following four sections we will apply this result to specific data generating processes.  We begin with linear regression.

\subsection{Linear Regression}
\label{se:linear-regression}

\begin{figure}[H]
    \centering
    \includegraphics[width=\textwidth]{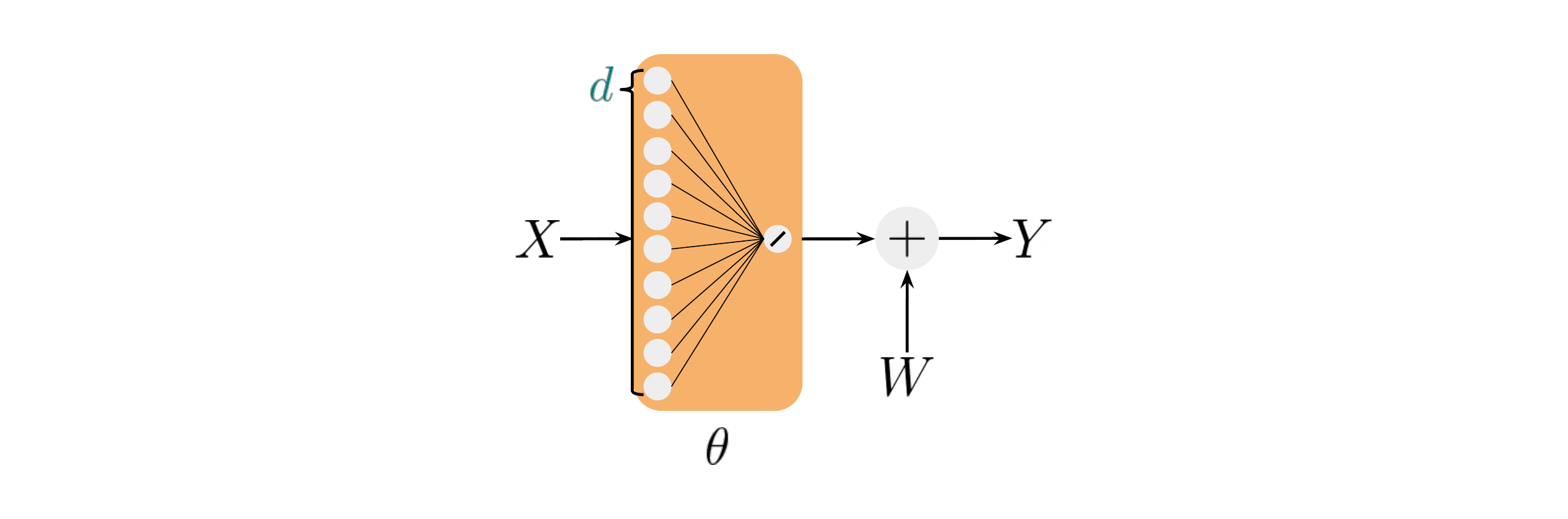}
    \caption{We depict our linear regression data generating process above.  It consists of an input vector $X$ of dimension $d$, an unknown weight vector $\theta$ of dimension $d$, and a final output $Y$ which is the sum of $\theta^\top X$ and independent Gaussian noise $W$.}
    \label{fig:lin_reg}
\end{figure}

\subsubsection{Data Generating Process}

In linear regression, the variable that we are interested in estimating is a random vector $\theta\in\Re^d$.  In our Bayesian framework, a known prior distribution $\Pr(\theta\in\cdot)$ expresses uncertainty about $\theta$.  We assume that $\Pr(\theta\in\cdot) = \normal(0, I_d/d)$.  For all $t \in \Z_{+}$, inputs and outputs are generated according to a random vector $X_t \overset{iid}{\sim} \normal(0, I_d)$ and
$$Y_{t+1} = \theta^\top X_t + W_{t+1},$$
where $W_{t}\overset{iid}{\sim} \normal(0, \sigma^2)$ for known variance $\sigma^2$.  It is straightforward to extend results of this section to Gaussian input and prior distributions with different means and covariance matrices and even distributions that are not Gaussian.  We study our specific Gaussian case to reduce clutter without compromising insight.

\subsubsection{Preliminary Results}

In this section, we establish preliminary results that will allow us to streamline our error analysis using the tools established in Section \ref{sec:info_learn}.  These results characterize the rate-distortion function for our linear regression model.

Our first result pertains to levels of error that require no information.  In particular, because no information is required to achieve the error delivered by an uninformed prediction, there should be a threshold $\overline{\epsilon}$ such that $\H_\epsilon(\theta) = 0$ for all $\epsilon \geq \overline{\epsilon}$.  This result establishes the threshold.
\begin{lemma}{\bf (linear regression 0 rate-distortion threshold)}\label{le:lin_reg_0}
    For all $d\in \Z_{++}, \sigma^2 \in \Re_{++}$, if for all $t\in \Z_{+}$, $(X_t, Y_{t+1})$ are generated according to the linear regression process and $\epsilon \geq \frac{1}{2}\ln(1 + 1/\sigma^2)$, then for all $T\in \Z_{++}$,
    $$\H_{\epsilon}(\theta) = 0.$$
\end{lemma}
\begin{proof}
    Let $\tilde{\theta} = \emptyset$. Then,
    \begin{align*}
        \E\left[\KL(P^*_0 \| \tilde{P}_0)\right]
        & = \I(Y_1;\theta|\tilde{\theta}, X_0)\\
        & = \diffentropy(Y_1|\tilde{\theta}, X_0) - \diffentropy(Y_1|\theta, \tilde{\theta}, X_0)\\
        & = \diffentropy(Y_1|X_0) - \diffentropy(W)\\
        & \overset{(a)}{\leq} \E\left[\frac{1}{2}\ln\left(2\pi e \left(\sigma^2 + \text{Var}[\theta^\top X|X]\right)\right)\right] - \frac{1}{2}\ln\left(2\pi e \sigma^2\right)\\
        & \overset{(b)}{=} \E\left[\frac{1}{2}\ln\left(2\pi e \left(\sigma^2 + \frac{\|X\|^2_2}{d}\right)\right)\right] - \frac{1}{2}\ln\left(2\pi e \sigma^2\right)\\
        & \overset{(c)}{\leq} \frac{1}{2}\ln\left(2\pi e \left(\sigma^2 + \frac{\E\left[\|X\|^2_2\right]}{d}\right)\right) - \frac{1}{2}\ln\left(2\pi e \sigma^2\right)\\
        & = \frac{1}{2}\ln\left(1 + \frac{1}{\sigma^2}\right)\\
        & \leq \epsilon,
    \end{align*}
    where $(a)$ follows from Lemma \ref{le:max_entropy}, $(b)$ follows from the fact that $\theta$ has covariance matrix $I_{d}/d$, and $(c)$ follows from Jensen's inequality.  The above establishes that $\tilde{\theta}\in \tilde{\Theta}_{\epsilon}$.  The result follows from the fact that $\I(\theta;\tilde{\theta}) = 0$.
\end{proof}

The preceding result addresses the trivial case where error is so large that no information is required.  Our next result addresses all levels of error.
\begin{lemma}{\bf (linear regression $0$-horizon rate-distortion upper bound)}\label{le:lin_reg_rd_ub}
    For all $d \in \Z_{++}$, $\sigma^2\in \Re_{++}$, and $\epsilon \in \Re_{++}$, if for all $t \in \Z_{+}$, $(X_t, Y_{t+1})$ are generated according to the linear regression process, then for all $T$,
    $$\H_{\epsilon}(\theta) \leq \left(\frac{d}{2}\ln\left(\frac{1}{\sigma^2(e^{2\epsilon}-1)}\right)\right)_+.$$
\end{lemma}
\begin{proof}
    Fix $\sigma^2\in \Re_{++}$. Let $\tilde{\theta} = \theta + V$, where $V \sim \mathcal{N}(0, \delta^2 I_d / d)$ for $\delta^2 = \frac{\sigma^2(e^{2\epsilon}-1)}{1 - \sigma^2(e^{2\epsilon}-1)}$ and $V\perp \theta$. Note that $\delta^2 \geq 0$ for all $0 < \epsilon < \frac{1}{2}\ln\left(1 + \frac{1}{\sigma^2}\right)$. We begin by showing that $\tilde{\theta} \in\tilde{\Theta}_{\epsilon}$.
    \begin{align*}
        \E\left[\KL(P^*_0 \| \tilde{P}_0)\right]
        & = \I(Y_{1};\theta|\tilde{\theta}, X_0)\\
        & = \diffentropy(Y_1|\tilde{\theta}, X_0) - \diffentropy(Y_1|\theta,\tilde{\theta}, X_0)\\
        & = \diffentropy(W_1 + \theta^\top X|\tilde{\theta}, X_0) - \diffentropy(W_1)\\
        & = \diffentropy\left(W_1 + \theta^\top X_0 - \frac{\tilde{\theta}^\top X_0}{1+\delta^2}\bigg|\tilde{\theta}, X_0\right) - \diffentropy(W_1)\\
        & = \diffentropy\left(W_1 + \left(\frac{\delta^2}{1+\delta^2}\theta + \frac{1}{1+\delta^2}V\right)^\top X_0 \Big|\tilde{\theta}, X_0\right) - \diffentropy(W_1)\\
        & \leq \diffentropy\left(W_1 + \left(\frac{\delta^2}{1+\delta^2}\theta + \frac{1}{1+\delta^2}V\right)^\top X_0\Big|X_0\right) - \diffentropy(W_1)\\
        & \overset{(a)}{\leq} \E\left[\frac{1}{2}\ln\left(2\pi e\left(\sigma^2 + \left(\frac{\delta^4}{d(1 + \delta^2)^2} + \frac{\delta^2}{d(1 + \delta^2)^2}\right)\|X_0\|^2_2\right)\right)\right] - \frac{1}{2}\ln\left(2\pi e \sigma^2\right)\\
        & = \E\left[\frac{1}{2}\ln\left(1 + \frac{\delta^2\|X_0\|^2_2}{d(1 + \delta^2)\sigma^2}\right)\right]\\
        & \overset{(b)}{\leq} \frac{1}{2}\ln\left(1 + \frac{\delta^2}{(1+\delta^2)\sigma^2}\right)\\
        & = \frac{1}{2}\ln\left(e^{2\epsilon}\right)\\
        & = \epsilon,
    \end{align*}
    where $(a)$ follows from Lemma \ref{le:max_entropy} and $(b)$ follows from Jensen's inequality. Next, we upper bound the rate of $\tilde{\theta}$:
    \begin{align*}
        \I(\theta;\tilde{\theta})
        & = \diffentropy(\tilde{\theta}) - \diffentropy(\tilde{\theta}|\theta)\\
        & = \diffentropy(\tilde{\theta}) - \diffentropy(V)\\
        & \overset{(a)}{\leq} \frac{d}{2}\ln\left(2\pi e \left(\frac{\delta^2 + 1}{d}\right)\right) - \frac{d}{2}\ln\left(2\pi e \left(\frac{\delta^2}{d}\right)\right)\\
        & = \frac{d}{2}\ln\left(1 + \frac{1}{\delta^2}\right)\\
        & = \frac{d}{2}\ln\left(\frac{1}{\sigma^2\left(e^{2\epsilon}-1\right)}\right),
    \end{align*}
    where $(a)$ follows from Lemma \ref{le:max_entropy}. The result follows from Lemma \ref{le:lin_reg_0} and the fact that $\tilde{\theta} \in \tilde{\Theta}_{\epsilon}$ for all $T\in \Z_{++}$ and $\epsilon < 1/2\ln(1+1/\sigma^2)$.
\end{proof}

Proofs of our previous two results did not rely on the assumption that $\Pr(\theta\in\cdot) = \normal(0, I_d/d)$ but only on the fact that the elements of $\theta$ are independent with variances that sum to $1$.  To establish a \emph{lower bound} on the rate-distortion function, we rely on Gaussianity.
\begin{lemma}{\bf (linear regression $0$-horizon rate-distortion lower bound)}\label{le:lin_reg_rd_lb}
    For all $d \in \mathbb{Z}_{++}$ s.t. $d > 2$, $\sigma^2\in\Re_{++}$, and $\epsilon \in \Re_{++}$, if for all $t \in \Z_{+}$, $(X_t, Y_{t+1})$ are generated according to the linear regression process, then
    $$\H_{\epsilon}(\theta)\ \geq\ \left(\frac{d}{2}\ln\left(\frac{1}{(8+\frac{d}{d-2}\sigma^2)\epsilon}\right)\right)_+.$$ 
\end{lemma}
\begin{proof}
    Fix $\sigma^2 \in \Re_{++}$, $\epsilon \in \mathbb{Z}_+$, and $\tilde{\theta} \in \tilde{\Theta}_{\epsilon}$. Then,
    \begin{align*}
        \epsilon
        & \overset{(a)}{\geq} \E\left[\KL(P^*_0 \| \tilde{P}_0)\right]\\
        & =  \I(Y_1;\theta|\tilde{\theta}, X_0)\\
        & \overset{(b)}{\geq} \E\left[\frac{\|X_0\|^2_2}{2(4\|X_0\|_2^2+d\sigma^2)}\right]\E\left[\|\theta-\E[\theta|\tilde{\theta}]\|^2_2\right]\\
        & = \E\left[\frac{1}{\frac{2(4\|X_0\|_2^2+d\sigma^2)}{\|X_0\|_2^2}}\right]\E\left[\|\theta-\E[\theta|\tilde{\theta}]\|^2_2\right]\\
        & = \E\left[\frac{1}{8 + \frac{d\sigma^2}{\|X_0\|^2_2}}\right]\E\left[\|\theta-\E[\theta|\tilde{\theta}]\|^2_2\right]\\
        & \overset{(c)}{\geq} \frac{1}{\E\left[8 + \frac{d\sigma^2}{\|X_0\|^2_2}\right]}\E\left[\|\theta-\E[\theta|\tilde{\theta}]\|^2_2\right]\\
        & \overset{(d)}{=} \frac{1}{8 + \frac{d\sigma^2}{d-2}}\E\left[\|\theta-\E[\theta|\tilde{\theta}]\|^2_2\right]\\
        & \geq \frac{1}{8 + \frac{d}{d-2}\sigma^2}\E\left[\|\theta-\E[\theta|\tilde{\theta}]\|^2_2\right]
    \end{align*}
    where $(a)$ follows from the fact that $\tilde{\theta} \in \tilde{\Theta}_\epsilon$, $(b)$ follows from Lemma \ref{le:mse-mutual-info-inequality}, $(c)$ follows from Jensen's inequality, and $(d)$ follows from the fact that $\E[1/\chi^2(d)] = 1/(d-2)$.
    
    Since the above condition is an implication that holds for arbitrary $\tilde{\theta}\in\tilde{\Theta}_\epsilon$, minimizing the rate $\I(\theta;\tilde{\theta})$ over the set of proxies that satisfy 
    $\E[\|\theta-\E[\theta|\tilde{\theta}]\|^2_2] \leq (8+\frac{d}{d-2}\sigma^2) \epsilon$ will provide a lower bound. However, this is simply the rate-distortion problem for a multivariate source under squared error distortion. For this problem there exists a well known lower bound (Theorem 10.3.3 of \citep{10.5555/1146355}). The lower bound follows as a result.\newline
\end{proof}
It is worth noting that this result extends to sub-Gaussian $\theta$, as established in Appendix \ref{apdx:lin_reg_lb}.

\subsubsection{Main Result}

The rate-distortion bounds which we established in the previous section allow us to directly apply Theorem \ref{th:horizon-indep-rd} to arrive at error bounds.

\begin{theorem}{\bf(linear regression error bounds)}\label{th:lin_reg_error_bounds}
    For all $d \in \mathbb{Z}_{++}$, $\sigma^2 \geq 0$, if for all $t$, $(X_t, Y_{t+1})$ are generated according to the linear regression processes, then for all $T$,
     $$\frac{d}{2T}W\left(\frac{2T}{d(8+\frac{d}{d-2}\sigma^2)}\right)\ \leq\ \Lc_{T} \ \leq\ \left(\frac{d}{2T}\ln\left(\frac{T}{\sigma^2 d}\right)\right)_+ + \frac{1}{2T}\ln\left(1+\frac{d}{T}\right),$$
     where $W$ is the Lambert W function.
\end{theorem}
\begin{proof}
    \begin{align*}
        \Lc_T
        & \overset{(a)}{\leq} \inf_{\epsilon\geq 0}\ \frac{\H_{\epsilon}(\theta)}{T} + \epsilon\\
        & \overset{(b)}{\leq} \inf_{\epsilon\geq 0}\ \frac{\left(d\ln\left(\frac{1}{\sigma^2(e^{2\epsilon}-1)}\right)\right)_+}{2T} + \epsilon\\
        & \overset{(c)}{\leq} \frac{\left(d\ln\left(\frac{1}{\sigma^2\left(\left(1 + \frac{d}{T}\right)^T-1\right)}\right)\right)_+}{2T} + \frac{1}{2T}\ln\left(1 + \frac{d}{T}\right)\\
        & \leq \left(\frac{d}{2T}\ln\left(\frac{T}{\sigma^2 d}\right)\right)_+ + \frac{1}{2T}\ln\left(1 + \frac{d}{T}\right),
    \end{align*}
    where $(a)$ follows from Theorem \ref{th:horizon-indep-rd}, $(b)$ follows from Lemma \ref{le:lin_reg_rd_ub}, and $(c)$ follows from setting $\epsilon = \frac{1}{2T}\ln\left(1 + \frac{d}{T}\right)$.
    \begin{align*}
        \Lc_T
        & \geq \sup_{\epsilon\geq 0}\ \min\left\{\frac{\H_\epsilon(\theta)}{T}, \epsilon\right\}\\
        & \geq \sup_{\epsilon\geq 0}\ \min\left\{\left(\frac{d}{2T}\ln\left(\frac{1}{(8+\frac{d}{d-2}\sigma^2)\epsilon}\right)\right)_+,\ \epsilon \right\}\\
        & = \epsilon \text{ s.t. } \left(\frac{d}{2T}\ln\left(\frac{1}{(8+\frac{d}{d-2}\sigma^2)\epsilon}\right)\right)_+ = \epsilon\\
        & \overset{(a)}{=} \frac{e^{-W\left(\frac{8 + \frac{d}{d-2}\sigma^2}{2}\right)}}{\frac{2T}{d(8+\frac{d}{d-2}\sigma^2)}}\\
        & \overset{(b)}{\geq} \frac{e^{-\ln\left(8 + \frac{d}{d-2}\sigma^2\right)}}{\frac{2T}{d(8+\frac{d}{d-2}\sigma^2)}}\\
        & = \frac{d}{2T},
    \end{align*}
    where in $(a)$, $W$ denotes the Lambert $W$ function and $(b)$ follows from the fact that for all $x \geq e$, $W(x) \leq \ln(2x)$.
\end{proof}

Notably, this bound is consistent with classical results in statistics which dictate that mean squared error grows linearly in the problem dimension $d$ and decays linearly in the number of samples observed $T$.  In the following section, we will observe that qualitatively similar results also hold for logistic regression.

\subsection{Logistic Regression}
\begin{figure}[H]
    \centering
    \includegraphics[width=\textwidth]{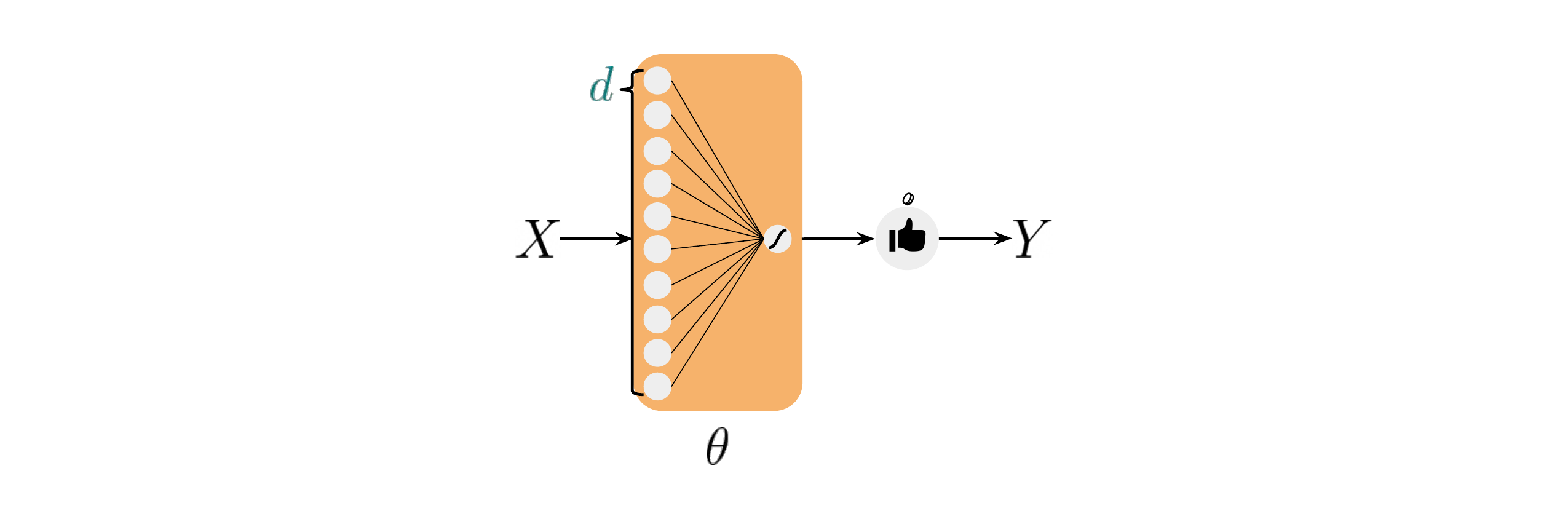}
    \caption{We depict our logistic regression data generating process above.  It consists of an input vector $X$ of dimension $d$, an unknown weight vector $\theta$ of dimension $d$, and a final binary output $Y\in\{0,1\}$.  $Y$ is \emph{sampled} according to probabilities generated by the sigmoid function applied to $\theta^\top X$.}
    \label{fig:log_reg}
\end{figure}

\subsubsection{Data Generating Process}
We next study logistic regression to demonstrate the application of our tools to \emph{classification}.  Just as in linear regression, in logistic regression, the variable that we are interested in estimating is a random vector $\theta \in \Re^d$.  We assume that $\Pr(\theta\in\cdot) = \mathcal{N}(0, I_d/d)$.  For all $t \in \mathbb{Z}_{+}$, inputs and outputs are generated according to a random vector $X_t \overset{iid}{\sim} \normal(0, I_d)$ and
$$Y_{t+1} = 
\begin{cases}
    1 & \text{ w.p. } \frac{1}{1+e^{-\theta^\top X_t}}, \\
    0 & \text{ otherwise.} 
\end{cases}$$
Our results can be generalized beyond these prior and input distributions.

\subsubsection{Preliminary Results}

We begin with the following result, which upper bounds the binary KL divergence of a sigmoidal output via the squared difference between the logits.
\begin{lemma}{\bf(squared error upper bounds binary KL divergence)}\label{le:kl_ub}
    For all $x, y \in \Re$,
    $$\frac{1}{1+e^{-x}}\ln\frac{1+e^{-y}}{1+e^{-x}} + \frac{1}{1+e^x}\ln\frac{1+e^{y}}{1+e^{x}} \leq \frac{(x-y)^2}{8}.$$
\end{lemma}

In contrast to a linear model with aleatoric randomness in the form of Gaussian noise, the logistic model exhibits aleatoric randomness via a Bernoulli sample.  This will complicate calculation of the $0$-horizon distortion $\E[\KL(P^*_t\| \tilde{P}_t)]$, since the prediction $\tilde{P}_t = \Pr(Y_{t+1} \in \cdot | \tilde{\theta}, X_0)$ is not Gaussian for natural choices of $\tilde{\theta}$.  To circumvent this issue, we provide a general upper bound for the $0$-horizon distortion function via the distortion of a prediction under a change of measure: $\tilde{P}'_t = \Pr(Y_{t+1} \in\cdot|\theta \leftarrow \tilde{\theta}, X_t)$.  Recall that the dependence of this prediction on $\tilde{\theta}$ and $X_t$ is expressed by a function $f$ such that, for all $(\tilde{\nu}, x) \in \tilde{\Theta} \times \mathcal{X}$, $f(x, \tilde{\nu}) = \Pr(Y_{t+1} \in\cdot|\theta=\tilde{\nu}, X_t=x)$.  The distortion incurred by $\tilde{P}'_t$, which does not depend on $t$, is in general much simpler to upper bound for concrete problem instances.
\begin{lemma}\label{le:change_measure_ub}{\bf ($\tilde{P}_0$ achieves lower error than the change of measure prediction)}
    If $\tilde{\theta}:\Omega\mapsto\tilde{\Theta}$ is a random variable such that for all $t\in \Z_{+}$, $Y_{t+1}\perp \tilde{\theta}|\theta, H_t$ and $\tilde{\Theta} \subseteq \Theta$, then
    $$\E\left[\KL(P^*_0 \| \tilde{P}_0)\right]\ \leq\ \E\left[\KL(P^*_0 \| \tilde{P}'_0)\right].$$
\end{lemma}
\begin{proof}
    \begin{align*}
        \E\left[\KL(P^*_0\|\tilde{P}_0)\right]
        &= \E\left[\KL(P^*_0\|\tilde{P}_0)\right] + \E\left[\KL\left(\tilde{P}_0 \| \tilde{P}'_0\right)\right]\\
        & \overset{(a)}{\leq} \E\left[\KL(P^*_0\|\tilde{P}_0)\right] + \E\left[\KL\left(\Pr(Y_{1}\in\cdot|\tilde{\theta}, X_0)\|\Pr(Y_{1}\in\cdot|\theta\leftarrow\tilde{\theta}, X_0)\right)\right]\\
        & = \E\left[\KL(P^*_0\|\tilde{P}_0)\right] + \E\left[\E\left[\ln\frac{\Pr(Y_{1}|\tilde{\theta}, X_0)}{\Pr(Y_{1}|\theta\leftarrow\tilde{\theta}, X_0)}\bigg|\tilde{\theta}, X_0\right]\right]\\
        & = \E\left[ \ln\frac{\Pr(Y_{1}|\theta, X_0)}{\Pr(Y_{1}|\tilde{\theta}, X_0)} \right] + \E\left[\ln\frac{\Pr(Y_{1}|\tilde{\theta}, X_0)}{\Pr(Y_{1}|\theta \leftarrow \tilde{\theta}, X_0)}\right]\\
        & = \E\left[ \ln\frac{\Pr(Y_{1}|\theta, X_0)}{\Pr(Y_{1}|\theta \leftarrow \tilde{\theta}, X_0)} \right]\\
        & = \E\left[\KL(P^*_0 \| \tilde{P}'_0)\right] \\
    \end{align*}
    where $(a)$ follows from the fact that KL divergence is non-negative.
\end{proof}

Using Lemmas \ref{le:kl_ub} and \ref{le:change_measure_ub}, we establish the following upper bound on the $0$-horizon rate-distortion function for logistic regression.

\begin{lemma}{\bf (logistic regression $0$-horizon rate-distortion upper bound)}\label{le:log_reg_rd}
    For all $d \in \mathbb{Z}_{++}$ and $\epsilon \in \Re_{++}$, if, for all $t \in \mathbb{Z}_{+}$, $(X_t, Y_{t+1})$ are generated according to the logistic regression process then
    $$\H_{\epsilon}(\theta) \ \leq\ \frac{d}{2}\ln\left(1 + \frac{1}{8\epsilon}\right).$$
\end{lemma}
\begin{proof}
    Let $\tilde{\theta} = \theta + Z$ where $Z \perp \theta$ and $Z\sim \normal(0, 8\epsilon I /d)$. Then,
    \begin{align*}
        \E\left[\KL(P^*_0\| \tilde{P}_0)\right]
        & \overset{(a)}{\leq} \E\left[\KL\left(P^*_0 \ \|\ \tilde{P}'_0\right)\right]\\
        & = \E\left[\KL\left(\Pr(Y_1\in\cdot|\theta, X_0) \ \|\ \Pr(Y_1\in\cdot|\theta \leftarrow \tilde{\theta}, X_0)\right)\right]\\
        & = \E\left[\frac{\ln\left(\frac{1+e^{-\tilde{\theta}^\top X_0}}{1+e^{-\theta^\top X_0}}\right)}{1+e^{-\theta^\top X_0}} + \frac{\ln\left(\frac{1+e^{\tilde{\theta}^\top X_0}}{1+e^{\theta^\top X_0}}\right)}{1+e^{\theta^\top X_0}}\right]\\
        & \overset{(b)}{\leq} \frac{\E\left[\left(\theta^\top X_0 - \tilde{\theta}^\top X_0\right)^2\right]}{8}\\
        & = \frac{\E\left[\|\theta - \tilde{\theta}\|^2_2\right]}{8}\\
        & = \epsilon,
    \end{align*}
    where $(a)$ follows from Lemma \ref{le:change_measure_ub} and $(b)$ follows from Lemma \ref{le:kl_ub}.
    
    We now upper bound the rate.
    \begin{align*}
        \I(\theta;\tilde{\theta})
        & = \diffentropy(\tilde{\theta}) - \diffentropy(\tilde{\theta}|\theta)\\
        & \overset{(a)}{\leq} \frac{d}{2}\ln\left(2\pi e \left(\frac{1}{d} + \frac{8\epsilon}{d}\right)\right) - \frac{d}{2}\ln\left(2\pi e \frac{8\epsilon}{d}\right)\\
        & = \frac{d}{2}\ln\left(1 + \frac{1}{8\epsilon}\right),
    \end{align*}
    where $(a)$ follows from Lemma \ref{le:max_entropy}.  The result follows.
\end{proof}

\subsubsection{Main Result}

With the $0$-horizon rate-distortion upper bound in place, we establish the following upper bound on the optimal horizon $T$ error.

\begin{theorem}{\bf (logistic regression estimation error upper bound)} For all $d\in \Z_{++}$, if for all $t$, $(X_t, Y_{t+1})$ is generated according to the logistic regression process then, for all $T,$
    $$ \Lc_T\ \leq\ \frac{d}{2T}\left(1 + \ln\left(1 + \frac{T}{4d}\right)\right).$$
\end{theorem}
\begin{proof}
    \begin{align*}
        \Lc_T
        & \overset{(a)}{\leq} \inf_{\epsilon \geq 0}\ \frac{\H_{\epsilon}(\theta)}{T} + \epsilon\\
        & \overset{(b)}{\leq} \inf_{\epsilon \geq 0}\ \frac{\frac{d}{2}\ln\left(1 + \frac{1}{8\epsilon}\right)}{T} + \epsilon\\
        & \overset{(c)}{\leq} \frac{d}{2T}\ln\left(1 + \frac{T}{4d}\right) + \frac{d}{2T},
    \end{align*}
    where $(a)$ follows from Theorem \ref{th:horizon-indep-rd}, $(b)$ follows from Lemma \ref{le:log_reg_rd}, and $(c)$ follows from setting $\epsilon = \frac{d}{2T}$.
\end{proof}

Just as with linear regression, this error bound is $\tilde{\mathcal{O}}(d/T)$.  In the next two sections, we consider much more complex data generating processes, which are associated with multi-layer perceptrons.

\subsection{Finite-Width Multi-Layer Perceptrons}\label{sec:dnn}
\begin{figure}[H]
    \centering
    \includegraphics[width=\textwidth]{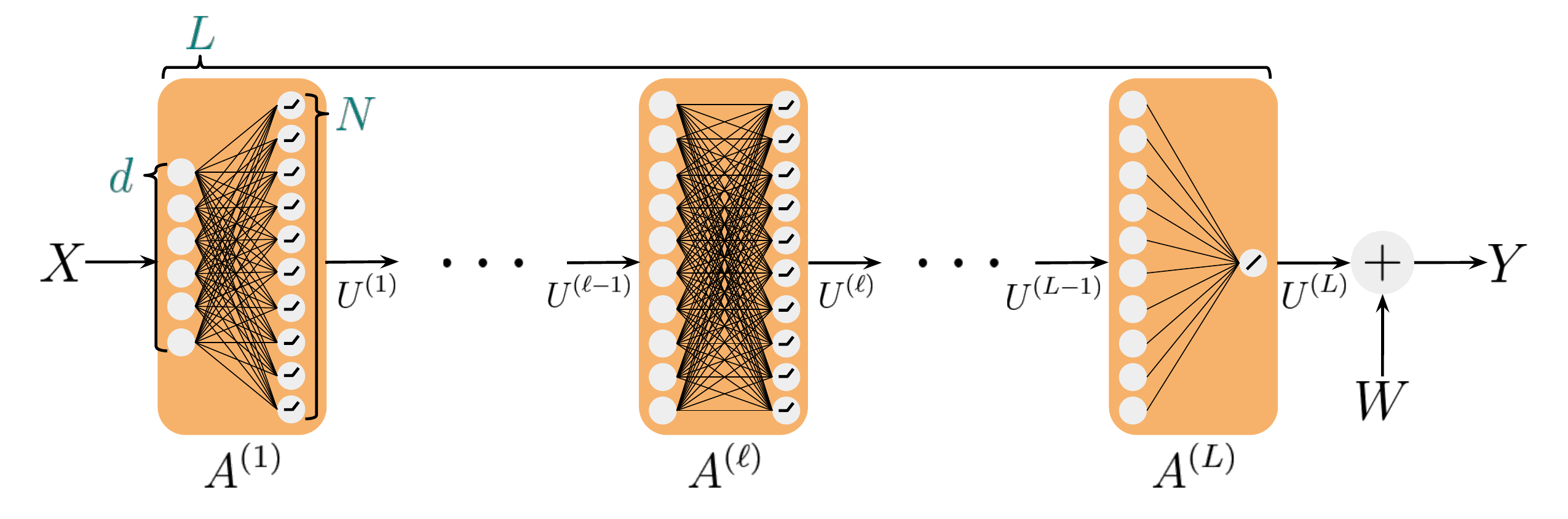}
    \caption{We depict our finite-width multi-layer perceptron data generating process above.  It consists of input dimension $d$, width $N$, depth $L$, with output dimension $1$, and ReLU activation units.  We denote the weights at layer $\ell$ by $A^{(\ell)}$ and the output of layer $\ell$ by $U^{(\ell)}$.  We assume that the final output $Y$ is the sum of the final network output $U^{(L)}$ and independent Gaussian noise $Z$.}
    \label{fig:dnn}
\end{figure}

\subsubsection{Data Generating Process}

For a finite-width multi-layer perceptron (MLP), $\theta$ consists of random matrices $(A^{(1)}, \ldots, A^{(L)})$, which represent the weights that parameterize each of the $L$ layers.  For simplicity, we assume that the width of every hidden layer is identical and equal to a positive integer $N$ and that the MLP has input dimension $d$ and output dimension $1$.  In our Bayesian framework, a known prior distribution $\Pr(\theta \in\cdot)$ expresses uncertainty over the MLP weights.  We assume the following prior distribution:
$$\Pr(A^{(\ell)}) = \begin{cases}
    \normal(0, \frac{I_{N\cdot d}}{d}) & \text{ if } \ell = 1, \\
    \normal(0, \frac{I_{N\cdot N}}{N}) & \text{ if } 2 \leq \ell \leq L-1, \\
    \normal(0, \frac{I_{N}}{N}) & \text{ if } \ell = L.
\end{cases}$$
We further make the assumption that the weights \emph{across} layers are independent.  These Gaussianity assumptions are not required to prove our results, though the specified covariance structure is.  We make the Gaussianity assumptions only to simplify the proofs.

For all $t \in \Z_+$, inputs and outputs are generated according to a random vector $X_t\overset{iid}{\sim}\normal(0, I_d)$ and
\begin{align*}
    U^{(0)}_t
    & = X_t\\
    U^{(\ell+1)}_t & = \text{ReLU}\left(A^{(\ell+1)}U^{(\ell)}_t\right)\\
    Y_{t+1} & = U^{(L)}_t + W_{t+1},
\end{align*}
where $W_{t} \overset{iid}{\sim} \normal(0, \sigma^2)$ for known variance $\sigma^2$.

\subsubsection{Preliminary Results}

In this section, we cover several lemmas which facilitate error analysis for MLPs.  Lemma \ref{le:nn_distortion_ub} establishes an upper bound on the horizon $0$ distortion function in the finite-width MLP setting.  This result allows us to easily derive upper bounds for the $0$-horizon rate-distortion function in Theorem \ref{th:nn_rd_bound}.  We begin with the following three Lemmas (\ref{le:decomp}, \ref{le:true_input_inequality}, \ref{le:nn_layer_dist}) which facilitate the derivation of distortion upper bound in Lemma \ref{le:nn_distortion_ub}.  We begin with an initial result which decomposes distortion into a sum of simpler terms.

We first introduce some requisite notation to necessary to understand the following result.  We restrict our attention to $\tilde{\theta} = (\tilde{A}^{1}, \tilde{A}^{2}, ..., \tilde{A}^{(L)})$ which are independent from each other.  Under this assumption, we abbreviate the prediction based on an approximation of the first $\ell$ layers $\Pr(Y_{t+1}\in\cdot|A^{(\ell+1,L)}, \tilde{A}^{(1:\ell)}, H_t)$ as $\tilde{P}^{(\ell)}_t$.  The following result expresses the $0$-horizon distortion as the sum of distortions between consecutive approximate predictions across all layers.

\begin{lemma}{\bf ($0$-horizon distortion decomposition)}\label{le:decomp}
    If $(X_0,Y_1)$ are generated by the finite-width MLP process and $\tilde{\theta} = (\tilde{A}_1, \ldots, \tilde{A}_L)$ are independent random variables such that, $Y_1\perp\tilde{A}^{(1:L)}|A^{(1:L)}, X_0$, then
    $$\E\left[\KL(P^*_0 \| \tilde{P}_0)\right]\ =\ \sum_{\ell=1}^{L} \E\left[\KL(\tilde{P}^{(\ell-1)}_0 \| \tilde{P}^{(\ell)}_0)\right].$$
\end{lemma}
\begin{proof}
    \begin{align*}
        \E\left[\KL(P^*_0 \| \tilde{P}_0)\right]
        & = \I(Y_1;A^{(1:L)}|\tilde{A}^{(1:L)}, X_0)\\
        & \overset{(a)}{=} \sum_{\ell=1}^{L} \I(Y_1;A^{(\ell)}|A^{(\ell+1:L)}, \tilde{A}^{(1:L)}, X_0)\\
        & \overset{(b)}{=} \sum_{\ell=1}^{L} \I(Y_1;A^{(\ell)}|A^{(\ell+1:L)}, \tilde{A}^{(1:\ell)}, X_0)\\
        & = \sum_{\ell=1}^{L} \E\left[\KL(\Pr(Y_1\in\cdot|A^{(\ell:L)}, \tilde{A}^{(1:\ell-1)}, X_0) \| \Pr(Y_1\in\cdot|A^{(\ell+1:L)}, \tilde{A}^{(1:\ell)}, X_0))\right]\\
        & = \sum_{\ell=1}^{L} \E\left[\KL(\tilde{P}^{(\ell-1)}_0 \| \tilde{P}^{(\ell)}_0)\right],
    \end{align*}
    where $(a)$ follows from the chain rule and $(b)$ follows from conditional independence assumptions.
\end{proof}

With this decomposition in mind, we can derive a suitable upper bound for the total distortion by bounding each \emph{individual} term of the decomposition.  The following two results establish upper bounds for the individual terms.

\begin{lemma}
\label{le:true_input_inequality}
If $(X_0,Y_1)$ are generated by the finite-width MLP process and $\tilde{A}_1, \ldots, \tilde{A}_L$ are independent random variables such that, $Y_1\perp\tilde{A}^{(1:L)}|A^{(1:L)}, X_0$, then for all $\ell \in \{1, \hdots, L\}$,
$$\E\left[\KL(\tilde{P}^{(\ell-1)}_0\|\tilde{P}^{(\ell)}_0)\right]\ \leq\ 
\E\left[\KL(\tilde{P}_0^{(\ell-1)}\| \Pr(Y_1\in\cdot|A^{(\ell+1:L)}, \tilde{A}^{(\ell)}, U^{(\ell)}_0))\right].$$
\end{lemma}
\begin{proof}
    \begin{align*}
        \E\left[\KL(\tilde{P}^{(\ell-1)}_0\|\tilde{P}^{(\ell)}_0)\right]
        & = \I(Y_1; A^{(\ell)}|A^{(\ell+1:L)}, \tilde{A}^{(1:\ell)}, X_0)\\
        & = \I(Y_1, X_0, \tilde{A}^{(1:\ell-1)}; A^{(\ell)}|A^{(\ell+1:L)},\tilde{A}^{(\ell)}) - \I(X_0, \tilde{A}^{(1:\ell-1)}; A^{(\ell)}|A^{(\ell+1:L)},\tilde{A}^{(\ell)})\\
        & = \I(Y_1, X_0, \tilde{A}^{(\ell-1:1)} ;A^{(\ell)}|A^{(\ell+1:L)},\tilde{A}^{(\ell)})\\
        & = \I(Y_1; A^{(\ell)}|A^{(\ell+1:L)}, \tilde{A}^{(\ell)}) + \I(X_0, \tilde{A}^{(1:\ell-1)};A^{(\ell)}|Y_1, A^{(\ell+1:L)}, \tilde{A}^{(\ell)})\\
        & \overset{(a)}{\leq} \I(Y_1;A^{(\ell)}|A^{(\ell+1:L)}, \tilde{A}^{(\ell)}) + \I(X_0, A^{(1:\ell-1)};A^{(\ell)}|Y_1, A^{(\ell+1:L)}, \tilde{A}^{(\ell)})\\
        & = \I(Y_1, X_0, A^{(1:\ell-1)}; A^{(\ell)}|A^{(\ell+1:L)}, \tilde{A}^{(\ell)})\\
        & \overset{(b)}{=} \I(Y_1;A^{(\ell)}|X_0, A^{(1:\ell-1)}, A^{(\ell+1:L)}, \tilde{A}^{(\ell)})\\
        & \overset{(c)}{=} \I(Y_1;A^{(\ell)}|U^{(\ell-1)}_0, A^{(\ell+1:L)}, \tilde{A}^{(\ell)})\\
        & = \E\left[\KL(\tilde{P}_0^{(\ell-1)}\| \Pr(Y_1\in\cdot|A^{(\ell+1:L)}, \tilde{A}^{(\ell)}, U^{(\ell)}_0))\right],
    \end{align*}
    where $(a)$ follows from the fact that $A^{(\ell)} \perp \tilde{A}^{(1:\ell-1)}|(X_0, Y_1, A^{(1:\ell-1)})$ and the data processing inequality, $(b)$ follows from the fact that $\I(A^{(\ell)};A^{(1:\ell-1)}, X_0|A^{(\ell+1:L)}, \tilde{A}^{(\ell)}) = 0$, and $(c)$ follows from the fact that $Y_1\perp (A^{(1:\ell-1)}, X_0)|U^{(\ell-1)}$.
\end{proof}

This result allows us to simplify our analysis since the prediction in the RHS consists of just the layer $\ell$ approximation $\tilde{A}^{(\ell)}$ as opposed to approximations for the first $\ell$ layers $\tilde{A}^{(1:\ell)}$.  The following result uses Lemma \ref{le:true_input_inequality} to derive an upper bound for each individual term in the decomposition of Lemma \ref{le:decomp}.

\begin{lemma}\label{le:nn_layer_dist}
    For all $d,N,L\in \Z_{++}, \sigma^2\in\Re_{++}$, if $(X_t,Y_{t+1})$ are generated according to the finite-width MLP process and $\tilde{A}_1, \ldots, \tilde{A}_L$ are independent random variables such that for all $Y_{t+1}\perp\tilde{A}^{(1:L)}|(A^{(1:L)}, X_t)$, then for all $\ell \in [L]$,
    $$ \E\left[\KL(\tilde{P}^{(\ell-1)}_0 \| \tilde{P}^{(\ell)}_0)\right]\ \leq\ \frac{1}{2}\ln\left(1 + \frac{\E\left[\left\|(A^{(\ell)}-\tilde{A}^{(\ell)})U^{(\ell-1)}\right\|^2_2\right]}{\sigma^2 N}\right).$$
\end{lemma}
\begin{proof}
    In the proof below, we use the notation $f_{A^{(\ell:L)}}$ to denote the depth $L+1-\ell$ MLP with ReLU activation units and weights parameterized by $A^{(\ell:L)}$.
    \begin{align*}
        \E\left[\KL(\tilde{P}^{(\ell-1)}_0 \| \tilde{P}^{(\ell)}_0)\right]
        & \overset{(a)}{\leq} \E\left[\KL(\tilde{P}_0^{(\ell-1)}\| \Pr(Y_1\in\cdot|A^{(\ell+1:L)}, \tilde{A}^{(\ell)}, U^{(\ell)}_0))\right]\\
        & = \I(Y_1;A^{(\ell)}|A^{(\ell+1:L)}, \tilde{A}^{(\ell)}, U_0^{(\ell-1)})\\
        & = \diffentropy(Y_1|A^{(\ell+1:L)}, \tilde{A}^{(\ell)}, U_0^{(\ell-1)}) - \diffentropy(Y_1|A^{(\ell+1:L)}, U_0^{(\ell-1)})\\
        & = \diffentropy(Y_1|A^{(\ell+1:L)}, \tilde{A}^{(\ell)}, U_0^{(\ell-1)}) - \frac{1}{2}\ln\left(2\pi e \sigma^2\right)\\
        & \overset{(b)}{\leq} \E\left[\frac{1}{2}\ln\left(\frac{\E\left[\left(Y_1 - f_{\tilde{A}^{(\ell)}, A^{(\ell+1:L)}}\left(U_0^{(\ell-1)}\right)\right)^2\Big|A^{(\ell+1:L)},\tilde{A}^{(\ell)}, U_0^{(\ell-1)}\right]}{\sigma^2}\right)\right]\\
        & \leq \E\left[\frac{1}{2}\ln\left(1 + \frac{\left(f_{A^{(\ell:L)}}\left(U_0^{(\ell-1)}\right) - f_{\tilde{A}^{(\ell)}, A^{(\ell+1:L)}}\left(U_0^{(\ell-1)}\right)\right)^2}{\sigma^2}\right)\right]\\
        & \overset{(c)}{\leq} \E\left[\frac{1}{2}\ln\left(1 + \frac{\left(A^{(L)}A^{(L-1)}\ldots A^{(\ell+1)}(A^{(\ell)}-\tilde{A}^{(\ell)})U_0^{(\ell-1)}\right)^2}{\sigma^2}\right)\right]\\
        & \overset{(d)}{\leq} \frac{1}{2}\ln\left(1 + \frac{\E\left[\left(A^{(L)}A^{(L-1)}\ldots A^{(\ell+1)}(A^{(\ell)}-\tilde{A}^{(\ell)})U_0^{(\ell-1)}\right)^2\right]}{\sigma^2}\right)\\
        & \overset{(e)}{=} \frac{1}{2}\ln\left(1 + \frac{\E\left[\left\|(A^{(\ell)}-\tilde{A}^{(\ell)})U_0^{(\ell-1)}\right\|^2_2\right]}{\sigma^2 N}\right)\\
    \end{align*}
    where $(a)$ follows from Lemma \ref{le:true_input_inequality}, $(b)$ follows from Lemma \ref{le:max_entropy}, $(c)$ comes from the fact that for all $n$ and $x,y\in\Re^n$, $\|x-y\|^2 \geq \|\relu(x)-\relu(y)\|^2_2$, $(d)$ follows from Jensen's inequality, and $(e)$ follows from the independence and variance assumptions of the finite-width MLP data generating process.
\end{proof}

With this result, we now derive an upper bound for the $0$-horizon distortion.

\begin{lemma}{\bf ($0$-horizon distortion upper bound)}\label{le:nn_distortion_ub}
    For all $d,N,L\in \Z_{++}, \sigma^2\in\Re_{++}$, if $(X_t,Y_{t+1})$ are generated by the finite-width MLP process and $\tilde{A}_1, \ldots, \tilde{A}_L$ are independent random variables such that for all $t$, $Y_{t+1} \perp\tilde{A}^{(1:L)}|(A^{(1:L)}, X_t)$, then
    $$\E\left[\KL(P^*_t \| \tilde{P}_t)\right]\ \leq\ \sum_{\ell = 1}^{L} \frac{1}{2}\ln\left(1 + \frac{\E\left[\left\|\left(A^{(\ell)} - \tilde{A}^{(\ell)}\right) U^{(\ell-1)}\right\|^2_2\right]}{\sigma^2 N}\right).$$
\end{lemma}
\begin{proof}
    \begin{align*}
        \E\left[\KL(P^*_t \| \tilde{P}_t)\right]
        & \overset{(a)}{=} \sum_{\ell=1}^{L} \E\left[\KL(\tilde{P}^{(\ell+1)}_0 \| \tilde{P}^{(\ell)}_0)\right]\\
        & \overset{(b)}{\leq} \sum_{\ell = 1}^{L}\frac{1}{2}\ln\left(1 + \frac{\E\left[\left\|\left(A^{(\ell)} - \tilde{A}^{(\ell)}\right) U^{(\ell-1)}\right\|^2_2\right]}{\sigma^2 N}\right),
    \end{align*}
    where $(a)$ follows from Lemma \ref{le:decomp} and $(b)$ follows from Lemma \ref{le:nn_layer_dist}.
\end{proof}

It is worth noting that the above upper bound exhibits only a linear dependence on $L$ (ignoring logarithmic factors).  This is an improvement over results whose analyses depend on VC dimension \citep{bartlett1998almost,bartlett2019nearly} for which the dependence on depth $L$ is quadratic.  This is because in worst-case analysis, an analogous error term to the logarithm of Lemma \ref{le:nn_distortion_ub} may compound \emph{exponentially} in the depth of the network.  The average-case framework allows us to avoid such a penalty since this error is \emph{in expectation} not exponential in the depth.  We now present the final preliminary result which upper bounds the $0$-horizon rate-distortion function.

\begin{theorem}{\bf (finite-width MLP $0$-horizon rate-distortion upper bound)}\label{th:nn_rd_bound}
     For all $d,N,L\in \Z_{++}, \sigma^2,\epsilon \in\Re_{++}$, if $(X_t,Y_{t+1})$ are generated according to the finite-width MLP process then
     $$\H_\epsilon(\theta)\ \leq\ \left(\frac{(L-2)N^2+N+dN}{2}\right)\ln\left(1 + \frac{1}{\sigma^2\left(e^{\frac{2\epsilon}{L}}-1\right)}\right).$$
\end{theorem}
\begin{proof}
    For all $\ell \in [L]$, let $\tilde{A}^{(\ell)} = A^{(\ell)} + Z^{(\ell)}$ where $Z^{(\ell)} \perp A^{(\ell)}$ and each element of $Z^{(\ell)}$ is $\overset{iid}{\sim} \normal(0, \delta^2)$, where $\delta^2 = \sigma^2(e^{2\epsilon/L}-1)/d^{(\ell-1)}$ and $d^{(\ell-1)}$ denotes the dimension of $U^{(\ell-1)}$. Then,
    \begin{align*}
        \E\left[\KL(P^*_0 \| \tilde{P}_0)\right]
        & = \I(Y_1;A^{(1:L)}|\tilde{A}^{(1:L)},X_0)\\
        & \leq \sum_{\ell=1}^{L}\frac{1}{2}\ln\left(1 + \frac{\E\left[\left\|\left(A^{(\ell)} - \tilde{A}^{(\ell)}\right) U^{(\ell-1)}_0\right\|^2_2\right]}{\sigma^2 N}\right)\\
        & = \sum_{\ell=1}^{L}\frac{1}{2}\ln\left(1 + \frac{\E\left[\left\|Z^{(\ell)} U^{(\ell-1)}_0\right\|^2_2\right]}{\sigma^2 N}\right)\\
        & = \sum_{\ell=1}^{L}\frac{1}{2}\ln\left(1 + \frac{\sigma^2N(e^{\frac{2\epsilon}{L}}-1)\E\left[\|U^{(\ell-1)}_0\|^2_2\right]}{\sigma^2 N d^{(\ell-1)}}\right)\\
        & = \epsilon,
    \end{align*}
    where $(a)$ follows from Lemma \ref{le:nn_distortion_ub}, and $(b)$ follows from the fact that $\E[Z^{(\ell)\top}Z^{(\ell)}] = \frac{\sigma^2N(e^{2\epsilon/L}-1)}{d^{(\ell-1)}}I_{d^{(\ell-1)}}$.
    \begin{align*}
        \I(A^{(1:L)};\tilde{A}^{(1:L)})
        & = \sum_{\ell=1}^{L}\ \I(A^{(\ell)};\tilde{A}^{(\ell)})\\
        & = \sum_{\ell=1}^{L}\ \diffentropy(\tilde{A}^{(\ell)}) - \diffentropy(\tilde{A}^{(\ell)}|A^{(\ell)})\\
        & \overset{(a)}{\leq} \sum_{\ell=1}^{L} \frac{d^{(\ell-1)}d^{(\ell)}}{2}\ln\left(2\pi e \left(\delta^2 + \frac{1}{d^{(\ell-1)}}\right)\right) - \frac{d^{(\ell-1)}d^{(\ell)}}{2}\ln\left(2\pi e \delta^2\right)\\
        & = \sum_{\ell=1}^{L}\frac{d^{(\ell-1)}d^{(\ell)}}{2}\ln\left(1 + \frac{1}{\delta^2 d^{(\ell-1)}}\right)\\
        & = \sum_{\ell=1}^{L}\frac{d^{(\ell-1)}d^{(\ell)}}{2}\ln\left(1 + \frac{1}{\sigma^2(e^{\frac{2\epsilon}{L}}-1)}\right)\\
        & = \left(\frac{(L-2)N^2+N+dN}{2}\right)\ln\left(1 + \frac{1}{\sigma^2\left(e^{\frac{2\epsilon}{L}}-1\right)}\right),
    \end{align*}
    where $(a)$ follows from Lemma \ref{le:max_entropy}.  The result follows from the definition of the rate-distortion function.
\end{proof}

Note that the bound in Theorem \ref{th:nn_rd_bound} is only \emph{linear} in the parameter count of the MLP.  In the following subsection, we will leverage this rate-distortion upper bound and Theorem \ref{th:horizon-indep-rd} to arrive at an upper bound on the error for the finite-width MLP setting.

\subsubsection{Main Result}

With the theoretical tools developed in the previous section, we now establish the main result, which upper bounds optimal error with data generated by the finite-width MLP process.

\begin{theorem}{\bf (finite-width MLP error upper bound)}\label{th:nn_est_ub}
    For all $d, N, L \in \Z_{++}, \sigma^2 \in \Re_{++}$, if for all $t$, $(X_t, Y_{t+1})$ are generated according to the finite-width MLP process then, for all $T$, 
    $$\Lc_{T} \ \leq\ \frac{P}{2T}\left(1 + \ln\left(1 + \frac{2LT}{\sigma^2P}\right)\right),$$
    where $P = (L-2)N^2 + N + dN$ denotes the total parameter count of the network.
\end{theorem}
\begin{proof}
    \begin{align*}
        \Lc_{T}
        & \overset{(a)}{\leq} \inf_{\epsilon\geq 0}\ \frac{\H_\epsilon(A^{(1:L)})}{T} + \epsilon\\
        & \overset{(b)}{\leq} \inf_{\epsilon\geq 0}\ \frac{P\ln\left(1 + \frac{1}{\sigma^2\left(e^{\frac{2\epsilon}{L}}-1\right)}\right)}{2T} + \epsilon\\
        & \overset{(c)}{\leq} \frac{P\left(1 + \ln \left(1 + \frac{1}{\sigma^2\left(e^{\frac{P}{2LT}} - 1\right)}\right)\right)}{2T}\\
        & \overset{(d)}{\leq} \frac{P}{2T}\left(1 + \ln\left(1 + \frac{2LT}{\sigma^2P}\right)\right),
    \end{align*}
    where $(a)$ follows from Theorem \ref{th:horizon-indep-rd}, $(b)$ follows from Theorem \ref{th:nn_rd_bound}, $(c)$ follows by setting $\epsilon = P/2T$, and $(d)$ follows from the fact that for all $x \in \Re_{+}$, 
    $$\ln\left(1+\frac{1}{\sigma^2(e^x-1)}\right) \leq \ln\left(1+\frac{1}{\sigma^2 x}\right).$$
\end{proof}

Theorem \ref{th:nn_est_ub} establishes an upper bound which is only \emph{linear} in the total parameter count of the network ($O(P)$).  This notably improves upon existing results from the frequentist line of analysis \citep{bartlett1998almost,pmlr-v65-harvey17a} which derive an upper bound which is $\tilde{O}(PL)$.  As mentioned in the previous section, we are able to arrive at these stronger results by leveraging an expectation with respect to the prior distribution as opposed to a worst-case assumption over a confidence set.

Since we observe empirically that \emph{deep} MLPs are able to effectively learn even in the presence of limited data, our error analysis in the Bayesian framework provides results which are closer to qualitative observations made in empirical studies.  However, the results of this section are not sufficient to explain how learning may be possible when the dataset size is \emph{smaller} than the parameter count of the model which generated the data.  Such results require stronger assumptions about the \emph{dependence} between weights in the neural network.  In the following section, we explore this phenomenon in a setting in which an \emph{infinite} width MLP generates the data, but still a learning algorithm can attain small error relatively modest amounts of data.

\subsection{Infinite-Width Multi-Layer Perceptrons}\label{subsec:nonparametric}
\begin{figure}[H]
    \centering
    \includegraphics[width=\textwidth]{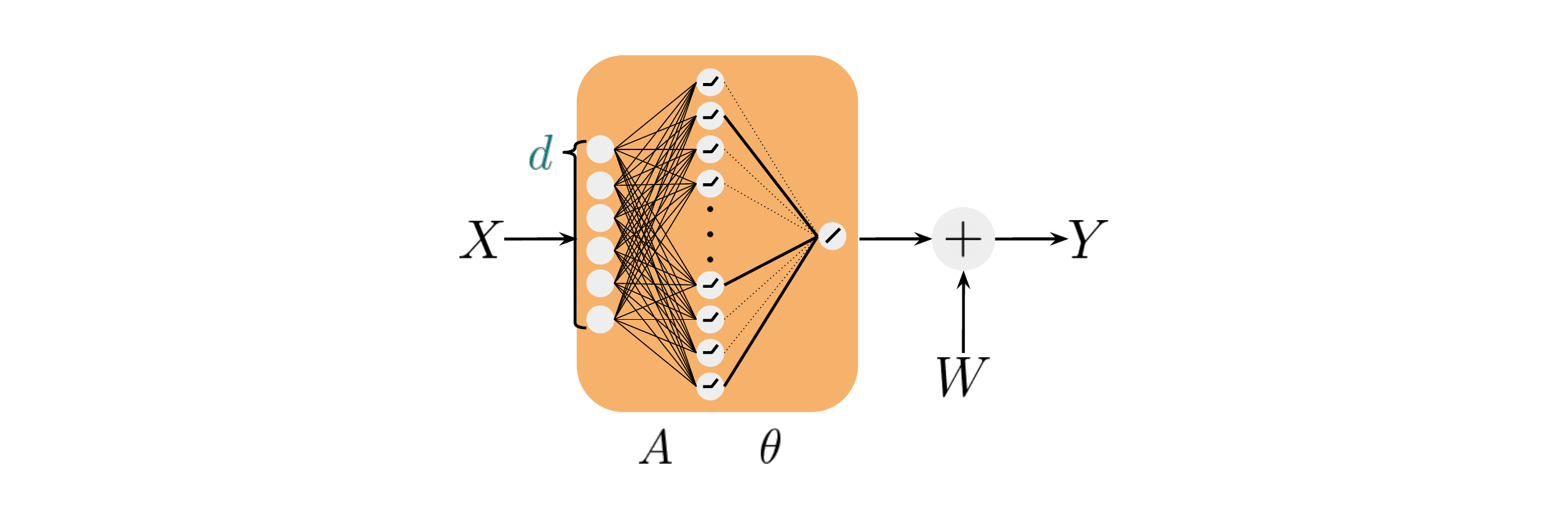}
    \caption{We depict our infinite-width MLP data generating process above.  It consists of input dimension $d$, an infinite number of hidden units with ReLU activations, and output dimension $1$.  We denote the weights of the first layer via an infinite dimensional matrix $A$ and the weights of the output layer by an infinite dimensional vector $\theta$.  To enforce structure which enables learning, we assume an appropriate prior distribution which results in a \emph{concentration} of the weights of $\theta$ as depicted by the solid (as opposed to dotted) lines in the above diagram.  We assume that the final output $Y$ is the sum of the final network output and independent Gaussian noise $Z$.}
    \label{fig:nonparamteric}
\end{figure}

In this section, we study a data generating process based on a two-layer neural network of \emph{infinite} width.  While the results we present in this section are limited to a two-layer architecture, the techniques can be generalized to treat infinite-width MLPs with any number of layers.

\subsubsection{Data Generating Process}\label{sec:inf_mlp}

For an infinite-width MLP, we are interested in estimating a function which can be uniquely identified by an infinite-dimensional matrix $A \in \Re^{\infty\times d}$ which represents the first-layer weights, and an infinite-dimensional vector $\theta\in\Re^{\infty}$ of output-layer weights.  In our Bayesian framework, there is a known prior distribution $\Pr((A,\theta)\in\cdot)$.  The neural network has input dimension $d$ and output dimension $1$.  In this analysis, we restrict our attention to a particular prior distribution on the weights of the network.  This prior distribution is describe by a \emph{Dirichlet process}, which we now describe.

A Dirichlet process takes as input a \emph{scale parameter} $K$ and a base distribution $P$.  In our case, we $P$ to be the uniform distribution $\text{Unif}(\sphere)$ over the unit sphere in $d$ dimensions.  Each realization of our Dirichlet process is a probability mass functions with countable support in $\sphere$.  Hence, each realization can be expressed as a countable subset $A_1,A_2,A_3, \ldots \in \sphere$ and probabilities $\bar{\theta}_1,\bar{\theta}_2,\bar{\theta}_3, \cdots$ assigned to elements of that subset.  Notably, with probability one, $\bar{\theta} \geq 0$ and $\sum_{n=1}^{\infty} \bar{\theta}_n=1$.  

For all $t$, let $X_{t} \overset{iid}{\sim} \normal(0, I_d)$ and 
$$Y_{t+1} = W_{t+1} + \sqrt{K}\cdot \sum_{n=1}^{\infty} \theta_n \relu(A_n^\top X_t),$$
with
$$\theta_n = \begin{cases}
    \bar{\theta}_n & \text{ w.p. } 0.5, \\
    -\bar{\theta}_n & \text{ w.p. } 0.5,
\end{cases}$$
where $A_1,A_2,A_3, \ldots$ and $\bar{\theta}_1,\bar{\theta}_2,\bar{\theta}_3, \cdots$ express an independent sample from our Dirichlet process and $W_{t}\overset{iid}{\sim} \normal(0,\sigma^2)$ is independent Gaussian noise of known variance $\sigma^2$.

The scale parameter $K$ influences diffusiveness of  $\bar{\theta}_1,\bar{\theta}_2,\bar{\theta}_3, \cdots$.  As $K$ decreases, the distribution becomes increasingly sparse, with probability eventually concentrating on a single component $\bar{\theta}_n = 1$ when $K=0$.  As $K$ increases, the distribution tends toward uniform.

In spite of the infinite width, the variance $Y_{t+1}$ is finite.  This property is shared by the neural tangent kernel (NTK), which also represents an infinite width neural network.  However, unlike the NTK setting, with the Dirichlet prior, weights are sparse.

\subsubsection{Preliminary Results}

Techniques used in previous sections to bound the rate-distortion function do not extend gracefully to infinite-width MLPs.  In particular, they would yield infinite, and thus vacuous, upper bounds.  We now introduce an alternative technique, which builds on the following lemma, which we establish via adapting the analysis of  \cite{barron1993universal}.
\begin{lemma}{\bf(approximation via multinomial)}\label{le:dir_apx}
    For all $K, m\in \Z_{++},$ let $(A,\bar{\theta})$ be drawn from the Dirichlet process with scale parameter $K$ and base distribution ${\rm Unif}(\sphere)$ and for all $n\in \Z_{++}$,
    $$\theta_n = \begin{cases}
        \bar{\theta}_n & \text{ w.p. } 0.5\\
        -\bar{\theta}_n & \text{ w.p. } 0.5\\
    \end{cases}.$$
    If for all $i \in [m]$, $c_i \overset{iid}{\sim} {\rm Categorical}(\bar{\theta})$ and $\tilde{A}_i = A_{c_i}$, then
    $$\E\left[ \left(\underbrace{\sqrt{K}\cdot \sum_{n=1}^{\infty} \theta_n \relu(A_n^\top X_t)}_{\rm true\ function} - 
    \underbrace{\frac{\sqrt{K}}{m}\sum_{i=1}^{m} {\rm sign}(\theta_{c_i})\cdot \relu(\tilde{A}_i^\top X_t)}_{\rm approximation} \right)^2\right]\ \leq\ \frac{K}{m}.$$
\end{lemma}
\begin{proof}
    \begin{align*}
        &\quad \E\left[ \left(\sqrt{K}\cdot \sum_{n=1}^{\infty} \theta_n \relu(A_n^\top X_t) - 
        \frac{\sqrt{K}}{m}\sum_{i=1}^{m} {\rm sign}(\theta_{c_i})\cdot\relu(\tilde{A}_i^\top X_t) \right)^2\right]\\
        & \overset{(a)}{=} \E\left[\left(\frac{\sqrt{K}}{m}\sum_{i=1}^{m} {\rm sign}(\theta_{c_i})\cdot\relu(\tilde{A}_i^\top X_t) \right)^2 -  \left(\sqrt{K}\cdot \sum_{n=1}^{\infty} \theta_n \relu(A_n^\top X_t)\right)^2\right]\\
        & \leq \E\left[\left(\frac{\sqrt{K}}{m}\sum_{i=1}^{m} {\rm sign}(\theta_{c_i})\cdot\relu(\tilde{A}_i^\top X_t) \right)^2\right]\\
        & \overset{(b)}{=} \frac{K}{m^2} \sum_{i=1}^{m} \E\left[ (\relu(\tilde{A}_i^\top X_t))^2\right]\\
        & \leq \frac{K}{m^2} \sum_{i=1}^{m} \E\left[ (\tilde{A}_i^\top X_t)^2\right]\\
        & = \frac{K}{m},
    \end{align*}
    where $(a)$ follows from the fact that the two expressions in the difference have the same expectation and $(b)$ follows from the fact that the $c_i$'s are independent and $\E[{\rm sign}(\theta_{c_i})] = 0$.
\end{proof}

Despite the fact that the MLP is of \emph{infinite} width, we can achieve finite distortion via a simple finite-sample approximation.  We now upper bound the nats of information contained in this finite-sample approximation to arrive at a rate-distortion upper bound.

\begin{restatable}{lemma}{dirMulEnt}{\bf(Dirichlet-multinomial entropy bound)}\label{le:dir_concen}
    For all $K, m \in \Z_{++}$, let $(A,\bar{\theta})$ be drawn from a Dirichlet process with scale parameter $K$ and base distribution ${\rm Unif}(\sphere)$.  If for all $i \in [m]$, $\bar{A}_i \overset{iid}{\sim} {\rm Categorical}(\bar{\theta})$, and $\tilde{A}_i = \argmin_{a\in \sphere_\epsilon} \|\bar{A}_i - a\|^2_2$ is the closest vector in an epsilon-cover of $\sphere$, then
    $$\H(\tilde{A}) \ \leq\  K\ln\left(1 + \frac{m}{K} \right)\left(\ln m + d\ln\frac{3}{{\epsilon}^2}\right),$$
    where $\tilde{A} = (\tilde{A}_1, \tilde{A}_2, \ldots, \tilde{A}_m)$.
\end{restatable}
\begin{proof}
    As the proof of this result requires significant mathematical machinery, we refer the reader to Appendix \ref{apdx:dir_mult} for the result.
\end{proof}

We now present one more preliminary result, which uses Lemmas \ref{le:dir_apx} and \ref{le:dir_concen} to upper bound the $0$-horizon rate-distortion function.

\begin{theorem}{\bf(infinite-width MLP $0$-horizon rate-distortion upper bound)}\label{th:nonparametric_rd_ub}
    For all $d, K \in \Z_{++}$ and $\sigma^2,\epsilon \in \Re_{++}$, if $(X,Y)$ are generated according to the infinite-width MLP process with parameters $\theta, A$, then
    $$\H_{\epsilon}(\theta,A)\ \leq\ K\ln\left(1 + \frac{2}{\sigma^2\epsilon}\right)\ln\left(\frac{2K}{\sigma^2\epsilon}\right) + 2dK\ln\left(1 + \frac{2}{\sigma^2\epsilon}\right)\ln\left(1+\frac{4}{\epsilon}\right).$$
\end{theorem}
\begin{proof}
    Suppose we set $\epsilon' = m\sigma^2(e^{2K/(m\sigma^2)}-1)/2K - 1 \geq 0$ and $m = \lceil\frac{K}{\sigma^2\epsilon}\rceil$.
    For all $i \in [m]$, let $\bar{A}_i \overset{iid}{\sim} {\rm Categorical}(\bar{\theta})$.  Let $\sphere_{\epsilon'}$ denote an $\epsilon'$-cover of $\sphere$ with respect to $\|\cdot\|^2_2$ and for all $i \in \Z_{++}$, let
    $$\tilde{A}_{i,\epsilon} = \argmin_{a \in \mathcal{A}_\epsilon}\ \|\bar{A}_i - a\|^2_2.$$
    Let $\tilde{A} = (\tilde{A}_1, \ldots, \tilde{A}_m)$ and  $\tilde{\theta} = (\tilde{A}, {\rm sign}(\theta_{\bar{A}_1}), \ldots, {\rm sign}(\theta_{\bar{A}_m}))$.
    Then,
    \begin{align*}
        \E\left[\KL\left(P^*_0 \| \tilde{P}_0\right)\right]
        & =\I(Y_1;\theta,A|\tilde{\theta},X_0)\\
        & = \diffentropy(Y_1|\tilde{\theta},X_0) - \diffentropy(Y_1|\theta,A,X_0)\\
        & = \diffentropy\left(Y_1 - \frac{\sqrt{K}}{m}\sum_{i=1}^{m}{\rm sign}(\theta_{\bar{A}_i})\cdot\relu(\tilde{A}_i^\top X_0) \bigg|\tilde{\theta}, \tilde{A}\right) - \diffentropy(W)\\
        & \overset{(a)}{\leq} \E\left[\frac{1}{2}\ln\left(\frac{\left(Y_1  - \frac{\sqrt{K}}{m}\sum_{i=1}^{m}{\rm sign}(\theta_{\bar{A}_i})\cdot\relu(\tilde{A}_i^\top X_0) \right)^2}{\sigma^2}\right)\right]\\
        & \overset{(b)}{\leq} \frac{1}{2}\ln\left(\frac{\E\left[\left(Y_1 - \frac{\sqrt{K}}{m}\sum_{i=1}^{m}{\rm sign}(\theta_{\bar{A}_i})\cdot\relu(\tilde{A}_i^\top X_0) \right)^2\right]}{\sigma^2}\right)\\
        & = \frac{1}{2}\ln\left(1 + \frac{\left(\sqrt{K}\cdot \sum_{n=1}^{\infty} \theta_n \relu(A_n^\top X_0) - 
        \frac{\sqrt{K}}{m}\sum_{i=1}^{m} {\rm sign}(\theta_{\bar{A}_i})\cdot\relu(\tilde{A}_i^\top X_0)\right)^2}{\sigma^2}\right)\\
        & \overset{(c)}{\leq} \frac{1}{2}\ln\left(1 + \frac{\frac{2K}{m} + \frac{2K}{m^2}\E\left[\left(\sum_{i=1}^{m} {\rm sign}(\theta_{\bar{A}_i}) \relu(\bar{A}_i^\top X_0) - \sum_{i=1}^{m} {\rm sign}(\theta_{\bar{A}_i})\cdot\relu(\tilde{A}_i^\top X_0)\right)^2\right]}{\sigma^2}\right)\\
        & = \frac{1}{2}\ln\left(1 + \frac{\frac{2K}{m} + \frac{2K}{m^2}\sum_{i=1}^{m}\E\left[(\relu(\bar{A}^\top_i X_0) -\relu(\tilde{A}_i^\top X_0))^2\right]}{\sigma^2}\right)\\
        & \overset{(d)}{\leq} \frac{1}{2}\ln\left(1 + \frac{\frac{2K}{m} + \frac{2K}{m}\E\left[(\bar{A}^\top_i X_0 -\tilde{A}_i^\top X_0)^2\right]}{\sigma^2}\right)\\
        & \overset{(e)}{\leq} \frac{1}{2}\ln\left(1 + \frac{2K(1+\epsilon')}{\sigma^2 m}\right)\\
        & = \frac{K}{\sigma^2 m}\\
        & \leq \epsilon,
    \end{align*}
    where $(a)$ follows from Theorem \ref{le:max_entropy} and the fact that conditioning reduces differential entropy, $(b)$ follows from Jensen's inequality, $(c)$ follows from Lemma \ref{le:dir_apx} and the fact that for all $x,y,z\in\Re$, $(x-y)^2 \leq 2(x-z)^2 + 2(z-y)^2$, $(d)$ follows from the fact that for all $x,y\in\Re,\ (x-y)^2 \geq (\relu(x)-\relu(y))^2$, and $(e)$ follows from the fact that $\tilde{A}_i$ is the closest element in an $\epsilon'$ cover of $\sphere$.
    
    We now upper bound the rate of $(\tilde{\theta})$
    \begin{align*}
        \I(\theta, A; \tilde{\theta})
        & \leq \H(\tilde{\theta})\\
        & \overset{(a)}{\leq} K\ln\left(1 + \frac{m}{K}\right)\left(\ln 2m + d\ln\frac{3}{{\epsilon'}^2}\right)\\
        & \overset{(b)}{\leq} K\ln\left(1 + \frac{2}{\sigma^2\epsilon}\right)\left(\ln\frac{2K}{\sigma^2\epsilon} + 2d\ln\left(\frac{\epsilon\sqrt{3}}{e^{\epsilon}-1-\epsilon}\right)\right)\\
        & \overset{(c)}{\leq} K\ln\left(1 + \frac{2}{\sigma^2\epsilon}\right)\ln\left(\frac{2K}{\sigma^2\epsilon}\right) + 2dK\ln\left(1 + \frac{2}{\sigma^2\epsilon}\right)\ln\left(1+\frac{4}{\epsilon}\right),
    \end{align*}
    where $(a)$ follows from Lemma \ref{le:dir_concen} plus the factor of 2 within the logarithm due to the signs, $(b)$ follows by upper bounding the quantity via $m = 2K/(\sigma^2\epsilon)$ as opposed to $\lceil K/(\sigma^2\epsilon)\rceil$, and $(c)$ follows from the fact that for all $\epsilon \geq 0$, $\ln(\epsilon\sqrt{3}/(e^\epsilon - 1 - \epsilon)) \leq \ln(1+4/\epsilon)$.
\end{proof}

\subsubsection{Main Result}
With the theoretical tools derived in the previous section, we now establish the following upper bound on the error of an optimal agent learning from data generated by the Dirichlet process.
\begin{theorem}{\bf (infinite-width MLP error upper bound)}
    For all $d, K \in \Z_{++}$, if for all $t$, $(X_t,Y_{t+1})$ are generated according to the finite-width MLP process, then for all $T$,
    $$\Lc_{T}\ \leq\ \frac{K}{T}\ln\left(1 + \frac{T}{\sigma^2 d}\right)\ln\left(\frac{T}{\sigma^2 d}\right) + \frac{2dK}{T}\left(1 + \ln\left(1 + \frac{T}{\sigma^2 d}\right) \ln\left(1 + \frac{T}{dK}\right)\right).$$
\end{theorem}
\begin{proof}
    The result follows from Theorems \ref{th:horizon-indep-rd} and \ref{th:nonparametric_rd_ub} and setting $\epsilon = 2dk/T$.
\end{proof}

Notably, this result is $\tilde{\mathcal{O}}(dK/T)$ despite the fact that the Dirichlet neural network process is described by a neural network with \emph{infinite} width.  The scale parameter $K$ determines the degree of concentration which occurs in the output-layer weights of the network and hence controls the difficulty of learning.  This example of learning under data generated by an infinite-width MLP process further demonstrates the flexibility and ingenuity of proof techniques which are encompassed in our framework.  We hope that the suite of examples provided in this section enables the reader to analyze whichever complex processes they may have in mind.

\newpage
\begin{summary}
\begin{itemize}
    \item We abbreviate the $0$-horizon rate-distortion function $H_{\epsilon,0}$ as $\H_{\epsilon}$
    \item {\bf(monotonicity of per-timestep error)}
    If for all $s \in \Z_{+}$, $(X_s, Y_{s+1})$ is sampled iid from some distribution $\Pr(\cdot|\theta)$, then for all $t\in\Z_{+}$,
    $$\I(Y_{t+2};\theta|H_{t+1})\ \leq \ \I(Y_{t+1};\theta|H_{t}).$$
    \item {\bf ($0$-horizon distortion upper bounds $T$-horizon distortion)}
    If $((X_t, Y_{t+1}) : t\in \Z_{+})$ is iid conditioned on $\theta$ and, for all $t \in\Z_{+}$, $Y_{t+1}\perp\tilde{\theta}|(\theta, X_t)$ then, for all $T$,
    $$\underbrace{\frac{1}{T}\sum_{t=0}^{T-1} \E\left[\KL(P^*_t \| \tilde{P}_t )\right]}_{T-{\rm horizon\ distortion}} \leq \underbrace{\E\left[ \KL (P^*_0 \| \tilde{P}_0)\right]}_{0-{\rm horizon\ distortion}}.$$
    \item We abbreviate the $0$-horizon rate-distortion function $\H_{\epsilon, 0}$ with $\H_\epsilon$.
    \item {\bf (0-horizon rate-distortion error bound for iid data)}
    If $((X_t, Y_{t+1}) : t\in \Z_{+})$ is iid conditioned on $\theta$, then, for all $T\in \Z_{++},$
    $$\sup_{\epsilon \geq 0}\ \min\left\{ \frac{\H_\epsilon(\theta)}{T},
    \epsilon \right\}\ \leq\ \Lc_{T} \ \leq\ \inf_{\epsilon \geq 0}\ \frac{\H_{\epsilon}(\theta)}{T} + \epsilon.$$
    \item {\bf(linear regression error bounds)}
    For all input dimensions $d$ and noise variance $\sigma^2$, if for all $t$, $(X_t, Y_{t+1})$ are generated according to the linear regression processes, then for all $T$,
     $$\frac{d}{2T}W\left(\frac{2T}{d(8+\frac{d}{d-2}\sigma^2)}\right)\ \leq\ \Lc_{T} \ \leq\ \left(\frac{d}{2T}\ln\left(\frac{T}{\sigma^2 d}\right)\right)_+ + \frac{1}{2T}\ln\left(1+\frac{d}{T}\right),$$
     where $W$ is the Lambert W function.
    \item {\bf (logistic regression error upper bound)} For all input dimensions $d$, if for all $t$, $(X_t, Y_{t+1})$ is generated according to the logistic regression process, then for all $T,$
    $$ \Lc_T\ \leq\ \frac{d}{2T}\left(1 + \ln\left(1 + \frac{T}{4d}\right)\right).$$
    \item {\bf (finite-width MLP error upper bound)}
    For all input dimensions $d$, widths $N$, depths $L$, and noise variances $\sigma^2$, if for all $t$, $(X_t, Y_{t+1})$ are generated according to the finite-width MLP process, then for all $T$,
    $$\Lc_{T} \ \leq\ \frac{P}{2T}\left(1 + \ln\left(1 + \frac{2LT}{\sigma^2P}\right)\right),$$
    where $P = (L-2)N^2 + N + dN$ denotes the total parameter count of the network.
    \item {\bf (infinite-width MLP error upper bound)}
    For all input dimensions $d$, scale parameters $K$, if for all $t$, $(X_t,Y_{t+1})$ are generated according to the infinite-width MLP process, then for all $T$,
    $$\Lc_{T}\ \leq\ \frac{K}{T}\ln\left(1 + \frac{T}{\sigma^2 d}\right)\ln\left(\frac{T}{\sigma^2 d}\right) + \frac{2dK}{T}\left(1 + \ln\left(1 + \frac{T}{\sigma^2 d}\right) \ln\left(1 + \frac{T}{dK}\right)\right).$$
\end{itemize}
\end{summary}
\clearpage

%% file: sections/sequence_learning.tex
\section{Learning from Sequences}
\label{se:continual-vs-convergent}

In the previous section, we focused on the special case of learning from a sequence of data which is iid when conditioned on $\theta$.  However, in general, machine learning systems will have to reason about data which does not obey this rigid structure.  For instance, suppose that $X_0, X_1, X_2, \ldots$ is a sequence of text tokens from a book.  It's clear that such a sequence would not be iid even when conditioned on $\theta$, as the order of the tokens plays a critical role in deriving meaning.  When we relax the iid assumption, a natural question is whether the sample complexity changes.  Sample complexity bounds established in the literature make assumptions about the mixing time of the data generating process \citep{mohri2008rademacher, kuznetsov2017generalization, roy2021empirical, pmlr-v235-ziemann24a}.  It is unclear whether these assumptions hold for practical problems.  In our framework, Theorem \ref{th:rd_bounds}, which upper and lower bounds error via the $T$-horizon rate-distortion function, does not make any assumptions about the mixing time.  In this section, we apply this theorem to general auto-regressive processes.

\subsection{Data Generating Process}

In this section, we will analyze learning from sequential data generated by a binary AR(K) process and by a transformer model with context length $K$.  In each case, the data sequence is a random process $(X_0, X_1, X_2, \ldots)$ and the learned model is parameterized by a random index $\theta$.

\subsection{Binary AR(K) Process}
\subsubsection{Data Generating Process}
We begin with a simple AR(K) process over a binary alphabet.  Our framework can easily adapt to learning from such \emph{sequential} data.  In this problem, we are interested in estimating random vectors $\theta = (\theta_1, \ldots, \theta_K) \in \Re^{d \times K}$ from data generated in the following way.  Let $(X_0, X_{1}\ldots,X_{K-1}) \sim {\rm Unif}(\{0, 1\}^K)$.  For all $t \geq K$, let
$$X_{t+1} =
\begin{cases}
    1 & \text{ w.p. } \sigma\left(\sum_{k=1}^{K} \theta_{k}^\top \phi_{t-k+1}\right)\\
    0 & \text{ otherwise,}
\end{cases},$$
where $\sigma$ denotes the sigmoid function and $\phi_t = \Phi_{X_t}$ for some known vectors $\Phi_0, \Phi_1 \in \sphere$.  For all $k$, we assume that each $\theta_k$ is independent and distributed $\normal(0, I_d/K)$.

\subsubsection{Preliminary Results}

As in all prior examples, our strategy is to leverage rate-distortion theory to arrive at an error upper bound.  We begin with the following result which upper bounds the rate-distortion function in the binary AR(K) problem setting.

\begin{lemma}{\bf (binary AR(K) $T$-horizon rate-distortion upper bound)}\label{le:ARK_rd}
    For all $d, K \in \mathbb{Z}_{++}$ and $\epsilon \in \Re_{++}$, if for all $t \in \mathbb{Z}_{+}, X_t$ is generated according to the binary AR(K) process, then for all $T$,
    $$\H_{\epsilon, T}(\theta_{1:K}) \ \leq\ \frac{dK}{2}\ln\left(1 + \frac{1}{8\epsilon}\right).$$
\end{lemma}
\begin{proof}
    For $k \in \{0, 1, \ldots, K-1\}$, let $\tilde{\theta}_k = \theta_k + Z_k$ where $Z_k \perp \theta_k$ and $Z_k\sim \normal(0, I_d \cdot 8\epsilon/K)$.  For all $t$, let $H_t$ denote $(X_0, X_1, \ldots, X_t)$.  Then,
    \begin{align*}
        \frac{1}{T}\sum_{t=0}^{T-1}\E\left[\KL(P^*_t\|\tilde{P}_t)\right]
        & = \frac{\I(H_T;\theta_{1:K}|\tilde{\theta}_{1:K})}{T}\\
        & = \frac{1}{T}\sum_{t=0}^{T-1} \I(X_{t+1};\theta_{1:K}|\tilde{\theta}_{1:K}, H_t)\\
        & \leq \frac{1}{T}\sum_{t=0}^{T-1} \I(X_{t+1};\theta_{1:K}|\tilde{\theta}_{1:K}, X_{t:t-K+1})\\
        & \leq \frac{1}{T} \sum_{t=0}^{T-1} \E\left[\KL\left(\Pr(X_{t+1}\in\cdot|\theta_{1:K}, X_{t:t-K+1}) \ \|\ \Pr(X_{t+1}\in\cdot|\tilde{\theta}_{1:K}, X_{t:t-K+1})\right)\right]\\
        & \overset{(a)}{\leq} \frac{1}{T} \sum_{t=0}^{T-1}\E\left[\KL\left(\Pr(X_{t+1}\in\cdot|\theta_{1:K}, X_{t:t-K+1}) \ \|\ \Pr(X_{t+1}\in\cdot|\theta_{1:K} \leftarrow \tilde{\theta}_{1:K}, X_{t:t-K+1})\right)\right]\\
        & = \frac{1}{T}\sum_{t=0}^{T-1}\E\left[\frac{\ln\left(\frac{1+e^{-\sum_{k=1}^{K} \theta_{k}^\top \phi_{t-k+1}}}{1+e^{-\sum_{k=1}^{K} \theta_{k}^\top \phi_{t-k+1}}}\right)}{1+e^{-\sum_{k=1}^{K} \theta_{k}^\top \phi_{t-k+1}}} + \frac{\ln\left(\frac{1+e^{\sum_{k=1}^{K} \theta_{k}^\top \phi_{t-k+1}}}{1+e^{\sum_{k=1}^{K} \theta_{k}^\top \phi_{t-k+1}}}\right)}{1+e^{\sum_{k=1}^{K} \theta_{k}^\top \phi_{t-k+1}}}\right]\\
        & \overset{(b)}{\leq} \frac{1}{T}\sum_{t=0}^{T-1}\frac{\E\left[\left(\sum_{k=1}^{K} \theta_{k}^\top \phi_{t-k+1} - \tilde{\theta}^\top_k \phi_{t-k+1}\right)^2\right]}{8}\\
        & = \frac{1}{T}\sum_{t=0}^{T-1}\frac{\E\left[\sum_{k=1}^{K}\phi_{t:t-K+1}^\top(\theta_k - \tilde{\theta}_k)(\theta_k- \tilde{\theta}_k)^\top\phi_{t:t-K+1}\right]}{8}\\
        & = \frac{\epsilon\cdot \E\left[ \sum_{k=1}^{K} \phi^\top_{t-k} \phi_{t-k+1}\right]}{K}\\
        & = \epsilon,
    \end{align*}
    where $(a)$ follows from Lemma \ref{le:change_measure_ub} and $(b)$ follows from Lemma \ref{le:kl_ub}.
    
    We now upper bound the rate.
    \begin{align*}
        \I(\theta_{1:K};\tilde{\theta}_{1:K})
        & = \diffentropy(\tilde{\theta}_{1:K}) - \diffentropy(\tilde{\theta}_{1:K}|\theta_{1:K})\\
        & \overset{(a)}{\leq} \frac{dK}{2}\ln\left(2\pi e \left(\frac{1}{K} + \frac{8\epsilon}{K}\right)\right) - \frac{dK}{2}\ln\left(2\pi e \frac{8\epsilon}{K}\right)\\
        & = \frac{dK}{2}\ln\left(1 + \frac{1}{8\epsilon}\right),
    \end{align*}
    where $(a)$ follows from Lemma \ref{le:max_entropy}.  The result follows.
\end{proof}

\subsubsection{Main Result}

We now present the main result of this section which upper bounds the error of learning under the binary AR(K) data generating process.  The result follows directly as a result of Lemma \ref{le:ARK_rd}.

\begin{theorem}{\bf (binary AR(K) error upper bound)}
    For all $d, K\in \Z_{++}$, if for all $t \in \Z_{+}, X_t$ is generated according to the binary AR(K) process, then for all $T$,
    $$\Lc_{T}\ \leq\ \frac{dK}{2T}\left(1 + \ln\left(1 + \frac{T}{4dK}\right)\right).$$
\end{theorem}
\begin{proof}
    \begin{align*}
        \Lc_{T}
        & \overset{(a)}{\leq} \inf_{\epsilon \geq 0} \frac{\H_{\epsilon, T}(\theta_{1:K})}{T} + \epsilon\\
        & \overset{(b)}{\leq} \inf_{\epsilon \geq 0} \frac{dK\ln\left(1 + \frac{1}{8\epsilon}\right)}{2T} + \epsilon\\
        & \overset{(c)}{\leq} \frac{dK}{2T}\left(1 + \ln\left(1 + \frac{T}{4dK}\right)\right),
    \end{align*}
    where $(a)$ follows from Theorem \ref{th:rd_bounds}, $(b)$ follows from Lemma \ref{le:ARK_rd}, and $(c)$ follows by setting $\epsilon = (dK)/(2T)$.
\end{proof}

An interesting aspect of this result is that both the proof techniques and final result are hardly impacted by the fact that the sequence is \emph{not iid} when conditioned on $\theta_{1:K}$.  Existing tools for statistical analysis can struggle in the setting without the appropriate averages of iid quantities.  However, our analytical tools involving rate-distortion theory allow us to handle such problem instances with relative ease and produce reasonable upper bounds on error in learning settings involving sequences of data.  In the following section, we extend these tools to analyze a more complex transformer data generating process. 

\subsection{Transformer Process}\label{subsec:transformer}
\subsubsection{Data Generating Process}

We now study a much more complex transformer data generating process over a vocabulary of size $d$.  We are interested in estimating the random matrices $\theta^{(1:L)} = (A^{(1)}, V^{(1)}, \ldots, A^{(L)}, V^{(L)})$ which are the weights of the transformer across $L$ layers.  Let $X_0:\Omega\mapsto \{1, 2, \ldots, d\}$ be a categorical random variable with arbitrary pmf.  Then, for all $t \geq 0$, let
$$X_{t+1} \sim {\rm Softmax}(f_{\theta^{(1:L)}}(X_{t:t-K+1})),$$
where $f_{\theta^{(1:L)}}$ is the layer $L$ transformer with weights $\theta^{(1:L)}$ which we will detail shortly.  Notably, we assume that $X_{t+1}$ only depends on the history via the most recent length $K$ context $X_{t:t-K+1}$.

We use $U_{t}^{(\ell)}$ to denote the output of layer $\ell$ of the transformer at time $t$.  Beginning at the input, let
$$U_{t}^{(0)} = \begin{bmatrix}
    \phi_{t-K+1} & \phi_{t-K+2} & \ldots & \phi_{t}
\end{bmatrix},$$
where $\phi_t = \Phi_{X_t}$ for some known vectors $\Phi_1, \ldots, \Phi_d \in \Re^r$ with unit L2 norm.  For $\ell > 0$, let
$$U_{t}^{(\ell)} = \text{Clip}\left(V^{(\ell)} U_{t}^{(\ell-1)} \text{Attn}^{(\ell)}\left(U_{t}^{(\ell-1)}\right)\right),$$
where $\text{Clip}(A)$ clips the columns of matrix $A$ to have L2 norm at most 1 and
$$\text{Attn}^{(\ell)}\left(U_{t}^{(\ell-1)}\right) = {\rm Softmax}\left( \frac{U^{(\ell-1)\top}_{t} A^{(\ell)} U_{t}^{(\ell-1)}}{\sqrt{r}}\right)$$
denotes the attention matrix of layer $(\ell)$ where ${\rm Softmax}$ denotes the softmax function applied elementwise along the columns.  The matrix $A^{(\ell)} \in \Re^{r\times r}$ can be interpreted as the product of the key and query matrices and without loss of generality, we assume that the elements of the matrices $A^{(\ell)}$ are distributed iid $\normal(0,1)$ (Gaussian assumption is not crucial but unit variance is).  The matrix $V^{(\ell)}$ resembles the value matrix and we assume that the rows of $V^{(\ell)}$ are distributed iid ${\rm Unif}(\sphere)$.  For $\ell < L$, we have that $V^{(\ell)} \in \Re^{r\times r}$, whereas $V^{(L)}\in\Re^{d\times r}$.

Finally, $f_{\theta^{(1:L)}}(X_{t:t-K+1}) = U_t^{(L)}[-1]$, which denotes the final column of the output matrix as is common in autoregressive generation with transformers.

\subsubsection{Preliminary Results}

In this section, we provide several Lemmas which culminate in a $T$-horizon rate-distortion upper bound for the multilayer transformer process.  The techniques mirror those introduced in the analysis of the finite-width MLP process.  Just as for the MLP process, we restrict our attention to approximations $\tilde{\theta} = (\tilde{\theta}^{(1)}, \ldots, \tilde{\theta}^{(L)})$ which are independent from each other.  Under this assumptions, we can upper bound the $T$-horizon distortion in the following way:

\begin{lemma}{\bf ($T$-horizon distortion upper bound)}\label{le:seq_real_input_inequality}
    For all $T, L\in\mathbb{Z}_{++}$ and $\ell \in \{1, \ldots, L\}$, if $\theta^{(i)} \perp \theta^{(j)}$, $\tilde{\theta}^{(i)} \perp \tilde{\theta}^{(j)}$, and $\theta^{(i)} \perp \tilde{\theta}^{(j)}$ for $i\neq j$, then

    $$
        \frac{1}{T}\sum_{t=0}^{T-1}\E\left[\KL(P^*_t\| \tilde{P}_t)\right]
        \ \leq \ \frac{1}{T}\sum_{t=0}^{T-1} \E\left[\KL(\Pr(H_{t+1}\in\cdot|\theta^{(1:L)}, X_0)\|\Pr(H_{t+1}\in\cdot|\theta^{(\ell+1:L)}, \tilde{\theta}^{(\ell)}, U_t^{(\ell)}))\right]
    $$
\end{lemma}
\begin{proof}
    Fix arbitrary $t < T$.  Recall that $\tilde{P}_t^{(\ell)} = \Pr(Y_{t+1}\in\cdot|\theta^{(\ell+1)}, \tilde{\theta}^{(1:\ell)}, H_t)$ is the prediction derived from approximating the first $\ell$ layers with $\tilde{\theta}^{(1:\ell)}$.  Then,
    \begin{align*}
        \E\left[\KL(\tilde{P}^{(\ell-1)}_t\|\tilde{P}^{(\ell)}_t)\right]
        & =
        \I(X_{t+1};\theta^{(\ell)}|\theta^{(\ell+1:L)}, \tilde{\theta}^{(1:\ell)}, H_{t})\\
        & \overset{(a)}{=} \I(H_{t+1}, \tilde{\theta}^{(1:\ell-1)};\theta^{(\ell)}|\theta^{(\ell+1:L)}, \tilde{\theta}^{(\ell)}) - \I(H_{t},\tilde{\theta}^{(1:\ell-1)} ;\theta^{(\ell)}|\theta^{(\ell+1:L)}, \tilde{\theta}^{(\ell)})\\
        & \overset{(b)}{\leq} \I(H_{t+1}, \tilde{\theta}^{(1:\ell-1)};\theta^{(\ell)}|\theta^{(\ell+1:L)}, \tilde{\theta}^{(\ell)}, X_0)\\
        & \overset{(c)}{\leq} \I(H_{t+1}, \theta^{(1:\ell-1)};\theta^{(\ell)}|\theta^{(\ell+1:L)}, \tilde{\theta}^{(\ell)}, X_0)\\
        & \overset{(d)}{=} \I(H_{t+1}, \theta^{(1:\ell-1)};\theta^{(\ell)}|\theta^{(\ell+1:L)}, \tilde{\theta}^{(\ell)}, X_0) - \I(\theta^{(1:\ell-1)};\theta^{(\ell)}|\theta^{(\ell+1:L)}, \tilde{\theta}^{(\ell)}, X_0)\\
        & \overset{(e)}{=} \I(H_{t+1};\theta^{(\ell)}|\theta^{(\ell+1:L)}, \theta^{(1:\ell-1)}, \tilde{\theta}^{(\ell)}, X_0)\\
        & = \E\left[\KL(\Pr(H_{t+1}\in\cdot|\theta^{(1:L)}, X_0)\| \Pr(H_{t+1}\in\cdot|\theta^{(\ell+1:L)}, \tilde{\theta}^{(\ell)}, U^{(\ell)}_0))\right],
    \end{align*}
    where $(a)$ follows from the chain rule of mutual information, $(b)$ follows from the independence assumptions, $(c)$ follows from the data processing inequality applied to the markov chain $\theta_i \perp \tilde{\theta}^{(1:\ell-1)}|(H_{t+1}, \theta^{(\ell+1:L)}, \theta^{(1:\ell-1)}, X_{0}^{K})$, $(d)$ follows from the fact that $\I(\theta^{(1:\ell-1)};\theta_i|\theta^{(\ell+1:L)},\tilde{\theta}_i, X_0) = 0$, and $(e)$ follows from the chain rule of mutual information.  The result follows by summing for all $t \in \{0, 1, ..., T-1\}$.
\end{proof}

The terms in the RHS quantify the expected prediction error incurred from approximating the weights at layer $\ell$ with $\tilde{\theta}$.  This approximation at layer $\ell$ incurs deviation which will propagate to later layers of the network.  The following result bounds this propagation via an upper bound on the squared Lipschitz constant of the transformer layers.  Note that we use $\|\cdot\|_\sigma$ to denote the \emph{operator norm} of a matrix and $\|\cdot\|_F$ to denote the \emph{frobenius norm}.

\begin{lemma}{\bf (transformer layer Lipschitz)}\label{le:transformer_lipschitz}
    For all $d,r,K,\ell\in\Z_{++}$, if
    $$\tilde{U}^{(\ell)} = {\rm Clip}\left(V^{(\ell)} \tilde{U}^{(\ell-1)} {\rm Attn}^{(\ell)}\left(\tilde{U}^{(\ell-1)}\right)\right),$$
    then
    $$\left\|U^{(\ell)} - \tilde{U}^{(\ell)}\right\|^2_F\ \leq\ 2K\|V^{(\ell)}\|^2_\sigma\left(1 + \frac{4K\|A^{(\ell)}\|^2_{\sigma}}{r}\right)\cdot \left\|U^{(\ell-1)} - \tilde{U}^{(\ell-1)}\right\|^2_F.$$    
\end{lemma}
\begin{proof}
    \begin{align*}
        &\ \left\|U^{(\ell)} - \tilde{U}^{(\ell)}\right\|^2_F\\
        & = \left\|\text{Clip}\left(V^{(\ell)}U^{(\ell-1)}\sigma\left(\frac{U^{(\ell-1)\top} A^{(\ell)} U^{(\ell-1)}}{\sqrt{r}}\right)\right) - \text{Clip}\left(V^{(\ell)}\tilde{U}^{(\ell-1)}\sigma\left(\frac{\tilde{U}^{(\ell-1)\top} A^{(\ell)} \tilde{U}^{(\ell-1)}}{\sqrt{r}}\right)\right)  \right\|^2_F\\
        & \overset{(a)}{\leq} \left\|V^{(\ell)}U^{(\ell-1)}\sigma\left(\frac{U^{(\ell-1)\top} A^{(\ell)} U^{(\ell-1)}}{\sqrt{r}}\right) - V^{(\ell)}\tilde{U}^{(\ell-1)}\sigma\left(\frac{\tilde{U}^{(\ell-1)\top} A^{(\ell)} \tilde{U}^{(\ell-1)}}{\sqrt{r}}\right)  \right\|^2_F \\
        & \overset{(b)}{=}
        \sum_{k=1}^{K}\|V^{(\ell)}\|^2_{\sigma} \left\|U^{(\ell-1)}\sigma\left(\frac{U^{(\ell-1)\top} A^{(\ell)} U^{(\ell-1)}_{k}}{\sqrt{r}}\right) - \tilde{U}^{(\ell-1)}\sigma\left(\frac{\tilde{U}^{(\ell-1)\top} A^{(\ell)} \tilde{U}^{(\ell-1)}_k}{\sqrt{r}}\right)\right\|^2_2 \\
        & \leq \sum_{k=1}^{K}2\|V^{(\ell)}\|^2_{\sigma}\left\|U^{(\ell-1)}\sigma\left(\frac{U^{(\ell-1)\top} A^{(\ell)} U^{(\ell-1)}_k}{\sqrt{r}}\right) - \tilde{U}^{(\ell-1)}\sigma\left(\frac{U^{(\ell-1)\top} A^{(\ell)} U^{(\ell-1)}_k}{\sqrt{r}}\right)\right\|^2_2\\
        &\quad + \sum_{k=1}^{K} 2\|V^{(\ell)}\|^2_{\sigma}\left\|\tilde{U}^{(\ell-1)}\sigma\left(\frac{U^{(\ell-1)\top} A^{(\ell)} U^{(\ell-1)}_k}{\sqrt{r}}\right) - \tilde{U}^{(\ell-1)}\sigma\left(\frac{\tilde{U}^{(\ell-1)\top} A^{(\ell)} \tilde{U}^{(\ell-1)}_k}{\sqrt{r}}\right)\right\|^2_2\\
        & \overset{(c)}{\leq} \sum_{k=1}^{K} 2\|V^{(\ell)}\|^2_{\sigma}\|U^{(\ell-1)}-\tilde{U}^{(\ell-1)}\|^2_F\\
        &\quad + \sum_{k=1}^{K}\frac{2K}{r}\|V^{(\ell)}\|^2_{\sigma}\left\|U^{(\ell-1)\top} A^{(\ell)} U^{(\ell-1)}_k - \tilde{U}^{(\ell-1)\top} A^{(\ell)} \tilde{U}^{(\ell-1)}_k \right\|^2_2\\
        & \overset{(d)}{\leq} 2K\|V^{(\ell)}\|^2_\sigma \left\|U^{(\ell-1)}- \tilde{U}^{(\ell-1)}\right\|^2_F + \frac{4K}{r}\sum_{k=1}^{K} \|V^{(\ell)}\|^2_\sigma\left\|U^{(\ell-1)\top} A^{(\ell)} U^{(\ell-1)}_k - U^{(\ell-1)\top} A^{(\ell)} \tilde{U}^{(\ell-1)}_k \right\|^2_2\\
        &\quad + \frac{4K}{r}\sum_{k=1}^{K} \|V^{(\ell)}\|^2_\sigma\left\|U^{(\ell-1)\top} A^{(\ell)} \tilde{U}^{(\ell-1)}_k - \tilde{U}^{(\ell-1)\top} A^{(\ell)} \tilde{U}^{(\ell-1)}_k \right\|^2_2 \\
        & \overset{(e)}{\leq} 2K\|V^{(\ell)}\|^2_\sigma\left\|U^{(\ell-1)}- \tilde{U}^{(\ell-1)}\right\|^2_F + \frac{4K^2}{r}\sum_{k=1}^{K}\|V^{(\ell)}\|^2_\sigma\|A^{(\ell)}\|^2_\sigma \left\|U^{(\ell-1)}_k - \tilde{U}^{(\ell-1)}_k \right\|^2_2\\
        &\quad + \frac{4K}{r}\sum_{k=1}^{K} \|V^{(\ell)}\|^2_\sigma \|A^{(\ell)}\|^2_\sigma \left\|U^{(\ell-1)} - \tilde{U}^{(\ell-1)} \right\|^2_F\\
        & = 2K\|V^{(\ell)}\|^2_\sigma\left\|U^{(\ell-1)}- \tilde{U}^{(\ell-1)}\right\|^2_F + \frac{8K^2}{r}\|V^{(\ell)}\|^2_\sigma\|A^{(\ell)}\|^2_\sigma \left\|U^{(\ell-1)} - \tilde{U}^{(\ell-1)} \right\|^2_F\\
        & = 2K\|V^{(\ell)}\|^2_\sigma\left(1 + \frac{4K\|A^{(\ell)}\|^2_{\sigma}}{r}\right)\cdot \left\|U^{(\ell-1)} - \tilde{U}^{(\ell-1)}\right\|^2_F,
    \end{align*}
    where $(a)$ follows from the fact that Clip is a contraction mapping, where in $(b)$, $U^{(\ell-1)}_k$ denotes the $k$th column of $U^{(\ell-1)}\in\Re^{d\times K}$, $(c)$ follows from the fact $\|\tilde{U}^{(\ell-1)}\|^2_\sigma \leq K$ and the fact that softmax is $1$-Lipschitz, $(d)$ follows from the fact that $(x+z)^2 \leq 2(x-y)^2 + 2(y+z)^2$, and $(e)$ follows from the fact that $\|U^{(\ell-1)}U^{(\ell-1)\top}\|_\sigma \leq K$.
\end{proof}

In the above result, $\tilde{U}^{(\ell)}$ is the result of passing a perturbed input $\tilde{U}^{(\ell-1)}$ through layer $\ell-1$ of the transformer.  The result upper bounds the squared frobenius norm of $U^{(\ell)} - \tilde{U}^{(\ell)}$ via the squared frobenius norm of the input perturbation $U^{(\ell-1)} - \tilde{U}^{(\ell-1)}$.  For our analysis, the perturbed input $\tilde{U}^{(\ell-1)} = {\rm Clip}(\tilde{V}^{(\ell-1)} U^{(\ell-2)} \tilde{{\rm Attn}}^{(\ell-1)}(U^{(\ell-2)}))$, the result of passing the layer $\ell-2$ output $U^{(\ell-2)}$ through the approximation $(\tilde{A}^{(\ell-1)}, \tilde{V}^{(\ell-1)})$.  The following result bounds the expected squared frobenius norm of this perturbation.

\begin{lemma}\label{le:tsfm_dist_part}
    For all $d, r, K \in \Z_{++}$ and $\epsilon \geq 0$, if $V\in\Re^{r\times r}$ consists of elements distributed iid $\normal(0, 1/r)$, $A\in\Re^{r\times r}$ consists of elements distributed $\normal(0, 1)$, for all $\ell \in [L]$,
    $$\tilde{U}^{(\ell)} = {\rm Clip}\left(\tilde{V}^{(\ell)} U^{(\ell-1)} \tilde{{\rm Attn}}^{(\ell)}\left(U^{(\ell-1)}\right)\right),$$
    where for all $\ell < L$, $\E[\|V^{(\ell)}-\tilde{V}^{(\ell)}\|^2_F] \leq \epsilon,\ \E[\|A^{(\ell)}-\tilde{A}^{(\ell)}\|^2_F] \leq \epsilon$, and $\E[\|V^{(L)}-\tilde{V}^{(L)}\|^2_F] \leq \epsilon,\ \E[\|A^{(L)}-\tilde{A}^{(L)}\|^2_F] \leq r\epsilon/d$, then
    $$
        \E\left[\|U^{(\ell)} - \tilde{U}^{(\ell)}\|^2_F\right] \leq 
        \begin{cases}
            2K^2\left(1 + K\right)\epsilon & \text{ if } \ell < L\\
            2K\left(1+K\right)\epsilon & \text{ if } \ell = L\\
        \end{cases}.
    $$
\end{lemma}
\begin{proof}
    \begin{align*}
        &\ \E\left[\left\|U^{(\ell)} - \tilde{U}^{(\ell)}\right\|^2_F\right]\\
        & \leq \E\left[\sup_{u \in\mathcal{U}}\ \left\|V^{(\ell)} u {\rm Attn}^{(\ell)}(u) - \tilde{V}^{(\ell)} u \tilde{{\rm Attn}}^{(\ell)}(u)\right\|^2_F\right]\\
        & = \E\left[\sup_{u\in\mathcal{U}}\left\|V u \sigma\left(\frac{u^\top A u}{\sqrt{r}}\right) - \tilde{V} u \sigma\left(\frac{u^\top \tilde{A} u}{\sqrt{r}}\right)\right\|^2_F\right]\\
        & \overset{(a)}{\leq} 2\E\left[\sup_{u\in\mathcal{U}}\left\|\left(V - \tilde{V}\right) u \sigma\left(\frac{u^\top \tilde{A} u}{\sqrt{r}}\right)\right\|^2_F\right] + 2\E\left[\sup_{u\in\mathcal{U}} \left\|Vu\left(\sigma\left(\frac{u^\top A u}{\sqrt{r}}\right) 
 - \sigma\left(\frac{u^\top \tilde{A} u}{\sqrt{r}}\right)\right)\right\|^2_F\right]\\
        & \overset{(b)}{\leq} 2\E\left[\sup_{u\in\mathcal{U}}\left\|V- \tilde{V}\right\|^2_F \left\| u \sigma\left(\frac{u^\top \tilde{A} u}{\sqrt{r}}\right) \right\|^2_F\right] + 2\E\left[\sup_{u\in\mathcal{U}} \left\|V\right\|^2_F\left\|u\sigma\left(\frac{u^\top A u}{\sqrt{r}}\right) 
 - u\sigma\left(\frac{u^\top \tilde{A} u}{\sqrt{r}}\right)\right\|^2_F\right]\\
        & \overset{(c)}{\leq} 2\epsilon \cdot \sup_{u\in\mathcal{U}} \|u\|^2_F \cdot \left\|\sigma\left(\frac{u^\top \tilde{A} u}{\sqrt{r}}\right)\right\|^2_F + 2\E\left[\|V\|^2_F\cdot \sup_{u\in\mathcal{U}} \|u\|^2_F\left\|\sigma\left(\frac{u^\top A u}{\sqrt{r}}\right) 
 - \sigma\left(\frac{u^\top \tilde{A} u}{\sqrt{r}}\right)\right\|^2_F\right]\\
        & \overset{(d)}{\leq} 2\epsilon K^2 + 2\E\left[Kr\cdot \sup_{u\in\mathcal{U}} \left\|\frac{u^\top A u}{\sqrt{r}} 
 - \frac{u^\top \tilde{A} u}{\sqrt{r}}\right\|^2_F\right]\\
        & \overset{(e)}{=} 2\epsilon K^2 + 2K\cdot \E\left[\sup_{u\in\mathcal{U}}\sum_{i=1}^{K}\sum_{j=1}^{K}\left(u_i^\top(A-\tilde{A}) u_j\right)^2 \right]\\
        & \leq 2\epsilon K^2 + 2K\cdot \E\left[\sum_{i=1}^{K}\sum_{j=1}^{K}\left\|A-\tilde{A} \right\|^2_F \right]\\
        & = 2K^2\epsilon + 2K^3\epsilon,
    \end{align*}
    where $(a)$ follows from the fact that $\|a + b\|^2_F \leq 2\|a\|^2_F + 2\|b\|^2_F$ for all matrices $a, b$, $(b)$ follows from the fact that $\|ab\|^2_F \leq \|a\|^2_\sigma\|b\|^2_F$ and $\|a\|^2_\sigma \leq \|a\|^2_F$ for all matrices $a, b$, $(c)$ follows from the fact that $\E\left[\|V - \tilde{V}\|^2_F\right] = \epsilon$, $(d)$ follows from the fact that $\E[\| V\|^2_F] = r$, and the fact that softmax is $1$-Lipschitz, and where in $(e)$, $x_i$ denotes the $i$th column of matrix $x$.

    For $\ell = L$,
    \begin{align*}
        &\ \E\left[\left\|U^{(L)} - \tilde{U}^{(L)}\right\|^2_F\right]\\
        & \leq \E\left[\sup_{u \in\mathcal{U}}\ \left\|V^{(L)} u {\rm Attn}^{(L)}(u)[-1] - \tilde{V}^{(L)} u \tilde{{\rm Attn}}^{(L)}(u)[-1]\right\|^2_F\right]\\
        & = \E\left[\sup_{u\in\mathcal{U}}\left\|V^{(L)} u \sigma\left(\frac{u^\top A^{(L)} u[-1]}{\sqrt{r}}\right) - \tilde{V}^{(L)} u \sigma\left(\frac{u^\top \tilde{A}^{(L)} u[-1]}{\sqrt{r}}\right)\right\|^2_F\right]\\
        & \overset{(a)}{\leq} 2\E\left[\sup_{u\in\mathcal{U}}\left\|\left(V^{(L)} - \tilde{V}^{(L)}\right) u \sigma\left(\frac{u^\top \tilde{A}^{(L)} u[-1]}{\sqrt{r}}\right)\right\|^2_F\right]\\
        &\quad + 2\E\left[\sup_{u\in\mathcal{U}} \left\|V^{(L)}u\left(\sigma\left(\frac{u^\top A^{(L)} u[-1]}{\sqrt{r}}\right) 
 - \sigma\left(\frac{u^\top \tilde{A}^{(L)} u[-1]}{\sqrt{r}}\right)\right)\right\|^2_F\right]\\
        & \overset{(b)}{\leq} 2\E\left[\sup_{u\in\mathcal{U}}\left\|V^{(L)}- \tilde{V}^{(L)}\right\|^2_F \left\| u \sigma\left(\frac{u^\top \tilde{A}^{(L)} u[-1]}{\sqrt{r}}\right) \right\|^2_F\right]\\
        &\ + 2\E\left[\sup_{u\in\mathcal{U}} \left\|V^{(L)}\right\|^2_F\left\|u\sigma\left(\frac{u^\top A^{(L)} u[-1]}{\sqrt{r}}\right) 
 - u\sigma\left(\frac{u^\top \tilde{A}^{(L)} u[-1]}{\sqrt{r}}\right)\right\|^2_F\right]\\
        & \overset{(c)}{\leq} 2\epsilon \cdot \sup_{u\in\mathcal{U}} \|u\|^2_F \cdot \left\|\sigma\left(\frac{u^\top \tilde{A}^{(L)} u[-1]}{\sqrt{r}}\right)\right\|^2_F\\
        &\ + 2\E\left[\|V^{(L)}\|^2_F\cdot \sup_{u\in\mathcal{U}} \|u\|^2_F\left\|\sigma\left(\frac{u^\top A^{(L)} u[-1]}{\sqrt{r}}\right) 
 - \sigma\left(\frac{u^\top \tilde{A}^{(L)} u[-1]}{\sqrt{r}}\right)\right\|^2_F\right]\\
        & \overset{(d)}{\leq} 2\epsilon K + 2\E\left[Kd\cdot \sup_{u\in\mathcal{U}} \left\|\frac{u^\top A^{(L)} u[-1]}{\sqrt{r}} 
 - \frac{u^\top \tilde{A}^{(L)} u[-1]}{\sqrt{r}}\right\|^2_F\right]\\
        & \overset{(e)}{=} 2\epsilon K + \frac{2Kd}{r}\cdot \E\left[\sup_{u\in\mathcal{U}}\sum_{i=1}^{K}\left(u_i^\top(A^{(L)}-\tilde{A}^{(L)}) u[-1]\right)^2 \right]\\
        & \leq 2\epsilon K + \frac{2Kd}{r}\cdot \E\left[\sum_{i=1}^{K}\left\|A^{(L)}-\tilde{A}^{(L)} \right\|^2_F \right]\\
        & = 2K\epsilon + 2K^2\epsilon,
    \end{align*}
    where $(a)$ follows from the fact that $\|a + b\|^2_F \leq 2\|a\|^2_F + 2\|b\|^2_F$ for all matrices $a, b$, $(b)$ follows from the fact that $\|ab\|^2_F \leq \|a\|^2_\sigma\|b\|^2_F$ and $\|a\|^2_\sigma \leq \|a\|^2_F$ for all matrices $a, b$, $(c)$ follows from the fact that $\E\left[\|V - \tilde{V}\|^2_F\right] = \epsilon$, $(d)$ follows from the fact that $\E[\| V\|^2_F] = d$, and the fact that softmax is $1$-Lipschitz, and where in $(e)$, $u_i$ denotes the $i$th column of matrix $u$.
\end{proof}

With these preliminary results in place, we present the following upper bound on the $T$-horizon distortion.

\begin{lemma}{\bf (transformer $T$-horizon distortion upper bound)}\label{le:transformer_distortion}
    For all $d,r,t,K,L\in\Z_{++}$, and $\ell \in [L]$, if 
    $$\tilde{\theta}^{(\ell)}\ =\ (\tilde{V}^{(\ell)}, \tilde{A}^{(\ell)}) = (V^{(\ell)} + Z^{(\ell)}, A^{(\ell)} + B^{(\ell)}),$$
    where
    $$(Z^{(\ell)}, B^{(\ell)}) \perp (V^{(\ell)}, A^{(\ell)}),\quad Z^{(\ell)}\overset{iid}{\sim}
    \begin{cases}
        \normal(0, \epsilon/r^2) & \text{ if } \ell < L\\
        \normal(0, \epsilon/(rd) & \text{ if } \ell = L\\
    \end{cases};\quad B^{(\ell)}\overset{iid}{\sim}
        \normal(0, \epsilon/r^2),$$
    then
    $$\frac{1}{T}\sum_{t=0}^{T-1}\E\left[\KL(P^*_t \| \tilde{P}_t)\right]\ \leq \ \frac{1}{T} \sum_{t=0}^{T-1}\epsilon KL(t+1)\left(8K(1+16K)\right)^{L}.$$
\end{lemma}
\begin{proof}
    For all $d,r,t,K,L \in \Z_{++}$, $0 \leq \epsilon$, and $\ell \in [L]$, let 
    $$\tilde{\theta}^{(\ell)}\ =\ (\tilde{V}^{(\ell)}, \tilde{A}^{(\ell)}) = (V^{(\ell)} + Z^{(\ell)}, A^{(\ell)} + B^{(\ell)}),$$
    where
    $$(Z^{(\ell)}, B^{(\ell)}) \perp (V^{(\ell)}, A^{(\ell)}),\quad Z^{(\ell)}\overset{iid}{\sim}
    \begin{cases}
        \normal(0, \epsilon/r^2) & \text{ if } \ell < L\\
        \normal(0, \epsilon/(rd) & \text{ if } \ell = L\\
    \end{cases};\quad B^{(\ell)}\overset{iid}{\sim}
        \normal(0, \epsilon/r^2).$$

    Then,
    \begin{align*}
        \frac{1}{T}\sum_{t=0}^{T-1}\E\left[\KL(P^*_t \| \tilde{P}_t)\right]
        & = \frac{1}{T}\sum_{t=0}^{T-1}\sum_{\ell=1}^{L}\ \I(X_{t+1};\theta^{(\ell)}|\tilde{\theta}^{(1:L)},\theta^{(\ell+1:L)}, H_t)\\
        & \overset{(a)}{\leq} \frac{1}{T}\sum_{t=0}^{T-1}\sum_{\ell=1}^{L}\ \I(H_{t+1};\theta^{(\ell)}|\theta^{(\ell+1:L)},\theta^{(1:\ell-1)}, \tilde{\theta}^{(\ell)}, X_0)\\
        & = \frac{1}{T}\sum_{t=0}^{T-1} \sum_{\ell=1}^{L} \sum_{k=0}^{t}\  \I(X_{k+1};\theta^{(\ell)}|\theta^{(\ell+1:L)}, \theta^{(1:\ell-1)}, \tilde{\theta}^{(\ell)}, H_k)\\
        & = \frac{1}{T}\sum_{t=0}^{T-1} \sum_{\ell=1}^{L} \sum_{k=0}^{t}\ \E\left[\KL\left(\Pr(X_{k+1}\in\cdot|\theta^{(1:L)}, H_k)\ \|\ \Pr(X_{k+1}\in\cdot|\theta^{(1:L-1)}, \tilde{\theta}^{(L)}, H_k)\right)\right]
    \end{align*}
    where $(a)$ follows from Lemma \ref{le:seq_real_input_inequality}.  We now upper bound each term in the final equation above.

    We begin with the base case $\ell=L$.  For all $t$,
    \begin{align*}
        & \E\left[\KL\left(\Pr(X_{t+1}\in\cdot|\theta^{(1:L)}, H_t)\ \|\ \Pr(X_{t+1}\in\cdot|\theta^{(1:L-1)}, \tilde{\theta}^{(L)}, H_t)\right)\right]\\
        & \overset{(a)}{\leq} \E\left[\KL\left(\Pr(X_{t+1}\in\cdot|\theta^{(1:L)}, H_t)\ \|\ \Pr(X_{t+1}\in\cdot|\theta^{(1:L-1)}, \theta^{(L)}\leftarrow\tilde{\theta}^{(L)}, H_t)\right)\right]\\
        & \leq \E\left[\left\| f_{\theta^{(1:L)}}(X_{t:t-K+1}) - f_{\tilde{\theta}^{(L)}}\left(f_{\theta^{(1:L-1)}}(X_{t:t-K+1})\right) \right\|^2_2\right]\\
        & \overset{(b)}{\leq} 2K(1+K)\epsilon,
    \end{align*}
    where $(a)$ follows from Lemma \ref{le:change_measure_ub} and $(b)$ follows from Lemma \ref{le:tsfm_dist_part}.
    
    For $\ell < L$ and for all $t$, we have:
    \begin{align*}
        & \E\left[\KL\left(\Pr\left(X_{t+1}\in\cdot|\theta^{(1:L)}, H_t\right)\ \|\ \Pr\left(X_{t+1}\in\cdot|\theta^{(1:\ell-1)},\theta^{(\ell+1:L)}, \tilde{\theta}^{(\ell)}, H_t\right)\right)\right]\\
        & \overset{(a)}{\leq} \E\left[\KL\left(\Pr\left(X_{t+1}\in\cdot|\theta^{(1:L)}, H_t\right)\ \|\ \Pr\left(X_{t+1}\in\cdot|\theta^{(1:\ell-1)},\theta^{(\ell+1:L)}, \theta^{(\ell)} \leftarrow \tilde{\theta}^{(\ell)}, H_t\right)\right)\right]\\
        & \overset{(b)}{\leq} \E\left[\left\| f_{\theta^{(1:L)}}(X_{t:t-K+1}) - \left(f_{\theta^{(\ell+1:L)}}\circ f_{\tilde{\theta}^{(\ell)}} \circ f_{\theta^{(1:\ell-1)}}\right)(X_{t:t-K+1}) \right\|^2_2\right]\\
        & \overset{(c)}{\leq} \E\left[2K\|V^{(L)}\|^2_\sigma\left(1 + \frac{4K\|A^{(L)}\|^2_{\sigma}}{r}\right)\cdot \left\| f_{\theta^{(1:L-1)}}(X_{t:t-K+1}) - \left(f_{\theta^{(\ell+1:L-1)}}\circ f_{\tilde{\theta}^{(\ell)}} \circ f_{\theta^{(1:\ell-1)}}\right)(X_{t:t-K+1}) \right\|^2_2\right]\\
        & \overset{(d)}{\leq} \E\left[\left(\prod_{l=\ell+1}^{L} 2K\|V^{(l)}\|^2_\sigma\left(1 + \frac{4K\|A^{(l)}\|^2_{\sigma}}{r}\right)\right)\cdot \left\| f_{\theta^{(1:\ell)}}(X_{t:t-K+1}) - \left(f_{\tilde{\theta}^{(\ell)}} \circ f_{\theta^{(1:\ell-1)}}\right)(X_{t:t-K+1}) \right\|^2_2\right]\\
        & \overset{(e)}{\leq} \E\left[\left(\prod_{l=\ell+1}^{L} 2K\|V^{(l)}\|^2_\sigma\left(1 + \frac{4K\|A^{(l)}\|^2_{\sigma}}{r}\right)\right)\cdot 2K^2(1+K)\epsilon\right]\\
        &\overset{(f)}{\leq} \left(8K(1+16K)\right)^{L-\ell}\cdot2K^2(1+K)\epsilon\\
        & \leq \epsilon K \left(8K(1+16K)\right)^{L-\ell+1}
    \end{align*}
    where $(a)$ follows from Lemma \ref{le:change_measure_ub}, $(b)$ follows from Lemma \ref{le:kl_ub}, $(c)$ and $(d)$ follow from Lemma \ref{le:transformer_lipschitz}, $(e)$ follows from Lemma \ref{le:tsfm_dist_part} and the fact that $(V^{(i)}, A^{(i)} \perp (V^{(j)}, A^{(j)}))$ and $(V^{(i)}, A^{(i)} \perp (\tilde{V}^{(j)}, \tilde{A}^{(j)}))$ for $i\neq j$, and $(f)$ follows from the fact that $A^{(L)}\perp V^{(L)}$ and  $\E\left[\|A^{(L)}\|^2_\sigma\right] \leq 4r$ and $\E\left[\|V^{(L)}\|^2_\sigma\right] \leq 4$.

    As a result,
    \begin{align*}
        \frac{1}{T}\sum_{t=0}^{T-1}\E\left[\KL(P^*_t \| \tilde{P}_t)\right]
        & \leq \frac{1}{T}\sum_{t=0}^{T-1} \sum_{\ell=1}^{L} \sum_{k=0}^{t}\ \E\left[\KL\left(\Pr(X_{k+1}\in\cdot|\theta^{(1:L)}, H_k)\ \|\ \Pr(X_{k+1}\in\cdot|\theta^{(1:L-1)}, \tilde{\theta}^{(L)}, H_k)\right)\right]\\
        & \leq \frac{1}{T}\sum_{t=0}^{T-1} \sum_{\ell=1}^{L} \sum_{k=0}^{t}\ \epsilon K \left(8K(1+16K)\right)^{L-\ell+1}\\
        & \leq \frac{1}{T}\sum_{t=0}^{T-1} \epsilon K L (t+1)\left(8K(1+16K)\right)^{L-\ell+1}.
    \end{align*}
\end{proof}

We conclude this section with an upper bound for the $T$-horizon rate-distortion upper bound.

\begin{lemma}{\bf (transformer $T$-horizon rate-distortion upper bound)}\label{le:tsfm_rd}
    For all $d,r,t,K,L\in\Z_{++}$, if for all $t$, $X_{t}$ is generated by the transformer process, then
    $$\H_{\epsilon, T}(\theta^{(1:L)})\ \leq\ r\cdot \max\{r,d\}L\ln\left(1 + \frac{r\cdot \max\{r,d\}KLT(8K(1+16K))^L }{\epsilon}\right).$$
\end{lemma}
\begin{proof}
    For all $\ell \in [L]$, let
    $$\tilde{\theta}^{(\ell)}\ =\ (\tilde{V}^{(\ell)}, \tilde{A}^{(\ell)}) = (V^{(\ell)} + Z^{(\ell)}, A^{(\ell)} + B^{(\ell)}),$$
    where
    $$(Z^{(\ell)}, B^{(\ell)}) \perp (V^{(\ell)}, A^{(\ell)}),\quad Z^{(\ell)}\overset{iid}{\sim}
    \begin{cases}
        \normal(0, \epsilon'/r^2) & \text{ if } \ell < L\\
        \normal(0, \epsilon'/(rd) & \text{ if } \ell = L\\
    \end{cases};\quad B^{(\ell)}\overset{iid}{\sim}
        \normal(0, \epsilon'/r^2).$$
    Let $\epsilon' = \frac{\epsilon}{KLT(8K(1+16K))^L}$.  Then,
    \begin{align*}
        \I(\theta^{(1:L)};\tilde{\theta}^{(1:L)})
        & = \diffentropy(\tilde{\theta}^{(1:L)}) - \diffentropy(\tilde{\theta}^{(1:L)}|\theta^{(1:L)})\\
        & = \sum_{\ell=1}^{L} \diffentropy(\tilde{\theta}^{(\ell)}) - \diffentropy(\tilde{\theta}^{(\ell)}|\theta^{(\ell)})\\
        & = \sum_{\ell=1}^{L}\left(\diffentropy(\tilde{V}^{(\ell)}) - \diffentropy(\tilde{V}^{(\ell)}|V^{(\ell)}) + \diffentropy(\tilde{A}^{(\ell)}) - \diffentropy(\tilde{A}^{(\ell)}|A^{(\ell)}) \right)\\
        & \leq \sum_{\ell=1}^{L-1}\left(\frac{r^2}{2}\ln\left(2\pi e\left(\frac{1}{r} + \frac{\epsilon'}{r^2}\right)\right) - \frac{r^2}{2}\ln\left(2\pi e\left(\frac{\epsilon'}{r^2}\right)\right) \right)\\
        & + \left(\frac{rd}{2}\ln\left(2\pi e\left(\frac{1}{r} + \frac{\epsilon'}{rd}\right)\right) - \frac{rd}{2}\ln\left(2\pi e\left(\frac{\epsilon'}{rd}\right)\right) \right)\\
        & + \sum_{\ell=1}^{L}\left(\frac{r^2}{2}\ln\left(2\pi e\left(1 + \frac{\epsilon'}{r^2}\right)\right) - \frac{r^2}{2}\ln\left(2\pi e\left(\frac{\epsilon'}{r^2}\right)\right) \right)\\
        & \leq \frac{Lr\cdot\max\{r,d\}}{2}\ln\left(1 + \frac{\max\{r, d\} KLT(8K(1+16K))^L)}{\epsilon}\right)\\
        &\quad + \frac{Lr^2}{2}\ln\left(1 + \frac{r^2KLT(8K(1+16K)))^L}{\epsilon}\right)\\
        & \leq Lr\cdot \max\{r,d\}\ln\left(1 + \frac{r\cdot \max\{r,d\}KLT(8K(1+16K))^L }{\epsilon}\right),
    \end{align*}
    We now verify that the distortion is $\leq \epsilon.$
    \begin{align*}
        \frac{1}{T}\sum_{t=0}^{T-1}\E\left[\KL(P^*_t\|\tilde{P}_t)\right]
        & = \frac{1}{T}\sum_{t=0}^{T-1}\I(X_{t+1};\theta^{(1:L)}|\tilde{\theta}^{(1:L)}, H_t)\\
        & \overset{(a)}{\leq} \frac{1}{T}\sum_{t=0}^{T-1} \epsilon' KL(t+1)\left(8K(1+16K)\right)^{L}\\
        & \leq \epsilon' KLT\left(8K(1+16K)\right)^{L}\\
        & = \epsilon,
    \end{align*}
    where $(a)$ follows from Lemma \ref{le:transformer_distortion}.
\end{proof}

\subsubsection{Main Result}

We now present an upper bound on the optimal expected error when the data is generated by the transformer process.

\begin{theorem}{\bf (transformer error upper bound)}
    For all $d,r,L,K,T \in \Z_{++}$, if for all $t \in \Z_{+}$, $X_t$ is generated according to the transformer process, then
    $$\Lc_T\ \leq\ \frac{r\cdot\max\{r,d\}L^2\ln(8eK(1+16K))}{T} + \frac{r\cdot\max\{r,d\}L\ln\left(\frac{2KT^2}{L}\right)}{T}.$$
\end{theorem}
\begin{proof}
    \begin{align*}
        \Lc_T
        & \overset{(a)}{\leq} \inf_{\epsilon\geq 0}\ \frac{\H_{\epsilon, T}(\theta^{(1:L)})}{T} + \epsilon\\
        & \overset{(b)}{\leq} \inf_{\epsilon\geq 0}\ \frac{r\cdot \max\{r,d\}L\ln\left(1 + \frac{r\cdot \max\{r,d\}KLT(8K(1+16K))^L }{\epsilon}\right)}{T} + \epsilon\\
        & \overset{(c)}{\leq} \frac{r\cdot\max\{r, d\}L\ln\left(1 + \frac{r\cdot\max\{r,d\}KLT(8K(1+16K))^L}{\frac{r\cdot\max\{r,d\}L^2}{T}}\right)}{T} + \frac{r\cdot\max\{r,d\}L^2}{T}\\
        & \leq \frac{r\cdot\max\{r, d\}L\ln\left(\frac{2KT^2(8K(1+16K))^L}{L}\right)}{T} + \frac{r\cdot\max\{r,d\}L^2}{T}\\
        & = \frac{r\cdot\max\{r, d\}L^2\ln\left(8eK(1 + 16K)\right)}{T} + \frac{r\cdot\max\{r,d\}L \ln\left(\frac{2KT^2}{L}\right)}{T},
    \end{align*}
    where $(a)$ follows from Theorem \ref{th:rd_bounds}, $(b)$ follows from Lemma \ref{le:tsfm_rd}, $(c)$ follows by setting $\epsilon = r\cdot\max\{r, d\}L^2/T$
\end{proof}

Notably, the result scales linearly in the product of the parameter count of the transformer and the depth of the transformer.  Unlike in the MLP setting, we are unable to eliminate the quadratic depth dependence.  This is due to the fact that it is unknown whether softmax attention obeys the condition that the expected squared lipschitz constant is $\leq 1$.  In this work we upper bounded this lipschitz value by $(2K+8K^2)$ which results in the additional dependence on depth.
\newpage
\begin{summary}
    \begin{itemize}
        \item In this section, we demonstrate that the analytic tools which we developed seamlessly transfer to the setting in which data is no longer iid under some unknown distribution.  To demonstrate concrete use of our tools, we derive error upper bounds for 2 settings involving learning from sequential data: 1) a simple binary AR(K) process, 2) a transformer process.
        \item {\bf (binary AR(K) error upper bound)}
        For all embedding dimensions $d$ and context lengths $K$, if for all $t \in \Z_{+}, X_t$ is generated according to the binary AR(K) process, then for all $T$,
        $$\Lc_{T}\ \leq\ \frac{dK}{2T}\left(1 + \ln\left(1 + \frac{T}{4dK}\right)\right).$$
        \item Consider a sequence which is generated by a transformer with vocabulary size $d$, embedding dimension $r$, depth $L$, and context length $K$.  The following result holds:
        \item {\bf (transformer error upper bound)}
        For all vocabulary sizes $d$, embedding dimensions $r$, depths $L$, and context lengths $K$, if for all $t \in \Z_{+}$, $X_t$ is generated according to the transformer process, then for all $T$,
        $$\Lc_T\ \leq\ \frac{r\cdot\max\{r,d\}L^2\ln(8eK(1+16K))}{T} + \frac{r\cdot\max\{r,d\}L\ln\left(\frac{2KT^2}{L}\right)}{T}.$$
    \end{itemize}
\end{summary}
\newpage

%% file: sections/meta_learning.tex
\section{Meta-Learning}

\begin{figure}[H]
    \centering
    \includegraphics[width=0.9\textwidth]{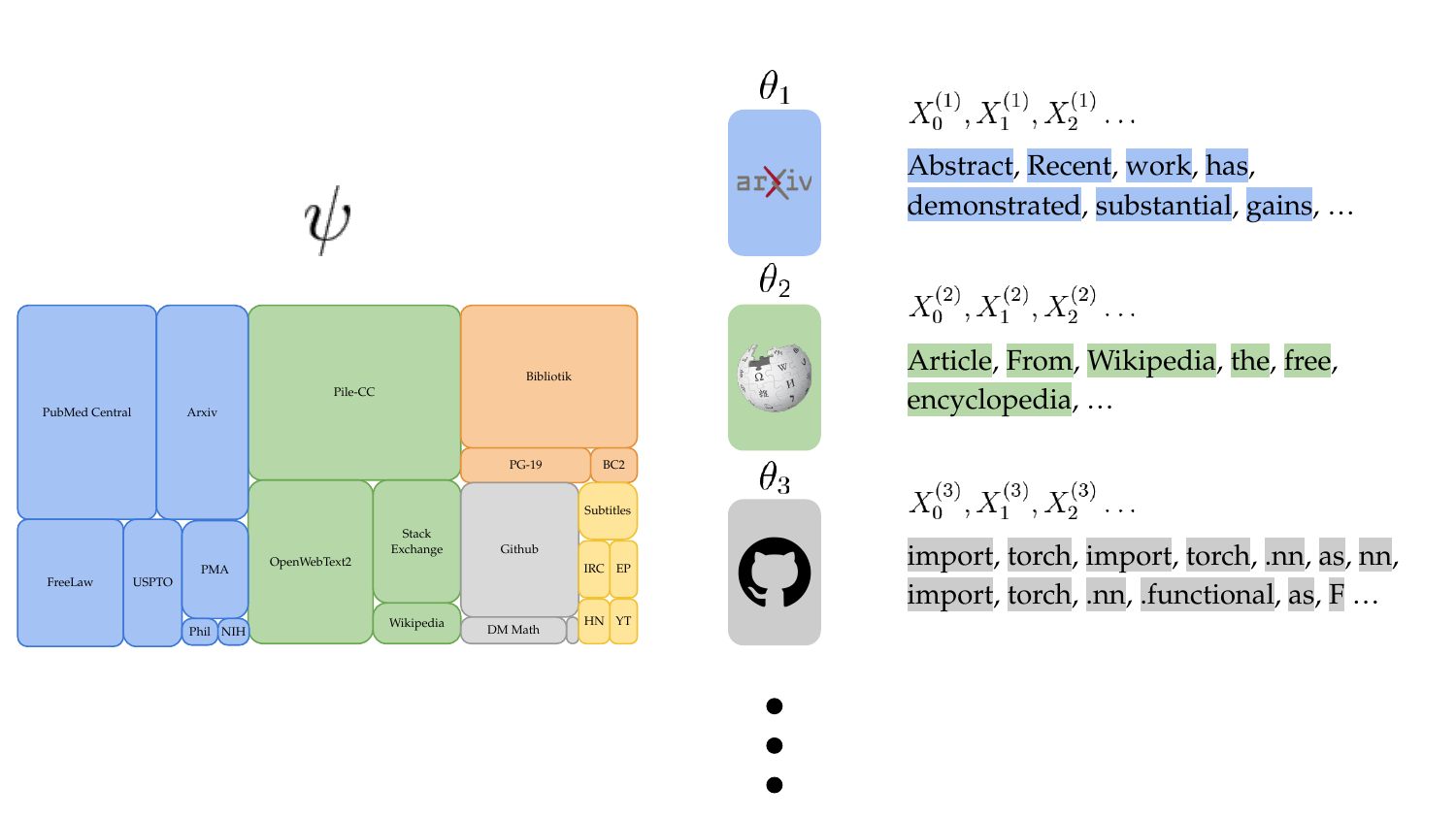}
    \caption{The above diagram depicts the pre-training of large language models as a meta-learning problem.  The meta-parameters $\psi$ specify a distribution over document types (tasks). Each document type (task) $\theta_m$ specifies an autoregressive random process. Each document $(X^{(m)}_0, X^{(m)}_1, X^{(m)}_2,\ldots)$, is represented by a sequence of tokens which is distributed according to the random process designated by the document type.}
    \label{fig:meta-learning}
\end{figure}

\subsection{Data Generating Process}

We now consider data generating processes with hierarchical structure.  There are $M$ data sequences, which we intuitively associate with {\it tasks}.  For each task $m \in [M]$, we denote the sequence by $(X_0^{(m)}, X_1^{(m)}, \ldots)$.  Each $m$th sequence is parameterized by random variable $\theta_m$.  Conditioned on $\psi$, $\theta_1, \ldots, \theta_M$ are iid.

The variable $\psi$, which we refer to as the \emph{meta-parameters}, represents the information that is shared \emph{across} the models.  Conditioned on $\psi$, $(\theta_1, \theta_2, \ldots)$ is iid.  Hence, there is no further information shared between models beyond that present in $\psi$.

As a concrete example, $\theta_1, \theta_2, \ldots, \theta_M$ could be the parameters of language models, each specialized to produce text for different tasks across diverse domain such as medicine, code, mathematics, law, etc.  And $\psi$ could encode the frequencies at which different document types appear in the text corpus used for training.  Then, $(X^{(m)}_0, X^{(m)}_1, X^{(m)}_2,\ldots)$ could be text tokens produced by a model parameterized by $\theta_m$.

\subsection{Meta-Learning Error}

We define notions of error for meta-learning, which directly parallel the definitions of Section \ref{sec:avg_reward}.  We consider learning under a setting for which there exist $M$ tasks and $T$ observations per task.  A learning algorithm produces for each $(m,t) \in [M] \times [T]$, a predictive distribution $P_{m,t}$ of $X_{t+1}^{(m)}$ after observing the concatenated history which we denote by
$$H_{m,t} = \left(X^{(1)}_{0:T}, X^{(2)}_{0:T}, \ldots, X^{(m-1)}_{0:T},X^{(m)}_{0:t}\right).$$
$H_{m,t}$ consists of all observations from tasks $1, \ldots, m-1$ and up to the $t$-th observation of task $m$.  We express our meta-learning algorithm in terms of a function $\pi$ for which $P_{m,t} = \pi(H_{m,t})$.  For all $M, T \in \Z_{++}$, we measure the error realized by our predictions $P_{m,t}$ with respect to the gold-standard clairvoyant prediction $P^*_{m,t} = \Pr(X^{(m)}_{t+1}\in\cdot|\theta, \psi, H_{m,t})$ via the expected KL divergence:
$$\frac{1}{MT} \sum_{m=1}^{M}\sum_{t=0}^{T-1}\E_{\pi} \left[ \KL\left( P^*_{m,t}\|P_{m,t} \right) \right].$$

\subsection{Optimal Achievable Error}

For all $(M,T)$, what $\pi$ minimizes error?  The following result states that the optimal algorithm assigns for all $(m,t)$, $P_{m,t} = \Pr(X^{(m)}_{t+1}\in\cdot|H_{m,t})$.  We will denote this \emph{posterior predictive distribution} as $\hat{P}_{m,t}$.

\begin{lemma}{\bf (posterior predictive distribution is optimal)}\label{le:meta_bayes_opt}
    For all $m,t\in \Z_{+}$,
    $$\E\left[ \KL( P^*_{m,t} \| \hat{P}_{m,t})\right] = \min_{\pi}\ \E_{\pi}\left[\KL(P^*_{m,t} \| P_{m,t})\right].$$
\end{lemma}
\begin{proof}
    \begin{align*}
        \E_{\pi}\left[\KL(P^*_{m,t} \| P_{m,t})\right]
        & = \E_{\pi}\left[\KL(P^*_{m,t} \| \hat{P}_{m,t})\right] + \E\left[ \ln \frac{\hat{P}_{m,t}(X^{(m)}_{t+1})}{P_{m,t}(X^{(m)}_{t+1})}\right]\\
        & = \E_{\pi}\left[\KL(P^*_{m,t} \| \hat{P}_{m,t})\right] + \E\left[\E\left[\ln \frac{\hat{P}_{m,t}(X^{(m)}_{t+1})}{P_{m,t}(X^{(m)}_{t+1})}\bigg|H_{m,t}\right]\right]\\
        & =  \E_{\pi}\left[\KL(P^*_{m,t} \| \hat{P}_{m,t})\right] + \E\left[\KL(\hat{P}_{m,t}\| P_{m,t})\right]
    \end{align*}
    The result follows from the fact that KL divergence is non-negative.
\end{proof}

Going forward, we will restrict our analysis to $\hat{P}_{m,t}$.  We will denote this error as:
$$\Lc_{M,T} = \frac{1}{MT}\sum_{m=1}^{M}\sum_{t=0}^{T-1} \E\left[ \KL(P^*_{m,t}\|\hat{P}_{m,t}) \right].$$
In the following section, we will establish general results which bound error via the total number of observations $M\cdot T$.

\subsection{General Theoretical Results}

We begin with a general result which expresses error in terms of two interpretable quantities.
\begin{restatable}{theorem}{metaSeqBayesError}{\bf (error decomposition)}\label{th:meta_est_error}
    For all $M,T \in \Z_{+}$,
    \begin{align*}
        \Lc_{M,T}
        & = \frac{1}{MT}\sum_{m=1}^{M}\sum_{t=0}^{T-1}\underbrace{\E\left[\KL(\Pr(X^{(m)}_{t+1}\in\cdot|\psi, H_{m,t})\ \|\ \Pr(X^{(m)}_{t+1}\in\cdot|H_{m,t}))\right]}_{{\rm meta\ error}}\\
        &\quad +\frac{1}{MT}\sum_{m=1}^{M}\sum_{t=0}^{T-1}\underbrace{\E\left[\KL(\Pr(X^{(m)}_{t+1}\in\cdot|\theta_m, \psi, H_{m,t})\ \|\ \Pr(X^{(m)}_{t+1}\in\cdot|\psi, H_{m,t}))\right]}_{{\rm intra-task\ error}}.
    \end{align*}
\end{restatable}
\begin{proof}
    \begin{align*}
        & \quad \frac{1}{MT}\sum_{m=1}^{M}\sum_{t=0}^{T-1} \E\left[ \KL(P^*_{m,t}\|\hat{P}_{m,t}) \right]\\
        & = \frac{1}{MT}\sum_{m=1}^{M}\sum_{t=0}^{T-1}\ \I(X_{t+1}^{(m)}; \psi, \theta_m|H_{m,t})\\
        & \overset{(a)}{=} \frac{1}{MT}\sum_{m=1}^{M}\I(X^{(m)}_{0:T};\psi, \theta_m|H_{m-1,T})\\
        & \overset{(b)}{=} \frac{1}{MT}\sum_{m=1}^{M}\I(X^{(m)}_{0:T};\psi|H_{m-1,T}) +\frac{1}{MT}\sum_{m=1}^{M}\I(X^{(m)}_{0:T};\theta_m|\psi, H_{m-1,T})\\
        & \overset{(c)}{=} \frac{1}{MT}\sum_{m=1}^{M}\I(X^{(m)}_{0:T};\psi|H_{m-1,T}) +\frac{1}{MT}\sum_{m=1}^{M}\I(X^{(m)}_{0:T};\theta_m|\psi)\\
        & \overset{(d)}{=} \frac{\I(H_{M,T};\psi)}{MT} + \sum_{m=1}^{M}\frac{\I(X^{(m)}_{0:T};\theta_m|\psi)}{MT}\\
        & = \frac{1}{MT}\sum_{m=1}^{M}\sum_{t=0}^{T-1}\underbrace{\E\left[\KL(\Pr(X^{(m)}_{t+1}\in\cdot|\psi, H_{m,t})\| \Pr(X^{(m)}_{t+1}\in\cdot|H_{m,t}))\right]}_{{\rm meta\ error}}\\
        &\quad +\frac{1}{MT}\sum_{m=1}^{M}\sum_{t=0}^{T-1}\underbrace{\E\left[\KL(\Pr(X^{(m)}_{t+1}\in\cdot|\theta_m, \psi, H_{m,t}) \| \Pr(X^{(m)}_{t+1}\in\cdot|\psi, H_{m,t}))\right]}_{{\rm intra-task\ error}},
    \end{align*}
    where $(a), (b)$, and $(d)$ follow from the chain rule of mutual information, and $(c)$ follows from the meta-learning conditional independence assumptions.
\end{proof}

The meta error represents error which is attributed to learning about the meta-parameters $\psi$.  The second term consists of the error which is incurred from learning about the task specific parameters $\theta_{1}, \ldots, \theta_{M}$ \emph{conditioned} on $\psi$.  This error encompasses the remaining information to be extracted for a particular task beyond what is present in $\psi$.

While this result is useful for conceptual understanding, it is abstract and we require further tools to produce more concrete bounds for specific model classes.  We again leverage rate-distortion theory to establish suitable error bounds.  Concretely, we define two rate-distortion functions.  The first is the $(M,T)$-horizon rate-distortion function for $\psi$.  First let $\tilde{\Psi}$ denote the set of random variables $\tilde{\psi}$ such that for all $m,t$, $X^{(m)}_{t+1}\perp \tilde{\psi}|(\psi, H_{m,t})$.  Then, let
$$\H_{\epsilon, M,T}(\psi) \ =\ \inf_{\tilde{\psi}\in\tilde{\Psi}_{\epsilon, M, T}}\ \I(\psi;\tilde{\psi}),$$
where
$$\tilde{\Psi}_{\epsilon, M, T} \ =\ \left\{\tilde{\psi} \in \tilde{\Psi}: \frac{1}{MT}\sum_{m=1}^{M}\sum_{t=0}^{T-1}\E\left[\KL(\Pr(X^{(m)}_{t+1}\in\cdot|\psi, H_{m,t}) \| \Pr(X^{(m)}_{t+1}\in\cdot|\tilde{\psi}, H_{m,t}))\right]\leq \epsilon\right\}.$$
For the intra-task parameters, let $\tilde{\Theta}_m$ denote the set of random variables $\tilde{\theta}_m$ such that for all $t$, $X^{(m)}_{t+1} \perp \tilde{\theta}_m|(\theta_m, H_{m,t})$.  Then, let
$$\H_{\epsilon, T}(\theta_m|\psi)\ =\ \inf_{\tilde{\theta}_m\in\tilde{\Theta}_{\epsilon, m, T}}\ \I(\theta_m;\tilde{\theta}_m|\psi),$$
where
$$\tilde{\Theta}_{\epsilon, m, T}\ =\ \left\{\tilde{\theta}\in\tilde{\Theta}_m: \frac{1}{T}\sum_{t=0}^{T-1}\frac{\I(X^{(m)}_{0:T};\theta_m|\tilde{\theta}, \psi)}{T} \leq \epsilon\right\}.$$
With these definitions in place, we establish error bounds in terms of the above $(M,T)$-horizon rate-distortion functions.

\begin{theorem}{\bf ($(M,T)$-horizon rate-distortion error bound)}\label{th:meta_rd}
    For all $M, T \in \Z_{+}$, and $m \in \{1, \ldots, M\}$,
    \begin{align*}
        \Lc_{M,T}
        &\ \leq\ \inf_{\epsilon\geq 0}\left(\frac{\H_{\epsilon,M,T}(\psi)}{MT} + \epsilon\right) + \inf_{\epsilon \geq 0}\left(\frac{\H_{\epsilon, T}(\theta_m|\psi)}{T} + \epsilon\right),
    \end{align*}
    and
    \begin{align*}
        \Lc_{M,T}
        &\ \geq\ \sup_{\epsilon\geq 0}\min\left\{\frac{\H_{\epsilon,M,T}(\psi)}{MT}, \epsilon\right\}\quad +\quad \sup_{\epsilon \geq 0}\min\left\{ \frac{\H_{\epsilon, T}(\theta_m|\psi)}{T}, \epsilon\right\}.
    \end{align*}
\end{theorem}
\begin{proof}
    We begin by showing the upper bound:
    \begin{align*}
        &\quad \frac{1}{MT}\sum_{m=1}^{M}\sum_{t=0}^{T-1}\E\left[\KL(P^*_{m,t}\|\hat{P}_{m,t})\right]\\
        & = \frac{\I\left(H_{M,T};\psi,\theta_{1:m} \right)}{MT}\\
        & = \frac{\I(H_{M,T};\psi)}{MT} + \frac{\I(X^{(m)}_{0:T};\theta_m|\psi)}{T}\\
        & = \frac{\I(H_{M,T};\psi,\tilde{\psi})}{MT} + \frac{\I(X^{(m)}_{0:T};\theta_m,\tilde{\theta}_m|\psi)}{T}\\
        & = \frac{\I(H_{M,T};\tilde{\psi})}{MT} + \frac{\I(H_{M,T};\psi|\tilde{\psi})}{MT} + \frac{\I(X^{(m)}_{0:T};\theta_m,\tilde{\theta}_m|\psi)}{T}\\
        & \overset{(a)}{\leq} \frac{\I(\psi;\tilde{\psi})}{MT} + \frac{\I(H_{M,T};\psi|\tilde{\psi})}{MT} + \frac{\I(X^{(m)}_{0:T};\theta_m,\tilde{\theta}_m|\psi)}{T}\\
        & = \frac{\I(\psi;\tilde{\psi})}{MT} + \frac{\I(H_{M,T};\psi|\tilde{\psi})}{MT} + \frac{\I(X^{(m)}_{0:T};\tilde{\theta}_m|\psi)}{T} + \frac{\I(X^{(m)}_{0:T};\theta_m|\tilde{\theta}_m, \psi)}{T}\\
        & \overset{(b)}{\leq} \frac{\I(\psi;\tilde{\psi})}{MT} + \frac{\I(H_{M,T};\psi|\tilde{\psi})}{MT} + \frac{\I(\theta_m;\tilde{\theta}_m|\psi)}{T} + \frac{\I(X^{(m)}_{0:T};\theta_m|\tilde{\theta}_m, \psi)}{T}\\
        & \overset{(c)}{\leq} \frac{\H_{\epsilon,M,T}(\psi)}{MT} + \epsilon + \frac{\H_{\epsilon',M,T}(\tilde{\theta}_m|\psi)}{T} + \epsilon',
    \end{align*}
    where $(a)$ and $(b)$ follow from the data processing inequality and $(c)$ follows from the definition of the rate-distortion functions.  The upper bound follows from the fact that inequality $(c)$ holds for all nonnegative $\epsilon$ and $\epsilon'$.
    
    We now prove the lower bound.
    Suppose that $\I(H_{M,T};\psi) < \H_{\epsilon,M,T}(\psi)$
    Let $\tilde{\psi} = \tilde{H}_{M,T} \notin \tilde{\Psi}_{\epsilon,M,T}$ where $\tilde{H}_{M,T}$ is another history sampled in the same manner as $H_{M,T}$.
    \begin{align*}
        \I(H_{M,T};\psi)
        & = \sum_{m=1}^{M}\sum_{t=0}^{T-1}\I(X^{(m)}_{t+1};\psi|H_{m,t})\\
        & \overset{(a)}{\geq} \sum_{m=1}^{M}\sum_{t=0}^{T-1}\I(X^{(m)}_{t+1};\psi|\tilde{H}_{M,T}, H_{m,t})\\
        & = \sum_{m=1}^{M}\sum_{t=0}^{T-1}\I(X^{(m)}_{t+1};\psi|\tilde{\psi}, X^{(m)}_1,\ldots, X^{(m)}_t)\\
        & \overset{(b)}{\geq} \epsilon MT,
    \end{align*}
    where $(a)$ follows from the fact that conditioning reduces entropy and that $X^{(m)}_{t+1}\perp \tilde{H}_{M,T}|(\psi, H_{m,t})$ and $(b)$ follows from the fact that $\tilde{\psi} \notin \tilde{\Psi}_{\epsilon,M, T}$.  Therefore, for all $\epsilon \geq 0$, $\I(H_{M,T};\psi) \geq \min\{ H_{\epsilon, M, T}(\psi), \epsilon MT\}$.

    Suppose that $\I(H_{T}^{(m)};\theta_{m}|\psi) < \H_{\epsilon,T}(\theta_{m}|\psi)$.  Let $\tilde{\theta}_{m} = \tilde{D}_m \notin \tilde{\Theta}_{\epsilon,T}$ where $\tilde{D}_{m}$ is another history sampled in the same manner as $X^{(m)}_{0:T}$.
    \begin{align*}
        \I(X^{(m)}_{0:T};\theta_m|\psi)
        & = \sum_{t=0}^{T-1}\I(X^{(m)}_{t+1};\theta_m|X^{(m)}_1,\ldots, X^{(m)}_t, \psi)\\
        & \overset{(a)}{\geq} \sum_{t=0}^{T-1}\I(X^{(m)}_{t+1};\theta_m|\tilde{D}_{m}, X^{(m)}_1,\ldots, X^{(m)}_t, \psi)\\
        & = \sum_{t=0}^{T-1}\I(X^{(m)}_{t+1};\theta_m|\tilde{\theta}_m, H_{m,t},\psi)\\
        & \overset{(b)}{\geq} \epsilon T,
    \end{align*}
    where $(a)$ follows from the fact that conditioning reduces entropy and that $X^{(m)}_{t+1}\perp \tilde{D}_{m}|(\psi, X^{(m)}_1,\ldots,X^{(m)}_t )$ and $(b)$ follows from the fact that $\tilde{\theta}_m \notin \tilde{\Theta}_{\epsilon,T}$.  Therefore, for all $\epsilon \geq 0$, $\I(X^{(m)}_{0:T};\theta_{m}|\psi) \geq \min\{ H_{\epsilon, M, T}(\theta_m), \epsilon T\}$.  The lower bound follows as a result.
\end{proof}
Theorem \ref{th:meta_rd} establishes an intimate connection between the error of meta learning and the aforementioned rate-distortion functions via upper and lower bounds.  This result will allow us to derive concrete error bounds for various meta-learning problems.  In the subsequent sections, we will study two such problems.  The first is a simple linear representation learning problem which is provided as exposition to acclimate the reader to the styles of analysis required for meta-learning.  The second setting involves a mixture of transformers which resembles LLM pre-training.

\subsection{Linear Representation Learning}

\begin{figure}[H]
    \centering
    \includegraphics[width=0.75\textwidth]{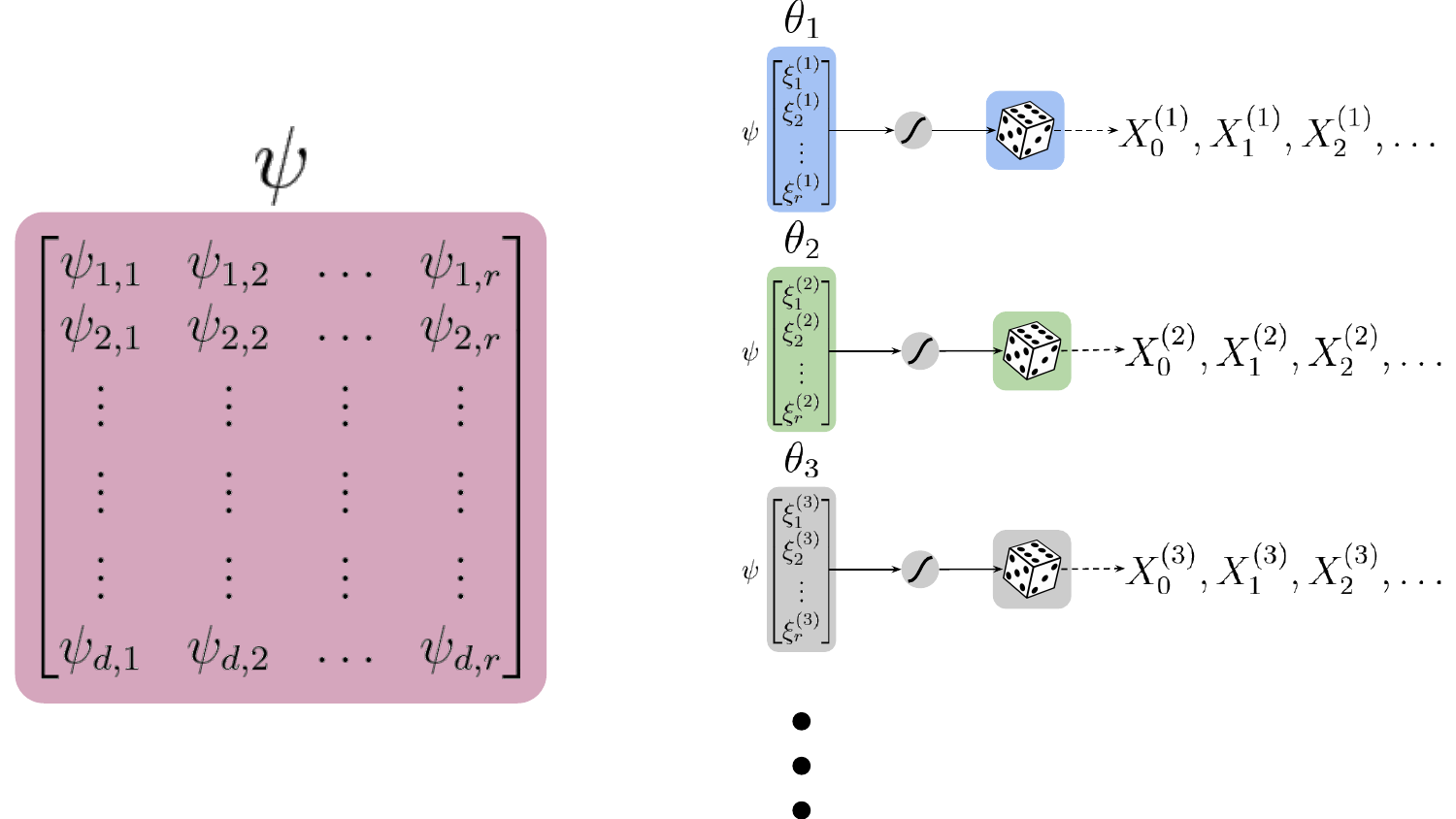}
    \caption{The above diagram depicts the linear representation learning problem.  The matrix $\psi\in\Re^{d\times r}$ represents the information which is shared across the tasks.  The intra-task information $\theta_m$ is the product $\psi \xi^{(m)}$, where $\xi^{(m)}\in\Re^r$.  The resulting vector $\theta_m$ is passed through a softmax to produce a categorical distribution over $d$ classes.  Each corresponding sequence $X^{(m)}_0, X^{(m)}_1, X^{(m)}_2, \ldots$ is produced via iid sampling of the aforementioned categorical distribution.}
    \label{fig:lin-rep-learning}
\end{figure}

\subsubsection{Data Generating Process}
We introduce a simple linear representation learning problem as a concrete example of meta-learning to demonstrate our method of analysis.  In this example, the intra-task observations are iid, but we begin with such an example for simplicity and to demonstrate this as a special case of general meta-learning under our framework.

For all $d,r\in \mathbb{Z}_{++}$, we let $\psi:\Omega\mapsto\Re^{d\times r}$ be distributed uniformly over the set of $d\times r$ matrices with orthonormal columns.  We assume that $d \gg r$.  For all $m$, let $\xi^{(m)}:\Omega\mapsto\Re^{r}$ be distributed iid $\normal(0, I_r/r)$.  We let $\theta_m = \psi\xi^{(m)}$.  For each $(m,t)$, let $X_{t+1}^{(m)}$ be drawn according to the following probability law:
$$X^{(m)}_{t+1} =
\begin{cases}
    1 & \text{w.p. } \sigma(\theta_{m})_1\\
    2 & \text{w.p. } \sigma(\theta_{m})_2\\
    \vdots & \vdots \\
    d & \text{w.p. } \sigma(\theta_{m})_d\\
\end{cases},$$
where $\sigma(\theta_{m})_j = e^{\theta_{m,j}}/\sum_{k=1}^{d} e^{\theta_{m,k}}$.  For each task $m$, the algorithm is tasked with estimating a vector $\theta_m$ from sampled observations $(X_1^{(m)}, X_2^{(m)},\ldots)$.  By reasoning about data from previous tasks, the algorithm can estimate $\psi$, which reduces the burden of estimating $\theta_m$ to just estimating $\xi^{(m)}$ for each task.  This is significant given the assumption that $d \gg r$.

\subsubsection{Preliminary Results}

We begin by establishing several lemmas which streamline our analysis.  We begin with the following upper bound on the rate-distortion function for the meta parameter.

\begin{lemma}{\bf (meta parameter rate-distortion upper bound)}\label{le:meta_rd_ub}
    For all $d,r,M,T\in\Z_{++}$, if for all $(m,t) \in [M]\times[T]$, $X_{t}^{(m)}$ is generated according to the linear representation learning process, then
    $$\H_{\epsilon, M,T}(\psi)\ \leq\ \frac{dr}{2MT}\ln\left(1 + \frac{d}{r\left(e^{\frac{2\epsilon T}{r}}-1\right)}\right).$$
\end{lemma}
\begin{proof}
    Let $\tilde{\psi} = \psi + Z$ where $Z\in \Re^{d\times r}$ is independent of $\psi$ and consists of elements which are distributed iid $\normal(0, \epsilon')$ where $\epsilon' = (e^{\frac{2\epsilon T}{r}} -1)/d$
    We begin by upper bounding the rate.
    \begin{align*}
        \frac{\I(\psi;\tilde{\psi})}{MT}
        & = \frac{\diffentropy(\tilde{\psi}) - \diffentropy(\tilde{\psi}|\psi)}{MT}\\
        & \overset{(a)}{\leq} \frac{\frac{dr}{2}\ln\left(2\pi e \left(\epsilon' + \frac{1}{r}\right)\right) - \frac{dr}{2}\ln\left(2\pi e \epsilon'\right)}{MT}\\
        & = \frac{dr\ln\left(1 + \frac{1}{r\epsilon'}\right)}{2MT}\\
        & = \frac{dr}{2MT}\ln\left(1 + \frac{d}{r\left(e^{\frac{2\epsilon T}{r}}-1\right)}\right),
    \end{align*}
    where $(a)$ follows from the maximum differential entropy of a random variable of fixed variance being upper bounded by a Gaussian random variable.

    We now upper bound the distortion.  Let $V \sim \mathcal{N}(0, I_d)$ and let $V \perp \psi$.  Then,
    \begin{align*}
        &\quad \frac{1}{MT}\sum_{m=1}^{M}\sum_{t=0}^{T-1}\E\left[\KL(\Pr(X_{t+1}^{(m)}\in\cdot|\psi,H_{m,t})\| \Pr(X_{t+1}^{(m)}\in\cdot|\tilde{\psi},H_{m,t})\right]\\
        & = \frac{\I(H_{M,T};\psi|\tilde{\psi})}{MT}\\
        & = \frac{\H(H_{M,T}|\tilde{\psi}) - \H(H_{M,T}|\psi)}{MT}\\
        & = \frac{\sum_{m=1}^{M}\H(X_{0:T}^{(m)}|\tilde{\psi}, H_{m-1,T}) - \H(X_{0:T}^{(m)}|\psi, H_{m-1, T})}{MT}\\
        & \leq \frac{\H(X_{0:T}^{(1)}|\tilde{\psi}) - \H(X_{0:T}^{(1)}|\psi)}{T}\\
        & = \frac{\I(X_{0:T}^{(1)};\psi|\tilde{\psi})}{T}\\
        & \leq \frac{\I(\theta_1;\psi|\tilde{\psi})}{T}\\
        & = \frac{\E\left[\KL(\Pr( \theta_1\in\cdot|\psi)\|\Pr(\theta_1\in\cdot|\tilde{\psi}))\right]}{T}\\
        & \overset{(a)}{\leq} \frac{\E\left[\KL(\Pr(\theta_1\in\cdot|\psi)\|\Pr(\theta_1 \in\cdot|\psi\leftarrow\tilde{\psi}))\right]}{T}\\
        & \leq \frac{\E\left[\KL(\lim_{\delta\to 0}\Pr(\theta_1 + \sqrt{\delta} \cdot V \in\cdot|\psi)\|\lim_{\delta\to 0}\Pr(\theta_1 + \sqrt{\delta} \cdot V|\psi\leftarrow\tilde{\psi}))\right]}{T}\\
        & \overset{(b)}{=} \frac{1}{T}\E\left[\lim_{\delta\to 0 }\frac{1}{2}\ln\left(\frac{\left|\delta I_d + \frac{\tilde{\psi}\tilde{\psi}^\top}{r}\right|}{\left|\delta I_d + \frac{\psi\psi^\top}{r}\right|}\right) - d + {\rm Tr}\left(\left(\delta I_d + \frac{\tilde{\psi}\tilde{\psi}^\top}{r}\right)^{-1}\left(\delta I_d + \frac{\psi\psi^\top}{r}\right)\right)\right]\\
        & \overset{(c)}{\leq} \frac{1}{T}\E\left[\lim_{\delta\to 0 }\frac{1}{2}\ln\left(\frac{\left|\delta I_d + \frac{\tilde{\pi}\tilde{\psi}^\top}{r}\right|}{\left|\delta I_d + \frac{\psi\psi^\top}{r}\right|}\right)\right]\\
        & \overset{(d)}{=} \frac{1}{T}\E\left[\lim_{\delta\to 0 }\frac{1}{2}\ln\left(\frac{\left|\delta I_d\right|\cdot\left|I_r + \frac{\tilde{\psi}^\top\tilde{\psi}}{r\delta}\right|}{\left|\delta I_d \right|\cdot\left|I_r + \frac{\psi^\top\psi}{r\delta}\right|}\right)\right]\\
        & = \frac{1}{T}\E\left[\lim_{\delta\to 0 }\frac{1}{2}\ln\left(\frac{\left|I_r + \frac{\tilde{\psi}^\top\tilde{\psi}}{r\delta}\right|}{\left|I_r + \frac{I_r}{r\delta}\right|}\right)\right]\\
        & \overset{(e)}{\leq} \lim_{\delta\to 0 } \frac{1}{2T}\ln\left(\frac{\left|I_r + \frac{\E\left[\tilde{\psi}^\top\tilde{\psi}\right]}{r\delta}\right|}{\left|I_r + \frac{I_r}{r\delta}\right|}\right)\\
        & = \lim_{\delta\to 0 } \frac{1}{2T}\ln\left(\frac{\left|I_r + \frac{\E\left[I_r + d\epsilon' I_r\right]}{r\delta}\right|}{\left|I_r + \frac{I_r}{r\delta}\right|}\right)\\
        & = \lim_{\delta\to 0 } \frac{r}{2T}\ln\left(\frac{1 + \frac{1+d\epsilon'}{r\delta}}{1+\frac{1}{r\delta}}\right)\\
        & = \frac{r}{2T}\ln\left(1 + d\epsilon \right)\\
        & = \epsilon,
    \end{align*}
    where $(a)$ follows from Lemma \ref{le:change_measure_ub}, $(b)$ follows from continuity of the KL divergence between two multivariate normal distributions w.r.t the covariance matrix, $(c)$ follows from the fact that the trace term is upper bounded by $d$, $(d)$ follows from the matrix determinant lemma, 
    $\epsilon = \frac{1}{m}$, and $(e)$ follows from Jensen's inequality.  The result follows.
\end{proof}

Recall that for meta-learning, the total estimation error is upper bounded by the sum of rate-distortion functions for the meta parameter and individual tasks.  The following result establishes an upper bound on this intra-task rate-distortion function.

\begin{lemma}{\bf (intra-task rate-distortion upper bound)}\label{le:intra_rd_ub}
    For all $r, T \in \mathbb{Z}_{++}$,
    $$\H_{\epsilon, T}(\theta_m|\psi)\ \leq\ \frac{r}{2}\ln\left(1 +\frac{1}{\epsilon}\right).$$
\end{lemma}
\begin{proof}
    Let $\tilde{\xi} = \xi^{(m)} + Z$ where $Z \perp \xi^{(m)}$ and $Z \sim \normal(0, \epsilon I_r/r)$.  Let $\tilde{\theta} = \psi\tilde{\xi}$.  We begin by upper bounding the rate.
    \begin{align*}
        \I(\theta_m;\tilde{\theta}|\psi)
        & \leq \I(\xi;\tilde{\xi}|\psi)\\
        & = \diffentropy(\tilde{\xi}|\psi)- \diffentropy(\tilde{\xi}|\psi,\xi)\\
        & = \diffentropy(\tilde{\xi})- \diffentropy(\tilde{\xi}|\xi)\\
        & = \diffentropy(\tilde{\xi})- \diffentropy(Z)\\
        & = \frac{r}{2}\ln\left(2\pi e (\frac{\epsilon}{r} + \frac{1}{r})\right) - \frac{r}{2}\ln\left(\frac{2\pi e\epsilon}{r}\right)\\
        & = \frac{r}{2}\ln\left(1 + \frac{1}{\epsilon}\right).
    \end{align*}
    We now upper bound the distortion:
    \begin{align*}
        & \quad \frac{1}{T}\sum_{t=0}^{T-1} \E\left[\KL(\Pr(X^{(m)}_{t+1}\in\cdot|\theta_m, \psi, H_{m,t}) \| \Pr(X^{(m)}_{t+1}\in\cdot|\tilde{\theta},\psi, H_{m,t})\right] \\
        & = \frac{\I(X^{(m)}_{0:T};\theta_m|\tilde{\theta}, \psi)}{T} \\
        &\leq \frac{\I(X^{(m)}_{0:T};\theta_m|\tilde{\theta},\psi)}{T} \\
        &\leq \I(X_{1}^{(m)};\theta_m|\tilde{\theta})\\
        & =\E\left[\KL\left(\Pr\left(X^{(m)}_1\in\cdot|\theta_m\right)\|\Pr\left(X^{(m)}_{1}\in\cdot|\tilde{\theta}\right)\right)\right]\\
        & \overset{(a)}{\leq} \E\left[\KL\left(\Pr\left(X^{(m)}_1\in\cdot|\theta_m\right)\|\Pr\left(X^{(m)}_{1}\in\cdot|\theta_m\leftarrow\tilde{\theta}\right)\right)\right]\\
        & \overset{(b}{\leq} \E\left[\|\tilde{\theta}-\theta_m\|^2_2\right]\\
        & = \E\left[\left(\xi^{(m)}-\tilde{\xi}\right)^\top\psi^\top\psi\left(\xi^{(m)}-\tilde{\xi}\right)\right]\\
        & = \E\left[\left(\xi^{(m)}-\tilde{\xi}\right)^\top \left(\xi^{(m)}-\tilde{\xi}\right)\right]\\
        & = \E\left[Z^\top Z\right]\\
        & = \epsilon
    \end{align*}
    where $(a)$ follows from Lemma \ref{le:bayes_opt}, and $(b)$ follows from Lemma \ref{le:kl_ub}.
\end{proof}

\subsubsection{Main Result}

With the rate-distortion upper bounds established from the previous section, we present the following result which upper bounds the error of linear representation learning.

\begin{theorem}{\bf (linear representation learning error bound)}
    For all $d,r,M,T\in\Z_{++}$, if for all $(m,t)$, $X^{(m)}_t$ is generated according to the linear representation learning process, then
    \begin{align*}
        \Lc_{M,T}
        &\ \leq\ \frac{dr\ln\left(e\left(1 + \frac{M}{r}\right)\right)}{2MT} + \frac{r\ln\left(e\left(1+\frac{2T}{r}\right)\right)}{2T}.
    \end{align*}
\end{theorem}
\begin{proof}
    \begin{align*}
        \frac{1}{MT}\sum_{m=1}^{M}\sum_{t=0}^{T-1}\E\left[\KL(P^*_{m,t}\| \hat{P}_{m,t})\right]
        & \overset{(a)}{\leq} \inf_{\epsilon\geq 0}\ \frac{\H_{M,T\epsilon}(\psi)}{MT} + \epsilon + \inf_{\epsilon'\geq 0}\ \frac{\H_{T, \epsilon'}(\theta_m|\psi)}{T} + \epsilon'\\
        & \overset{(b)}{\leq} \inf_{\epsilon \geq 0}\ \frac{dr}{2MT}\ln\left(1 + \frac{d}{r\left(e^{\frac{2\epsilon T}{r}} -1\right)}\right) + \epsilon + \inf_{\epsilon'\geq 0} \frac{r}{2n}\ln\left(1 + \frac{1}{r\epsilon'}\right) + \epsilon'\\
        & \overset{(c)}{\leq} \frac{dr}{2MT}\ln\left(1 + \frac{d}{r\left(e^{\frac{d\epsilon}{M}} - 1\right)}\right) + \frac{dr}{2MT} + \frac{r\ln(1+\frac{2T}{r})}{2T} + \frac{r}{2T}\\
        & \leq \frac{dr\ln\left(1 + \frac{M}{r}\right)}{2MT} + \frac{dr}{2MT} + \frac{r\ln(1+\frac{2T}{r})}{2T} + \frac{r}{2T}\\
        & = \frac{dr\ln\left(e\left(1 + \frac{M}{r}\right)\right)}{2MT} + \frac{r\ln\left(e\left(1+\frac{2T}{r}\right)\right)}{2T},
    \end{align*}
    where $(a)$ follows from Theorem \ref{th:meta_est_error}, 
    $(b)$ follows directly from Lemmas \ref{le:meta_rd_ub} and \ref{le:intra_rd_ub}, and $(c)$ follows from setting $\epsilon = \frac{dr}{2MT}$ and $\epsilon' = \frac{r}{2T}$.
\end{proof}

The first term describes the error which is incurred in the process of estimating the meta-parameters $\psi$.  Notably, this term is linear in $dr$, the parameter count of $\psi$, and decays linearly in $MT$, the total number of observations.  This is both intuitive and desirable since each observation ought to provide information about $\psi$.  Meanwhile, the second term describes the error which is incurred in the process of estimating the intra-task parameters $\theta_{m}$ conditioned on $\psi$.  In our problem instance, $\theta = \psi\xi$ for $\xi \in \Re^{r}$, so the remaining uncertainty in $\theta$ when conditioning on $\psi$ is simply $\xi$.  Therefore, the intra-task error is linear in $r$, the parameter count of $\xi$, and decays linearly in $T$, the number of observations pertinent to task $m$.

In the following section, we will characterize the learning performance under a much more complex process which resembles LLM pre-training

\subsection{Mixture of Transformers}

\begin{figure}[H]
    \centering
    \includegraphics[width=0.8\textwidth]{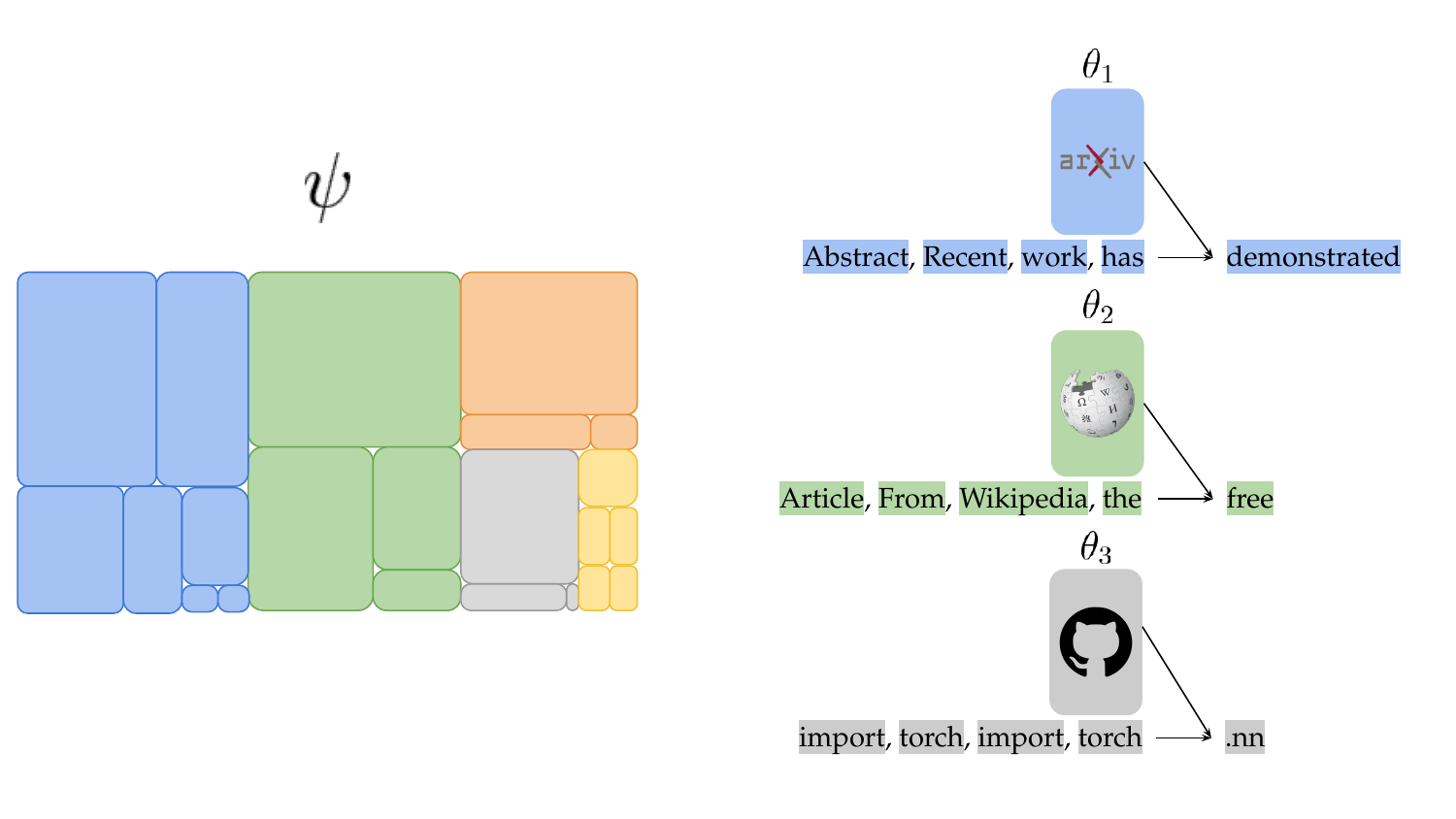}
    \caption{The above diagram depicts the mixture of transformers problem.  $\psi$ designates a mixture of many transformer models, each which is responsible for producing documents of a particular type.  Each document is generated by sampling a transformer model $\theta_m$ from the mixture $\psi$, and generating a sequence $(X_0^{(m)}, X_1^{(m)}, X_2^{(m)}, \ldots)$ autoregressively.  }
    \label{fig:incontext_learning}
\end{figure}

\subsubsection{Data Generating Process}

We consider a data generating process that produces $M$ documents.  For each $m$th document, we denote its sequence of tokens by $(X^{(m)}_1, \ldots, X^{(m)}_T)$.  Each token takes value in $[d]$, where $d$ is the size of the vocabulary.  Each of the $d$ token types is assigned a \emph{known} embedding vector, which we denote as $\Phi_j$ for $j \in [d]$.  We assume that, for all $j$, $\|\Phi_j\|_2 = 1$.  As shorthand, we let $\phi_t^{(m)} = \Phi_{X^{(m)}_t}$; this is the embedding associated with token $X^{(m)}_t$.

Each document is generated by a transformer model which is sampled iid from an unknown (random) categorical distribution $\alpha$.  There are $N$ categories, each corresponding to one transformer model.  The distribution $\alpha$ is sampled according to $\Pr(\alpha\in\cdot) = \text{Dirichlet}(N, [R/N, \ldots, R/N])$, with scale parameter $R \ll N$.

For each of the transformer models, let $K$ denote the context length, $L$ the depth, and $r$ the attention dimensions.  For all $i \in [n]$, let $\psi_i$ denote the parameters of the $i$th transformer model and let $\psi = (\alpha, \psi_1, \ldots, \psi_N)$.

For all $t$, $X^{(m)}_{t+1}$ is generated by transformer model $i_m$, where this index is sampled for document $m$.  Hence, $X^{(m)}_{t+1}$ is generated by a transformer model with weights $\psi_{i_m}$, given context $X^{(m)}_{t-K+1},\ldots, X^{(m)}_{t}$.  We use $U^{(m)}_{t,\ell}$ to denote the output of layer $\ell$ at time $t$ for document $m$.  For all $m, t$, we let $(U^{(m)}_{t,0} = \phi_{t-K+1:t})$ be the embeddings associated with the past $K$ tokens.  For $\ell > 0$, we let $U^{(m)}_{t,\ell}$ denote the output of layer $\ell$ of the transformer with input $U^{(m)}_{t,0}$.  For all $t \leq T, i < L, m \leq M$, let
$$\text{Attn}^{(\ell)}\left(U^{(m)}_{t,\ell-1}\right) = {\rm Softmax}\left( \frac{U^{(m)\top}_{t,\ell-1} A^{(m)}_\ell U^{(m)}_{t,\ell-1}}{\sqrt{r}}\right)$$
denote the attention matrix of layer $\ell$ for document $m$ where the softmax function is applied to each column.  The matrix $A^{(m)}_\ell \in \Re^{r\times r}$ can be interpreted as the product of the key and query matrices and we assume that the elements of the matrices $A^{(m)}_\ell$ are distributed iid $\normal(0,1)$.  Note that these weights belong to the transformer with parameters $\psi_{i_m}$.  

Subsequently, we let
$$U^{(m)}_{t,\ell} = \text{Clip}\left(V^{(m)}_\ell U^{(m)}_{t,\ell-1} \text{Attn}^{(\ell)}\left(U^{(m)}_{t,\ell-1}\right)\right),$$
where Clip ensures that each column of the matrix input has $L_2$ norm at most $1$.  The matrix $V^{(m)}_\ell$ resembles the value matrix and we assume that the elements of $V^{(m)}_\ell$ are distributed iid $\normal(0,1/d)$.

Finally, the next token is generated via sampling from the softmax of the final layer:
$$X^{(m)}_{t+1} \sim {\rm Softmax}\left(U^{(m)}_{t,L}[-1]\right),$$
where $U^{(m)}_{t,L}[-1]$ denotes the right-most column of $U^{(m)}_{t,L}$.  At each layer $\ell$, the parameters consist of the matrices $A^{(m)}_\ell, V^{(m)}_\ell$.  Therefore, $\theta_m = (i_m, \psi_{i_m})$.

\subsubsection{Preliminary Results}

To obtain tighter bounds for this problem instance, instead of considering two separate rate-distortion functions as in Theorem \ref{th:meta_rd}, we must instead analyze the joint $(M,T)$-horizon rate-distortion function.  Let $\tilde{\Theta}_{M}$ denote the set of random variables $\tilde{\theta}$ such that for all $t\in\mathbb{Z}_{+}$ and $m \in \{1, 2, \ldots, M\}$, $X^{(m)}_{t+1}\perp \tilde{\theta}|(\theta_{1:M}, \psi, H_{m,t})$.  Then let
$$\H_{\epsilon, M,T}(\psi, \theta_{1:M})\ =\ \inf_{\tilde{\theta}\in\tilde{\Theta}_{\epsilon, M,T}}\ \I(\psi,\theta_{1:M};\tilde{\theta}),$$
where
$$\tilde{\Theta}_{\epsilon, M,T}\ =\ \left\{\tilde{\theta}\in\tilde{\Theta}_M: \frac{1}{MT}\sum_{m=1}^{M}\sum_{t=0}^{T-1}\E\left[\KL(\Pr(X^{(m)}_{t+1}\in\cdot|\psi,\theta_{1:m},H_{m,t}) \| \Pr(X^{(m)}_{t+1}\in\cdot|\tilde{\theta},H_{m,t}))\right] \leq \epsilon\right\}.$$
We now present the following error bounds in terms of the joint $(M,T)$-horizon rate-distortion function.
\begin{theorem}{\bf (joint $(M,T)$-horizon rate-distortion error bound)}\label{th:meta_rd2}
    For all $M,T\in\Z_{+}$,
    $$\frac{1}{MT}\sum_{m=1}^{M}\sum_{t=0}^{T-1} \E\left[\KL(P^*_t\|\hat{P}_t)\right]\ \leq\ \inf_{\epsilon \geq 0}\ \frac{\H_{\epsilon,M,T}(\psi, \theta_{1:M})}{MT} + \epsilon,$$
    and
    $$\frac{1}{MT}\sum_{m=1}^{M}\sum_{t=0}^{T-1} \E\left[\KL(P^*_t\|\hat{P}_t)\right]\ \geq\ \sup_{\epsilon \geq 0}\ \min\left\{\frac{\H_{\epsilon,M,T}(\psi,\theta_{1:M})}{MT},\ \epsilon \right\}.$$
\end{theorem}
\begin{proof}
    We begin by establishing the upper bound.
    \begin{align*}
        \frac{1}{MT}\sum_{m=1}^{M}\sum_{t=0}^{T-1} \E\left[\KL(P^*_t\|\hat{P}_t)\right]
        & = \frac{\I(H_{M,T};\psi, \theta_{1:M})}{MT}\\
        & = \inf_{\tilde{\theta}\in\tilde{\Theta}_{\epsilon,M,T}}\ \frac{\I(H_{M,T};\psi, \theta_{1:M}, \tilde{\theta})}{MT}\\
        & = \inf_{\tilde{\theta}\in\tilde{\Theta}_{\epsilon,M,T}}\ \frac{\I(H_{M,T};\tilde{\theta})}{MT} + \frac{\I(H_{M,T};\theta_{1:M}|\tilde{\theta})}{MT}\\
        & \leq \frac{\H_{\epsilon,M,T}(\psi, \theta_{1:M})}{MT} + \epsilon.
    \end{align*}
    We now establish the lower bound.  Suppose that $\I(H_{M,T};\psi,\theta_{1:M}) < \H_{\epsilon,M,T}(\psi, \theta_{1:M})$.  Let $\tilde{\theta} = \tilde{H}_{M,T} \notin \tilde{\Theta}_{\epsilon,M,T}$ where $\tilde{H}_{M,T}$ is another independent history sampled in the same manner as $H_{M,T}$.
    \begin{align*}
        \frac{1}{MT}\sum_{m=1}^{M}\sum_{t=0}^{T-1} \E\left[\KL(P^*_t\|\hat{P}_t)\right]
        & = \I(H_{M,T};\psi, \theta_{1:M})\\
        & = \H(H_{M,T}) - \H(H_{M,T}|\psi, \theta_{1:M})\\
        & \overset{(a)}{=} \H(H_{M,T}) - \H(H_{M,T}|\psi, \theta_{1:M}, \tilde{\theta})\\
        & \overset{(b)}{\geq} \H(H_{M,T}|\tilde{\theta}) - \H(H_{M,T}|\psi, \theta_{1:M}, \tilde{\theta})\\
        & = \I(H_{M,T};\psi,\theta_{1:M}|\tilde{\theta})\\
        & \overset{(c)}{\geq} \epsilon M T
    \end{align*}
    where $(a)$ follows from conditional independence assumptions, $(b)$ follows from the fact that conditioning reduces entropy, and $c)$ follows from the fact that $\tilde{\theta} \notin \tilde{\Theta}_{\epsilon,M,T}$.  Therefore, for all $\epsilon \geq 0,\ \I(H_{M,T};\psi,\theta_{1:M}) \geq \min\{\H_{\epsilon,M,T}(\psi,\theta_{1:M}), \epsilon MT\}$.  The result follows.
\end{proof}

Using the above combined rate-distortion result instead of one that separates meta parameter versus intra-task rate-distortion functions (Theorem \ref{th:meta_rd}) will allow us to dramatically tighten our error bound.  This is because it suffices to learn about models that generate significant numbers of documents.

Consider a random variable $\tilde{\psi}$ which approximates $\psi$ but additionally identifies which transformer models \emph{actually} generated the documents.  Recall that $\tilde{\Psi}$ denotes the set of random variables over which we minimize to attain the meta parameter rate-distortion function.  Note that $\tilde{\psi} \notin \tilde{\Psi}$ because $\tilde{\psi}$ is not independent of $H_{M,T}$, even when conditioned on $\psi$.  However, $\tilde{\psi}$ is independent of $H_{M,T}$ when conditioned on $(\psi, \theta_{1:M})$, so $\tilde{\psi} \in \tilde{\Theta}_M$ and thus can be used to bound the joint rate-distortion function.

In the following result, we establish an upper bound on the joint $(M,T)$-horizon rate-distortion function for the mixture of transformers problem.

\begin{lemma}{\bf (mixture of transformers joint $(M,T)$-horizon rate-distortion upper bound)}\label{le:icl_rd}
    For all $d,r,K,L,M,T$, if for all $(m,t)\in[M]\times[T]$, $X^{(m)}_t$ is generated by the mixture of transformers process, then
    $$\H_{\epsilon,M,T}(\psi, \theta_{1:M}) \ \leq\ M\ln(N) + R\ln\left(1 + \frac{M}{R}\right)\cdot r\cdot\max\{r,d\}L\ln\left(1 + \frac{r\cdot \max\{r,d\}KLT(8K(1+16K))^L}{\epsilon}\right).$$
\end{lemma}
\begin{proof}
    Let $\tilde{\theta} = (i_1, i_2, \ldots, i_M)$ and let $\tilde{\psi} = (\tilde{\psi}_j: j \in \{1, 2, \ldots, N\})$, where
    $$\tilde{\psi}_j =
    \begin{cases}
        \left(\tilde{V}_{j,\ell}, \tilde{A}_{j,\ell} : \ell \in \{1, 2, \ldots, L\} \right) & {\rm if\ } j \in \mathcal{I}_M\\
        \emptyset & {\rm otherwise}\\
    \end{cases},$$
    $$\left(\tilde{V}_{j,\ell}, \tilde{A}_{j,\ell}\right) = \left(V_{j,\ell} + Z_{j,\ell},\  A_{j,\ell} + B_{j,\ell}\right)$$
    $$(Z_{j,\ell}, B_{j,\ell}) \perp (V_{j,\ell}, A_{j,\ell}),\quad Z_{j,\ell}\overset{iid}{\sim}
    \begin{cases}
        \normal(0, \epsilon'/r^2) & \text{ if } \ell < L\\
        \normal(0, \epsilon'/(rd) & \text{ if } \ell = L\\
    \end{cases};\quad B_{j,\ell}\overset{iid}{\sim}
        \normal(0, \epsilon'/r^2),$$
    $\mathcal{I}_M$ denotes the set of unique outcomes in $(i_1, i_2, \ldots, i_M)$, and $\epsilon' = \frac{\epsilon}{KLT(8K(1+16K))^L}$.

    We begin by upper bounding the rate
    \begin{align*}
        \I\left(\tilde{\psi},\tilde{\theta} ;\psi,\theta_{1:M}\right)
        & =  \I(\tilde{\theta};\psi, \theta_{1:M}) + \I\left(\tilde{\psi};\psi, \theta_{1:M}|\tilde{\theta}\right)\\
        & = \I(\tilde{\theta};\psi, \theta_{1:M}) + \I\left(\tilde{\psi};\psi|\tilde{\theta}\right)\\
        & \leq \H(\tilde{\theta}) + \I\left(\tilde{\psi};\psi|\tilde{\theta}\right)\\
        & \leq M\ln(N) + \E\left[\sum_{i=1}^{N} \mathbbm{1}_{i\in\mathcal{I}_M} \cdot \I(\tilde{\psi}_i, \psi_i)\right]\\
        & \overset{(a)}{\leq} M\ln(N) + R\ln\left(1 + \frac{M}{R}\right)\cdot r\cdot\max\{r,d\}L\ln\left(1 + \frac{r\cdot \max\{r,d\}KLT(8K(1+16K))^L}{\epsilon}\right),
    \end{align*}
    where $(a)$ follows from Lemmas \ref{le:dir_concen} and \ref{le:tsfm_rd}
    
    We now upper bound the distortion:
    \begin{align*}
        \frac{1}{MT}\sum_{m=1}^{M}\sum_{t=0}^{T-1} \E\left[\KL(P^*_t\|\hat{P}_t)\right]
        & =\frac{\I(H_{M,T};\psi,\theta_{1:M}|\tilde{\theta})}{MT}\\
        & = \frac{\I(H_{M,T};\theta_{1:M}|\tilde{\theta})}{MT}\\
        & = \sum_{m=1}^{M} \frac{\I(X^{(m)}_{0:T};\theta_{1:M}|\tilde{\theta}, H_{m-1,T})}{MT}\\
        & \leq \frac{\I(X^{(1)}_{0:T};\theta_{1:M}|\tilde{\theta})}{T}\\
        & = \frac{\I(X^{(1)}_{0:T};\theta_1|\tilde{\psi}_{i_1})}{T}\\
        & \overset{(a)}{\leq} \epsilon,
    \end{align*}
    where $(a)$ follows from Lemma \ref{le:tsfm_rd}.
\end{proof}
\subsubsection{Main Result}

Equipped with the rate-distortion results from the previous section, we establish the following upper bound on the error of the mixture of transformers process.
\begin{theorem}{\bf (mixture of transformers error upper bound)}
    For all $d,r,L,K,R,N,M,T \in \Z_{++}$, if for all $(m,t) \in [M]\times [T]$, $X^{(m)}_t$ is generated according to the mixture of transformers process, then
    \begin{align*}
        \Lc_{M,T} \leq \frac{r\max\{r,d\}RL^2\ln\left(1+\frac{M}{R}\right)\ln\left(8Ke(1+16K)\right)}{MT} + \frac{r\max\{r,d\}RL\ln\left(1+\frac{M}{R}\right)\ln\left(\frac{2KMT^2}{L}\right)}{MT} + \frac{\ln(N)}{T}.
    \end{align*}
\end{theorem}
\begin{proof}
    \begin{align*}
        &\quad \frac{1}{MT}\sum_{m=1}^{M}\sum_{t=0}^{T-1} \E\left[\KL(P^*_t\|\hat{P}_t)\right]\\
        & \leq \inf_{\epsilon\geq 0} \frac{\H_{M,T,\epsilon}(\psi, \theta_{1:M})}{MT} + \epsilon\\
        & \leq \inf_{\epsilon\geq 0} \frac{\ln(N)}{T} + \frac{r\max\{r,d\}RL\ln\left(1+\frac{M}{R}\right)\ln\left(1+\frac{r\max\{r,d\}KLT(8K(1+16K))^L}{\epsilon}\right)}{MT} + \epsilon \\
        & \leq \frac{\ln(N)}{T} + \frac{r\max\{r,d\}RL\ln\left(1+\frac{M}{R}\right)\ln\left(1+\frac{KMT^2(8K(1+16K))^L}{L}\right)}{MT} + \frac{r\max\{r,d\}RL^2}{MT}\\
        & \leq \frac{\ln(N)}{T} + \frac{r\max\{r,d\}RL\ln\left(1+\frac{M}{R}\right)\ln\left(\frac{2KMT^2(8K(1+16K))^L}{L}\right)}{MT} + \frac{r\max\{r,d\}RL^2}{MT}\\
        & \leq \frac{r\max\{r,d\}RL^2\ln\left(1+\frac{M}{R}\right)\ln\left(8Ke(1+16K)\right)}{MT} + \frac{r\max\{r,d\}RL\ln\left(1+\frac{M}{R}\right)\ln\left(\frac{2KMT^2}{L}\right)}{MT} + \frac{\ln(N)}{T}
    \end{align*}
\end{proof}
Notably, the term which decays linearly in $T$ indicates that once the transformer weights of the sampled models are learned, the only information which must be deduced from each document is its index.  As a result, the error grows logarithmically in $N$ (the size of the mixture) and decays linearly in $T$ (the length of the intra-task sequence).  The first term scales linearly in the product of the parameter count and depth of the transformer model.  Furthermore it also scales linearly in $R\ln(1+M/R)$, the expected number of unique mixture elements from $M$ documents.  This component of the error decays linearly in $MT$, so it will decay both as the number of documents and the number of tokens per document grow.  As a result, for a pre-training dataset with sufficiently large $M$, this portion of the error could become negligible, which would indicate that the loss would eventually be dominated by $\ln(N)/T$.  This is intuitive as once the transformer weights have been learned, the algorithm just has to disambiguate which index the next document belongs to.

\newpage
\begin{summary}
\begin{itemize}
    \item Meta-learning consists of a random process $(X_t^{(m)}: m, t \in \Z_{++})$ of observations.  For any fixed $m$, $(X_0^{(m)}, X_1^{(m)}, \ldots)$ represents the data associated with a particular \emph{task} indexed by $m$.
    \item The \emph{meta-parameters} $\psi$, represent information which is shared \emph{across} tasks.
    \item The \emph{intra-task parameters} $(\theta_1, \theta_2, \ldots)$ are iid when conditioned on $\psi$ and for all $m$, $X_{0:\infty}^{(m)}\perp \psi|\theta_m$.
    \item A meta-learning algorithm $\pi$ produces for each $(m,t) \in [M] \times [T]$, a predictive distribution $P_{m,t}$ of $X_{t+1}^{(m)}$ after observing the concatenated history which we denote by
    $$H_{m,t} = \left(X^{(1)}_{0:T}, X^{(2)}_{0:T}, \ldots, X^{(m-1)}_{0:T},X^{(m)}_{0:t}\right).$$
    \item For all $M, T \in \Z_{++}$, we measure the error realized by our predictions $P_{m,t}$ with respect to the gold-standard clairvoyant prediction $P^*_{m,t} = \Pr(X^{(m)}_{t+1}\in\cdot|\psi,\theta_{1:M}, H_{m,t})$ via the expected KL divergence:
    $$\frac{1}{MT}\sum_{m=1}^{M-1}\sum_{t=0}^{T-1}\E_{\pi}\left[\KL(P^*_{m,t}\| P_{m,t})\right].$$
    \item {\bf (Bayesian posterior is optimal)}
    For all $m,t\in \Z_{+}$,
    $$\E\left[ \KL( P^*_{m,t} \| \hat{P}_{m,t})\right] = \min_{\pi}\ \E_{\pi}\left[\KL(P^*_{m,t} \| P_{m,t})\right].$$
    \item{\bf (error decomposition)}
    For all $M,T \in \Z_{+}$,
    \begin{align*}
        &\quad \frac{1}{MT}\sum_{m=1}^{M}\sum_{t=0}^{T-1} \E\left[ \KL(P^*_{m,t}\|\hat{P}_{m,t}) \right]\\
        & = \frac{1}{MT}\sum_{m=1}^{M}\sum_{t=0}^{T-1}\underbrace{\E\left[\KL(\Pr(X^{(m)}_{t+1}\in\cdot|\psi, H_{m,t})\ \|\ \Pr(X^{(m)}_{t+1}\in\cdot|H_{m,t}))\right]}_{{\rm meta\ error}}\\
        &\quad +\frac{1}{MT}\sum_{m=1}^{M}\sum_{t=0}^{T-1}\underbrace{\E\left[\KL(\Pr(X^{(m)}_{t+1}\in\cdot|\theta_m, \psi, H_{m,t})\ \|\ \Pr(X^{(m)}_{t+1}\in\cdot|\psi, H_{m,t}))\right]}_{{\rm intra-task\ error}}.
    \end{align*}
    \item Let $\tilde{\Psi}$ denote the set of random variables $\tilde{\psi}$ such that for all $m,t$, $X^{(m)}_{t+1}\perp \tilde{\psi}|(\psi, H_{m,t})$.  Then, let
    $$\H_{\epsilon, M,T}(\psi) \ =\ \inf_{\tilde{\psi}\in\tilde{\Psi}_{\epsilon, M, T}}\ \I(\psi;\tilde{\psi}),$$
    where
    $$\tilde{\Psi}_{\epsilon, M, T} \ =\ \left\{\tilde{\psi} \in \tilde{\Psi}: \frac{1}{MT}\sum_{m=1}^{M}\sum_{t=0}^{T-1}\E\left[\KL(\Pr(X^{(m)}_{t+1}\in\cdot|\psi, H_{m,t}) \| \Pr(X^{(m)}_{t+1}\in\cdot|\tilde{\psi}, H_{m,t}))\right]\leq \epsilon\right\}.$$
    \item Let $\tilde{\Theta}_m$ denote the set of random variables $\tilde{\theta}_m$ such that for all $t$, $X^{(m)}_{t+1} \perp \tilde{\theta}_m|(\theta_m, H_{m,t})$.  Then, let
    $$\H_{\epsilon, T}(\theta_m|\psi)\ =\ \inf_{\tilde{\theta}_m\in\tilde{\Theta}_{\epsilon, m, T}}\ \I(\theta_m;\tilde{\theta}_m|\psi),$$
    where
    $$\tilde{\Theta}_{\epsilon, m, T}\ =\ \left\{\tilde{\theta}\in\tilde{\Theta}_m: \frac{1}{T}\sum_{t=0}^{T-1}\frac{\I(X^{(m)}_{0:T};\theta_m|\tilde{\theta}, \psi)}{T} \leq \epsilon\right\}.$$
    \item {\bf ($(M,T)$-horizon rate-distortion error bound)}
    For all $M, T \in \Z_{+}$, and $m \in \{1, \ldots, M\}$,
    \begin{align*}
        \frac{1}{MT}\sum_{m=1}^{M}\sum_{t=0}^{T-1}\E\left[\KL(P^*_{m,t}\|\hat{P}_{m,t})\right]
        &\ \leq\ \inf_{\epsilon\geq 0}\ \frac{\H_{\epsilon,M,T}(\psi)}{MT} + \epsilon\ +\ \inf_{\epsilon' \geq 0}\ \frac{\H_{\epsilon', T}(\theta_m|\psi)}{T} + \epsilon',
    \end{align*}
    and
    \begin{align*}
        \frac{1}{MT}\sum_{m=1}^{M}\sum_{t=0}^{T-1}\E\left[\KL(P^*_{m,t}\|\hat{P}_{m,t})\right]
        &\ \geq\ \sup_{\epsilon\geq 0}\min\left\{\frac{\H_{\epsilon,M,T}(\psi)}{MT}, \epsilon\right\}\ +\ \sup_{\epsilon' \geq 0}\min\left\{ \frac{\H_{\epsilon', T}(\theta_m|\psi)}{T}, \epsilon'\right\}.
    \end{align*}
    \item  Let $\tilde{\Theta}_{M}$ denote the set of random variables $\tilde{\theta}$ such that for all $t\in\mathbb{Z}_{+}$ and $m \in \{1, 2, \ldots, M\}$,\\ $X^{(m)}_{t+1}\perp \tilde{\theta}|(\theta_{1:M}, \psi, H_{m,t})$.  Then let
    $$\H_{\epsilon, M,T}(\psi, \theta_{1:M})\ =\ \inf_{\tilde{\theta}\in\tilde{\Theta}_{\epsilon, M,T}}\ \I(\psi,\theta_{1:M};\tilde{\theta}),$$
    where
    $$\tilde{\Theta}_{M,T,\epsilon}\ =\ \left\{\tilde{\theta}\in\tilde{\Theta}_M: \frac{1}{MT}\sum_{m=1}^{M}\sum_{t=0}^{T-1}\E\left[\KL(\Pr(X^{(m)}_{t+1}\in\cdot|\psi,\theta_{1:m},H_{m,t}) \| \Pr(X^{(m)}_{t+1}\in\cdot|\tilde{\theta},H_{m,t}))\right] \leq \epsilon\right\}.$$
    \item {\bf (joint $(M,T)$-horizon rate-distortion error bound)}
    For all $M,T\in\Z_{+}$,
    $$\frac{1}{MT}\sum_{m=1}^{M}\sum_{t=0}^{T-1} \E\left[\KL(P^*_t\|\hat{P}_t)\right]\ \leq\ \inf_{\epsilon \geq 0}\ \frac{\H_{\epsilon,M,T}(\psi, \theta_{1:M})}{MT} + \epsilon,$$
    and
    $$\frac{1}{MT}\sum_{m=1}^{M}\sum_{t=0}^{T-1} \E\left[\KL(P^*_t\|\hat{P}_t)\right]\ \geq\ \sup_{\epsilon \geq 0}\ \min\left\{\frac{\H_{\epsilon,M,T}(\psi,\theta_{1:M})}{MT},\ \epsilon \right\}.$$
\end{itemize}
\end{summary}
\clearpage

%% file: sections/misspecification.tex
\section{Misspecification}

All results presented in previous chapters pertain to an agent that carries out Bayesian inference with a correctly specified prior distribution.  In this chapter, we establish results that characterize the impact of prior misspecification.  We apply these results to study error with misspecified linear and neural network models.

\subsection{General Theoretical Results}

In order to model the misspecified prior, we introduce an alternative probability measure $\mathbb{Q}$ on $(\Omega, \mathbb{F})$.  This measure $\mathbb{Q}$ may assign a different marginal distribution to $\theta$ but distributions conditioned on $\theta$ are identical: $\mathbb{Q}(\cdot|\theta) = \Pr(\cdot|\theta)$.  In particular, for all $t$, 
$\mathbb{Q}(H_t \in \cdot | \theta) = \mathbb{P}(H_t \in \cdot | \theta)$.
The posterior of $\theta$ conditioned on $H_t=h$ under the misspecified prior is $\mathbb{Q}(\theta \in \cdot|H_t = h)$.  We will use the notation $\mathbb{Q}(\theta \in \cdot|H_t \leftarrow H_t)$ to express the random distribution obtained by sampling $h$ from $\Pr(H_t \in \cdot)$ and calculating the conditional probability $\mathbb{Q}(\theta \in \cdot|H_t = h)$.

For all $t$, let
$$\hat{Q}_t = \mathbb{Q}(Y_{t+1}\in\cdot|H_t\leftarrow H_t).$$
Note that this is the posterior predictive distribution with the history drawn according to $\Pr$ and Bayesian inference carried out according to $\mathbb{Q}$.  Our analysis will make use of a set $\Theta^\dagger$ consisting of random variables $\theta^\dagger$ for which $\Pr(\theta^\dagger \in \cdot)$ and $\mathbb{Q}(\theta \in \cdot)$ are equivalent measures; that is, each is absolutely continuous with respect to the other.  We refer to such a random variable $\theta^\dagger$ as a {\it constrained approximation} of $\theta$.  When $\Pr(\theta \in \cdot)$ and $\mathbb{Q}(\theta \in \cdot)$ are equivalent measures, we will typically take the constrained approximation $\theta^\dagger$ to be identical to $\theta$.  Otherwise, $\theta^\dagger$ will typically forgo some information expressed by $\theta$; the amount of information lost depends on the degree of model misspecification.

Recall that the $T$-horizon rate-distortion function of $\theta$ is 
$$\H_{\epsilon, T}(\theta) = \ \inf_{\tilde{\theta} \in \tilde{\Theta}_{\epsilon, T}}\ \I(\theta;\tilde{\theta}),$$
where
$$
\tilde{\Theta}_{\epsilon, T} = \left\{\tilde{\theta} \in \tilde{\Theta}: \frac{1}{T}\sum_{t=0}^{T-1}\E\left[\KL(\Pr(Y_{t+1}\in\cdot|\theta, H_t) \| \Pr(Y_{t+1}\in\cdot|\tilde{\theta}, H_t))\right] \leq \epsilon\right\}.
$$
Recall that $\tilde{\Theta}$ is the set of random variables $\tilde{\theta}$ such that, for all $t$, $\tilde{\theta} \perp Y_{t+1} | (\theta, H_t)$.  Recall that we interpret $\tilde{\theta}$ as an (unconstrained) approximation of $\theta$.  In this chapter, we will use a variation of this definition that pertains to approximation of $\theta^\dagger$ instead of $\theta$.  In particular, for any random variable $\theta^\dagger \in \Theta^\dagger$, let
$$\H^\dagger_{\epsilon, T}(\theta^\dagger) = \ \inf_{\tilde{\theta}^\dagger \in \tilde{\Theta}^\dagger_{\epsilon, T}}\ \I(\theta^\dagger;\tilde{\theta}^\dagger),$$\
where
\begin{align*}
    \tilde{\Theta}^\dagger_{\epsilon, T}
    & = \left\{\tilde{\theta}^\dagger \in \tilde{\Theta}: \frac{1}{T}\sum_{t=0}^{T-1}\E\left[\KL(\Pr(Y_{t+1}\in\cdot|\theta^\dagger, \tilde{\theta}^\dagger, H_t) \| \Pr(Y_{t+1}\in\cdot|\tilde{\theta}^\dagger, H_t))\right] \leq \epsilon\right\}.
\end{align*}

The following result upper bounds the error of $\hat{Q}_t$ in terms of an irreducible misspecification error and the $T$-horizon rate-distortion function.

\begin{theorem}\label{th:unrealizable}{\bf (misspecified learner error bound)}\label{th:unrealizable}
    For all $T \in \mathbb{Z}_{++}$,
    \begin{align*}
        \frac{1}{T}\sum_{t=0}^{T-1}\ \E\left[\KL\left(P^*_t\|\hat{Q}_t\right)\right]
        \leq\ \inf_{\theta^\dagger\in \Theta^\dagger} \left(\underbrace{\frac{1}{T}\sum_{t=0}^{T-1}\E\left[\KL(P^*_t\|P^\dagger_t)\right] + \frac{\KL(\Pr(\theta^\dagger\in\cdot) \| \mathbb{Q}(\theta\in\cdot) )}{T}}_{\rm misspecification\ error}\ +\ \underbrace{\inf_{\epsilon\geq 0}\ \left(\frac{\H^\dagger_{\epsilon, T}(\theta^\dagger)}{T} + \epsilon\right)}_{\rm statistical\ error}\right),
    \end{align*}
    where $P^\dagger_t = \Pr(Y_{t+1}\in\cdot|\theta\leftarrow \theta^{\dagger}, H_t)$. 
\end{theorem}
\begin{proof}
While the stated result applies to continuous random variables, to keep this proof simple, we will restrict attention to the discrete case, leaving extension to the reader.

Fix some $\theta^\dagger \in \Theta^\dagger$ and let $f$ denote a random function such that $\Pr(f(\theta^\dagger)\in\cdot) \overset{d}{=} \mathbb{Q}(\theta\in\cdot)$.  Such a function $f$ must exist since $\Pr(\theta^\dagger\in\cdot)$ and $\mathbb{Q}(\theta\in\cdot)$ are equivalent measures.  Let $\Theta_{\mathbb{Q}}$ be the support of $\mathbb{Q}(\theta\in \cdot)$ and let
    $$\mathbb{Q}(H_T) = \sum_{\nu\in\Theta_{\mathbb{Q}}} \mathbb{Q}(\theta=\nu) \cdot \Pr(H_T|\theta=\nu).$$
    Then,
    \begin{align*}
        & \sum_{t=0}^{T-1}\ \E\left[\KL(P^*_t\|\hat{Q}_t)\right]\\
        & = \sum_{t=0}^{T-1}\E\left[\ln\frac{\Pr\left(Y_{t+1}|\theta, H_t\right)}{\hat{Q}_t(Y_{t+1})}\right]\\
        & = \E\left[\ln \frac{\mathbb{P}(H_T|\theta)}{\mathbb{Q}(H_T)}\right]\\
        & = \E\left[\ln \frac{\mathbb{P}(H_T|\theta)}{\Pr(H_T|\tilde{\theta}^{\dagger})}\right]
        + \E\left[\ln \frac{\mathbb{P}(H_T|\tilde{\theta}^{\dagger})}{\mathbb{Q}(H_T)}\right]\\
        & = \E\left[\ln \frac{\mathbb{P}(H_T|\theta)}{\Pr(H_T|\tilde{\theta}^{\dagger})}\right] + \E\left[\ln \frac{\mathbb{P}(H_T|\tilde{\theta}^{\dagger})}{\sum_{\nu\in \Theta_{\mathbb{Q}}}\mathbb{Q}(\theta=\nu)\cdot\Pr(H_T|\theta=\nu)}\right]\\
        & = \E\left[\ln \frac{\mathbb{P}(H_T|\theta)}{\Pr(H_T|\tilde{\theta}^{\dagger})}\right] + \E\left[\ln \frac{\mathbb{P}(H_T|\tilde{\theta}^{\dagger})}{\sum_{\nu\in \Theta_{\mathbb{Q}}}\Pr(f(\theta^{\dagger})=\nu)\cdot\Pr(H_T|\theta=\nu)}\right]\\
        & =  \E\left[\ln \frac{\mathbb{P}(H_T|\theta)}{\Pr(H_T|\tilde{\theta}^{\dagger})}\right] + \E\left[\ln \frac{\sum_{\nu\in\Theta_{\mathbb{Q}}}\mathbb{P}(H_T,\theta^{\dagger}=\nu|\tilde{\theta}^{\dagger})}{\sum_{\nu\in \Theta_{\mathbb{Q}}}\Pr(f(\theta^{\dagger})=\nu)\cdot\Pr(H_T|\theta=\nu)}\right]\\
        & =  \E\left[\ln \frac{\mathbb{P}(H_T|\theta)}{\Pr(H_T|\tilde{\theta}^{\dagger})}\right] + \E\left[\sum_{h\in\mathcal{H}}\sum_{\nu'\in\Theta_{\mathbb{Q}}}\mathbb{P}(H_T=h,\theta^{\dagger}=\nu'|\tilde{\theta}^{\dagger})\ln \frac{\sum_{\nu\in\Theta_{\mathbb{Q}}}\mathbb{P}(H_T=h,\theta^{\dagger}=\nu|\tilde{\theta}^{\dagger})}{\sum_{\nu\in \Theta^{\dagger}}\Pr(f(\theta^{\dagger})=\nu)\cdot\Pr(H_T=h|\theta=\nu)}\right]\\
        & \overset{(a)}{\leq}   \E\left[\ln \frac{\mathbb{P}(H_T|\theta)}{\Pr(H_T|\tilde{\theta}^{\dagger})}\right] + \E\left[\sum_{h\in\mathcal{H}}\sum_{\nu\in\Theta_{\mathbb{Q}}}\mathbb{P}(H_T=h,\theta^{\dagger}=\nu|\tilde{\theta}^{\dagger})\ln \frac{\mathbb{P}(H_T=h,\theta^{\dagger}=\nu|\tilde{\theta}^{\dagger})}{\Pr(f(\theta^{\dagger})=\nu)\cdot\Pr(H_T=h|\theta=\nu)}\right]\\
        & =  \E\left[\ln \frac{\mathbb{P}(H_T|\theta)}{\Pr(H_T|\tilde{\theta}^{\dagger})}\right] + \E\left[\ln \frac{\Pr(H_T|\tilde{\theta}^\dagger)}{\Pr(H_T|\theta\leftarrow \theta^\dagger)}\right] + \E\left[\sum_{h\in\mathcal{H}}\sum_{\nu\in\Theta_{\mathbb{Q}}}\mathbb{P}(H_T=h,\theta^{\dagger}=\nu|\tilde{\theta}^{\dagger})\ln \frac{\Pr(\theta^\dagger=\nu|H_t=h,\tilde{\theta}^{\dagger})}{\Pr(f(\theta^{\dagger})=\nu)}\right]\\
        & = \sum_{t=0}^{T-1}\E\left[\KL(P^*_t\|P^\dagger_t)\right] + \E\left[\sum_{h\in\mathcal{H}}\sum_{\nu\in\Theta_{\mathbb{Q}}}\mathbb{P}(H_T=h,\theta^{\dagger}=\nu|\tilde{\theta}^{\dagger})\ln \frac{\Pr(\theta^\dagger=\nu|H_t=h,\tilde{\theta}^{\dagger})}{\Pr(f(\theta^{\dagger})=\nu)}\right]\\
        & = \sum_{t=0}^{T-1}\E\left[\KL(P^*_t\|P^\dagger_t)\right] + \E\left[\sum_{h\in\mathcal{H}}\sum_{\nu\in\Theta_{\mathbb{Q}}}\mathbb{P}(H_T=h,\theta^{\dagger}=\nu|\tilde{\theta}^{\dagger}) \left(\ln \frac{\Pr(\theta^\dagger=\nu|H_t=h,\tilde{\theta}^{\dagger})}{\Pr(\theta^{\dagger}=\nu)} + \ln \frac{\Pr(\theta^{\dagger}=\nu)}{\Pr(f(\theta^\dagger)=\nu)}\right)\right]\\
        & \overset{(b)}{=} \sum_{t=0}^{T-1}\E\left[\KL(P^*_t\|P^\dagger_t)\right] + \KL(\Pr(\theta^\dagger\in\cdot) \| \Pr(f(\theta^\dagger) \in \cdot) + \E\left[\KL(\Pr(\theta^\dagger\in\cdot|H_T,\tilde{\theta}^{\dagger})\|\Pr(\theta^\dagger\in\cdot))\right]\\
    \end{align*}
    where $(a)$ follows from the log sum inequality and $(b)$ follows from the fact that $\theta^\dagger$ and $f(\theta^\dagger)$ have the same support and hence the expression can be written as a KL divergence.  Further,
    \begin{align*}
        & \KL(\Pr(\theta^\dagger\in\cdot) \| \Pr(f(\theta^\dagger) \in \cdot) ) + \E\left[\KL(\Pr(\theta^\dagger\in\cdot|H_T,\tilde{\theta}^{\dagger})\|\Pr(\theta^\dagger\in\cdot))\right]\\
        & = \KL(\Pr(\theta^\dagger\in\cdot) \| \mathbb{Q}(\theta \in \cdot) ) + \I(\theta^\dagger;H_T,\tilde{\theta}^{\dagger})\\
        & = \KL(\Pr(\theta^\dagger\in\cdot) \| \mathbb{Q}(\theta \in \cdot) ) + \I(\theta^\dagger;\tilde{\theta}^{\dagger}) + \I(\theta^\dagger;H_T|\tilde{\theta}^{\dagger})\\
        & = \KL(\Pr(\theta^\dagger\in\cdot) \| \mathbb{Q}(\theta \in \cdot) ) + \I(\theta^\dagger;\tilde{\theta}^{\dagger}) +  \sum_{t=0}^{T-1}\I(Y_{t+1};\theta^\dagger|\tilde{\theta}^{\dagger}, H_t) \\
        & =  \KL(\Pr(\theta^\dagger\in\cdot) \| \mathbb{Q}(\theta \in \cdot) ) + \I(\theta^\dagger;\tilde{\theta}^{\dagger}) + \sum_{t=0}^{T-1}\E\left[\KL(\Pr(Y_{t+1}\in\cdot|\theta^\dagger,\tilde{\theta}^\dagger, H_t)\|\Pr(Y_{t+1}\in\cdot|\tilde{\theta}^\dagger, H_t))\right],
    \end{align*}
    The theorem statement follows from the fact that these relations hold for all $\tilde{\theta}^\dagger$, the definition of $\H^\dagger_{\epsilon, T}$, and by dividing both sides of the inequality by $T$.
\end{proof}
$\H^\dagger_{\epsilon, T}(\theta^\dagger)$ represents error incurred in the process of learning $\theta^\dagger$.  When $\epsilon$ is optimized for each $T$, this error vanishes as $T$ grows.  However, the misspecification error will often remain bounded away from zero for all $T$ as it represents the shortfall of making predictions using $\mathbb{Q}$ instead of $\mathbb{P}$.

In the special case where $\Pr(\theta\in\cdot)$ and $\mathbb{Q}(\theta \in\cdot)$ are equivalent measures, the misspecification error vanishes as $T$ grows.  The following corollary of Theorem \ref{th:unrealizable} establishes this.
\begin{restatable}{corollary}{misspecified}{\bf (misspecified learner with matching support error bound )}\label{th:prior_ub}
    For all $T\in\mathbb{Z}_{++}$, if $\Pr(\theta\in\cdot)$ and $\mathbb{Q}(\theta\in \cdot)$ are equivalent measures then 
    \begin{align*}
        \frac{1}{T}\sum_{t=0}^{T-1} \E\left[\KL\left(P^*_{t}\ \big\|\ \hat{Q}_{t}\right)\right]\ \leq\ \frac{\KL\left(\Pr(\theta\in\cdot)\|\mathbb{Q}(\theta\in\cdot)\right)}{T} + \inf_{\epsilon \geq 0} \left(\frac{\H_{\epsilon, T}(\theta)}{T} + \epsilon\right).
    \end{align*}
\end{restatable}
\begin{proof}
    Let $\theta^\dagger \overset{a.s.}{=} \theta$.  The result follows from Theorem \ref{th:unrealizable} since $\theta^\dagger \overset{a.s.}{=} \theta$ results in $\E[\KL(P^*_t\|P^\dagger_t)] = 0$ for all $t$ and $\H^\dagger_{\epsilon, T}(\theta) = \H_{\epsilon,T}(\theta)$.
\end{proof}

{\bf Remarks about $\theta^\dagger$}\\
Since the upper bound of Theorem \ref{th:unrealizable} holds for all $\theta^\dagger$ for which $\Pr(\theta^\dagger \in \cdot)$ and $\mathbb{Q}(\theta\in\cdot)$ are equivalent measures, deriving an upper bound reduces to applying Theorem \ref{th:unrealizable} with one such $\theta^\dagger$.  The choice of $\theta^\dagger$ may dramatically impact the resulting upper bound.  One may consider an extreme case in which $\theta^\dagger \perp \theta$, and hence $\H_{\epsilon, T}(\theta^\dagger) = 0$.  However, for this choice of $\theta^\dagger$, the misspecification error becomes very large, as 
$$\E\left[\KL(\Pr(Y_{t+1}\in\cdot|\theta, H_t) \| \Pr(Y_{t+1}\in\cdot|\theta\leftarrow\theta^\dagger, H_t))\right]$$
is very large when $\theta^\dagger$ is independent of $\theta$.

In the section to follow, we study two concrete examples involving misspecified linear models.  These examples serve to illustrate application of Theorems \ref{th:prior_ub} and \ref{th:unrealizable}.  Section \ref{se:neural-scaling-laws} then demonstrates that these theorems can be used to study misspecification of much more complex models such as neural networks.

\subsection{Linear Regression with Misspecification}

We now study two kinds of misspecified linear models.  One with a misspecified prior mean and one with a missing feature.

\subsubsection{Data Generating Process}

Recall the linear regression model of Section \ref{se:linear-regression}, which is parameterized by a random vector $\theta\in\Re^d$ of feature coefficients with prior distribution $\Pr(\theta\in\cdot) = \normal(0, I_d/d)$.  For each $t \in \Z_{+}$, inputs and outputs are generated according to a random vector $X_t \overset{iid}{\sim} \normal(0, I_d)$ and
$$Y_{t+1} = \theta^\top X_t + W_{t+1},$$
where $W_{t}\overset{iid}{\sim} \normal(0, \sigma^2)$ for known variance $\sigma^2$.

\subsubsection{Mean Misspecified Algorithm}

We first study the case of a misspecified mean.  In particular, we consider  $\mathbb{Q}(\theta\in\cdot) = \normal(\mu, I_d)$.  The following result upper bounds the error introduced by this misspecified prior.
\begin{theorem}{\bf (mean misspecification error bound)}
For all $d,T \in \mathbb{Z}_{++}$, if for all $t \in \{0,1,\ldots, T-1\}$, $(X_t, Y_{t+1})$ is generated by the linear regression process and $\mathbb{Q}(\theta\in\cdot) = \normal(\mu, I_d/d)$ then
     $$\frac{1}{T}\sum_{t=0}^{T-1} \E\left[\KL\left(P^*_t \| \hat{Q}_t\right)\right]\ \leq\ \frac{\|\mu\|^2_2}{2T} + \frac{d}{2T}\ln\left(\frac{T}{\sigma^2d}\right) + \frac{1}{2T}\ln\left(1 + \frac{d}{T}\right).$$
\end{theorem}
\begin{proof}
    Let $\theta^\dagger = \theta$.  Then,
    \begin{align*}
        \frac{1}{T}\sum_{t=0}^{T} \E\left[\KL\left(P^*_t \| \hat{Q}_t\right)\right]
        & \overset{(a)}{\leq} \frac{\KL\left(\Pr(\theta\in\cdot)\|\mathbb{Q}(\theta\in\cdot)\right)}{T} + \inf_{\epsilon \geq 0} \left(\frac{\H^\dagger_{\epsilon, T}(\theta)}{T} + \epsilon\right)\\
        & \overset{(b)}{=} \frac{\|\mu\|^2_2}{2T} + \inf_{\epsilon \geq 0} \left(\frac{\H_{\epsilon, T}(\theta)}{T} + \epsilon\right)\\
        & \overset{(c)}{\leq} \frac{\|\mu\|^2_2}{2T} + \frac{d}{2T}\ln\left(\frac{T}{\sigma^2d}\right) + \frac{1}{2T}\ln\left(1 + \frac{d}{T}\right),
    \end{align*}
    where $(a)$ follows from Corollary \ref{th:prior_ub}, $(b)$ follows from the formula of KL divergence between two multivariate normal distributions and the fact that $\H^\dagger_{\epsilon,T}(\theta)= \H_{\epsilon, T}(\theta)$, and $(c)$ follows from Theorem \ref{th:lin_reg_error_bounds}.
\end{proof}

Note that the error vanishes as $T$ grows.  This is intuitive, as the misspecified prior shares the support of the correctly specified prior.  In the following example, we will study a case where the support of the misspecified prior is a small subset.

\subsubsection{Missing Feature Misspecified Algorithm}

Consider an agent that ignores the final component of $\theta$.  Let $\1_i$ denote the $i$th standard basis vector. To model this, let $\mathbb{Q}(\theta\in\cdot) = \normal(0, (I_d - \1_d \1_d^\top)/d)$, noting that $I_d  - \1_d \1_d^\top$ is the identity matrix with the $d$th diagonal element changed to zero.  In this case, $\Pr$ and $\mathbb{Q}$ are \emph{not} equivalent measures since $\mathbb{Q}(\theta_d=0) = 1$.  The following result upper bounds the error incurred by this misspecified prior.
\begin{theorem}
    {\bf (missing feature error bound)}
    \label{th:misspecified_2}
    For all $d, T \in \Z_{++}$, if for all $t\in \{0, 1, \ldots, T-1\}$, $(X_t, Y_{t+1})$ is generated by the linear regression process and $\mathbb{Q}(\theta\in\cdot) = \normal(0, (I_d - \1_d \1_d^\top)/d)$ then
    $$\frac{1}{T}\sum_{t=0}^{T-1}\ \E\left[\KL\left(P^*_t\|\hat{Q}_t\right)\right]\ \leq\ \frac{d-1}{2T}\left(\ln(T) + \frac{1}{d\sigma^2}\right) + \frac{1}{2d\sigma^2}.$$
\end{theorem}
\begin{proof}
    For $i \in \{1, 2, \ldots, d-1\}$, let $\theta^\dagger_i = \theta_i$ and let $\theta^\dagger_d = 0$.  Then, 
    \begin{align*}
        \frac{1}{T}\sum_{t=0}^{T-1}\ \E\left[\KL\left(P^*_t\|\hat{Q}_t\right)\right]
        & \overset{(a)}{\leq} \frac{1}{T}\sum_{t=0}^{T-1}\E\left[\KL(P^*_t\|P^\dagger_t)\right] + \frac{\KL(\Pr(\theta^\dagger\in\cdot) \| \mathbb{Q}(\theta\in\cdot) )}{T} + \inf_{\epsilon\geq 0}\ \left(\frac{\H^\dagger_{\epsilon, T}(\theta^\dagger)}{T} + \epsilon\right)\\
        & = \E\left[\KL(P^*_0\|P^\dagger_0)\right] + \inf_{\epsilon\geq 0}\ \left(\frac{\H^\dagger_{\epsilon, T}(\theta^\dagger)}{T} + \epsilon\right)\\
        & = \frac{\E\left[\left(\theta^\top X_t - \theta^{\dagger \top} X_t\right)^2\right]}{2\sigma^2}+ \inf_{\epsilon\geq 0}\ \left(\frac{\H^\dagger_{\epsilon, T}(\theta^\dagger)}{T} + \epsilon\right)\\
        & = \frac{\E\left[\left(\theta - \theta^\dagger\right)^\top \left(\theta - \theta^\dagger\right)\right]}{2\sigma^2}+ \inf_{\epsilon\geq 0}\ \left(\frac{\H^\dagger_{\epsilon, T}(\theta^\dagger)}{T} + \epsilon\right)\\
        & = \frac{1}{2d\sigma^2}+ \inf_{\epsilon\geq 0}\ \left(\frac{\H^\dagger_{\epsilon, T}(\theta^\dagger)}{T} + \epsilon\right)
    \end{align*}
    where $(a)$ follows from Theorem \ref{th:unrealizable}.  We now bound the rate-distortion function.

    Let $\tilde{\theta}^\dagger \in\Re^d$ be a random vector such that for all $i \in \{1, 2, \ldots, d-1\}$, $\theta_i = \tilde{\theta}_i^\dagger + Z_i$, where $Z_i \perp \tilde{\theta}$ and $Z_{1:d-1} \sim \normal(0, \epsilon I_{d-1}/d)$.  Meanwhile assume that $\tilde{\theta}^\dagger_d = \theta_d$.  Then, 
    \begin{align*}
        \E\left[\KL(\Pr(Y_{t+1}\in\cdot|\tilde{\theta}^\dagger,\theta^\dagger, H_t) \| \Pr(Y_{t+1}\in\cdot|\tilde{\theta}^\dagger,H_t))\right]
        & =  \E\left[\KL(\Pr(Y_{t+1}\in\cdot|\theta, H_t) \| \Pr(Y_{t+1}\in\cdot|\tilde{\theta}^\dagger,H_t))\right]\\
        & \leq \E\left[\KL(\Pr(Y_{t+1}\in\cdot|\theta, H_t) \| \Pr(Y_{t+1}\in\cdot|\theta \leftarrow \tilde{\theta}^\dagger,H_t))\right]\\
        & = \frac{\E\left[\left(\theta^\top X_t - \theta^{\dagger\top} X_t\right)^2\right]}{2\sigma^2}\\
        & = \frac{d-1}{2d\sigma^2}\epsilon.
    \end{align*}
    Meanwhile we have,
    \begin{align*}
        \frac{\I(\theta^\dagger;\tilde{\theta}^\dagger)}{T}
        & = \frac{1}{T}\left(\diffentropy(\theta^\dagger) - \diffentropy(\theta^\dagger|\tilde{\theta}^\dagger)\right)\\
        & = \frac{1}{T}\left(\frac{d-1}{2}\ln\left(\frac{2\pi e}{d-1}\right) - \diffentropy(\theta^\dagger-\tilde{\theta}^\dagger|\tilde{\theta}^\dagger)\right)\\
        & = \frac{1}{T}\left(\frac{d-1}{2}\ln\left(\frac{2\pi e}{d-1}\right) - \diffentropy(Z_{1:d-1})\right)\\
        & = \frac{1}{T}\left(\frac{d-1}{2}\ln\left(\frac{2\pi e}{d-1}\right) - \frac{d-1}{2}\ln \left(\frac{\epsilon 2\pi e}{d-1} \right) - \frac{1}{2}\ln\left(\frac{2\pi e}{d-1}\right)\right)\\
        & = \frac{d-1}{2T} \ln \left(\frac{1}{\epsilon}\right).
    \end{align*}
    Therefore, the result follows:
    \begin{align*}
        \frac{1}{T}\sum_{t=0}^{T-1}\ \E\left[\KL\left(P^*_t\|\hat{Q}_t\right)\right]
        & \leq \frac{1}{2d\sigma^2}+ \inf_{\epsilon\geq 0}\ \left(\frac{\H^\dagger_{\epsilon, T}(\theta^\dagger)}{T} + \epsilon\right)\\
        & \leq \frac{1}{2d\sigma^2}+ \inf_{\epsilon\geq 0}\ \left(\frac{d-1}{2T} \ln \left(\frac{1}{\epsilon}\right) + \frac{d-1}{2d\sigma^2}\epsilon\right)\\
        & \overset{(a)}{\leq} \frac{1}{2d\sigma^2}+ \frac{d-1}{2T}\left(\ln \left(T\right) + \frac{1}{d\sigma^2}\right),
    \end{align*}
    where $(a)$ follows from setting $\epsilon= 1/T$.
\end{proof}

Note that the first term in this bound vanishes as $T$ grows.  The second term, on the other hand, represents irreducible error due to the fact that $\tilde{\theta}$ ignores the final feature.

Our analysis can be extended to treat a prior that misspecifies the noise variance $\sigma^2$.  And as we will see, when a feature is missing, this additional misspecification can greatly reduce error.  To model this, we use model parameters to encode not only $\theta$ but also $\sigma$.  We take the distribution of the random variable $\sigma$ to be a Dirac delta function centered at some value $\sigma_* \in \Re_{++}$.  For the misspecified prior, we let $\mathbb{Q}(\theta \in\cdot) = \normal(0, (I_{d}- \1_d \1_d^\top)/d)$, as before, and let $\mathbb{Q}(\sigma \in\cdot)$ be the Dirac delta function centered at $\sqrt{\frac{1}{d} + \sigma^2_*}$.  The following result establishes that the error with this misspecified prior is smaller than with $\mathbb{Q}(\sigma\in\cdot) = \Pr(\sigma\in\cdot)$.  In particular, the misspecification error $1/2d\sigma^2$ is replaced by $\ln(1 + 1/d\sigma^2) / 2$.  This represents an enormous reduction when $\sigma$ is very small.
\begin{theorem}
    {\bf (missing feature with inflated noise variance error bound)}
    For all $d, T \in \Z_{++}$, if for all $t\in \{0, 1, \ldots, T-1\}$, $(X_t, Y_{t+1})$ is generated by the linear regression process and $\mathbb{Q}(\theta\in\cdot) = \normal(0, (I_d - \1_d \1_d^\top)/d)$ and $\mathbb{Q}(\sigma=\sqrt{1/d + \sigma^2_*}) = 1$ then
    $$\frac{1}{T}\sum_{t=0}^{T-1}\ \E\left[\KL\left(P^*_t\|\hat{Q}_t\right)\right]\ \leq\ \frac{d-1}{2T}\left(\ln(T) + \frac{1}{d\sigma^2}\right) + \frac{1}{2}\ln \left(1 + \frac{1}{d\sigma^2}\right).$$
\end{theorem}
\begin{proof}
    Let $\theta^\dagger_i = \theta_i$ for $i \in \{1,2,\ldots, d-1\}$ and $\theta^\dagger_d \overset{a.s.}{=} 0$.  Let $\sigma^\dagger = \sqrt{\sigma^2 + \frac{1}{d}}$.
    Then,
    \begin{align*}
        \frac{1}{T}\sum_{t=0}^{T-1} \E\left[\KL(P^*_t \| \hat{Q}_t)\right]
        & \leq \E\left[\KL(P^*_t \| \Pr(Y_{1}\in\cdot|\theta\leftarrow\theta^\dagger, \sigma\leftarrow \sigma^\dagger, X_0))\right] + \frac{\H^\dagger_{\epsilon, T}(\theta^\dagger, \sigma^\dagger) }{T} + \epsilon\\
        & = \E\left[\KL(P^*_t \| \Pr(Y_{1}\in\cdot|\theta\leftarrow\theta^\dagger, \sigma\leftarrow \sigma^\dagger, X_0))\right] + \frac{\H^\dagger_{\epsilon, T}(\theta^\dagger) }{T} + \epsilon\\
        & = \frac{1}{2}\ln \left(1 + \frac{1}{d\sigma^2}\right) + \frac{\sigma^2 + \E\left[(\theta^\top X_0 - \theta^{\dagger \top} X_0)^2\right]}{2(\frac{1}{d}+\sigma^2)} - \frac{1}{2} + \frac{\H^\dagger_{\epsilon, T}(\theta^\dagger) }{T} + \epsilon\\
        & \overset{(a)}{\leq} \frac{1}{2}\ln \left(1 + \frac{1}{d\sigma^2}\right) + \frac{\sigma^2 + \E\left[(\theta^\top X_0 - \theta^{\dagger \top} X_0)^2\right]}{2(\frac{1}{d}+\sigma^2)} - \frac{1}{2} + \frac{d-1}{2T}\left(\ln(T) + \frac{1}{d\sigma^2}\right)\\
        & =  \frac{d-1}{2T}\left(\ln(T) + \frac{1}{d\sigma^2}\right) + \frac{1}{2}\ln \left(1 + \frac{1}{d\sigma^2}\right),
    \end{align*}
    where $(a)$ follows from Theorem \ref{th:misspecified_2}
\end{proof} 

\subsection{Neural Scaling Laws}
\label{se:neural-scaling-laws}

Theorem \ref{th:unrealizable} also offers insights into trade-offs arising in the development of more complex machine learning models.  One example where this result can be used to generate new insight arises in the study of neural scaling laws.  In particular, how to optimally balance misspecification versus statistical errors subject to computational constraints \citep{hoffmann2022training, kaplan2020scaling}.

Consider a simple model of computation in deep learning, where computation is measured in FLOPs, which is the product of the parameter count of the trained model and the size of the training dataset.  Increasing the parameter count reduces misspecification error but at the cost of statistical error.  Meanwhile, increasing the dataset size reduces statistical error.  In this section, we discuss the problem studied by \cite{jeon2024informationtheoreticfoundationsneuralscaling} of sizing the model versus the dataset to minimize loss subject to a \emph{fixed} FLOP count.

For a FLOP constraint $p \cdot T \leq C$, where $p$ denotes the parameter count of the model and $T$ the dataset size, there is a tension between $p$ and $T$ in minimizing the upper bound of Theorem \ref{th:unrealizable}.  This can be seen by first fixing a FLOP count $C$ and substituting $T= C/p$.  The support of the misspecified prior distribution $\mathbb{Q}(\theta\in\cdot)$ includes only learning models with at most $p$ parameters.  If there is a constrained approximation $\theta^\dagger$ that satisfies $\Pr(\theta^\dagger\in\cdot) = \mathbb{Q}(\theta\in\cdot)$, as there will be for the model we formulate, then the upper bound  of Theorem \ref{th:unrealizable} implies
$$
\frac{1}{T} \sum_{t=0}^{T-1} \E\left[\KL(P^*_t \| \hat{Q}_t)\right]\ \leq\ 
\underbrace{\frac{p\cdot \I(\theta^\dagger;H_T)}{C}}_{\rm statistical\ error}\ +\ \underbrace{\E\left[\KL\left(P^*_t\|P^\dagger_t\right)\right]}_{\rm misspecification\ error}.$$
As our analysis will establish, the statistical error \emph{increases} in $p$ whereas the misspecification error \emph{decreases} in $p$.  Under a fixed FLOP budget, the model designer ought to select $p$ to balance between the two sources of error.  In the following section, we present a neural network model for which we can mathematically analyze this optimal allocation.  

Our analysis will lead to a bound (Theorem \ref{th:eff_frontier}) on the error suffered by a finite-width neural network used to predict labels generate by an infinite width neural network.  By choosing the number of parameters $p^*$ of the finite-width neural network to minimize this error bound subject to the FLOP count constraint $p \cdot T \leq C$, we will establish a scaling law (Theorem \ref{th:eff_frontier}) of the form
$$p^* = \tilde{\Theta}(\sqrt{C}).$$
In other words, up to logarithmic factors, the optimal number of parameters $p^*$ grows with the square root of the FLOP count.

\subsubsection{Data Generating Process}

\begin{figure}[H]
    \centering
    \includegraphics[width=\linewidth]{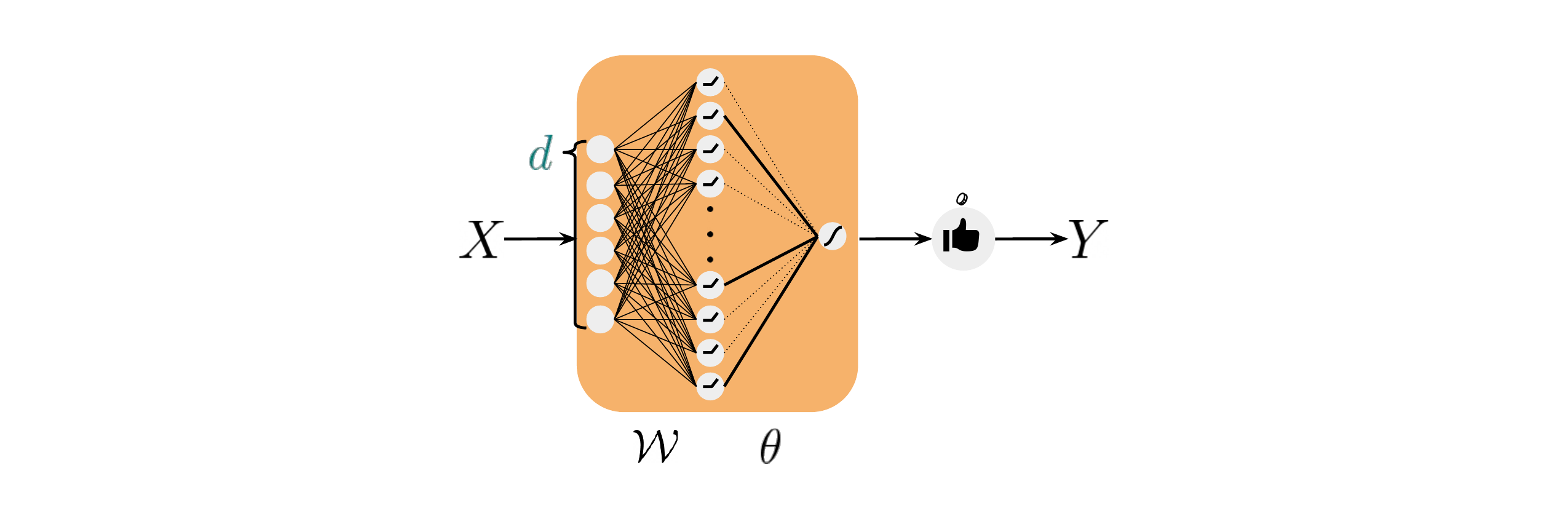}
    \caption{Our Dirichlet data generating process is specified by an input dimension $d$ and an infinite number of hidden units with ReLU activations.  The output is scalar and passed through a sigmoidal activation unit to produce a probability from which a binary label is drawn.}
    \label{fig:dirichlet}
\end{figure}

We consider a data generating process which resembles the infinite-width MLP process that we analyzed in Section \ref{sec:inf_mlp}.  The MLP is uniquely identified by an infinite-dimensional matrix $A \in \Re^{\infty\times d}$, which represents the first-layer weights, and an infinite-dimensional vector $\theta\in\Re^{\infty}$ of output-layer weights.  The neural network has input dimension $d$ and output dimension $1$.  In this analysis, we restrict attention to a particular prior distribution $\Pr((A,\theta)\in\cdot)$ over weights.  This prior distribution is characterized by a \emph{Dirichlet process}, which we now describe.

A Dirichlet process takes as input a \emph{scale parameter} $K$ and a base distribution $P$.  We take $P$ to be the uniform distribution $\text{Unif}(\sphere)$ over the unit sphere in $d$ dimensions.  Each realization of our Dirichlet process is a probability mass function with countable support in $\sphere$.  Hence, each realization can be expressed as a countable subset $A_1,A_2,A_3, \ldots \in \sphere$ and probabilities $\bar{\theta}_1,\bar{\theta}_2,\bar{\theta}_3, \cdots$ assigned to elements of that subset.  Notably, with probability one, $\bar{\theta} \geq 0$ and $\sum_{n=1}^{\infty} \bar{\theta}_n = 1$.

For all $t$, let $X_{t} \overset{iid}{\sim} \normal(0, I_d)$ and 
$$Y_{t+1} =  
\begin{cases}
    1 & {\rm \ w.p.\ }\ \sigma\left(\sqrt{K+1}\cdot \sum_{n=1}^{\infty} \theta_n \relu(A_n^\top X_t)\right) \\
    0 & {\rm \ otherwise,}
\end{cases}
$$
with
$$\theta_n = \begin{cases}
    \bar{\theta}_n & \text{ w.p. } 0.5 \\
    -\bar{\theta}_n & \text{ w.p. } 0.5,
\end{cases}$$
where $A_1,A_2,A_3, \ldots$ and $\bar{\theta}_1,\bar{\theta}_2,\bar{\theta}_3, \ldots$ express an independent sample from our Dirichlet process and $\sigma$ is the standard logistic function.

The scale parameter $K$ influences diffusiveness of  $\bar{\theta}_1,\bar{\theta}_2,\bar{\theta}_3, \ldots$.  As $K$ decreases, the distribution becomes increasingly sparse, with probability eventually concentrating on a single component $\bar{\theta}_n = 1$ when $K=0$.  As $K$ increases, the distribution tends toward uniform.

In the next section, we consider a width-$n$ approximation to this infinitely wide neural network.  The number of parameters that identify this approximation scales with $p = d n$.  As we will see, the data generating process exhibits two properties that are crucial for studying neural scaling laws: 1) error decreases as the dataset size $T$ increases and as the parameter count $p$ increases and $2)$ the reduction in error is a smooth function of $p$ and $T$.  \citet{hoffmann2022training} and \citet{kaplan2020scaling} both observe a smooth and monotonic improvement in the out-of-sample log loss with increasing computational resources and data.  In the following section, we will bound the error of predictions $\hat{Q}_t$ as a function of $n$.


\subsubsection{A Width-$n$ Misspecified Prior Distribution}

As in earlier sections, to represent misspecification we introduce an alternative probability measure $\mathbb{Q}$.  In particular, $\mathbb{Q}((\theta, A) \in\cdot)$ is the misspecified prior distribution for our infinite-dimensional MLP.  Under this misspecified prior, $A_1, \ldots, A_n$ are iid ${\rm Uniform}(\sphere)$, and $A_{n+1} = A_{n+2} = \cdots = 0$.  Further, $\overline{\theta}_{1:n}$ is independently sampled ${\rm Dirichlet}(K/n)$, where $K$ is the scale parameter introduced in the previous section, and $\overline{\theta}_{n+1} = \overline{\theta}_{n+2} = \cdots = 0$.  Finally, $\theta$ is determined by $\overline{\theta}$ as described in the previous section.

Note that, under our misspecified prior, $(\theta, A)$ can be interpreted as the parameters of an MLP of width $n$.  Furthermore, as $n$ grows, the misspecified prior $\mathbb{Q}((\theta, A) \in \cdot)$ converges in distribution to the true prior $\Pr((\theta, A) \in \cdot)$.

\subsubsection{A Random Variable of Equal Distribution}

To facilitate our analysis, we construct a random tuple $(\theta^\dagger, A^\dagger)$ such that $\mathbb{P}((\theta^\dagger, A^\dagger) \in\cdot) = \mathbb{Q}((\theta, A) \in\cdot)$.  Let $(\delta_1, \delta_2, \ldots)$ be an iid process for which $\Pr(\delta_j\in\cdot) =  {\rm Uniform}(\{0, 1, \ldots, n\})$.  To interpret $\delta_j$, consider a random partition of the positive integers into $n$ cells and think of $\delta_j$ as indexing the cell to which we assign $\theta_j$.  For all $i \in \{1, 2, \ldots, n\}$, $\theta^\dagger_i$ aggregates weights $\theta_j$ that are assigned to cell $i$ in the following way:
$$\theta^\dagger_i =  {\rm sign}\left( \sum_{j: \delta_j = i} \theta_j \right) \cdot \sum_{j: \delta_j = i} |\theta_j|.$$
Further, $\theta^\dagger_{n+1} = \theta^\dagger_{n+2} = \cdots = 0$.  There are many ways to sample $A^\dagger$ such that $\mathbb{P}((\theta^\dagger, A^\dagger) \in\cdot) = \mathbb{Q}((\theta, A) \in\cdot)$.  To fix on one, for $i \in \{1, 2, \ldots, n\}$, we let $A^\dagger_i = A_{j_i}$, where $j_i = \argmax_{j: \delta_j = i} |\theta_j|$.  Further, $A^\dagger_{n+1} = A^\dagger_{n+2} = \cdots = 0$.  It is straightforward to show that, by the aggregation property of the Dirichlet process, $\mathbb{P}((\theta^\dagger, A^\dagger) \in\cdot) = \mathbb{Q}((\theta, A) \in\cdot)$.

\subsubsection{Preliminary Results}
We first provide a result which bounds the mean squared error of our width-n approximation derived from $(\theta^\dagger, A^\dagger)$.  

\begin{lemma}{\bf (width-n approximation error bound)}\label{le:width_n_approx}
    For all $d, K, n \in \mathbb{Z}_{++}$, if $(\theta, A)$ are drawn according to the infinite-width MLP process and $(\theta^\dagger, A^\dagger)$ the finite-width approximation, then
    $$\E\left[ \left( \sum_{i = 1}^{\infty} \theta_i {\rm ReLU}(A_i^\top X_0) - \sum_{i=1}^{n} \theta^\dagger_i {\rm ReLU}(A_i^{\dagger\top} X_0) \right)^2 \right] \ \leq\ \frac{10}{n}.$$
\end{lemma}
\begin{proof}
    We begin by introducing the following abbreviated notation:
    $$\phi_i = {\rm ReLU}(A_i^\top X_0),\quad S_i = {\rm Sign}(\theta_i),\quad S^\dagger_i = {\rm sign}\left(\sum_{j:\delta_j=i} \theta_j\right).$$
    Furthermore, let
    $$\phi_i = \mu + \epsilon_i,$$
    where $\mu = \E[{\rm ReLU}(A_i^\top X_0)]$ and $\epsilon_i$ is zero-mean.  
    \begin{align*}
        & \E\left[\left(\sum_{i=1}^{\infty} \theta_i \phi_i - \sum_{i=1}^{n} \theta^\dagger_{i} \phi^\dagger_i\right)^2\right]\\
        & = \E\left[\left(\sum_{i=1}^{n} \theta^\dagger_i\left(\sum_{j: \delta_j = i} \frac{\theta_j}{\theta^\dagger_i} \phi_j\right) - \sum_{i=1}^{n} \theta^\dagger_{i} \phi^\dagger_i\right)^2\right]\\
        & = \E\left[\left( \sum_{i=1}^{n} \theta^\dagger_i \left(\left(\sum_{j: \delta_j = i} \frac{\theta_j}{\theta^\dagger_i} \phi_j\right) - \phi^\dagger_i \right) \right)^2\right]\\
        & \overset{(a)}{=} \sum_{i=1}^{n}\E\left[\theta^{\dagger 2}_i\right] \E\left[\left(\left(\sum_{j: \delta_j = i} S_j\cdot\left|\frac{\theta_j}{\theta^\dagger_i}\right| \phi_j\right) - S^\dagger_i \cdot \phi^\dagger_i \right)^2\right]\\
        & = \sum_{i=1}^{n}\E\left[\theta^{\dagger 2}_i\right] \E\left[\left(\left(\sum_{j: \delta_j = i} S_j\cdot\left|\frac{\theta_j}{\theta^\dagger_i}\right| (\mu + \epsilon_j)\right) - S^\dagger_i \cdot (\mu + \epsilon^\dagger_i) \right)^2\right]\\
        & = \sum_{i=1}^{n}\E\left[\theta^{\dagger 2}_i\right] \E\left[\left(\sum_{j: \delta_j = i} S_j\cdot\left|\frac{\theta_j}{\theta^\dagger_i}\right| \mu  - S^\dagger_i \cdot \mu  + \sum_{j:\delta_j = i} S_j\cdot\left|\frac{\theta_j}{\theta^\dagger_i}\right|\epsilon_j - S^\dagger_i \cdot \epsilon^\dagger_i \right)^2\right]\\
        & \overset{(b)}{=} \sum_{i=1}^{n}\E\left[\theta^{\dagger 2}_i\right] \E\left[\mu^2\left(\sum_{j: \delta_j = i} S_j\cdot\left|\frac{\theta_j}{\theta^\dagger_i}\right|  - S^\dagger_i \right)^2  + \left(\sum_{j:\delta_j = i} S_j\cdot\left|\frac{\theta_j}{\theta^\dagger_i}\right|\epsilon_j - S^\dagger_i \cdot \epsilon^\dagger_i \right)^2\right]\\
        & = \sum_{i=1}^{n}\E\left[\theta^{\dagger 2}_i\right] \E\left[\mu^2\left(1 - 2\sum_{j: \delta_j = i} S_i^\dagger S_j\cdot\left|\frac{\theta_j}{\theta^\dagger_i}\right| + \sum_{j:\delta_j=i} \left(\frac{\theta_j}{\theta^\dagger_i}\right)^2 \right)  + \sum_{j:\delta_j = i} \left(\frac{\theta_j}{\theta^\dagger_i}\right)^2\epsilon_j^2 + \epsilon^{\dagger 2}_i - 2 S^\dagger_i \cdot \frac{\theta_{j_i}}{|\theta^\dagger_i|} \epsilon^{\dagger 2}_i \right]\\
        & = \sum_{i=1}^{n}\E\left[\theta^{\dagger 2}_i\right] \E\left[\mu^2\left(1 - 2\sum_{j: \delta_j = i} S_i^\dagger S_j\cdot\left|\frac{\theta_j}{\theta^\dagger_i}\right| + \sum_{j:\delta_j=i} \left(\frac{\theta_j}{\theta^\dagger_i}\right)^2 \right)  + \epsilon^2 \left(\sum_{j:\delta_j = i} \left(\frac{\theta_j}{\theta^\dagger_i}\right)^2 + 1 - 2 S^\dagger_i \cdot \frac{\theta_{j_i}}{|\theta^\dagger_i|} \right) \right]\\
    \end{align*}
    where $(a)$ follows from the fact that the signs are independent and zero-mean across the partitions and Lukac's Proportion Theorem which allows us to split out the $\E[\theta^{\dagger2}_i]$ and $(b)$ follows from the fact that the cross terms are zero-mean.

    We now bound this sum of expectations by considering the events: $E_i : |\theta_{j_i}/\theta^\dagger_i| \geq 0.5$ for which $S^\dagger_i = S_{j_i}$ and the complement $\lnot E_i : |\theta_{j_i}/\theta^\dagger_i| < 0.5$.  
    \begin{align*}
         & \E\left[\left(\sum_{i=1}^{\infty} \theta_i \phi_i - \sum_{i=1}^{n} \theta^\dagger_{i} \phi^\dagger_i\right)^2\right]\\
         & \overset{(a)}{\leq} \sum_{i=1}^{n}\E\left[\theta^{\dagger 2}_i\right]\left( \Pr(E_i) \cdot \E\left[\left(\mu^2 + \epsilon^2\right)\left(\left(1 - \left|\frac{\theta_{j_i}}{\theta^\dagger_i}\right|\right)^2 + \sum_{j\neq j_i: \delta_j = i} \left(\frac{\theta_j}{\theta_i^\dagger}\right)^2 \right) \bigg| E_i \right] + \Pr(\lnot E_i) \cdot 4\right)\\
         & \overset{(b)}{\leq} \sum_{i=1}^{n}\E\left[\theta^{\dagger 2}_i\right]\left( \Pr(E_i) \cdot 2\E\left[\left(\mu^2 + \epsilon^2\right)\left(1 - \left|\frac{\theta_{j_i}}{\theta^\dagger_i}\right|\right)^2 \bigg| E_i \right]+ \Pr(\lnot E_i) \cdot 4\right)\\
         & \overset{(c)}{\leq} \sum_{i=1}^{n}\E\left[\theta^{\dagger 2}_i\right]\left( \Pr(E_i) \cdot 2\E\left[\left(\mu^2 + \epsilon^2\right)\left(1 - \left|\frac{\theta_{j_i}}{\theta^\dagger_i}\right|\right)^2 \right] + \Pr(\lnot E_i) \cdot 4\right)\\
         & \overset{(d)}{\leq} \sum_{i=1}^{n}\E\left[\theta^{\dagger 2}_i\right] \left( \frac{\frac{2K}{n}}{2 + \frac{K}{n}} + 4\left(1 - \left(\frac{1}{2}\right)^{K/n}\right)\right)\\
         & \overset{(e)}{\leq} \sum_{i=1}^{n}\E\left[\theta^{\dagger 2}_i\right] \left( \frac{2K}{2n + K} + \frac{4K}{n}\right)\\
         & \leq \sum_{i=1}^{n}\E\left[\theta^{\dagger 2}_i\right] \cdot \frac{5K}{n}\\
         & \overset{(f)}{=} \left(\frac{1 + \frac{K}{n}}{K+1}\right) \cdot \frac{5K}{n}\\
         & \leq \frac{10}{n}
    \end{align*}
    where $(a)$ follows from the fact that conditioned on the event $E_i$, $S^\dagger_i = S_{j_i} \perp S_j$ for all $j\neq j_i: \delta_j = i$.  $(b)$ follows from the fact that $\sum_{j\neq j_i:\delta_j=i}(\theta_j/\theta^\dagger_i)^2 \leq (\sum_{j\neq j_i:\delta_j=i} |\theta_j / \theta^\dagger_i|)^2 \leq (1 - |\theta_{j_i}/\theta^\dagger_i|)^2$, $(c)$ follows from the fact that the expected squared difference is larger when not conditioned on $E_i$, $(d)$ follows the fact that $|\theta_{j_i} / \theta^\dagger_{i}|$ is distributed ${\rm Beta}(1, K/n)$ and that $\mu^2 + \epsilon^2 < 1$,  $(e)$ follows from the fact that $(1 - 0.5^{K/n}) \leq K/n$ for $n \geq 1$, and $(f)$ follows from the fact that $|\theta^\dagger_i|$ is distributed ${\rm Dirichlet}(K/n)$ in $n$ dimensions. 
\end{proof}

With this bound on the width-n error in place, we derive the following upper bound on the misspecification error.

\begin{lemma}\label{le:miss_ub_3_new}{\bf (misspecification error bound)}
    For all $d, n, K\in\mathbb{Z}_{++}$ and $\epsilon \geq 0$,
    $$\E\left[\KL\left(P^*_0\|P^\dagger_0\right)\right]\ + \KL(\Pr((\theta^\dagger, A^\dagger)\in\cdot) \| \mathbb{Q}((\theta, A)\in\cdot)) \ \leq\ \  \frac{10(K+1)}{n}.$$
\end{lemma}
\begin{proof}
    \begin{align*}
        & \E\left[\KL\left(P^*_0\|P^\dagger_0\right)\right]\ + \KL(\Pr((\theta^\dagger, A^\dagger)\in\cdot) \| \mathbb{Q}(\theta, A\in\cdot))\\
        & \overset{(a)}{=} 
        \E\left[\KL\left(P^*_0\|P^\dagger_0\right)\right]\\
        & \overset{(b)}{\leq} \E\left[\left(\sqrt{K+1}\cdot\sum_{i=1}^{\infty}\theta_i {\rm ReLU}(A_i^\top X_0) - \sqrt{K+1}\cdot\sum_{i=1}^{\infty} \theta^\dagger_i {\rm ReLU}(A^\dagger_i X_0)\right)^2\right]\\
        & = (K+1)\E\left[\left(\sum_{i=1}^{\infty}\theta_i {\rm ReLU}(A_i^\top X_0) - \sum_{i=1}^{\infty} \theta^\dagger_i {\rm ReLU}(A^\dagger_i X_0)\right)^2\right]\\
        & \overset{(c)}{\leq} \frac{10(K+1)}{n},
    \end{align*}
    where $(a)$ follows from the aggregation property of Dirichlet process which implies $\mathbb{P}((\theta^\dagger, A^\dagger) \in \cdot) = \mathbb{Q}((\theta, A) \in\cdot)$, $(b)$ follows from Lemma \ref{le:kl_ub}, and $(c)$ follows from Lemma \ref{le:width_n_approx}.
\end{proof}

We now provide an upper bound for the $T$-horizon rate-distortion function.
\begin{lemma}{\bf (width-$n$ approximation $T$-horizon rate-distortion bound)}\label{le:misspecified_mlp_rd}
    For all $K, n \in \mathbb{Z}_{++}$, if $\theta, A$ are generated by the infinite-width MLP process and $\theta^\dagger, A^\dagger$ is the width-$n$ approximation then for all $\epsilon \geq 0$,
    $$\H^\dagger_{\epsilon,T}(\theta^\dagger, A^\dagger) \ \leq\ K\ln\left(1 + \frac{n}{\left(K+n\right)\epsilon}\right)\cdot\left(\ln\left(\frac{4Kn}{\left(K + n\right)\epsilon}\right) + d\ln \left(\frac{3(K+1)}{2\epsilon}\right)\right) + K \ln (2).$$
\end{lemma}
\begin{proof}
    We begin by bounding the rate.  Let $m = 2K/((K/n+1)\epsilon)$ and let $(C_1, C_2, \ldots, C_m)$ be an iid random process for which $\Pr(C_i) = {\rm Categorical}(|\theta^\dagger|)$.  For all $i \in \{1, 2, \ldots, n\}$, let
    $$\tilde{A}^\dagger_i = \argmin_{a \in \sphere_{\epsilon'}}\ \|A^\dagger_i - a\|^2_2,$$
    where $\epsilon' = 2\epsilon/(K+1)$ and $\sphere_{\epsilon'}$ is an $\epsilon'$-cover of $\sphere$ w.r.t the squared L2 distance.  For all $i$, we use $\mathcal{G}_i$ to denote the set $\{j: \delta_j = i\}$.  Recall that $j_i = \argmax_{j\in \mathcal{G}_i} |\theta_j|$.  Let
    $$\tilde{A}^\dagger = \left((\tilde{A}^\dagger_{C_1}, \tilde{A}^\dagger_{C_2}, \ldots, \tilde{A}^\dagger_{C_n}),\ (A_j : \forall j {\rm \ s.t.\ } \forall i,\ j\neq j_i)\right).$$
    Meanwhile let
    $$\tilde{\theta}^\dagger = \left( (C_1, C_2, \ldots, C_m),\ (S_{C_1}, S_{C_2}, \ldots, S_{C_m}),\ (S_j: \forall j {\rm \ s.t.\ } \forall i,\ j \neq j_i),\ \left(\left|\frac{\theta_j}{\theta^\dagger_i}\right|: \forall i \in \{1, \ldots, n\},\ j \in \mathcal{G}_i\right) \right).$$
    Let $\tilde{n}_m^\dagger = \sum_{i=1}^{n} \mathbbm{1}\left[\exists j \in \{1, 2, \ldots, m\} {\rm \ s.t.\ } C_j = i\right]$ denote the number of unique outcomes drawn from\\ $(C_{1:m})$.  We abbreviate $S_{C_1}, S_{C_2}, \ldots, S_{C_m}$ as $S_{C_{1:m}}$.  Finally, recall that for all $i$, $E_i = |\theta_{j_i}/\theta^\dagger_i| \geq 0.5$.  Then,
    \begin{align*}
        \I(\tilde{\theta}^\dagger, \tilde{A}^\dagger ; \theta^\dagger, A^\dagger)
        & \overset{(a)}{=} \I\left(C_{1:m},\tilde{A}^\dagger_{1:m}, S_{C_{1:m}}\ ;\ \theta^\dagger, A^\dagger\right)\\
        &\quad + \I\left((S_j: \forall j {\rm \ s.t.\ } \forall i,\ j \neq j_i);\theta^\dagger\bigg|C_{1:m}, \tilde{A}^\dagger_{1:m}, S_{C_{1:m}}, \left(\left|\frac{\theta_j}{\theta^\dagger_i}\right|: \forall i \in \{1, \ldots, n\},\ j \in \mathcal{G}_i\right) \right)\\
        & \overset{(b)}{\leq} \E\left[\tilde{n}^\dagger \ln \left( 2m \left(\frac{3}{\epsilon'}\right)^d \right)\right] + \sum_{i=1}^{n} \Pr(\lnot E_i) \cdot \ln(2)\\
        & \overset{(c)}{\leq} K\ln\left(1 + \frac{m}{K}\right) \cdot \left( \ln(2m) + d\ln\left(\frac{3}{\epsilon'}\right)\right) + n \left(1-\left(\frac{1}{2}\right)^{\frac{K}{n}}\right) \ln (2)\\
        & \leq K\ln\left(1 + \frac{n}{\left(K+n\right)\epsilon}\right)\cdot\left(\ln\left(\frac{4Kn}{\left(K + n\right)\epsilon}\right) + d\ln \left(\frac{3(K+1)}{2\epsilon}\right)\right) + K\ln(2),
    \end{align*}
    where $(a)$ follows from the fact that $(|\theta_j/\theta^\dagger_i|:  \forall i \in \{1, \ldots, n\},\ j \in \mathcal{G}_i) \perp (\theta^\dagger, A^\dagger)$ from Lukac's proportion theorem, $(b)$ follows from Theorem \ref{th:entropy_code} where $(C_{1:m}, \tilde{A}^\dagger_{1:m}, S_{C_{1:m}})$ are encoded as $\tilde{n}_m^\dagger$ words, each with length $\ln(2m)$ to denote the unique value in $(-1, -(m-1)/m, -(m-2)/m, \ldots, -1/m, 1/m, \ldots, (m-1)/m, 1)$ and length $\ln((3/\epsilon')^d)$ to denote the element of $\sphere_{\epsilon'}$ and the fact that if $E_i$, the conditional mutual information is $0$, and $(c)$ follows from Lemma \ref{le:num_unique}.

    We now bound the distortion. Let
    $$F(x) = \sqrt{K+1} \cdot \sum_{i=1}^{\infty}\theta_i {\rm ReLU}(A_i^\top x),$$
    and let
    $$\tilde{F}^\dagger(x) = \sqrt{K+1}\cdot\sum_{i=1}^{m} \left( \frac{\theta_{j_{C_i}}}{m|\theta^\dagger_{C_i}|}{\rm ReLU}(\tilde{A}^{\dagger\top}_{C_i} x) + \sum_{j\in\mathcal{G}_{C_i}\setminus\{j_{C_i}\}} \frac{\theta_j}{m|\theta^\dagger_{C_i}|}{\rm ReLU}(A^\top_j x) \right).$$

    Then,
    \begin{align*}
        & \frac{1}{T}\sum_{t=0}^{T-1} \E\left[\KL\left(\Pr(Y_{t+1}\in\cdot|\theta^\dagger, A^\dagger, \tilde{\theta}^\dagger, \tilde{A}^\dagger, X_t)\| \Pr(Y_{t+1}\in\cdot|\tilde{\theta}^\dagger, \tilde{A}^\dagger, X_t)\right)\right]\\
        & = \E\left[\KL\left(\Pr(Y_1\in\cdot|\theta^\dagger, A^\dagger, \tilde{\theta}^\dagger, \tilde{A}^\dagger, X_0)\| \Pr(Y_1\in\cdot|\tilde{\theta}^\dagger, \tilde{A}^\dagger, X_0)\right)\right]\\
        & \overset{(a)}{=} \E\left[\KL\left(\Pr(Y_1\in\cdot|\theta, A, X_0)\| \Pr(Y_1\in\cdot|\tilde{\theta}^\dagger, \tilde{A}^\dagger, X_0)\right)\right] \\
        & \overset{(b)}{\leq}\E\left[\KL\left(\Pr(Y_1\in\cdot|F(X_0)\| \Pr(Y_1\in\cdot|F(X_0) \leftarrow \tilde{F}^\dagger(X_0)\right)\right] \\
        & \overset{(c)}{\leq} \frac{\E\left[\left(F(X_0) - \tilde{F}^\dagger(X_0)\right)^2\right]}{8}\\
        & \leq  \frac{\E\left[\left( F(X_0)  -  \sqrt{K+1} \sum_{i=1}^{m} \sum_{j\in \mathcal{G}_{C_i}} \frac{\theta_j}{m |\theta^\dagger_{C_i}|} {\rm ReLU}(A_j^\top X_0) \right)^2\right]}{4}\\
        &\ + \frac{\E\left[\left( \sqrt{K+1} \sum_{i=1}^{m} \sum_{j\in \mathcal{G}_{C_i}} \frac{\theta_j}{m |\theta^\dagger_{C_i}|} {\rm ReLU}(A_j^\top X_0) - \tilde{F}^\dagger(X_0) \right)^2\right]}{4}\\
        & \overset{(d)}{\leq}  \frac{\mathbb{V}\left(\sqrt{K+1} \sum_{i=1}^{m} \sum_{j\in \mathcal{G}_{C_i}} \frac{\theta_j}{m |\theta^\dagger_{C_i}|} {\rm ReLU}(A_j^\top X_0) \ \Big|\ \theta,A,\theta^\dagger,X_0\right)}{4} + \frac{\E\left[\left( \sum_{i=1}^{m} \frac{1}{m}\sqrt{\epsilon'}\|X_0\|_2 \right)^2\right]}{4}\\
        & \overset{(e)}{\leq}  \mathbb{V}\left(\sqrt{K+1} \sum_{i=1}^{m} \sum_{j\in \mathcal{G}_{C_i}} \frac{\theta_j}{m |\theta^\dagger_{C_i}|} {\rm ReLU}(A_j^\top X_0)\right) - \mathbb{V}\left(F(X_0)\right) + \frac{(K+1)\epsilon'}{4}\\
        & \overset{(f)}{=} (K+1)\cdot \E\left[\left( \sum_{i=1}^{m} \sum_{j \in \mathcal{G}_{C_i}} \frac{\theta_j}{m|\theta^\dagger_{C_i}|} {\rm ReLU}(A_j^\top X_0) \right)^2 - \left( \sum_{i=1}^{\infty}\theta_i {\rm ReLU}(A_i^\top X_0) \right)^2 \right] + \frac{(K+1)\epsilon'}{4}\\
        & \overset{(g)}{=} (K+1)\cdot\E\left[ \sum_{i=1}^{m} \left(\sum_{j \in \mathcal{G}_{C_i}} \frac{\theta_j}{m|\theta^\dagger_{C_i}|} {\rm ReLU}(A_j^\top X_0)\right)^2 + \sum_{i}^{m}\sum_{k>i}^{m}\left( \sum_{j \in \mathcal{G}_{C_i}} \frac{\theta_j}{m|\theta^\dagger_{C_i}|} {\rm ReLU}(A_j^\top X_0)\right)\left(\sum_{l \in \mathcal{G}_{C_k}} \frac{\theta_l}{m|\theta^\dagger_{C_k}|} {\rm ReLU}(A_l^\top X_0)\right) \right]\\
        & - (K+1)\cdot\E\left[\sum_{i=1}^{\infty} \theta_i^2 {\rm ReLU}(A^\top_j X_0)^2\right] + \frac{(K+1)\epsilon'}{4}\\
        & \overset{(h)}{=} (K+1)\cdot\E\left[ \sum_{i=1}^{m} \sum_{j \in \mathcal{G}_{C_i}} \left(\frac{\theta_j}{m|\theta^\dagger_{C_i}|}\right)^2 {\rm ReLU}(A_j^\top X_0)^2 + \sum_{i}^{m}\sum_{k>i}^{m}\Pr(C_i = C_k)\cdot \sum_{j\in\mathcal{G}_i} \left(\frac{\theta_j}{m|\theta^\dagger_{C_i}|}\right)^2 {\rm ReLU}(A_j^\top X_0)^2 \right]\\
        & - (K+1)\cdot\E\left[\sum_{i=1}^{\infty} \theta_i^2 {\rm ReLU}(A^\top_j X_0)^2\right] + \frac{(K+1)\epsilon'}{4}\\
        & \overset{(i)}{\leq} (K+1)\cdot\E\left[\frac{1}{m^2}\sum_{i=1}^{m}\sum_{j\in\mathcal{G}_{C_i}} \left(\frac{\theta_j}{|\theta^\dagger_{C_i}|}\right)^2 + \frac{m-1}{m}\sum_{i=1}^{n} \left(\theta^\dagger_{i}\right)^2 \cdot \left(\sum_{j\in\mathcal{G}_1}\left(\frac{\theta_j}{\theta^\dagger_1}\right)^2\right) - \sum_{i=1}^{\infty} \theta_i^2 \right] + \frac{(K+1)\epsilon'}{4}\\
        & \overset{(j)}{=} \left(K+1\right)\left( \frac{1}{m(\frac{K}{n}+1)}  + \frac{m-1}{m}\cdot\frac{K+n}{n(K+1)}\cdot \frac{1}{\frac{K}{n}+1} - \frac{1}{K+1}\right) + \frac{(K+1)\epsilon'}{4}\\
        & = \left(K+1\right)\left( \frac{1}{m(\frac{K}{n}+1)}  - \frac{1}{m(K+1)} \right) + \frac{(K+1)\epsilon'}{4}\\
        & \leq \frac{K}{m\left(\frac{K}{n} + 1\right)} + \frac{(K+1)\epsilon'}{4},\\
        & = \epsilon,
    \end{align*}
    where $(a)$ follows from the fact that $(\theta^\dagger, A^\dagger, \tilde{\theta}^\dagger, \tilde{A}^\dagger)$ fully identifies $\theta, A$, $(b)$ follows from Lemma \ref{le:change_measure_ub}, $(c)$ follows from Lemma \ref{le:kl_ub}, $(d)$ follows from the definition of conditional variance, $(e)$ follows from the law of total variance, $(f)$ follows from the fact that the two expressions have the same expectation, $(g)$ and $(h)$ follows from the fact that the cross terms have mean $0$, $(i)$ follows from the fact that $\E[{\rm ReLU}(A_i^\top X_0)^2] \leq 1$, and $(j)$ follows from the fact that $|\theta_j/\theta^\dagger_{C_i}|$ is distributed ${\rm DirichletProcess}(K/n, {\rm Uniform}(\sphere))$, $|\theta^\dagger_i|$ is distributed ${\rm Dirichlet}([K/n, \ldots, K/n])$ in $n$ dimensions, and $\theta_i$ is distributed ${\rm Dirichlet Process}(K, {\rm Uniform}(\sphere))$.
\end{proof}

\subsubsection{Main Results}

Using the above results, we can now bound the error when learning with the width-$n$ approximation.
\begin{restatable}{theorem}{lossUb}\label{th:lossUb}{\bf (width-$n$ misspecified learner error upper bound)}
    For all $n,d,K,T \in \Z_{++}$, if for all $t \in \{0, 1, 2, \ldots, T-1\}$, $(X_t, Y_{t+1})$ is generated by the infinite-width MLP process and $\hat{Q}_t$ is generated by the width-$n$ approximation then
    \begin{align*}
        \frac{1}{T}\sum_{t=0}^{T-1}\E\left[\KL(P^*_t \| \hat{Q}_t)\right]
        &\ \leq\ \underbrace{\frac{11K}{n}}_{\rm misspecification\ error}\ +\ \underbrace{\frac{K(d+4)\ln^2\left(4T\right)}{T}}_{\rm statistical\ error}.
    \end{align*}
\end{restatable}
\begin{proof}
    \begin{align*}
        \inf_{\epsilon\geq 0}\ \frac{\H^\dagger_{\epsilon,T}(\theta^\dagger, A^\dagger)}{T} + \epsilon
        & \overset{(a)}{=} \inf_{\epsilon\geq 0}\ \frac{K}{T}\ln\left(1 + \frac{n}{\left(K+n\right)\epsilon}\right)\cdot\left(\ln\left(\frac{4Kn}{\left(K + n\right)\epsilon}\right) + d\ln \left(\frac{3(K+1)}{2\epsilon}\right)\right) + \frac{K \ln (2)}{T} + \epsilon\\
        & \overset{(b)}{\leq} \frac{K}{T}\ln\left(1 + \frac{Tn}{\left(K+n\right)(K+1)}\right)\cdot\left(\ln\left(\frac{4Tn}{\left(K + n\right)}\right) + d\ln \left(\frac{3T}{2}\right)\right) + \frac{K \ln (2)}{T} + \frac{K+1}{T}\\
        & \leq \frac{K(d+1)}{T}\ln\left(1 + \frac{T}{K+1}\right)\cdot\ln\left(4T\right) + \frac{3K}{T}\\
        & \leq \frac{K(d+4)}{T}\ln\left(1 + \frac{T}{K}\right)\cdot\ln\left(4T\right)\\
        & \leq \frac{K(d+4)}{T}\ln^2\left(4T\right),
    \end{align*}
    where $(a)$ follows from Lemma \ref{le:misspecified_mlp_rd} and $(b)$ follows by setting $\epsilon = (K+1)/T$.  The result follows from Theorems \ref{th:unrealizable} and Lemma \ref{le:miss_ub_3_new}.
\end{proof}

By sizing the finite-width approximation to minimize the preceding error bound, we obtain a scaling law.
\begin{restatable}{theorem}{effFront}{\bf(compute-optimal parameter count)}
\label{th:eff_frontier}
    For all $d, n, K, T\in \mathbb{Z}_{++}$ and FLOP counts $C \geq 5000$, if for all $t \in \{0, 1, \ldots, T-1\}$, $(X_t, Y_{t+1})$ is generated by the infinite-width MLP process, $\hat{Q}_t$ is generated by the width-$n$ approximation, and if $n^*$ minimizes the upper bound of Theorem \ref{th:lossUb} subject to $d\cdot n\cdot T \leq C$ then
    $$d\cdot n^* = \tilde{\Theta}\left(\sqrt{C}\right).$$
\end{restatable}
\begin{proof}
    By Theorem \ref{th:lossUb} we are looking for $n^*$ where
    \begin{equation}\label{eq:argmin}
        n^* = \argmin_{n\in\left[\frac{C}{d}\right]}\frac{11K}{n} + \frac{K(d+4)\ln^2(4t)}{t};\ \text{ s.t. } n\cdot d\cdot t \leq C.
    \end{equation}
    For all $C \geq 5000$, the loss of $n \geq 4C/(de^3)$ is suboptimal. For instance, if we let $n = \sqrt{C}/(d\ln(C))$ and if $C \geq 5000$ then for all $d$, 
    \begin{align*}
        \frac{d\cdot 11\cdot \ln(C)}{\sqrt{C}} + \frac{d+4\cdot \ln^2\left(4\ln(C)\cdot \sqrt{C}\right)}{\sqrt{C}\cdot \ln(C)} \leq \frac{11de^3}{4C} + \frac{d(d+4)\cdot 36}{de^3}.
    \end{align*}
    Since the expression in equation \ref{eq:argmin} is concave up with a unique minimizer, we know that $n^* < 4C/de^3$.  Then,
    \begin{align*}
        n^*
        & = \argmin_{n\in\left[\frac{4C}{de^3}\right]}\frac{11K}{n} + \frac{K(d+4)\ln^2(4t)}{t};\ \text{ s.t. } n\cdot d\cdot t \leq C\\
        & = \argmin_{n\in\left[\frac{4C}{de^3}\right]}\frac{11}{n} + \frac{(d+4)\ln^2(4t)}{t};\ \text{ s.t. } n\cdot d\cdot t \leq C\\
        & = \argmin_{n\in\left[\frac{4C}{de^3}\right]}\frac{11}{n}+ \frac{nd(d+4)\ln^2(\frac{4C}{nd})}{C} \\
        & \overset{(a)}{=} n \text{ s.t. } \frac{11}{n^2} = \left(\frac{d(d+4)\ln^2\left(\frac{4C}{nd}\right)}{C} - \frac{2d(d+4)\ln\left(\frac{4C}{nd}\right)}{C}\right)\\
        & = n \text{ s.t. } C = \frac{n^2d(d+4)\ln\left(\frac{4C}{nd}\right)}{11}\left(\ln\left(\frac{4C}{nd}\right) - 2\right),
    \end{align*}
    where $(a)$ follows from $1$st order optimality conditions

    Due to monotonicity, we can derive upper and lower bounds for the value of $n$ via lower and upper bounds of the above RHS respectively.  We begin with the upper bound for $n^*$:
    \begin{align*}
        n^* 
        & = n \leq \frac{4C}{de^3}\ \text{ s.t. } C = \frac{n^2d(d+4)\ln\left(\frac{4C}{nd}\right)}{11}\left(\ln\left(\frac{4C}{nd}\right)-2\right)\\
        & \overset{(a)}{\leq} n \text{ s.t. } C = \frac{d^2n^2}{11}\\
        & = \frac{\sqrt{11C}}{d}.
    \end{align*}
    where $(a)$ follows from the fact that for $n > 4C/(de^3)$, $d^2n^2/11$ lower bounds the above expression.  We now derive the lower bound for $n^*$:
    \begin{align*}
        n^* 
        & = n \text{ s.t. }  C = \frac{n^2d(d+4)\ln\left(\frac{4C}{nd}\right)}{11}\left(\ln\left(\frac{4C}{nd}\right)-2\right)\\
        & \overset{(a)}{\geq} n \text{ s.t. } C = n^2d^2\ln^2(4C)\\
        & = \tilde{\Omega}\left(\frac{\sqrt{C}}{d}\right). 
    \end{align*}
    where $(a)$ follows from the fact that $d^2n^2\ln^2(4C)$ upper bounds the above expression. The result follows.
\end{proof}

\subsubsection{Discussion}

Our theoretical analysis of neural scaling laws was inspired by extensive empirical analysis of transformer-based language models \citep{kaplan2020scaling, hoffmann2022training}.  Those two empirical studies posited differing functional forms -- both flawed -- to map training dataset size and learning model parameter count to out-of-sample loss.  The function proposed by \citet{kaplan2020scaling} overlooks an irreducible error term.  \citet{hoffmann2022training} corrected this and carried out a regression analysis to conclude that the optimal parameter count ought to grow \emph{linearly} in the optimal dataset size.

Theorem \ref{th:eff_frontier} establishes that the optimal parameter count is $d\cdot n^* = \tilde{\Theta}(\sqrt{C})$.  This asymptotic bound corroborates the insight of \cite{hoffmann2022training} that, up to logarithmic factors, the optimal parameter count grows \emph{linearly} in the optimal dataset size.  In spite of this congruence, our theoretical error bound differs significantly from the functional form posited by \citet{hoffmann2022training}.  We detail some of the differences now.  The empirical results of \citet{hoffmann2022training} suggest that the error decays as
\begin{equation}\label{eq:hoffmann}
    \E\left[\KL(P^*_{T}\| \hat{Q}_{T} )\right]\ \approx\ \frac{B}{p^{0.34}} +  \frac{A}{T^{0.28}},
\end{equation}
where $A$ and $B$ are constants.  This result is derived from empirical analysis of transformer models trained via stochastic gradient descent.  Meanwhile, the results of this section hold for the Bayesian posterior predictive with respect to a misspecified prior.  A natural question is: what is the gap between the error attained by current optimization algorithms and what is achievable based on information theory?

Theorem \ref{th:unrealizable} provides an interesting answer to this question.  Let $\Pr(\theta^\dagger\in\cdot) = \mathbb{Q}(\theta\in\cdot)$ and assume that the data pairs $(X_0, Y_1), \ldots, (X_{T}, Y_{T+1})$ are iid when conditioned on $\theta$.  Then, Theorem \ref{th:unrealizable} becomes:
$$\E\left[\KL(P^*_{T}\| \hat{Q}_{T} )\right] \ \leq\  \E\left[\KL\left(P_T^* \ \|\ P^\dagger_T\right)\right] + \inf_{\epsilon\geq 0} \left(\frac{\H^\dagger_{\epsilon}(\theta^\dagger)}{T} + \epsilon\right).$$
The first term is independent of $T$ and mirrors the first term in Equation \ref{eq:hoffmann}; error which diminishes as $p$ increases.  Meanwhile, we expect the second term, which represents statistical error, to be $\tilde{O}(1/T)$ (as in Theorem \ref{th:lossUb}).  On the other hand, the empirical rate expressed by Equation \ref{eq:hoffmann} suggests that, for stochastic gradient descent, statistical error decays at a rate of only $T^{0.28}$.  Developing algorithms that close this gap offers a promising direction for future research.

\newpage
\begin{summary}
    \begin{itemize}
        \item Let $\mathbb{Q}$ be an alternative probability measure on $(\Omega, \mathbb{F})$, which may assign a misspecified prior distribution $\mathbb{Q}(\theta \in \cdot) \neq \Pr(\theta \in \cdot)$ but with the correct conditional data distribution $\mathbb{Q}(\cdot|\theta) = \Pr(\cdot|\theta)$.
        \item The $t$th prediction generated based on the misspecified prior is $\hat{Q}_t = \mathbb{Q}(Y_{t+1}\in\cdot|H_t \leftarrow H_t)$.
        \item For all $\theta^\dagger$ such that $\Pr(\theta^\dagger\in\cdot)$ and $\mathbb{Q}(\theta\in\cdot)$ are equivalent measures, let
        $$\H^\dagger_{\epsilon, T}(\theta^\dagger) = \inf_{\tilde{\theta}^\dagger \in \tilde{\Theta}^\dagger_{\epsilon, T}}\ \I(\theta^\dagger; \tilde{\theta}^\dagger),$$
        where
        $$\tilde{\Theta}_{\epsilon, T} = \left\{ \tilde{\theta}^\dagger \in \tilde{\Theta} : \frac{1}{T}\sum_{t=0}^{T-1} \E\left[\KL(\Pr(Y_{t+1}\in\cdot|\theta^\dagger, \tilde{\theta}^\dagger, H_t) \| \Pr(Y_{t+1}\in\cdot|\tilde{\theta}^\dagger, H_t))  \right] \leq \epsilon \right\}.$$
        \item  {\bf (misspecified learner error bound)} For all $T \in \mathbb{Z}_{++}$,
        \begin{align*}
            & \frac{1}{T}\sum_{t=0}^{T-1}\ \E\left[\KL\left(P^*_t\|\hat{Q}_t\right)\right]\\
            & \leq \inf_{\theta^\dagger\in \Theta^\dagger} \left(\underbrace{\frac{1}{T}\sum_{t=0}^{T-1}\E\left[\KL(P^*_t\|P^\dagger_t)\right] + \frac{\KL(\Pr(\theta^\dagger\in\cdot) \| \mathbb{Q}(\theta\in\cdot) )}{T}}_{\rm misspecification\ error}\ +\ \underbrace{\inf_{\epsilon\geq 0}\ \left(\frac{\H^\dagger_{\epsilon, T}(\theta^\dagger)}{T} + \epsilon\right)}_{\rm statistical\ error}\right),
        \end{align*}
        where $P^\dagger_t = \Pr(Y_{t+1}\in\cdot|\theta\leftarrow \theta^{\dagger}, H_t)$.
    \end{itemize}
\end{summary}
\clearpage

%% file: sections/conclusion.tex
\section{Conclusion}

\textbf{Summary.} This monograph provided a rigorous mathematical framework for analysis of achievable performance in machine learning.  By adopting a Bayesian framing, we were able to leverage connections to Shannon's information theory.  In doing so, we elucidated the intimate connection between learning and optimal lossy compression via rate-distortion theory, providing a streamlined process for deriving error bounds for a plethora of settings ranging from simple to complex.  Regardless of whether the data is iid or exhibits sequential or hierarchical structure, our general results apply and provide accurate insights.  We also extended our framework and results to the study of misspecified models. 

\textbf{Future Research.} Our framework leaves open many directions of future research.  Except for the case of linear regression, we established only rate-distortion upper bounds, leaving lower bounds as a topic for future research.  It would be interesting to understand, for example, the extent to which tools such as the Donsker-Varadhan lower bound can be leveraged to derive general rate-distortion lower bounds.  Beyond that, our framework and results may be leveraged to understand puzzling new phenomena as they continue to emerge in the field of machine learning.